\documentclass[sigconf, nonacm]{acmart}

\newcommand\vldbdoi{10.14778/3476249.3476308}
\newcommand\vldbpages{2627 - 2641}
\newcommand\vldbvolume{14}
\newcommand\vldbissue{11}
\newcommand\vldbyear{2021}
\newcommand\vldbauthors{\authors}
 
\newcommand\vldbavailabilityurl{https://github.com/mkuchnik/PCR_Release}
\newcommand\vldbpagestyle{empty}

\usepackage{url}
\usepackage{hyperref}
\usepackage{xfrac}
\usepackage{subcaption}
\usepackage{todonotes}
\usepackage{graphicx}
\usepackage{booktabs} %
\usepackage{xcolor}
\usepackage{wrapfig}
\usepackage{adjustbox}
\usepackage{tabularx}
\usepackage{makecell,rotating}
\usepackage{enumitem}
\usepackage{siunitx}
\usepackage{amsthm}

\usepackage{amsmath,amsfonts,bm}

\def\eqref#1{equation~\ref{#1}}

\def\1{\bm{1}}
\newcommand{\train}{\mathcal{D}}

\DeclareMathAlphabet{\mathsfit}{\encodingdefault}{\sfdefault}{m}{sl}
\SetMathAlphabet{\mathsfit}{bold}{\encodingdefault}{\sfdefault}{bx}{n}

\def\sB{{\mathbb{B}}}

\newcommand{\E}{\mathbb{E}}

\graphicspath{{Figures/}}

\newcommand\leftRightCrop[4]{\adjustbox{trim={#1\width} {0} {#4\width} {0},clip,center}{\includegraphics[width=\linewidth * \real{#3}]{#2}}}
\newcommand\centerCrop[2]{\adjustbox{trim={#1\width} {#1\height} {#1\width} {#1\height},clip,center}{\includegraphics[width=\linewidth * \real{0.29} / \real{#1}]{#2}}}
\newcommand\centerCropEx[1]{\centerCrop{0.2}{#1}}

\newtheorem{theorem}{Theorem}[section]

\newtheorem{lemma}[theorem]{Lemma}

\newcommand\newcontent[1]{{#1}}

\newcounter{obscount}
\newcommand{\observation}{ %
\addtocounter{obscount}{1} %
\textbf{Observation \arabic{obscount}: }}

\DeclareMathOperator{\similarity}{sim}

\begin{document}
\setcounter{page}{1}
\title{Progressive Compressed Records:\\ Taking a Byte out of Deep Learning Data}

\author{Michael Kuchnik}
\affiliation{%
  \institution{Carnegie Mellon University}
}
\email{mkuchnik@cmu.edu}

\author{George Amvrosiadis}
\affiliation{%
  \institution{Carnegie Mellon University}
}
\email{gamvrosi@cmu.edu}

\author{Virginia Smith}
\affiliation{%
  \institution{Carnegie Mellon University}
}
\email{smithv@cmu.edu}

\begin{abstract}
Deep learning accelerators efficiently train over vast and growing
amounts of data, placing a newfound burden on commodity networks and storage
devices.
A common approach to conserve bandwidth involves resizing or compressing data prior to training.
We introduce \textit{Progressive Compressed Records} (PCRs), a data format that uses compression to reduce the overhead of fetching and transporting data, effectively reducing the training time required to achieve a target accuracy. 
PCRs deviate from previous storage formats by combining progressive compression with an efficient storage layout to view a single dataset at multiple fidelities---all without adding to the total dataset size.
We implement PCRs and evaluate them on a range of datasets, training tasks, and hardware architectures. %
Our work shows that: (i) the amount of compression a dataset can tolerate
exceeds 50\% of the original encoding for many DL training tasks; (ii) it is possible to automatically and efficiently select appropriate compression levels for a given task; and (iii) PCRs enable tasks to readily access compressed data at
runtime---\textit{utilizing as little as half the training bandwidth} and thus potentially doubling training speed.
 \end{abstract}

\maketitle

\pagestyle{\vldbpagestyle}
\begingroup\small\noindent\raggedright\textbf{PVLDB Reference Format:}\\
\vldbauthors. Progressive Compressed Records: Taking a Byte out of Deep Learning Data. PVLDB, \vldbvolume(\vldbissue): \vldbpages, \vldbyear.\\
\href{https://doi.org/\vldbdoi}{doi:\vldbdoi}
\endgroup
\begingroup
\renewcommand\thefootnote{}\footnote{\noindent
This work is licensed under the Creative Commons BY-NC-ND 4.0 International License. Visit \url{https://creativecommons.org/licenses/by-nc-nd/4.0/} to view a copy of this license. For any use beyond those covered by this license, obtain permission by emailing \href{mailto:info@vldb.org}{info@vldb.org}. Copyright is held by the owner/author(s). Publication rights licensed to the VLDB Endowment. \\
\raggedright Proceedings of the VLDB Endowment, Vol. \vldbvolume, No. \vldbissue\ %
ISSN 2150-8097. \\
\href{https://doi.org/\vldbdoi}{doi:\vldbdoi} \\
}\addtocounter{footnote}{-1}\endgroup

\ifdefempty{\vldbavailabilityurl}{}{
\vspace{.3cm}
\begingroup\small\noindent\raggedright\textbf{PVLDB Artifact Availability:}\\
The source code, data, and/or other artifacts have been made available at \url{\vldbavailabilityurl}.
\endgroup
}

\section{Introduction}%
\label{sec:introduction}
Deep learning training consists of three key components: the data loading pipeline (storage), the training computation (compute), and, in the case of parallel or distributed training, the parameter synchronization (network). A plethora of work has investigated scaling deep learning from a compute- or network-bound perspective~\citep[][]{NIPS2012_4687,cui2016geeps,tensorflow2015-whitepaper,184014,jouppi2017datacenter,lim20183lc,zhu2018tbd,alistarh2017qsgd,lin2017deep,wen2017terngrad,wangni2018gradient,zhang2017poseidon}. However, little attention has been paid to scaling the storage layer, where training data is sourced.

\begin{figure}[t]
  \begin{center}
  \includegraphics[width=.35\textwidth]{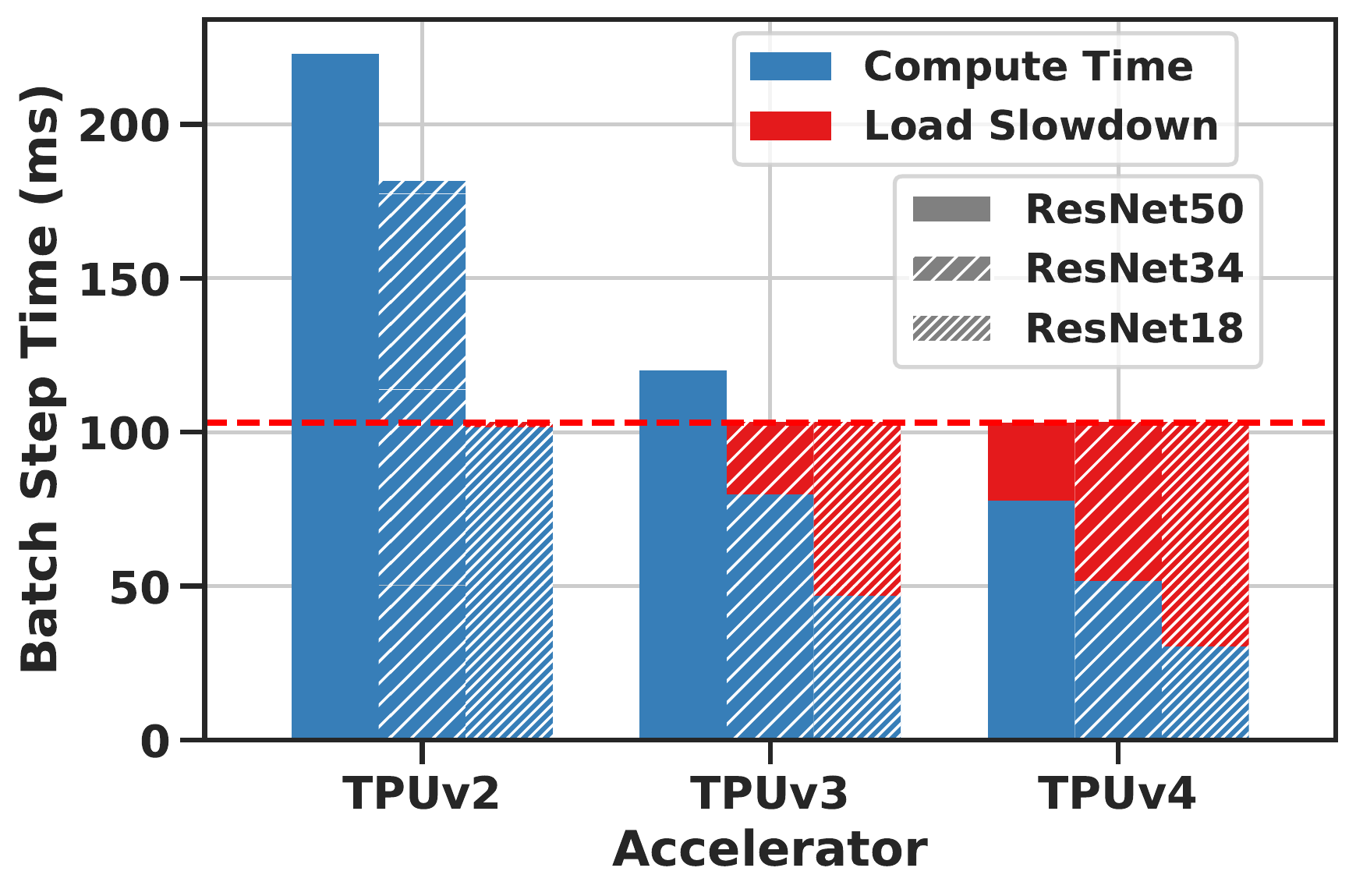}
  \end{center}
  \caption{Three generations of single-node TPU hardware performance on
  ResNet/ImageNet (batch of 1024).
  For illustrative purposes, a 1GiB/s limit of data bandwidth is shown
  ({\color{red}red} line) with the corresponding slowdown. ResNet18 on a TPUv3
  node can pull over 21k images/second, or over 2.1GiB/s---a challenge for both storage and network.}
  \label{fig:simulatedslowdown}%
\end{figure}

Current hardware trends point to an increasing divide between compute and the
rest of the hardware stack, including network or storage
bandwidth~\citep{li2016hippogriffdb,lim20183lc,kurth2018exascale} and main
memory~\citep{Wulf:1995:HMW:216585.216588,kwon2018beyond,sebastian2020memory}.
Indeed, in the last decade, the amount of compute available to deep learning (DL) has increased exponentially~\citep{aiindex2019}. %
However, I/O bandwidth has been slower to evolve, potentially resulting in I/O
becoming a dominating factor in the overall runtime of deep learning
tasks~\cite{li2016hippogriffdb,kurth2018exascale,wang2020systematic,MLPerfHPC07}.
Recent evidence highlights that, while accelerators are the workhorse of any ML
fleet, 30\% of resources are spent on the data
pipeline~\cite{murray2021tfdata} in industrial workloads and up to 65\% of epoch time
is spent in data pipelines in research workloads~\cite{dsanalyzer}.
While I/O is only a part of the data pipelines,
it has the possibility to create bottlenecks
and thus lower end-to-end ML training efficiency.

The resource requirements for I/O can be prohibitive, either due to cost,
scaling limits of filesystems, or quality of service requirements.
Figure~\ref{fig:simulatedslowdown} shows that this can be problematic even at
small scales.
We mark the time it takes to fetch one training batch using 1GiB/s of data
bandwidth with a dashed red line, since cloud instances are typically limited to
1--4 GiB/s of network~\cite{GCPNetworkBandwidth} \newcontent{and 1GiB/s of disk
bandwidth~\cite{GCPDiskBandwidth}}.
We have trained ResNet~\cite{he2016deep} models of varying complexity (ResNet18 being the least complex) using the ImageNet~\cite{imagenet_cvpr09} dataset. We find that TPUv2~\cite{jouppi2020domain}, Google's second version of their custom AI accelerator, completes computation on a given training batch within the time it takes to fetch the next batch (according to the dashed red line). However, the least complex of the models, ResNet18, is expected to toe that line. TPUv2 is 6 years old, and the third version of Google's accelerator manages to pack enough compute to speed up batch computations so that two out of the three ResNet models spend more time fetching the next training batch than computing on the current one.
This is projected to become a problem for even more complex models, according to
publicly available per-core performance numbers for TPUv4~\citep{MLPerf07}, which we include in the figure.

These trends highlight that I/O, if left unaltered, stands to dominate training costs.
Worse, if the underlying data used in training were to get larger, the problem
could become much worse.
Contrary to conventional wisdom, datasets like ImageNet consist of \textit{small} images with an average image resolution that is $7\times$ smaller than industry workloads~\cite{reddi2020mlperf}, and thus the combination of training \textit{rates} and \textit{data sizes} are likely to increase.

To cope with the divide between compute and I/O, architects have turned to
hardware/software co-design in an attempt to meet scaling and efficiency
goals~\citep{jouppi2020domain,kumar2020exploring}.
Two common, complementary approaches to optimize the storage layer include
\textit{caching}~\cite{quiver,dsanalyzer} and \textit{reducing data
volume}~\cite{karras2017progressive,DeepNJPEG}.
From the caching point of view, I/O pressure can be relieved by keeping a subset
of the workload in memory, and optimizing access patterns to hit in the cache
and thus avoid I/O.
However, for large datasets, one must choose between prohibitively large cache sizes or
weaker forms of sampling~\cite{meng2017convergence}.
From the point of view of data reduction, practitioners can resize images to
reduce their size.
However, choosing the resize parameters is task-specific, and is subject to
error~\cite{fixing_train_test}, which diminishes task performance.
In this work, we show deep neural network
training is amenable to a range of JPEG compression; however, unlike resizing,
this fact can be exploited as a
\textit{mechanism for dynamic data reduction}.
Notably, we show that different training tasks---a product of the dataset, model(s), and
objective---can tolerate different compression levels
(Section~\ref{sec:experiments}), and it is non-trivial to determine these levels
\textit{a priori}, which motivates a need for dynamic compression.

In this work, we propose \textit{Progressive Compressed Records} (PCRs), a novel data format that reduces the bandwidth cost associated with DL training. Our approach leverages a compression technique that decomposes each data item into a series of \textit{deltas}, each progressively increasing data fidelity. PCRs use deltas to \textit{dynamically} access entire datasets at a fidelity suitable for each task's needs, while avoiding inflating the dataset's size. This allows training tasks to control the trade-off between training data size and fidelity. The careful layout of PCRs ensures that data access is efficient at the storage level. Switching between data fidelities is lightweight, enabling adaptation to changing runtime conditions. Using PCRs for a variety of common deep learning models and image datasets, we find that bandwidth (and therefore training time) can be reduced by $2\times$ on average relative to simple JPEG compression without affecting model accuracy. This can allow for a larger fraction of the dataset to be cached in memory, complementing prior work~\cite{quiver}. Overall, we make the following contributions:

\begin{enumerate}[leftmargin=1.5em,topsep=0pt,itemsep=-1ex,partopsep=1ex,parsep=2ex]
  \item We introduce \textit{Progressive Compressed Records} (PCRs), a novel storage format for training data. PCRs combine progressive compression and careful data placement to enable tasks to \textit{dynamically} choose their data fidelity, increasing the effective training bandwidth.

  \item We demonstrate that by using PCRs, training speed can be improved by
    1.6--2.6$\times$ by selecting a lower data fidelity. These speedups are conservative given that the `raw' images we use are in fact already JPEG compressed; speedups are thus likely to be even larger for uncompressed datasets.

  \item In experiments with multiple architectures and  large-scale image datasets, we show deep neural network training is robust to data compression in terms of test accuracy and training loss.

  \item We introduce methodology for choosing the particular data fidelity necessary for a task, as well as a tuning heuristic that can be applied automatically. Using PCRs, our method can
dynamically switch between multiple data fidelities while training without loss of accuracy. %
\end{enumerate}

\begin{figure}[t!]
  \includegraphics[width=.48\textwidth]{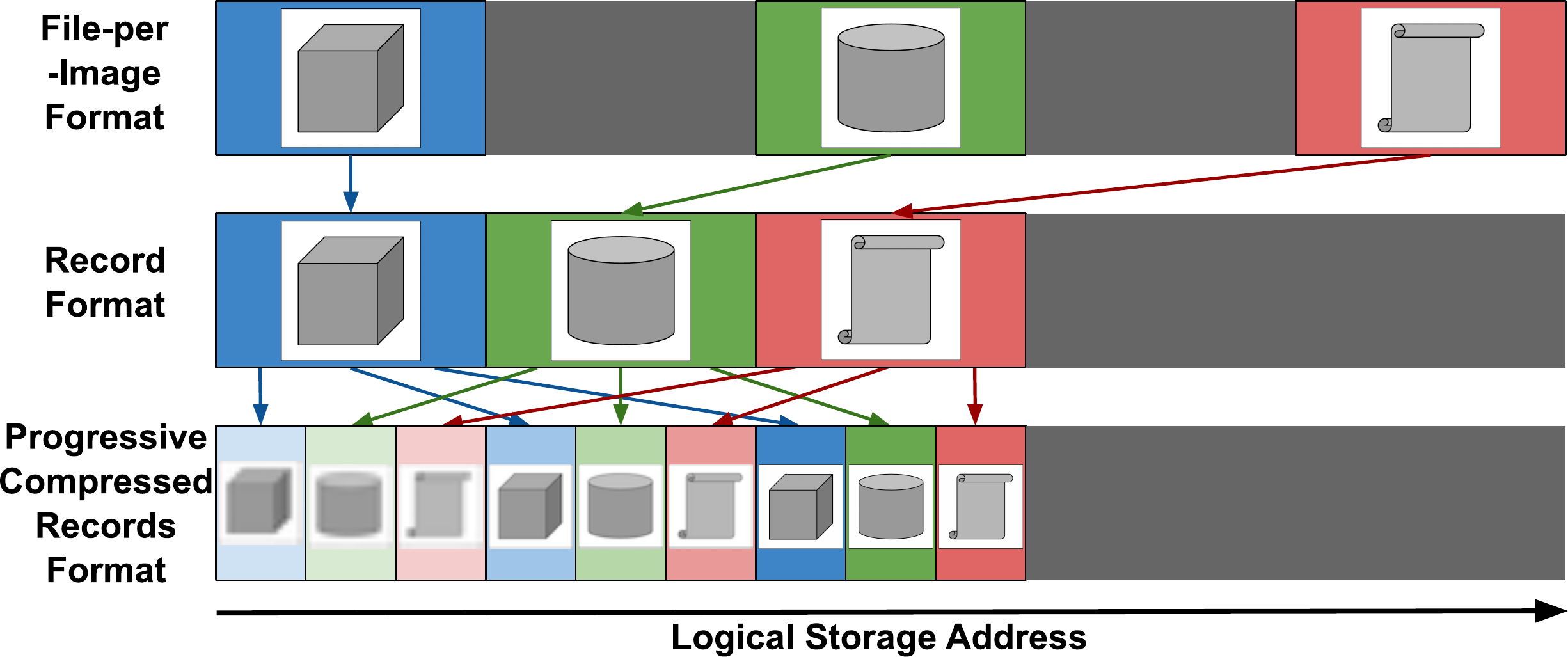}
  \caption{
    The design space of ML image file formats.
    File-per-Image (\textbf{Row 1}) formats randomly read files in the logical storage address space.
    Record layouts (\textbf{Row 2}) batch a subset of files into a single, large sequential read, promoting locality in address accesses.
    PCRs (\S\ref{sec:PCR}, \textbf{Row 3}) group by image fidelity (3 shown) to maintain the sequential
    behavior of record layouts while enabling dynamic compression without the need for duplicating data.
    Reading a full record recovers the full data fidelity for all images.
    Metadata (not shown) is small and can be kept in memory.
  }%
  \label{fig:PCR_encoding}
\end{figure}

\section{Background}%
\label{sec:background}
Advances in scalable training methods, software, and compute (e.g.,
accelerators) suggest that the time spent on training computation is decreasing
relative to time spent accessing
data~\citep{li2016hippogriffdb,lim20183lc,kurth2018exascale,wang2020systematic,dsanalyzer,quiver}.
Data bandwidth is therefore an increasingly important bottleneck to consider for machine
learning pipelines.
Two complementary concepts make up the process of storing/loading data: the 
\textit{data layout}, which helps to utilize the bandwidth of the underlying storage system, and
the \textit{data representation}, which can increase bandwidth by reducing the amount of data transferred. %
In this work, we develop a novel, flexible, and efficient storage format, PCRs, by
combining a data representation (progressive compression) with an efficient
data layout. Our work serves to lower three fundamental storage costs: storage
capacity, storage operations (IOPS), and storage/network bandwidth.

\paragraph{Record Layouts.} Learning from data requires sampling points from a training set. In the context of image data, perhaps the simplest way to access data is with a \textit{File-per-Image} layout, such as
PyTorch's ImageFolder, which can cause small, random accesses that are
detrimental to storage bandwidth and latency, \newcontent{while also stressing
filesystem metadata~\cite{panasas_fs,PyTorchWebDataset,largedirectory}.}
\textit{Record layouts}, such as TensorFlow's TFRecords~\citep{TFRecords},
MXNet's ImageRecord~\citep{RecordIOTutorial}, or even TAR files~\citep{PyTorchWebDataset},
attempt to alleviate this problem by batching data together into records.
Record layouts increase performance (i.e., read rate) by exploiting locality (Figure~\ref{fig:PCR_encoding}).
Our experiments indicate a single epoch can take $25\times$ longer with
File-per-Image formats compared to reading
Record formats---limiting their practicality without caching.
To achieve randomness, each Record is read into memory, where it may be shuffled with other
Records and broken into minibatches by the data loader.
While Record layouts improve over File-per-Image layouts, they are designed to
store data at a specific fidelity level, thus requiring multiple copies of each
dataset at different fidelities in order to realize efficient training across tasks.
In this work, our aim is to combine the efficiency of Record layouts with dynamic compression schemes (described below) to offer quick, easy access to data at multiple fidelity levels.

\paragraph{Image Compression.} Compressed forms are commonly used to represent training
data. JPEG~\citep{wallace1992jpeg} is one of the most popular
formats for image compression and is used ubiquitously in
machine
learning~\citep{imagenet_cvpr09,ILSVRC15,lin2014microsoft,Everingham10}.
Most compression formats (including JPEG) require
the compression level to
be set at encoding time, which often results in choosing this parameter in an application-agnostic manner.
However, as we show in Section~\ref{sec:experiments}, it is difficult to set the compression level for deep learning training without over- or
under-compressing, as the appropriate level may vary significantly across training tasks. Current approaches resort to storing multiple copies of the dataset at different compression levels, 
particularly for applications using multiple data fidelities within a single
training task~\citep{karras2017progressive}. This is infeasible for larger
datasets. For example, we find duplicating a 2GiB dataset at 9 resolutions can \textit{amplify the dataset size by
\textbf{$\mathbf{1.5{-}40\times}$} and require hours
of extra processing time}.
Terabyte-sized datasets rely on distributed frameworks to reduce dataset
creation from weeks to days~\cite{apachebeam}.

In Figure~\ref{fig:jpeg_algo}, we provide a simplified illustration of the JPEG
algorithm.
First, an image is split into blocks of size $8\times8$, which are then converted
into the frequency domain.
The low frequencies (top left of the matrix)
store the bulk of
the perceptually relevant content in the image.
Quantization, which discards information from the block and results in compression, is used to diminish the high frequency values, compressing the data.
\textit{Sequential formats} serialize the image's blocks from left to right, top to bottom. Decoding the data is simply a matter of inverting this process. %

\begin{figure}
  \centering
  \hspace{17pt}%
  \begin{subfigure}[b]{0.7\linewidth}
  \includegraphics[width=.99\linewidth]{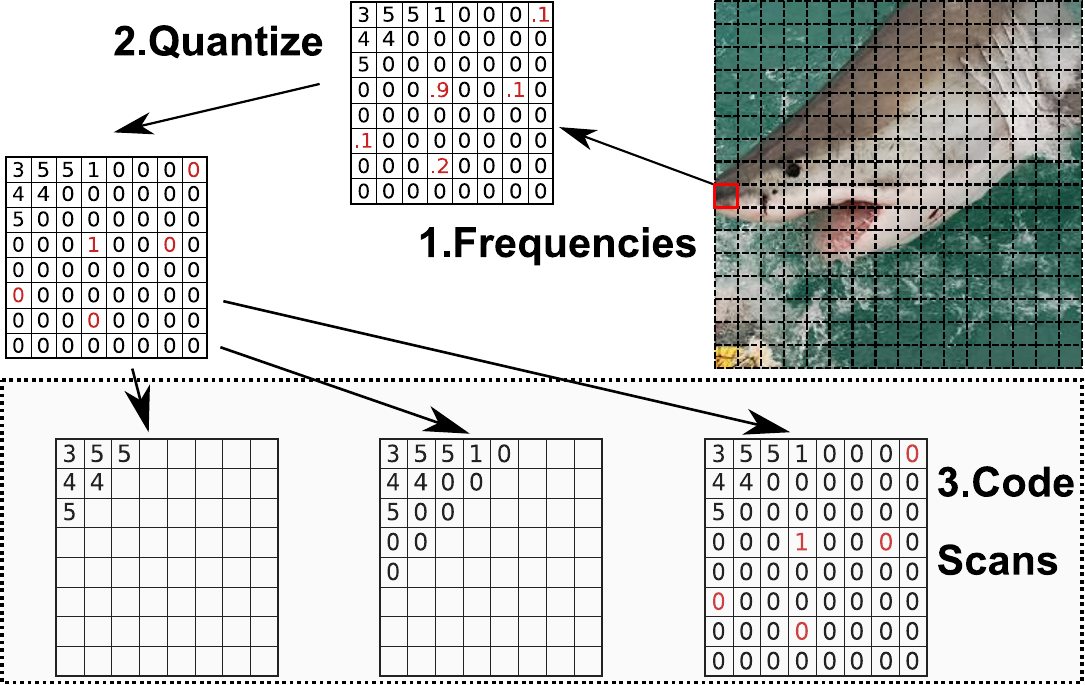}%
  \end{subfigure}%
  \newline
  \begin{subfigure}[b]{0.7\linewidth}
  \begin{subfigure}[t]{0.33\linewidth}%
    \center%
    \adjustbox{trim={0.0\width} {0.25\height} {0.5\width} {0.25\height},clip,center}{\includegraphics[width=2.0\linewidth]{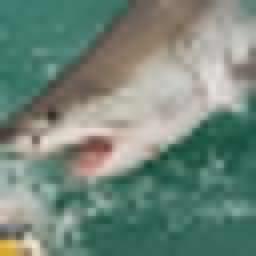}}
    \vspace{-1.0em}%
    \caption{Scan 1}%
  \end{subfigure}%
  \begin{subfigure}[t]{0.33\linewidth}%
    \center%
    \adjustbox{trim={0.0\width} {0.25\height} {0.5\width} {0.25\height},clip,center}{\includegraphics[width=2.0\linewidth]{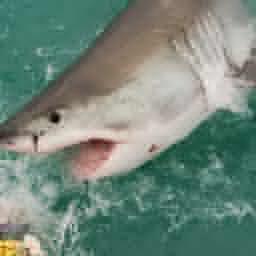}}
    \vspace{-1.0em}%
    \caption{Scan 3}%
  \end{subfigure}%
  \begin{subfigure}[t]{0.33\linewidth}%
    \center%
    \adjustbox{trim={0.0\width} {0.25\height} {0.5\width} {0.25\height},clip,center}{\includegraphics[width=2.0\linewidth]{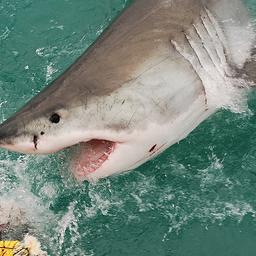}}
    \vspace{-1.0em}%
    \caption{Scan 10}%
  \end{subfigure}%
  \end{subfigure}%
  \caption{%
    \textbf{Top:}
    JPEG carves an image into blocks,
    which are then converted into frequencies, quantized, and serialized.
    Progressive compression writes out key coefficients from each block before re-visiting the block.
    \textbf{Bottom:}
    Higher scans (a$\rightarrow$c) have greater fidelity from more frequencies.
    }%
  \label{fig:jpeg_algo}%
\end{figure}

\paragraph{Progressive Image Compression.} Progressive compression is an alternative to standard image compression, which---combined with an additional rearrangement of data (Section~\ref{sec:PCR})---forms the
basis of the idea behind PCRs.
\textit{Progressive} formats allow data to be
read at varying degrees of compression without duplication.
As an example, over slow internet connections, these formats allow images to be
decoded \textit{dynamically} as they are transmitted over the network.
With the sequential case, data is ordered by blocks, and thus partially reading
the data results in ``holes'' in the image for unread blocks~\citep{wallace1992jpeg}.
Dynamic compression schemes interleave information (\textit{deltas}) from each
block, allowing all blocks (and thus the entire image) to be approximated
without reading the entire byte stream.
As progressive formats are simply a different traversal of the set of quantization
matrices, they contain the same information as
sequential JPEG~\citep{JPEGTranlibjpeg} and are actually often smaller in
practice.
As we depict in Figure~\ref{fig:jpeg_algo}, while non-progressive formats
serialize the image matrix in one pass, progressive formats serialize the
matrix in disjoint groups of deltas which are called \textit{scans}.
Scans are ordered by importance (e.g., the first few scans improve quality
more than subsequent scans).
Thus, any references to images generated from scan $n$ will implicitly assume
that the image decoder had access to prior scans.
Progressive formats exist not only for images, but also for modalities such as audio~\citep{OggVorbis} and video~\citep{schwarz2007overview}.
\section{Progressive Compressed Records}%
\label{sec:PCR}
We present \textit{Progressive Compressed Records} (PCRs), a novel storage format that reduces data bandwidth for ML training.  We specifically explore
PCRs for training deep neural networks with image data, but note that the ideas
behind PCRs could be readily extended to other modalities (e.g.,
audio~\citep{OggVorbis} or video~\cite{schwarz2007overview}), and
compression strategies (e.g.,
cropping~\citep[][]{952804}, interlaced PNG, or neural compression~\cite{toderici2017full}).
PCRs define a data layout that ensures bandwidth is fully utilized, and a data representation that permits accessing data at multiple levels of fidelity with minimal overhead.

PCRs are optimized to allow the entire training dataset to be read at a given fidelity.
To achieve this, data is rearranged into \textit{scan groups}, i.e., collections
of deltas of the same fidelity that are stored together in the address space.
To dynamically increase the fidelity of data read and decoded, a task then merely needs to read subsequent scan groups until the desired fidelity level is reached.
PCRs differ from other formats (e.g., TFRecord, RecordIO) because PCRs allow these
lower fidelity versions of each record to be accessed efficiently (without
space/throughput tradeoffs).
This efficiency is achieved by using progressive compression and changing the order that data is stored and accessed.
Space overhead for PCR conversion is negligible; PCRs are usually 5\% smaller than TFRecords.
This is because record format size is dominated by the image payload,
which is \textit{simply rearranged} with progressive compression.
File size
differences stem from the efficiency of entropy coding in JPEG, which typically
has higher compression ratios over progressive layouts~\cite{bookofspeed}.
Since PCRs allow a lower fidelity version of the entire dataset to be
accessed efficiently, they can drop the effective size and utilized
bandwidth of a record by a
factor of 2--10$\times$.

\begin{figure}
 \includegraphics[width=.48\textwidth]{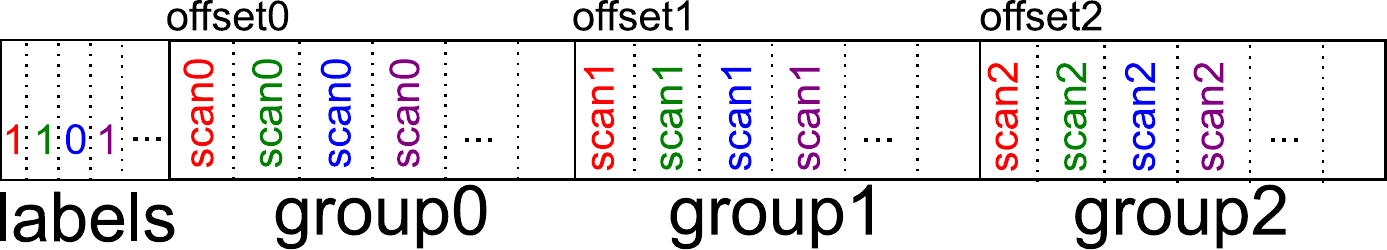}
 \caption{PCRs encode label metadata followed by all scan groups.
 Accessing the dataset at lower quality requires reading up to a given scan
  group. Reading all scan groups returns the full quality data, and decodes to
  identical bytes as the conventional JPEG format.
 }%
 \label{fig:PCR_encoding_2}
\end{figure}

Figure~\ref{fig:PCR_encoding_2} depicts the PCR format as it is organized on the storage medium.
PCRs logically consist of two parts: \textit{metadata} and \textit{data}.
Metadata consists of sample metadata (e.g., labels, bounding boxes) as well as
PCR metadata (e.g., mapping of files and scans to storage addresses).
Metadata is small in size (e.g., an image label can be
represented by a 32 bit \texttt{int}, while an image is 100kiB or more) and can be kept in a database, mapped in memory, or
pre-pended to PCRs (for per-sample metadata e.g., labels).
The data, which is orders of magnitude larger, consists of the images themselves,
organized in terms of increasing levels of fidelity.
Each fidelity level for an image is a \textit{scan} and each grouping of images of the
same fidelity is a \textit{scan group}.
For example, the scan 1 representation of the shark in Figure~\ref{fig:jpeg_algo} can be retrieved by reading its data from scan group 1.
Likewise, the scan 3 representation will be available once the records up to scan group 3 are read, and the reconstructed representation will be of higher fidelity than that of scan 1.
As scan groups consist of scans of the same fidelity, every image contained
in a record of the same group offset is available at the same fidelity.
Users of PCRs can read data at a given fidelity by simply reading the
on-disk byte stream from the start of the PCR
to the end of the corresponding scan group.
This way of dynamically selecting data fidelity allows for bandwidth savings without re-encoding the data.

\section{Design}%
\label{sec:design}
The key insight behind PCRs is that, for storage or network I/O bound workloads,
training tasks can be sped up by reducing data fidelity (and, thus, the amount of data read) to match the available I/O bandwidth.
Figure~\ref{fig:predictedThroughput} shows the training throughput obtained by
using PCRs vs.\ the traditional TFRecord format.
As a fair point of comparison, we use our
\texttt{tf.data}~\citep{TFData,murray2021tfdata} implementation, and thus only
the dataset reader operation has changed.
As we describe in Section~\ref{sec:speedup_analysis},
PCRs can reduce the bytes read per image, and thus
proportionally increase the throughput of the end-to-end training process (up to
the compute limits of the accelerator).
However, a speedup is only possible if the CPU overhead introduced by PCRs can
be absorbed by the machine (Section~\ref{sec:decoding_overhead}).
The final part of the design of PCRs involves choosing the image quality level
automatically, which we describe in Section~\ref{sec:autotuning_quality}.

\subsection{I/O Speedup Analysis}%
\label{sec:speedup_analysis}

\paragraph{End-to-End Slowdown.}
Amdahl's Law~\cite{amdahlslaw} states that if $p$ fraction of a program is
waiting for data (see red/blue bars in
Figure~\ref{fig:simulatedslowdown}), a $\frac{1}{1-p}$ speedup can be obtained
by removing the wait.
Recent work determines possible speedups empirically by finding the gap between the data preparation rate and the I/O rate~\cite{dsanalyzer}.
However, because PCRs are dynamic, it is important to know what PCR
configurations can actually lead to a speedup (i.e., what scan group to select).
Although we observe over $500\times$ $\text{max}/\text{min}$ range on ImageNet, mean size-per-sample, $\E_{x \sim \train}[s(x)]$, is all that is required for an accurate performance model.
We tabulate this information in Table~\ref{fig:PCR_expected_size} and motivate the model below.

\begin{table}[h]
  \caption{Image size reduction for various scans and the size of an average
  image, which can be combined to predict I/O speedups.
  Scan 10 is approximately the same size as baseline JPEG, and scan 5 is roughly
  half.
  }%
  \centering
  \begin{tabular}{lccccc} \toprule
      \textbf{Dataset}  & Scan 1 & Scan 2 & Scan 5 & Scan 10 & $\E_{x \sim \train}[s(x)]$ \\ \midrule
      \textbf{ImageNet}  & $16\times$ & $7\times$ & $2\times$ & $1\times$ & 110kB \\
      \textbf{HAM10000}  & $30\times$ & $15\times$ & $3\times$ & $1\times$ & 250kB \\
      \textbf{Cars}     & $14\times$ & $6\times$ & $2\times$ & $1\times$ & 110kB\\
      \textbf{CelebAHQ}  & $7\times$ & $4\times$ & $3\times$ & $1\times$ & 80kB \\
      \bottomrule
  \end{tabular}%
  \label{fig:PCR_expected_size}
\end{table}

\paragraph{Input Pipeline Throughput.}
Using closed-system Little's Law~\cite{littles_law,harchol2013performance} and basic assumptions on the
characteristics of a storage system (i.e., the cost of large reads is proportional to bytes read),
the image throughput, $X$ (e.g., images per second), of an image pipeline is explained
by the equation: $X=\cfrac{W}{\E_{x \sim \train}[s(x)]}$, where $W$ is the
bandwidth and $\E_{x \sim \train}[s(x)]$ is the mean image size (average size of
an image sample).
The number of bytes in a record (a large read) is, by linearity of expectation, the number of
images, $n$, times the average image size.
Thus, the amortized cost per image (dividing by $n$) is the average image size, and time taken is proportional to $W$.
If a model/accelerator trains at 500 images/second (a function of the resized and cropped input-matrix dimensions~\cite{krizhevsky2012imagenet,simonyan2014very,he2016deep}), we can conclude that, using
ImageNet images, it will use up to $110\text{kB} *
500\text{s}^{-1}=55\text{MB/s}$ (Table~\ref{fig:PCR_expected_size}), as demonstrated in Figure~\ref{fig:predictedThroughput}.

\paragraph{Dataset-Level Bounds.}
To remove dependence on the accelerator's speed, the equation can be applied on
both scans and the original data: Theorem~\ref{speedup_thm}
presents the asymptotic bounds for the impact of
data reduction on training speedups.
It is derived by noticing that, when a system with fixed $W$ is bound by
the throughput of the I/O subsystem, $X$, one can calculate the speedup ratio
$\hat{X} \mathbin{/} X$, where $\hat{X}$ is using a reduced image size.
In sum, reducing the mean data read results in proportionally
higher I/O throughput, which results in proportional speedup on
I/O bound workloads.
For example, using Table~\ref{fig:PCR_expected_size}, which displays the ratios, one can anticipate that a
$2\times$ speedup would be seen on ImageNet with scan 5.
We defer interested readers to our supplemental material for a more
formal discussion of the performance modeling.

\begin{theorem}%
  \label{speedup_thm}%
  If a training pipeline is data bound,
  then the maximum
  achievable system speedup, $S_\text{max}$, for switching from dataset $\train$ to $\train'$ is
  the ratio of mean sample sizes, $s(x)$:
  $$S_\text{max}\left(\train, \train'\right)=\cfrac{\E_{x \sim \train}[s(x)]}{\E_{x' \sim \train'}[s(x')]} \,.$$ 
\end{theorem}

\begin{figure}[t]
  \begin{center}
  \leftRightCrop{0.0}{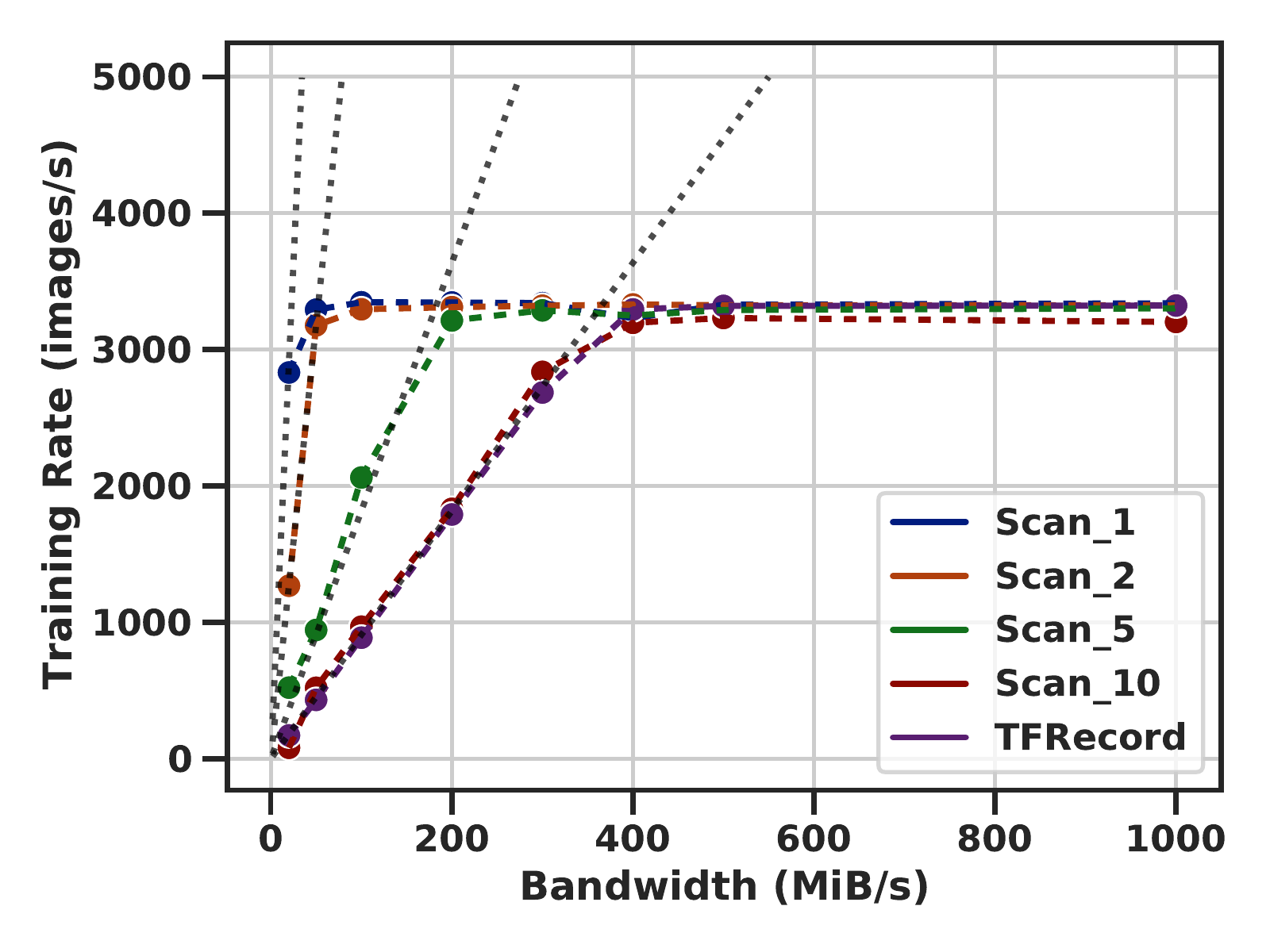}{0.95}{0.0}
  \end{center}
  \caption{%
      The training rate of a 10-node \texttt{TitanX} GPU cluster with a ResNet-18 workload using
      PCRs (the scans) and TFRecord.
      The throughput of the training process is dominated by I/O bandwidth until
      the compute limit of the GPU is reached.
      PCRs at scan 10 are approximately the same size as TFRecord,
      and thus have similar performance.
      Predicted rates are shown.
    }%
  \label{fig:predictedThroughput}%
\end{figure}

\newcontent{%
\subsection{Data Preparation Decoding Overhead}%
\label{sec:decoding_overhead}%
Changing the dataset encoding inevitably changes data-preparation work, which consumes CPU resources and must be managed.
The cost of progressive compression is dependent on the image size, scan
configuration, and decoder implementation~\citep{yan2017customizing}.
To analyze the cost of this progressive decoding, we test the rate of decoding $30k$
images using one thread---these rates can then be multiplied by the number of
cores on the machine in a parallel training scenario.
The figures are shown in Table~\ref{fig:PCR_decoding_cost}, indicating that
decoding with a subset of scans can be comparable to traditional
decoding.
On the other hand, using all the scans is over $2\times$ slower than traditional
decoding.
On a many-core machine (e.g., a 32+ core setup like the one used in Section~\ref{sec:evaluation_setup}), this overhead can be absorbed by idle cores---in Figure~\ref{fig:predictedThroughput}, we do not see any
slowdown by using PCRs relative to baseline JPEG TFRecords because each of the 10 machines can decode 3k images in the worst case.
However, for less core-heavy machines, we note three optimization paths to lower decoding overhead.
First, excessive and unused scans can be removed.
Second, using
spectral selection can lower decoding overhead.
Third, hardware acceleration can be used (Section~\ref{sec:discussion}).

\begin{table}[h]
  \caption{The single-core decoding rate (images/s) of various JPEG encodings
  across the datasets.
  Progressive decompression can be over $2\times$ more expensive than baseline decoding.
  }%
  \centering
  \begin{tabular}{lccccc} \toprule
    \textbf{Dataset} & Scan 1 & Scan 2 & Scan 5 & Scan 10 & Baseline \\ \midrule
    \textbf{ImageNet} & 433 & 412 & 340 & 146 & 419 \\
    \textbf{HAM10000} & 465 & 438 & 275 & 96 & 240 \\
    \textbf{Cars} & 266 & 240 & 225 & 127 & 268  \\
    \textbf{CelebAHQ} & 239 & 213 & 195 & 129 & 286 \\
    \bottomrule
  \end{tabular}%
  \label{fig:PCR_decoding_cost}
\end{table}
}%

\newcontent{%
\subsection{Autotuning Image Fidelity}%
\label{sec:autotuning_quality}
Lossy compression of input data creates concerns for the output of model training, and thus creates questions for how to select a tolerable scan group.
To analyze the effect of lower image fidelity, we observe that deep learning training is based on stochastic gradient descent, which involves
taking a ``step'' in the \textit{direction} (a vector) that improves the model.
If two datasets, $\train$ and $\train'$, yield the same direction, then they will yield the same model.
Therefore, we may intuitively find an alternative dataset $\train'$, which is
close to the original dataset $\train$ in terms of how the model views the
gradient direction of the loss function, $L$.
More formally, we want the \textit{angle} between gradient vectors to be
small.

To accomplish this, we freeze the model mid-training and empirically measure the gradient
direction, $\nabla_\theta L$, over the full fidelity dataset, $\train$, which
contains batches of images, $\mathbf{X}$, and labels, $\mathbf{y}$.
As the parameters are frozen, we can also measure the gradient on the lower fidelity
dataset, $\train'$, which has alternative images, $\mathbf{X}'$.
The angle between the lower fidelity dataset and the original dataset yields the similarity score, which ranges
between -1 and 1.
Maximizing similarity would yield an
\textit{identical} model.

$$\text{score}\left(\train, \train'\right)=\similarity\left(\nabla_\theta
L(\mathbf{X},\mathbf{y}),\nabla_\theta L(\mathbf{X}',
\mathbf{y})\right)$$
where similarity is the cosine similarity:
$$\similarity\left(\textbf{A}, \textbf{B}\right) =
\cos(\theta) = \cfrac{\textbf{A} \cdot \textbf{B}}{||\textbf{A}||
||\textbf{B}||}$$

We evaluate this procedure with HAM10000/ResNet, using 2560 images to estimate the gradient and
3 trials to get 95\% confidence intervals.
As shown in Figure~\ref{fig:PCR_similarity}, decreasing image fidelity decreases the similarity with
respect to the true gradient.
Scan 10 is bit-identical to the baseline dataset, and thus we observe maximum
similarity in that case.
Meanwhile, scan 1 has the lowest similarity, and the difference increases as the
model converges.
Given the gradients are well-behaved with respect to fidelity,
it is natural to parameterize scan tuning to match some level of
gradient similarity.

\begin{figure}
  \includegraphics[width=.49\textwidth]{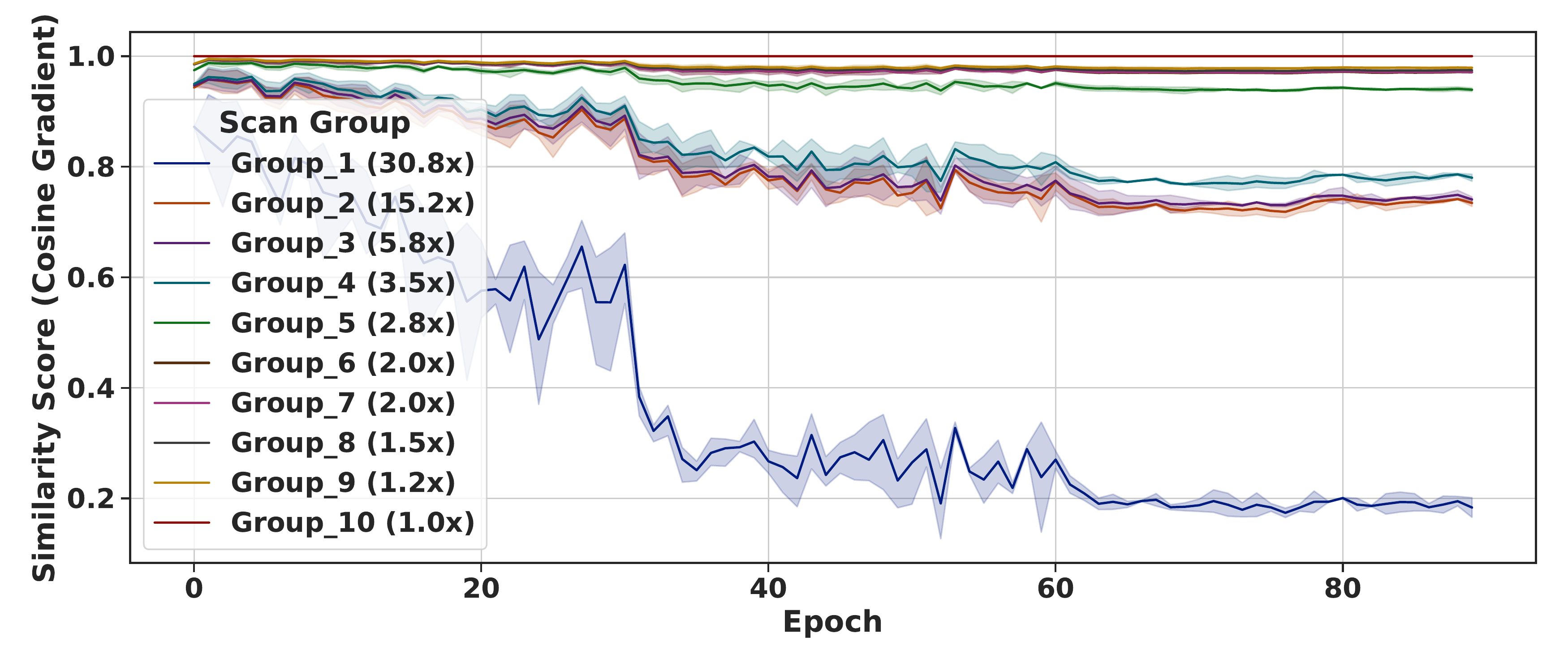}
  \caption{The similarity of gradients across epochs for
  ResNet/HAM10000 (max of 1.0).
  Legend shows bandwidth savings.
  Gradient similarity is exact for scan group 10 and decreases for other scans as the
  model converges.
  Higher
  quality scans lead to gradients within 0.1 of the baseline's gradient, and
  thus should result in similar final models.
 }%
 \label{fig:PCR_similarity}
\end{figure}

Using gradient similarity for autotuning quality requires choosing a minimum gradient similarity threshold for scans throughout
training, which is the main drawback of this approach.
As is shown in Figure~\ref{fig:PCR_similarity}, the similarity for high quality
data is bounded within a threshold of $0.8$---therefore, we find this threshold
a good default.
The computational cost of evaluation is on the order of tens of gradient steps,
and is proportional to the number of scans,
epochs, and minibatches used.
Our implementation tunes once every 20 epochs and does not tune for the first 5 epochs
because models are unstable during this
period~\cite{he2016deep,goyal2017accurate}.
As training progresses, low scans become too dissimilar (in terms of gradients) from the higher scans, and therefore are avoided.
This, in turn, allows the faster scans to apply a burst of speed to the training
process before fine-tuning at a higher fidelity (e.g., when learning rates drop).
Compared to static schedules~\cite{MLWeaving}, there is only one hyperparameter (the
threshold), which is independent of other schedules (e.g., learning rates),
and the parameter does not require validation data.
We leave tuning using QoS/congestion information~\citep{fouladi2018salsify} to future work.
}%

\section{Implementation}%
\label{sec:implementation}
There are three fundamental components in the PCR implementation: the encoder, the decoder, and the data loader.
The encoder transforms a set of JPEG files into a directory,
which contains: a database for PCR metadata and at least one PCR data file.
The decoder takes the directory as input and yields a set of
JPEG images, efficiently inverting a subset of the encoding.
The dataset is split into many records, and, thus, the training process is reading tens to hundreds of PCR data files per epoch.
The data loader is where the PCR decoding library interfaces with the inputs
provided to deep learning libraries (e.g.,
TensorFlow~\citep{tensorflow2015-whitepaper},
MXNet~\citep{chen2015mxnet}, PyTorch~\citep{NIPS2019_9015}).
Below, we describe these components in detail.

\paragraph{Encoding.}
Given a set of images, the PCR encoder breaks images into scans,
groups scans into scan groups, and sorts scan groups by fidelity.
Once groups are sorted, the PCR encoder can serialize groups while
taking note of their offsets (so that subsets may later be decoded).
The metadata (e.g., labels) is prepended to the serialized representation, and the resulting byte stream is written to disk.
Our implementation uses \texttt{JPEGTRAN}~\citep{jpegtran} to losslessly
transform JPEG images into progressive JPEG images.
With the default settings, each JPEG is broken up into $10$ scans.
The encoder scans the binary representation of the progressive JPEG files, searching
for the markers that designate the end of a scan group.
The encoder thus has access to all $10$ offsets within the JPEG files that can
be used to determine the boundaries between scan regions.
Forming scan groups requires grouping the scan regions with the same fidelity
together, which can be done in one pass over the set of images corresponding to
that PCR\@.
An index must be created for ungrouping the scans during decoding;
however, serialization libraries, such as
Protobuf~\citep{Protobuf}, handle both the packing and unpacking steps transparently.
\newcontent{%
As record format conversion can take hours (\S\ref{sec:background}), PCRs benefit from requiring only a single conversion for
multiple tasks.
The encoding time is within $2\times$ of conversion to
TFRecords in our implementation: for example, converting ImageNet takes 1.4
hours rather than 0.8 hours.
When using the widely available TFRecords
converter~\cite{tfrecords_converter},
our implementation for PCRs is
actually
faster
due to being parallelized---converters are
typically not optimized due to being one-time costs.
}%

\paragraph{Decoding.}
To decode a PCR file, the file's scan group offsets have to be located in the PCR metadata.
The offsets allow a partial read of the file, i.e., only the bytes of the desired scan group are read.
JPEG decoding requires serializing the image deltas of individual scan groups.
We terminate the byte stream with an
End-of-Image (EOI) JPEG token, which allows most JPEG decoders to render
the image with the available subset of scans.
The cost of these steps is negligible relative to that of the JPEG decoder, which
is the primary challenge facing PCRs.

\paragraph{Loader.} %
We implement PCR loaders using the DALI ExternalSource
operator~\citep{DALI}, as well as a C++ version compatible with
\texttt{tf.data}~\citep{TFData,murray2021tfdata}, including a Tensorflow Op~\citep{TFOp}.
SQLite and RocksDB are supported backing databases, and we support embedding
images and metadata in Protobufs or in ``raw struct'' form.
The PCR reader, like most readers, is cheap to evaluate; we can read over
400MiB/s using just a single CPU core.
This is because the bulk of the work is not computational, i.e., a file \texttt{read}
and a set of \texttt{memcpy} operations to re-arrange the PCR images.
Serialization libraries can add overhead (e.g., parsing, memory allocations);
however, ``raw struct'' formats avoid these entirely,
and flat formats minimize them~\citep{flatbuffers}.
Another design point is buffer allocation: in contrast to
traditional Record loaders, which can iteratively return individual data samples, PCRs must read (and
allocate) buffers for possibly the \textit{entire} record (10MiB+), since later
scan groups are used for even the first example (See $\S$\ref{sec:discussion}).
Thus, an optimized implementation of PCR loaders uses a double-buffer design,
where the buffers are re-used and read directly from disk using
\texttt{O\_DIRECT}.
Our implementations show that the main bottleneck with
using
PCRs is the image decode, which is downstream from the loader.
For I/O bound workloads, baseline and full-fidelity progressive record readers
perform the same (Figure~\ref{fig:predictedThroughput})
as image size differences are negligible (Section~\ref{sec:speedup_analysis}).

\section{Evaluation}%
\label{sec:experiments}
We evaluate the flexibility and efficiency of PCRs using a suite of large-scale image datasets.
We begin by describing our experimental setup (\S\ref{sec:evaluation_setup}) and
present an end-to-end evaluation of PCRs (\S\ref{sec:results}), demonstrating their ability to reduce training time. %
We show that dynamic compression is crucial because the appropriate level of compression varies across models and training tasks (\S\ref{sec:dynamic_compression}).
We explore metrics that can be used to explain the effectiveness of compression on a training task (\S\ref{sec:quality_analysis}),
introduce autotuning heuristics for dynamic training (\S\ref{sec:autotuning}),
and trace the speedups achieved by PCRs in terms of training time (\S\ref{sec:loading_rates}).
Our supplemental material contains additional experimental details and
training plots (e.g., training loss, other datasets).

\begin{table}[h]
  \caption{PCR dataset size and
  record count information.
  Datasets vary in terms of number of images, their JPEG quality, and the image
  sizes.
  Some datasets, such as HAM10000, have image sizes larger than average.
  Record sizes concentrate around the dataset size divided by the record count.
  }%
  \centering
  \begin{tabular}{lccccc} \toprule
    \textbf{Dataset} & Records & Images &
    Size & Quality & Classes \\ \midrule
      \textbf{ImageNet}  & 1251 & 1281167 & 129GiB & 91.7\% & 1000\\
      \textbf{HAM10000}  & 125 & 8012 & 2GiB & 100\% & 7 \\
      \textbf{Cars}  & 63 & 8144 & 887MiB & 83.8\% & 196 \\
      \textbf{CelebAHQ}  & 93 & 24000 & 2GiB & 75\% & 2 \\
      \bottomrule
  \end{tabular}
  \label{fig:PCR_record_info}%
\end{table}

\begin{figure*}[t]
  \centering
  \begin{subfigure}[t]{0.25\textwidth}
    \leftRightCrop{0.015}{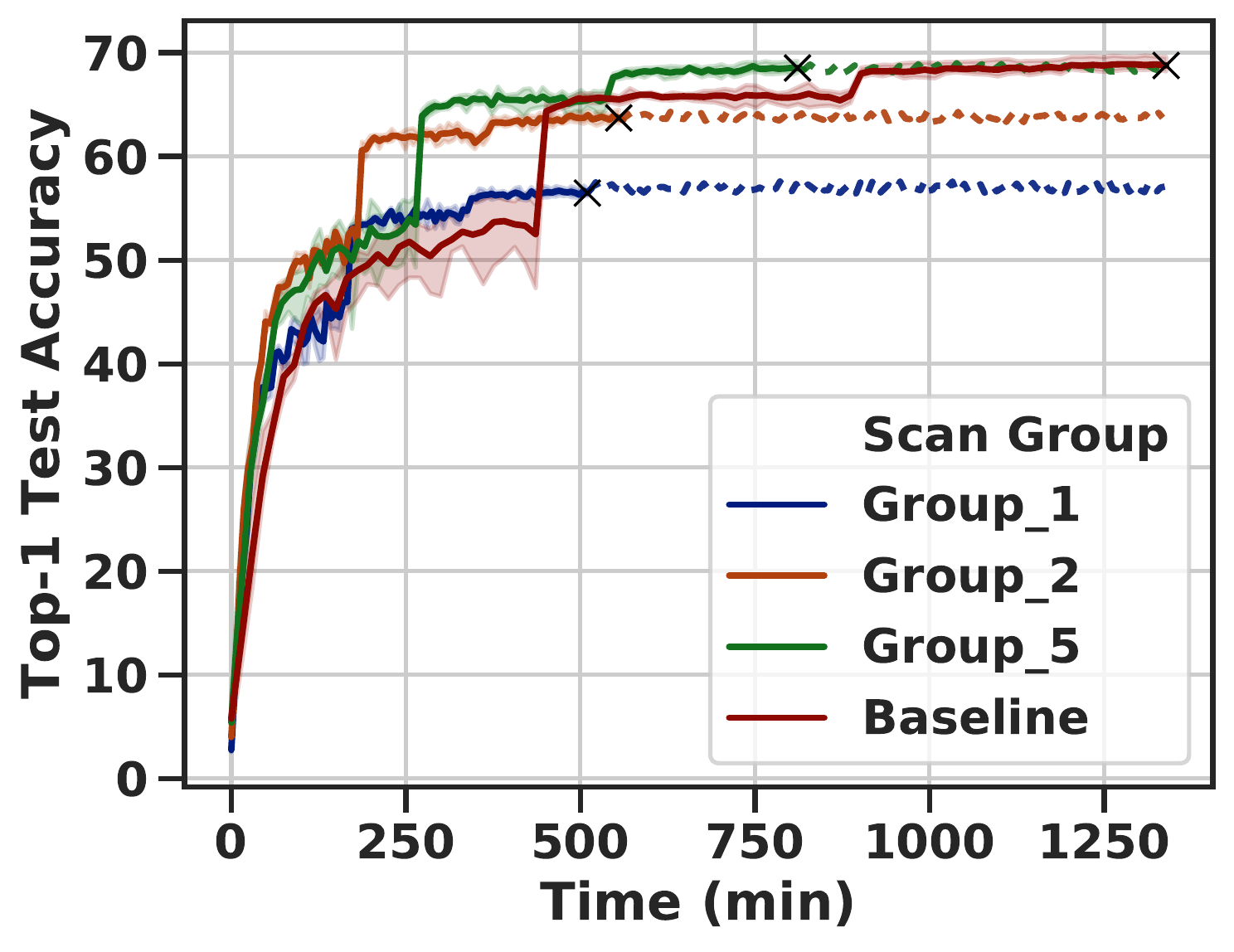}{1.03}{0.0}
    \caption{ImageNet ResNet}
  \end{subfigure}%
  \begin{subfigure}[t]{0.25\textwidth}
    \leftRightCrop{0.055}{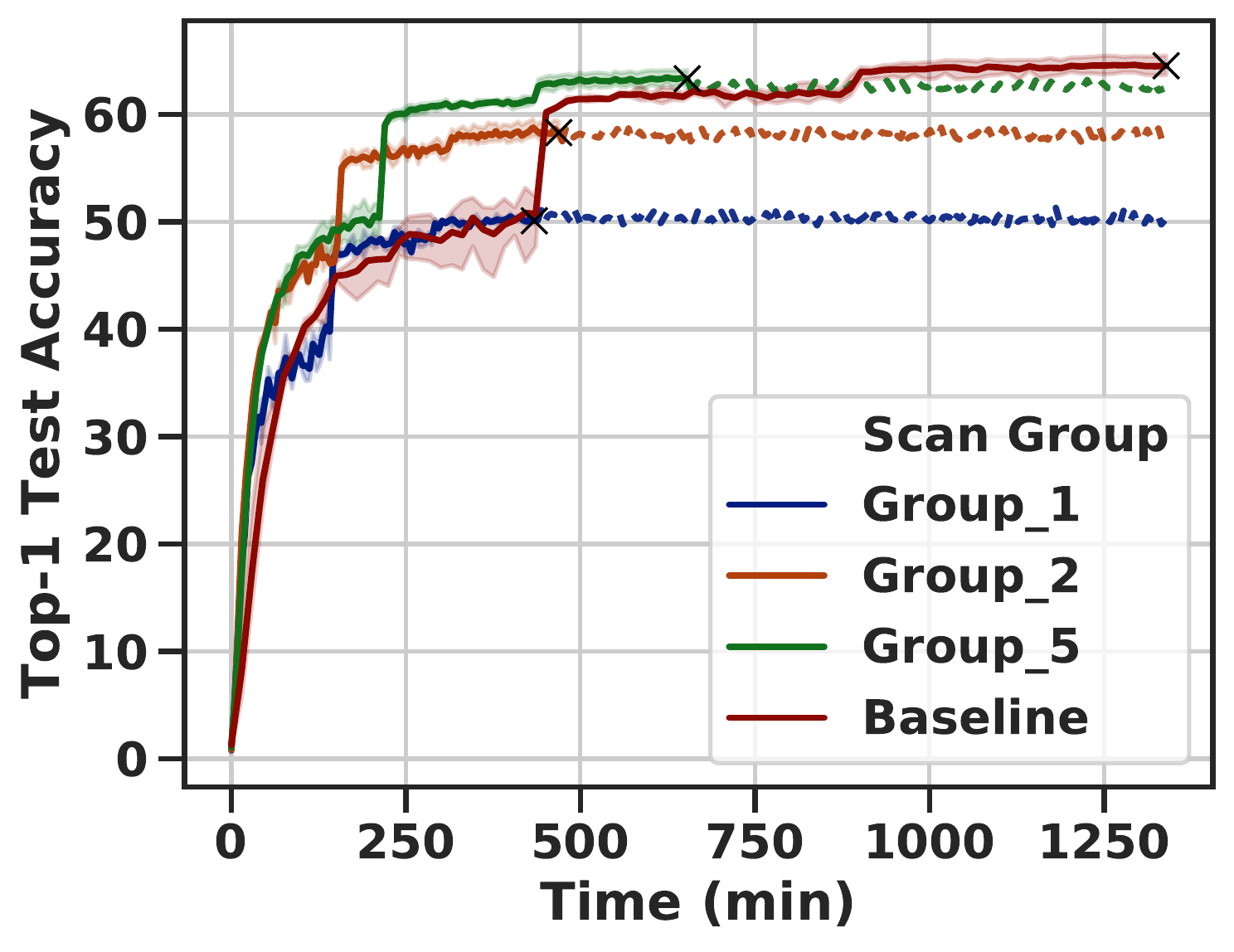}{1.03}{0.01}
    \caption{ImageNet ShuffleNet}
  \end{subfigure}%
  \begin{subfigure}[t]{0.25\textwidth}
    \leftRightCrop{0.07}{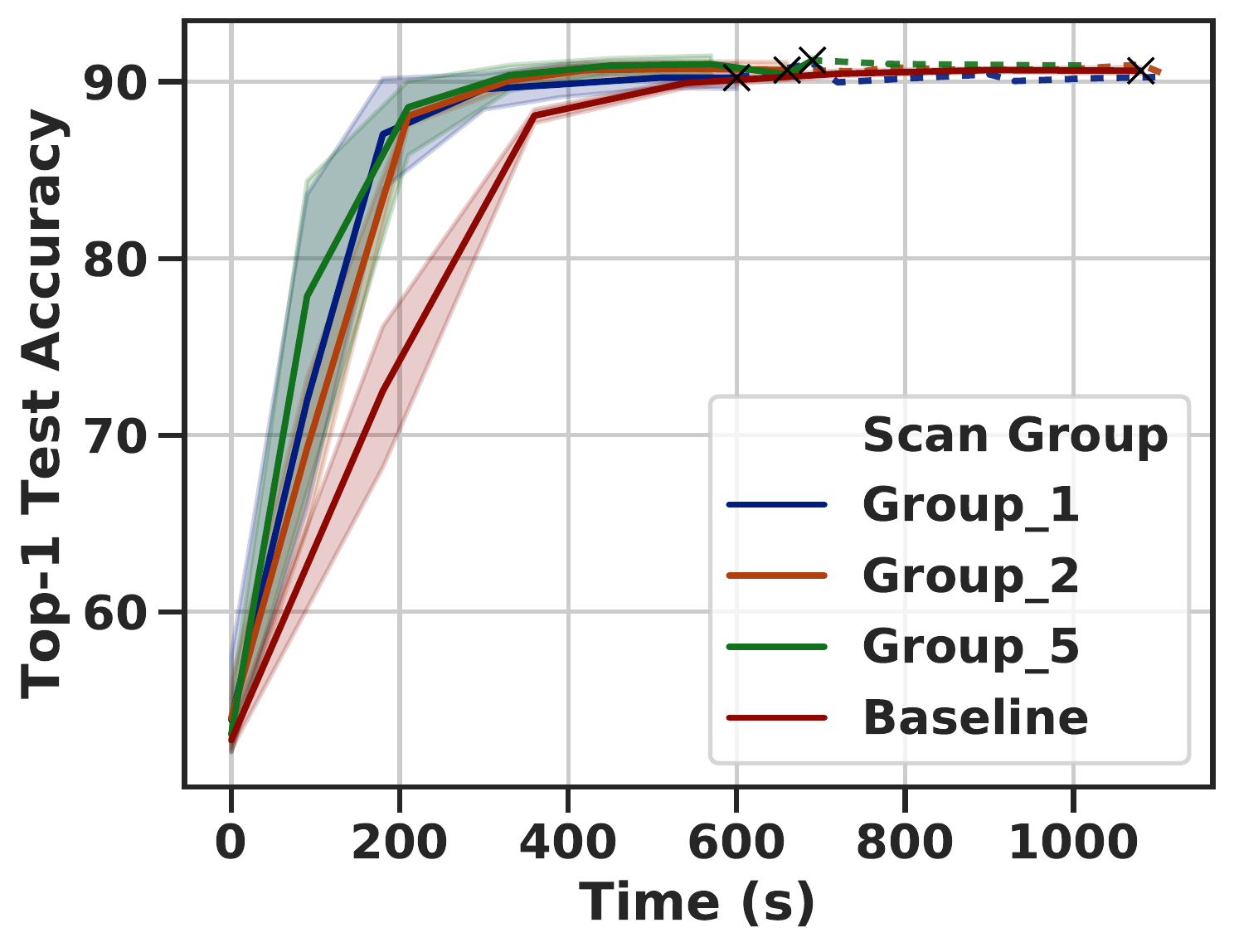}{1.03}{0.013}
    \caption{CelebAHQ ResNet}
  \end{subfigure}%
  \begin{subfigure}[t]{0.25\textwidth}
    \leftRightCrop{0.07}{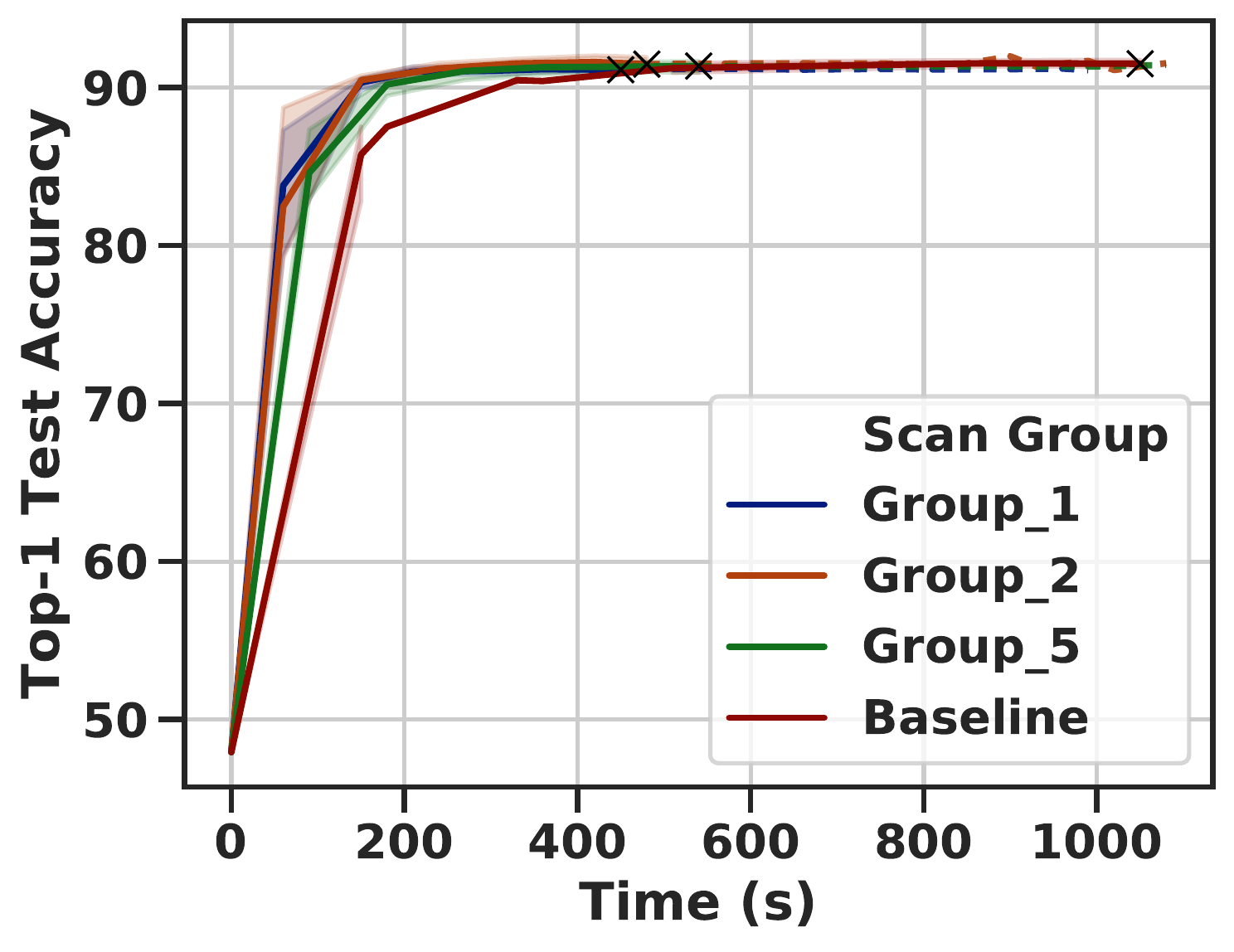}{1.03}{0.01}
    \caption{CelebAHQ ShuffleNet}
  \end{subfigure}%
  \caption{%
    Top-1 test performance (with 95\% confident intervals) using ResNet and ShuffleNet on ImageNet (a,b) and CelebAHQ (c,d). Lower scan groups (corresponding to further compressed data) can provide faster overall training, often without sacrificing accuracy. However, the appropriate level of compression depends on the model, infrastructure, and data---necessitating the ability to easily access data at multiple fidelities, as with PCRs. We explore these factors in Sections~\ref{sec:dynamic_compression}, \ref{sec:quality_analysis}, and \ref{sec:autotuning}.
  }%
  \label{fig:scan_performance_joined_orca_acc_time}%
\end{figure*}

\subsection{Experimental Setup}%
\label{sec:evaluation_setup}
Our evaluation uses the
ImageNet ILSVRC~\citep{imagenet_cvpr09,ILSVRC15},
CelebAHQ~\citep{karras2017progressive},
HAM10000~\citep{tschandl2018ham10000},
and
Stanford Cars~\citep{KrauseStarkDengFei-Fei_3DRR2013}
datasets.
For CelebAHQ, we classify if the celebrity is smiling or not.
A summary of each dataset is given in Table~\ref{fig:PCR_record_info}.
We aimed to select datasets that vary in terms of the image resolution, number of examples, number of classes, and image/scene type.
\textit{All of the datasets are in fact already compressed before progressive
compression is applied, making the presented speedups conservative estimates of the potential benefit of PCRs.} Specifically, the datasets natively use a
JPEG quality level varying from 75\% (CelebAHQ) to 100\% (HAM10000)
(\S\ref{fig:PCR_record_info}).
Experiments use resizing, crop, and horizontal-flip augmentations, as
is standard for ImageNet training~\cite{fixing_train_test,inception}
The sizes of each dataset's scan groups (used in
Table~\ref{fig:PCR_expected_size}) are shown in
Figure~\ref{fig:PCR_Inspector_Sizes}; sizes decrease for lower scan groups.
For examples of images under each scan group, see the supplement.

\begin{figure*}
  \centering
  \begin{subfigure}[t]{0.25\textwidth}
    \includegraphics[width=.99\linewidth]{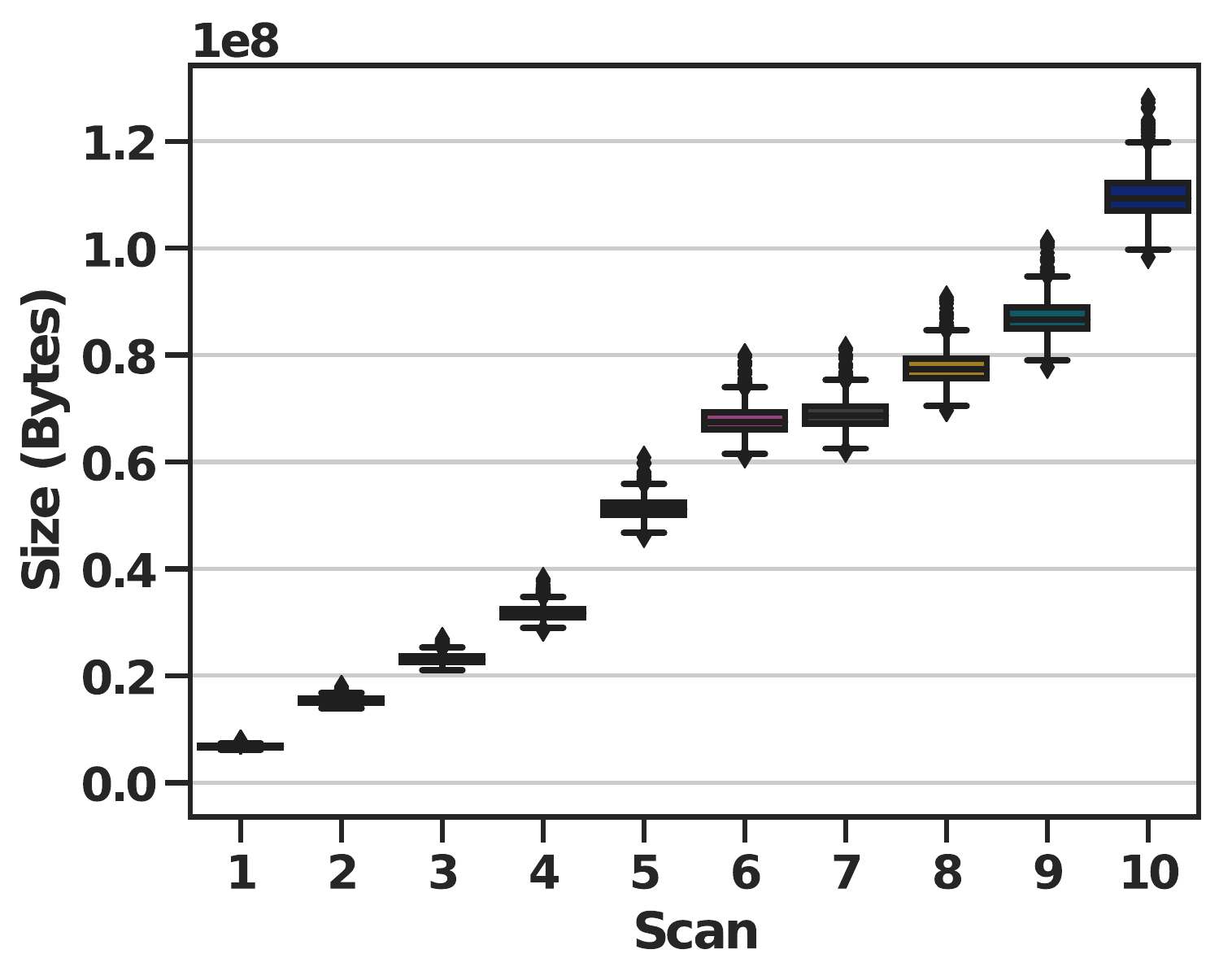}
    \caption{ImageNet}
  \end{subfigure}%
  \begin{subfigure}[t]{0.25\textwidth}
    \includegraphics[width=1.012\linewidth]{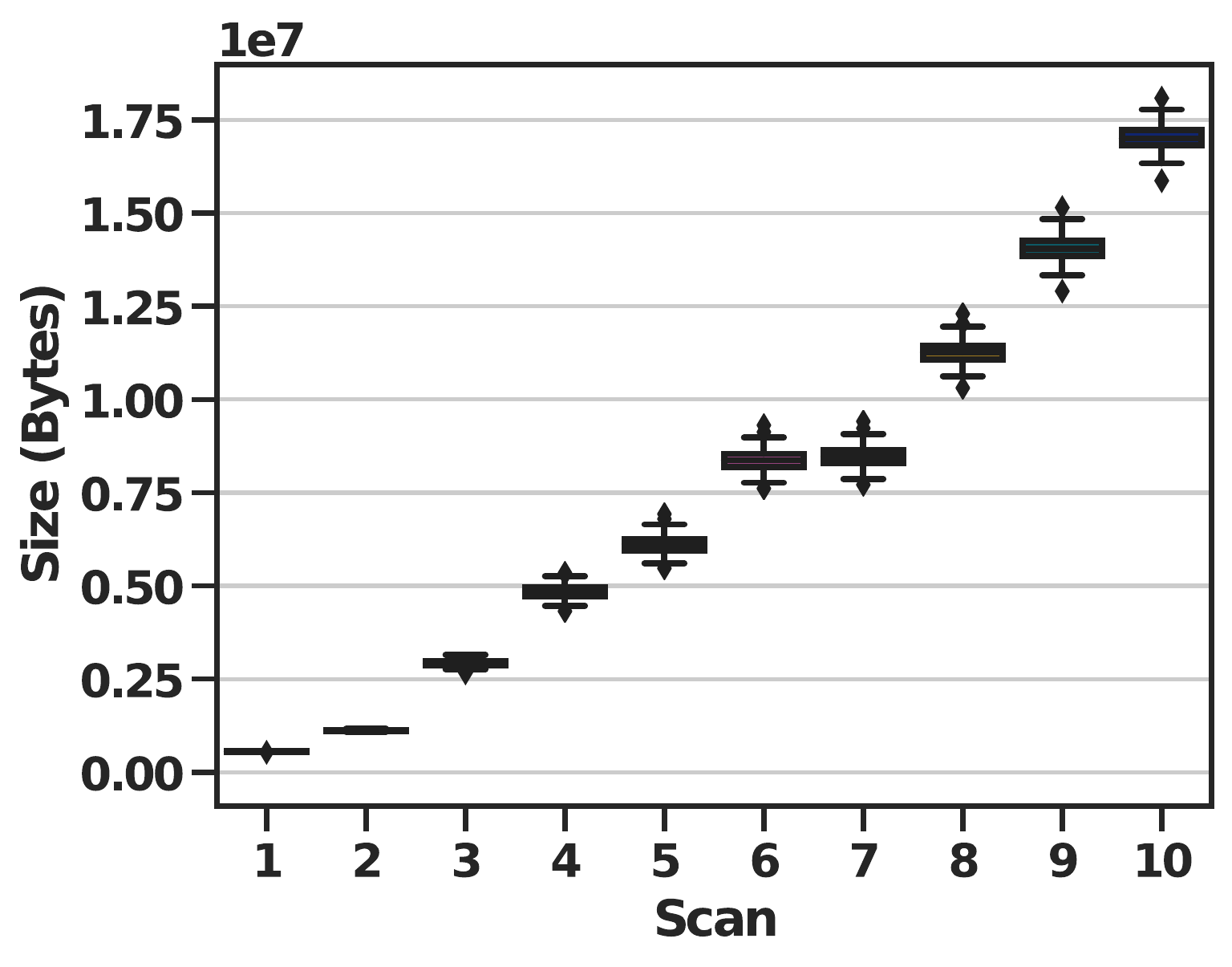}
    \caption{HAM10000}
  \end{subfigure}%
  \begin{subfigure}[t]{0.25\textwidth}
    \includegraphics[width=1.0\linewidth]{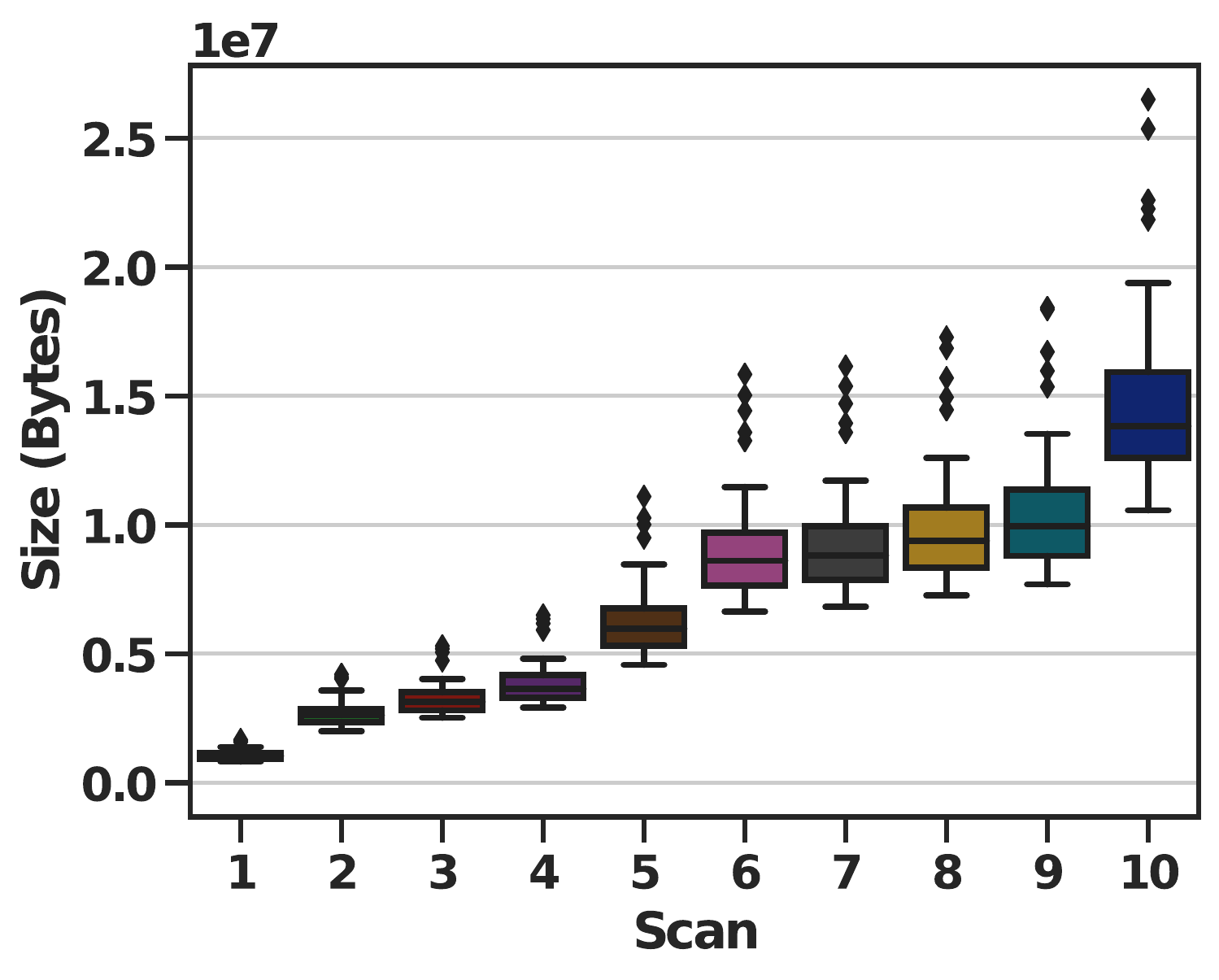}
    \caption{Stanford Cars}
  \end{subfigure}%
  \begin{subfigure}[t]{0.25\textwidth}
    \includegraphics[width=1.0\linewidth]{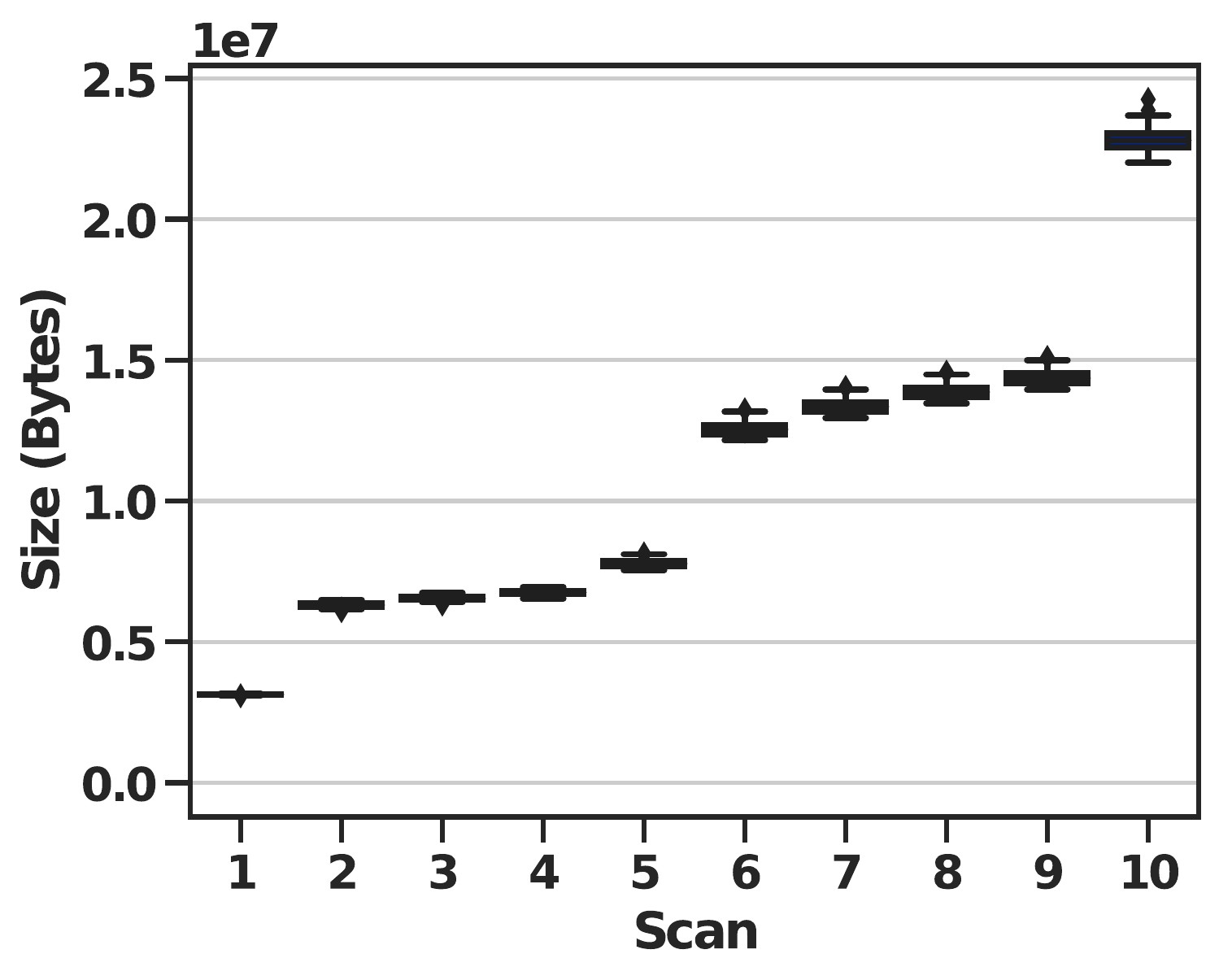}
    \caption{CelebAHQ}
  \end{subfigure}%
  \caption{%
    The size in bytes of various levels of scans read.
    Scan group 0 (not shown) contains only labels and is typically ${\sim}100$
    bytes.
    Each scan adds roughly a constant amount of data (i.e., linear scaling),
    although certain scans
    add considerably more than others (i.e., sizes sometimes cluster)
    due to techniques like chroma subsampling.
    Using all 10 scans can require over an order of magnitude more bandwidth
    than 1--2 scans.
    Interquartile ranges are shown.
  }%
  \label{fig:PCR_Inspector_Sizes}%
\end{figure*}

\paragraph{Training Regime.}
We evaluate two loader implementations of PCRs, comparing PCR scans against
themselves and then comparing PCRs against strong baselines (TFRecords).
For both setups, we use PyTorch~\citep{NIPS2019_9015} for all model training;
the two loaders are using DALI~\citep{DALI} Loader, which was used for
initial prototyping, and \texttt{tf.data}~\citep{TFData,murray2021tfdata}, which
we have since made native operator extensions to for maximum performance.
In our experiments we use pretrained ImageNet weights for HAM10000 and Cars due to the limited amount of
training data. %
We use standard ImageNet training, starting the learning rate
at $0.1$ with gradual
warmup~\citep{goyal2017accurate}, and dropping it on epochs 30 and 60 by $10\times$.
After augmentations, all inputs are of the same size; thus,
a model's update rates are the same across datasets.
The pretrained experiments (HAM10000 and Cars) start at a learning rate of
$0.01$ to avoid changing the initialization too aggressively.
We use mixed-precision training~\citep{micikevicius2017mixed,apex} for the
DALI runs.
We use ResNet-18~\citep{he2016deep}
and ShuffleNetv2~\citep{ma2018shufflenet}
architectures
for our experiments with a batch size of 128 per
worker.
We run each experiment at least 3 times to obtain confidence intervals. %
We sample test accuracy every 15 epochs for non-ImageNet datasets.
Our evaluation
focuses on the differences obtained by reading various amounts of scan groups.
For the DALI runs, we consider reading all the data (up to scan group 10) to be the baseline,
as the baseline formats will perform similarly under I/O bounds
(Figure~\ref{fig:predictedThroughput})---we later provide a direct comparison with
baseline TFRecords when using \texttt{tf.data}.
Our results are conservative as we are already utilizing pre-compressed data and
we include evaluation times in our results.
For the purpose of evaluation,
all scan groups within a dataset were run for the same number of
epochs (90 for ImageNet, 150 for HAM10k, 250 for Cars, and 90 for CelebAHQ).
We also provide annotated (dashed) lines for subsequent epochs.

\paragraph{System Setup.}
Our experiments were run on a 16 node Ceph~\citep{weil2006ceph}
cluster with  NVIDIA TitanX GPUs and 4TB 7200RPM Seagate ST4000NM0023 HDD.
We use six Ceph nodes: five dedicated Object Storage Device (OSD) nodes, and one Metadata Server (MDS).
The remaining 10 nodes are machine learning training workers. %
This 2:1 ratio between compute and storage nodes results in 400+ MiB/s of peak
storage bandwidth; we have also tested a heavily I/O bound 10:1 ratio and found the trends
comparable.
Ceph is a common production-grade open-source filesystem, but
our results would generalize to any setup with a mismatch
between compute power and data bandwidth (either storage or network).
In addition to microbenchmarks, we evaluate the generalization of PCRs to SSD
setups in Section~\ref{sec:discussion}.
Since state of the art compute is
$\mathbf{150\times}$ \textbf{faster} than our own
setup on a more expensive model (ResNet-50)~\citep{ying2018image},
we focus on models which are
fast to train (while still being modern; AlexNet~\cite{krizhevsky2012imagenet}
is potentially faster) while limiting read parallelism.
The DALI setup uses O\_DIRECT to ignore caching effects and highlight bandwidth
usage.
To reflect what PCRs may look like in realistic, heavy-load situations, we provide 20
node experiments in Figure~\ref{fig:scan_performance_joined_orca_acc_time_heavy}
with the same storage system and double the workers, which allows speedups to be
seen with full read parallelism per node (over 700MiB/s of peak bandwidth).
This setup uses our \texttt{tf.data} loader
implementation to fairly compare against TFRecords and
File-per-Image formats, showing its effectiveness.
We use the same setup in
Figure~\ref{fig:adaptive_shufflenet}.

\begin{figure}
  \centering
  \begin{subfigure}[t]{0.25\textwidth}
    \leftRightCrop{0.00}{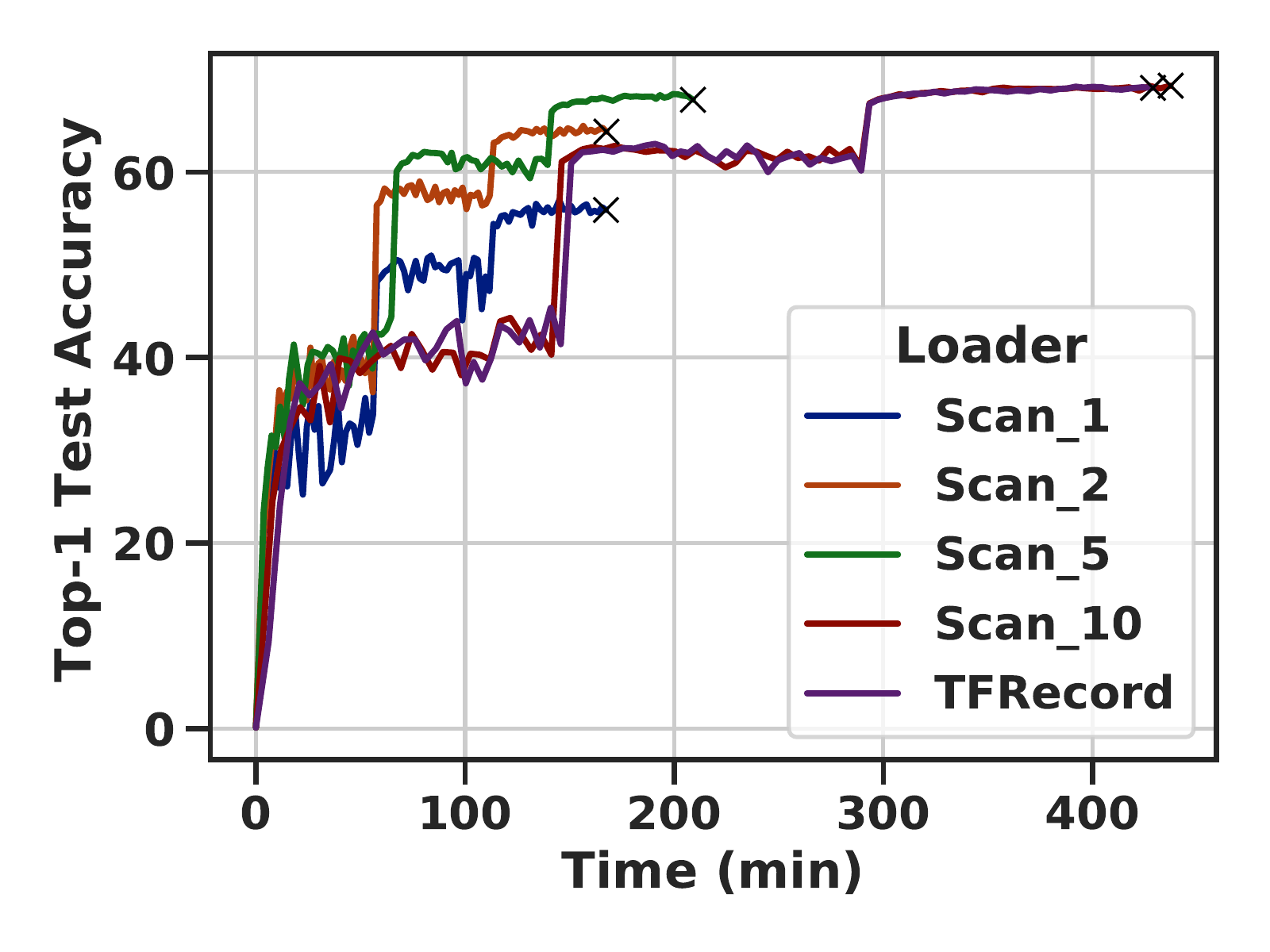}{1.03}{0.01}
    \caption{ImageNet ResNet (heavy)}
  \end{subfigure}%
  \begin{subfigure}[t]{0.25\textwidth}
    \leftRightCrop{0.00}{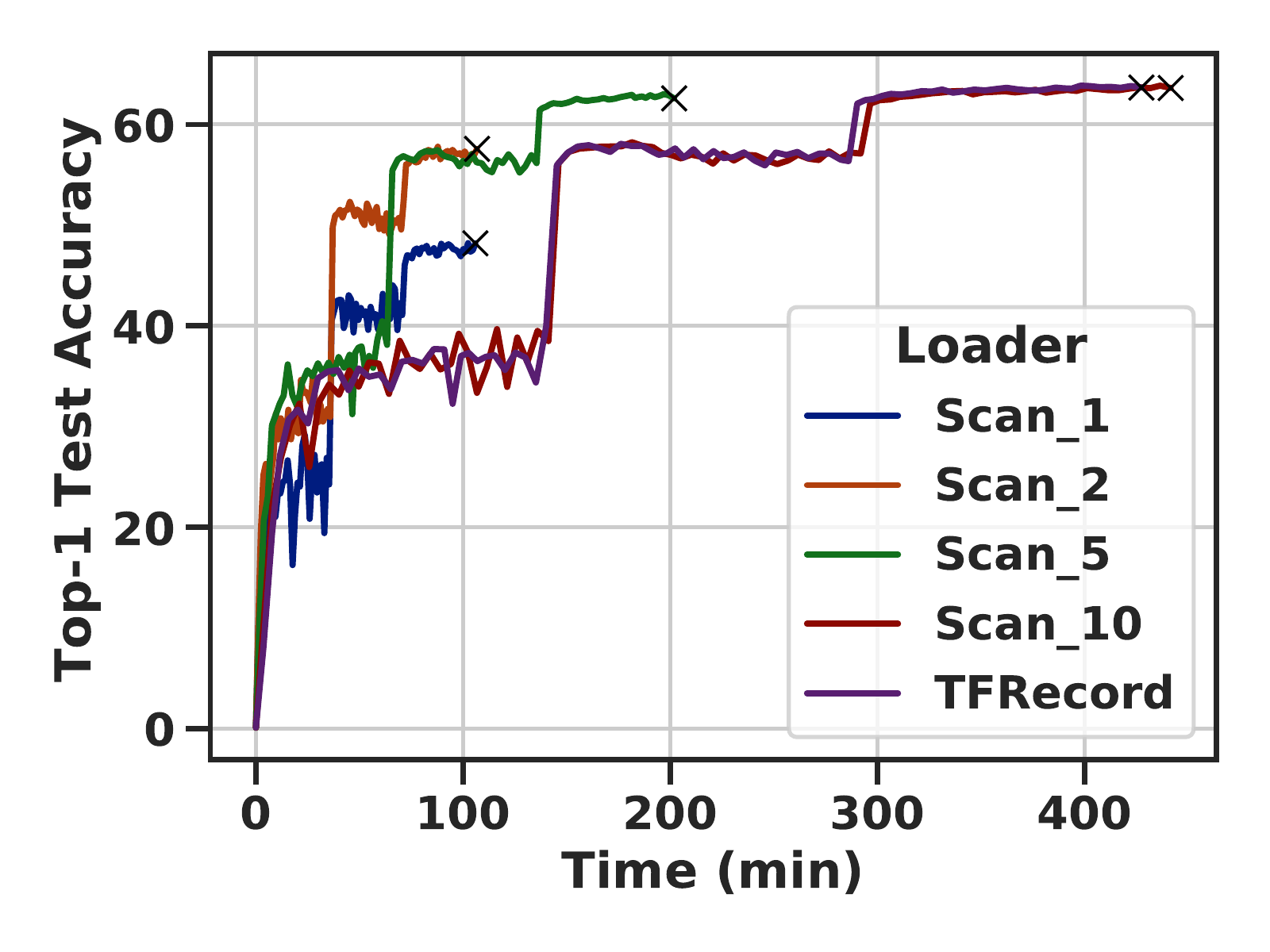}{1.03}{0.01}
    \caption{ImageNet ShuffleNet (heavy)}
  \end{subfigure}%
  \caption{%
    Top-1 test performance on ImageNet with ResNet and ShuffleNet, using double
    the compute
    (20 workers with same configuration).
    Doubling the compute forces bottlenecks to appear by approaching hardware
    limits of aggregate disk throughput.
    We run this experiment once and terminate at epoch 90, showing a $2\times$
    speedup for scan 5.
  }%
  \label{fig:scan_performance_joined_orca_acc_time_heavy}%
\end{figure}

\begin{figure*}[t!]
\centering
\begin{minipage}{.5\textwidth}
  \centering
\begin{subfigure}[t]{0.5\textwidth}
    \leftRightCrop{0.015}{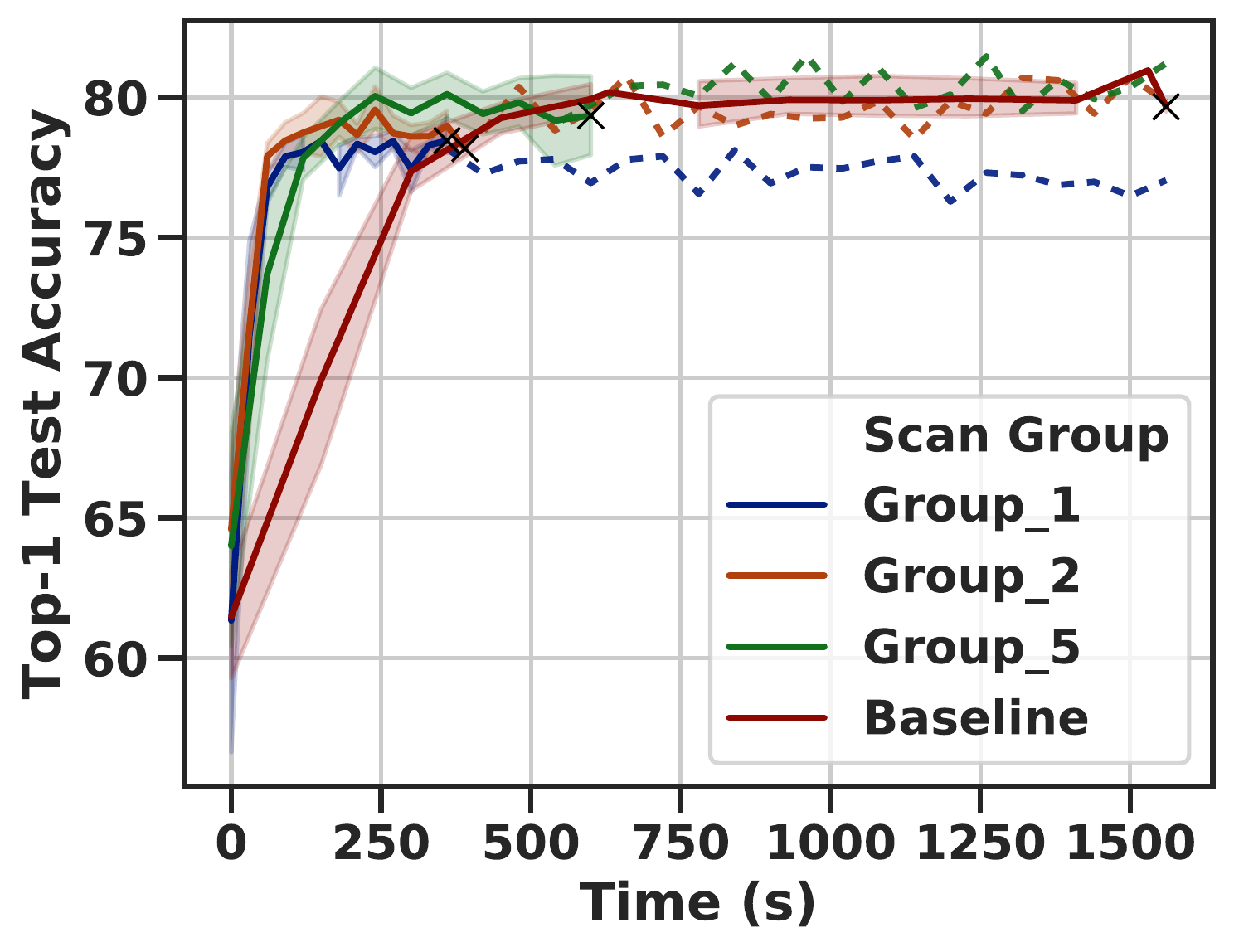}{0.97}{0.015}
    \caption{HAM10000 ResNet}
  \end{subfigure}%
  \begin{subfigure}[t]{0.5\textwidth}
    \leftRightCrop{0.055}{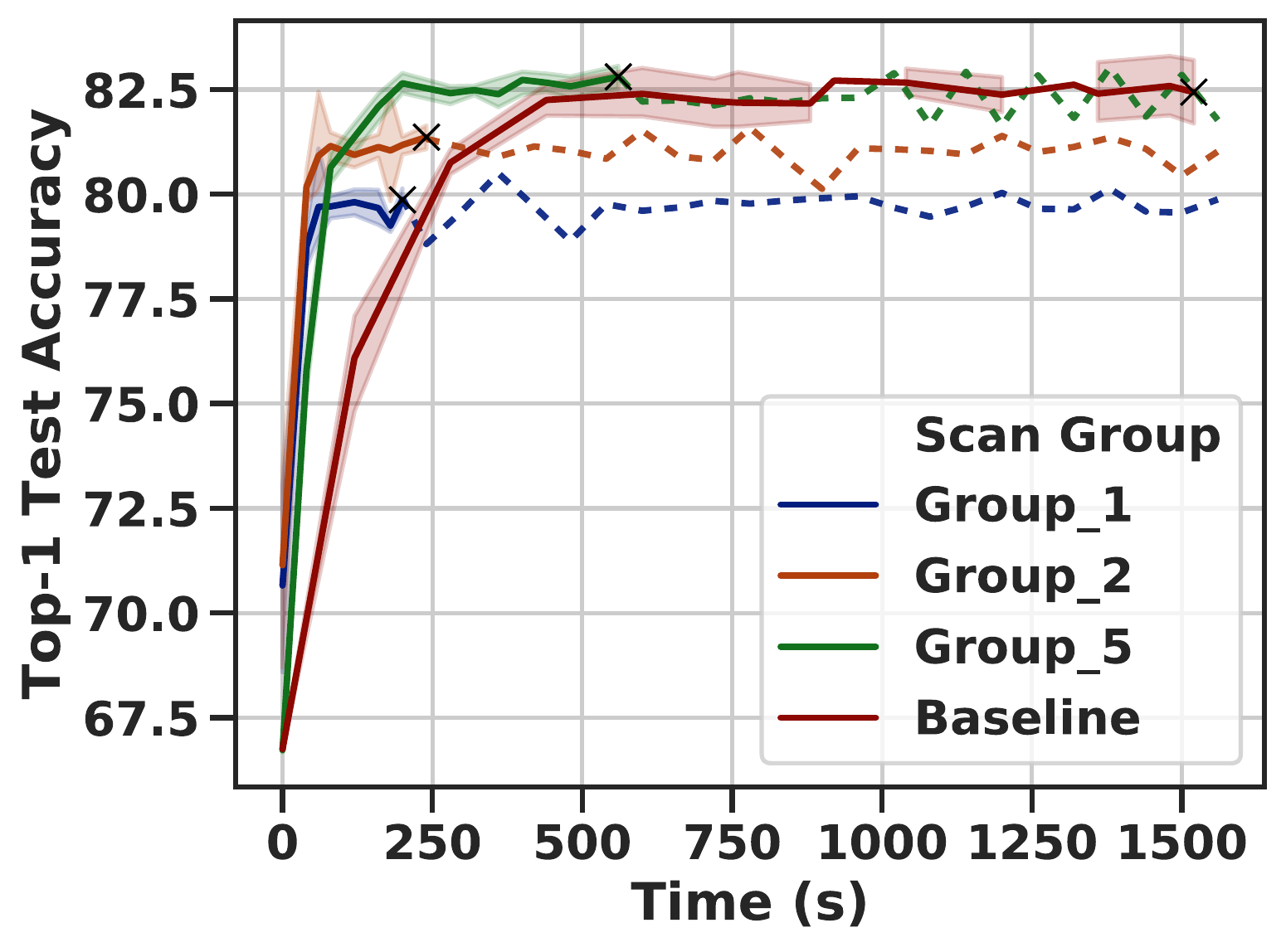}{1.01}{0.013}
    \caption{HAM10000 ShuffleNet}
  \end{subfigure}%
  \caption{Test accuracy on HAM10000.
  While ResNet is robust to additional compression, ShuffleNet requires
  higher fidelity data (at least scan group 5) for higher accuracy.
    Time is relative to first epoch.
    95\% confidence intervals are shown.
  }%
  \label{fig:scan_performance_ham10000_models}%
\end{minipage}%
\hfill
\begin{minipage}{.48\textwidth}
  \centering
\begin{subfigure}[t]{0.49\linewidth}
    \leftRightCrop{0.015}{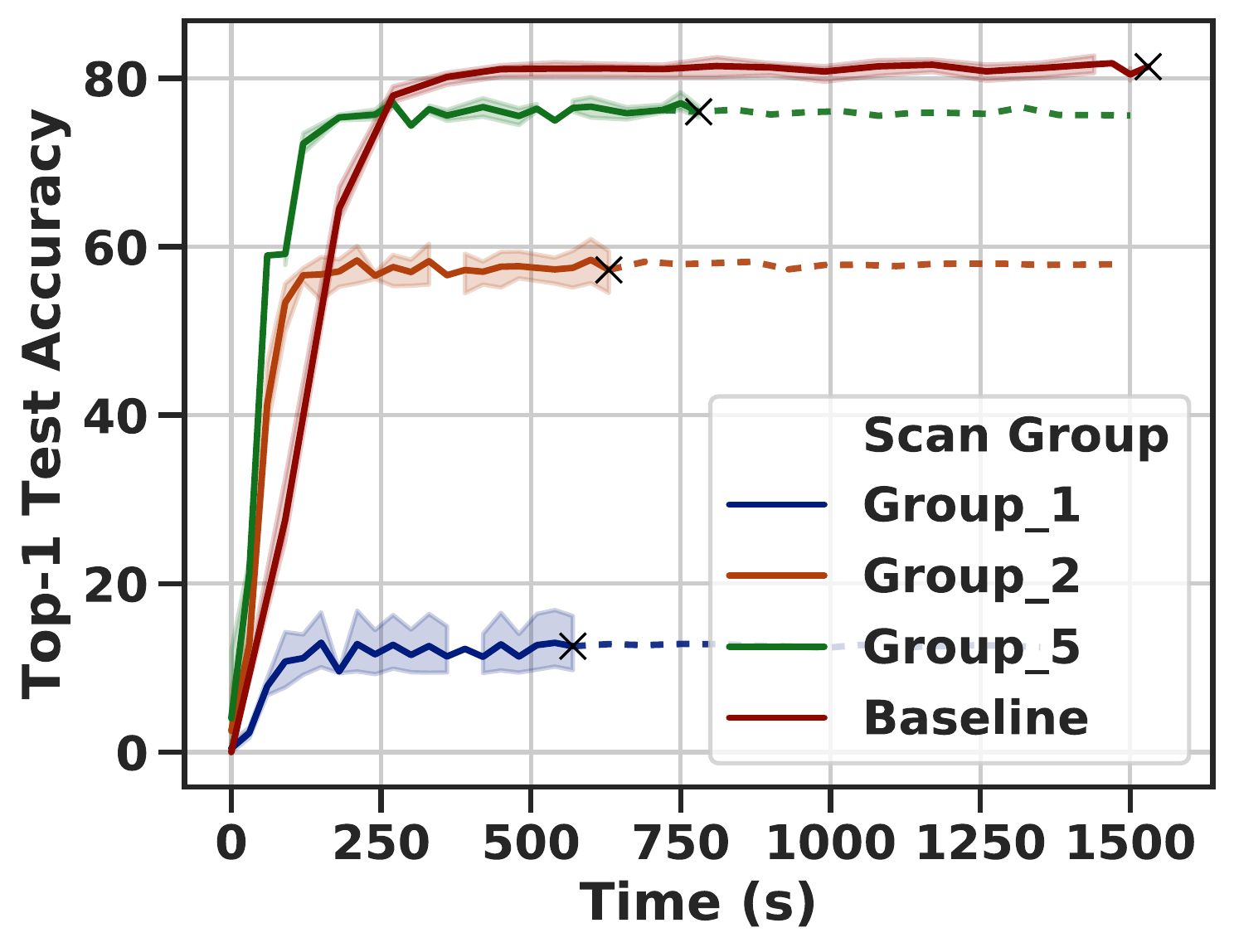}{1.02}{0.015}
    \caption{Original Multiclass}
  \end{subfigure}%
  \begin{subfigure}[t]{0.49\linewidth}
    \leftRightCrop{0.055}{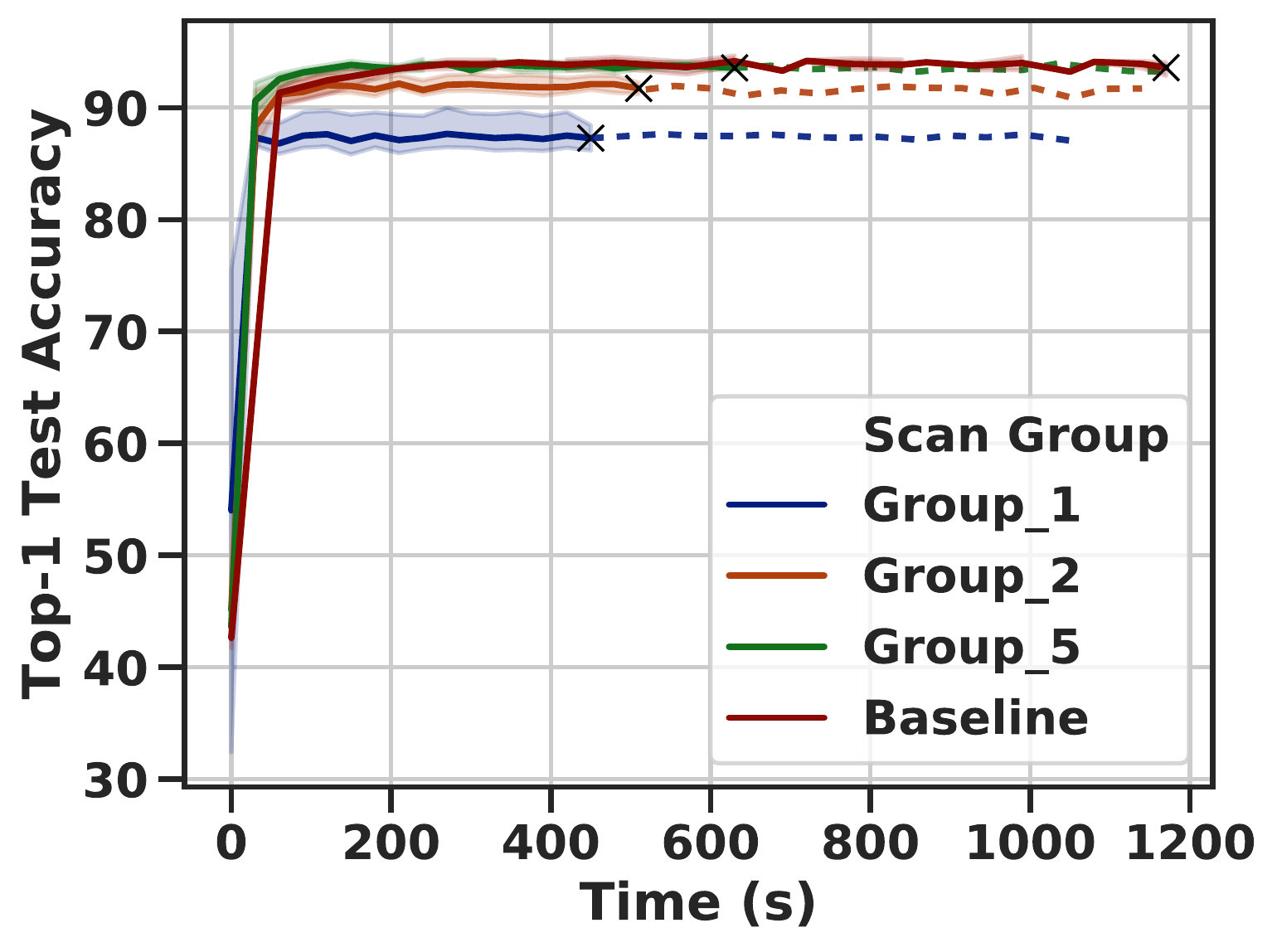}{1.045}{0.015}
    \caption{Binary Is-Corvette}
  \end{subfigure}%
  \caption{%
    Test accuracy with ResNet-18 
    on the Stanford Cars dataset and a binary variant. The gap between scan
    groups closes as the task is simplified.
    Time in x-axis is relative to first epoch.
    95\% confidence intervals are shown.
  }%
  \label{fig:coarse_make_cars_performance_resnet18_orca_acc_time_resnet_short}%
\end{minipage}
\end{figure*}

\subsection{Reducing Time to Accuracy via Compression}%
\label{sec:results}
\vspace{5pt}
\observation{\textit{Training time can be reduced by up to $2\times$ using data compression. PCRs capitalize on this by dynamically reducing training data size, all without adding space overhead.}}
\vspace{5pt}

We begin our empirical study by exploring the effect of data compression on training time and training loss/test accuracy. We provide time-to-accuracy results for ResNet-18 and ShuffleNetv2 training on ImageNet and
CelebAHQ (Figure~\ref{fig:scan_performance_joined_orca_acc_time}), HAM (Figure~\ref{fig:scan_performance_ham10000_models}), and Cars (Figure~\ref{fig:coarse_make_cars_performance_resnet18_orca_acc_time_resnet_short}).
Across these experiments, we find that PCRs can provide a
$2\times$ boost on average in time-to-accuracy compared to the baseline, by dynamically providing data at a higher level of compression.
We make several observations about these results. First, we note that we tend to
see larger speedups for smaller, faster models (e.g. ShuffleNet), than for bigger models (e.g., ResNet). %
Indeed, the current speedups may in fact become significantly larger with faster
compute~\citep[e.g.,][]{ying2018image,kurth2018exascale,kumar2019scale}.
Such a trend is visible in the heavy ImageNet experiments featured in
Figure~\ref{fig:scan_performance_joined_orca_acc_time_heavy}---both
TFRecords
and scan 10 are about the same size, and therefore finish simultaneously, but
scan 1 and 2 finish nearly an hour faster for ShuffleNet.
For this same setup, we observe Files-per-Image take over 2 hours per
epoch due to a lack of sequential reads---$25\times$ slower than
TFRecords, which scan 5 improves by $2\times$;
therefore, we conclude that progressive compression and record formats are both
necessary for performance.

Second, while time-to-accuracy is reduced as we move to lower scan
groups, there is a \textit{statistical efficiency cost}.
Namely, models trained on scans 1 and 2 may not always converge to an
acceptable solution, as shown for ImageNet (Figure~\ref{fig:scan_performance_joined_orca_acc_time}).
Certain tasks like CelebAHQ, however, can tolerate this fidelity loss, either because they consist of less compressed images or because the
training task is less dependent on high-frequency image features.
These results suggest that, while compression saves
bandwidth and offers a potential speedup, the ideal amount of
compression depends on two factors: (i) the speed of the model and the underlying compute infrastructure, and (ii) the structure of the task and the images in the dataset. We explore these factors in more detail below.

\subsection{Task Tolerance to Data Fidelity}%
\label{sec:dynamic_compression}%

\vspace{5pt}
\observation{\textit{Different models can tolerate different fidelities.}}
\vspace{5pt}

\begin{figure}
  \centering
  \begin{subfigure}[t]{0.24\textwidth}
  \leftRightCrop{0.015}{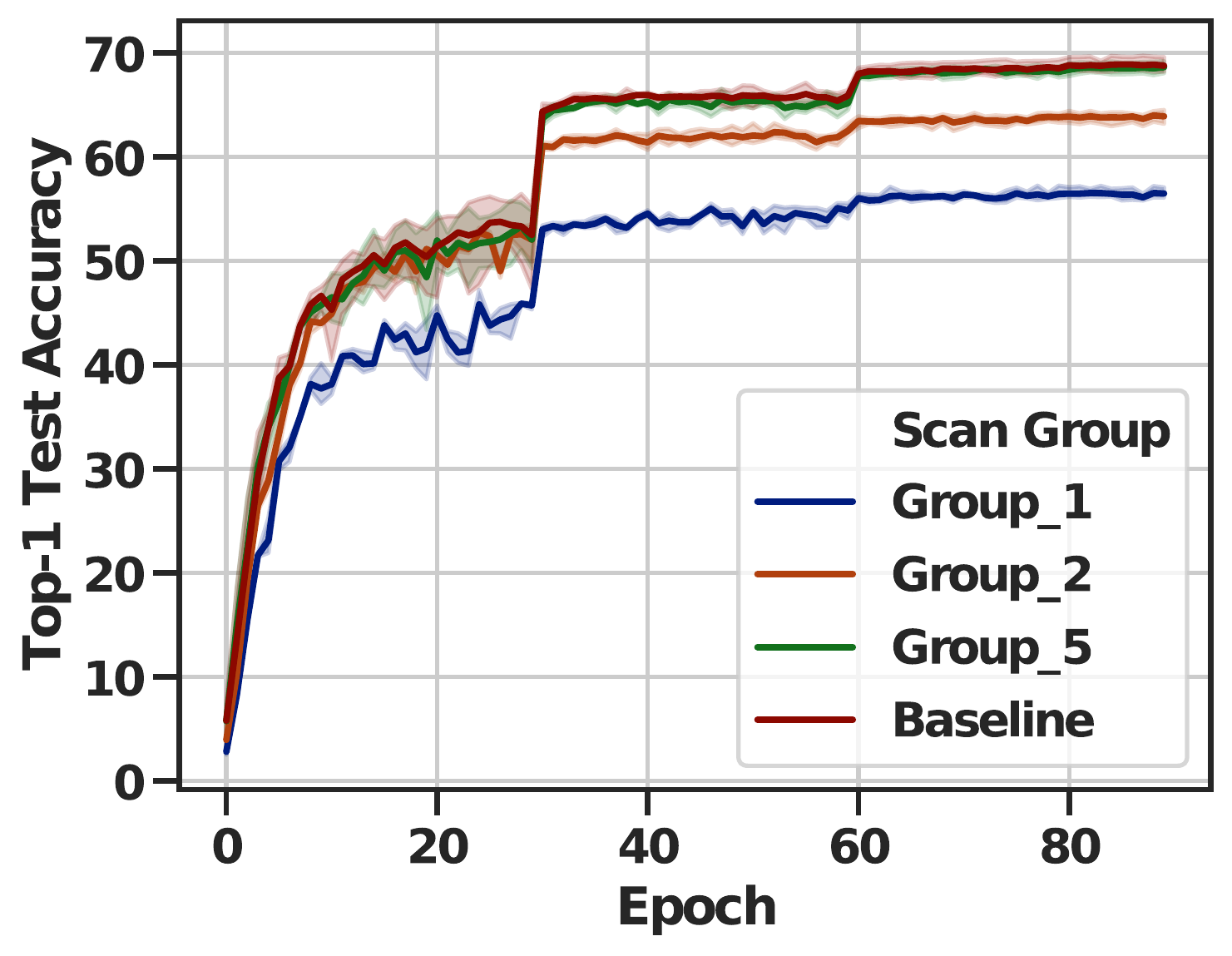}{0.93}{0.01}
  \caption{Epoch Accuracy}
  \end{subfigure}%
  \begin{subfigure}[t]{0.24\textwidth}
  \leftRightCrop{0.015}{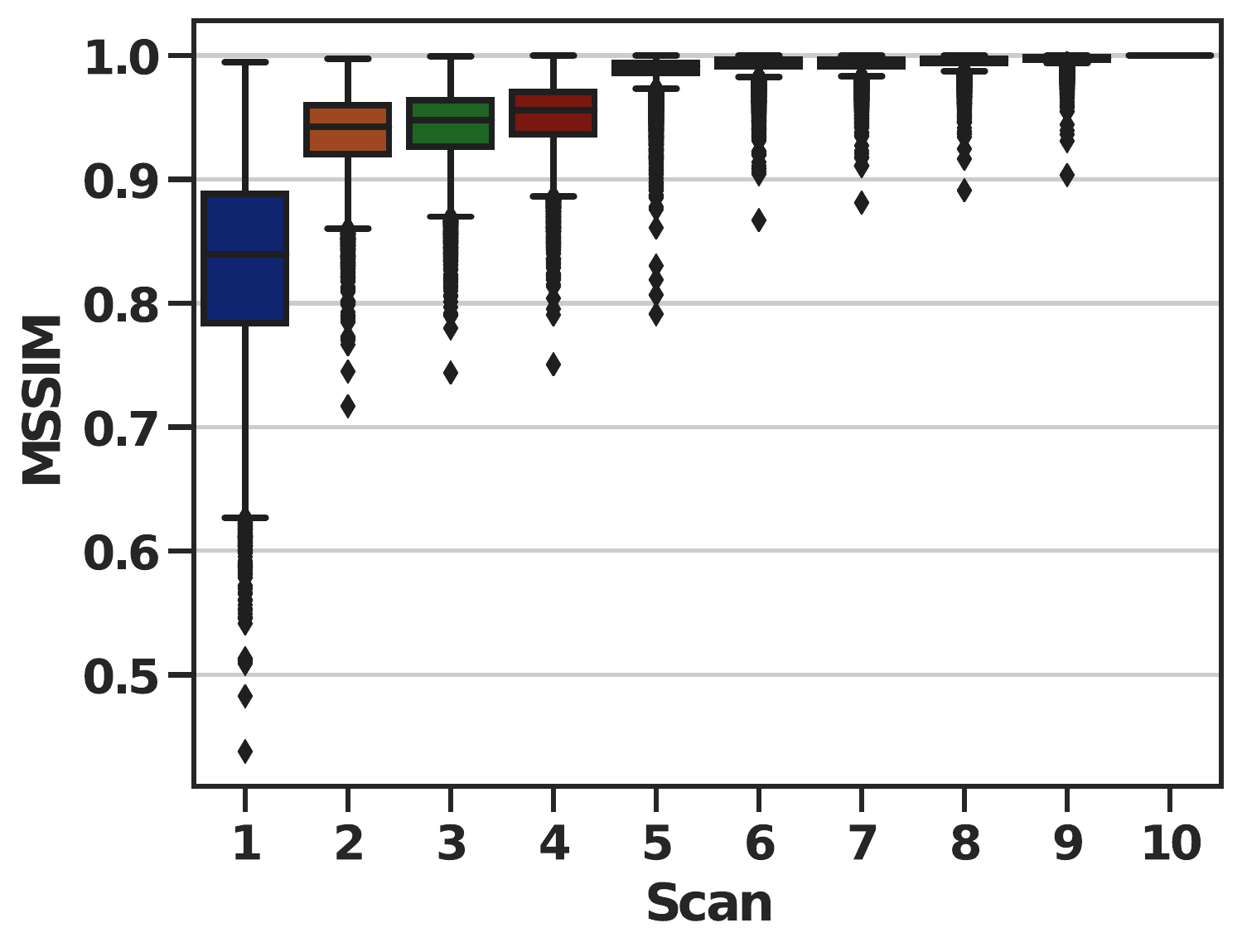}{0.94}{0.01}
  \caption{MSSIM}
  \end{subfigure}
  \vspace{-5pt}
  \caption{Left: Top-1 test performance vs.\ epoch on ResNet-18/ImageNet; other
  models/datasets are similar.
  Using lower quality scans can only degrade performance; it does not act as a
  beneficial data augmentation.
  Right: Corresponding image quality degradation according to MSSIM\@.}%
  \label{fig:resnet_imagenet_epoch_mssim}
\end{figure}

Given a fixed dataset, we show that there is variation in the data fidelity/compression level that different models can tolerate for training. This motivates an important use-case of PCRs, as the format allows data to be stored \textit{once} but then accessed at multiple compression levels while models are tuned or various models are applied to the problem at hand.  %
In Figure~\ref{fig:scan_performance_ham10000_models}, both ResNet and ShuffleNet are trained with the HAM10000 dataset.
While ResNet consistently tolerates low fidelity images, ShuffleNet training tends to degrade with low fidelity data.
ShuffleNet reaches its best accuracy at scan 5, but our
other results suggest that lowering fidelity results in lower
accuracy for the same epoch in nearly all cases (Figure~\ref{fig:resnet_imagenet_epoch_mssim}).
This suggests that different models will experience different speedups for similar accuracy levels, depending on their sensitivity to fine-grained features unavailable in low fidelity data.\vspace{.5em}

\observation{\textit{Different tasks,\ e.g., multi-class classification vs.\ binary classification, can tolerate different levels of data fidelity.
The same PCR dataset can service these different tasks.}}
\vspace{5pt}

The difficulty of a task, or training objective of interest, also affects the amount of compression
that can be tolerated.
Harder tasks, e.g., multi-class classification with a large number of classes, require higher fidelity data.
We  validate this empirically in 
Figure~\ref{fig:coarse_make_cars_performance_resnet18_orca_acc_time_resnet_short}
(and additional evidence is provided in the supplement).
This experiment reduces the number of classes for the classification task, demonstrating that lower scan groups can be used for easier tasks.
The full range of classes investigated includes:
\textit{Baseline} (i.e., Car Make, Model, Year),
\textit{Make-Only} (i.e., car Make only),
and \textit{Is-Corvette}, a binary classification task of Corvette detection.
Compared to the original task, the coarser tasks reduce the
gap between scan groups, decreasing the gap from baseline to the binary case. %
These results suggest that the optimal image encoding can be dependent on
the exact labeling or task complexity.
Thus, while static approaches may need one encoding per task, a fixed PCR
encoding can support multiple tasks at optimal fidelity by simply changing the
scan group depending on how the labels (metadata) are remapped.

\vspace{1em}
\subsection{Compression Level Estimation}%
\label{sec:quality_analysis}%

\vspace{5pt}
\observation{\textit{MSSIM image similarity is a reliable estimator of the accuracy loss between scan groups, and can be used to determine appropriate levels of compression for training with PCRs.
}}
\vspace{5pt}

\begin{figure}
  \centering%
  \begin{subfigure}[t]{0.23\textwidth}%
    \leftRightCrop{0.015}{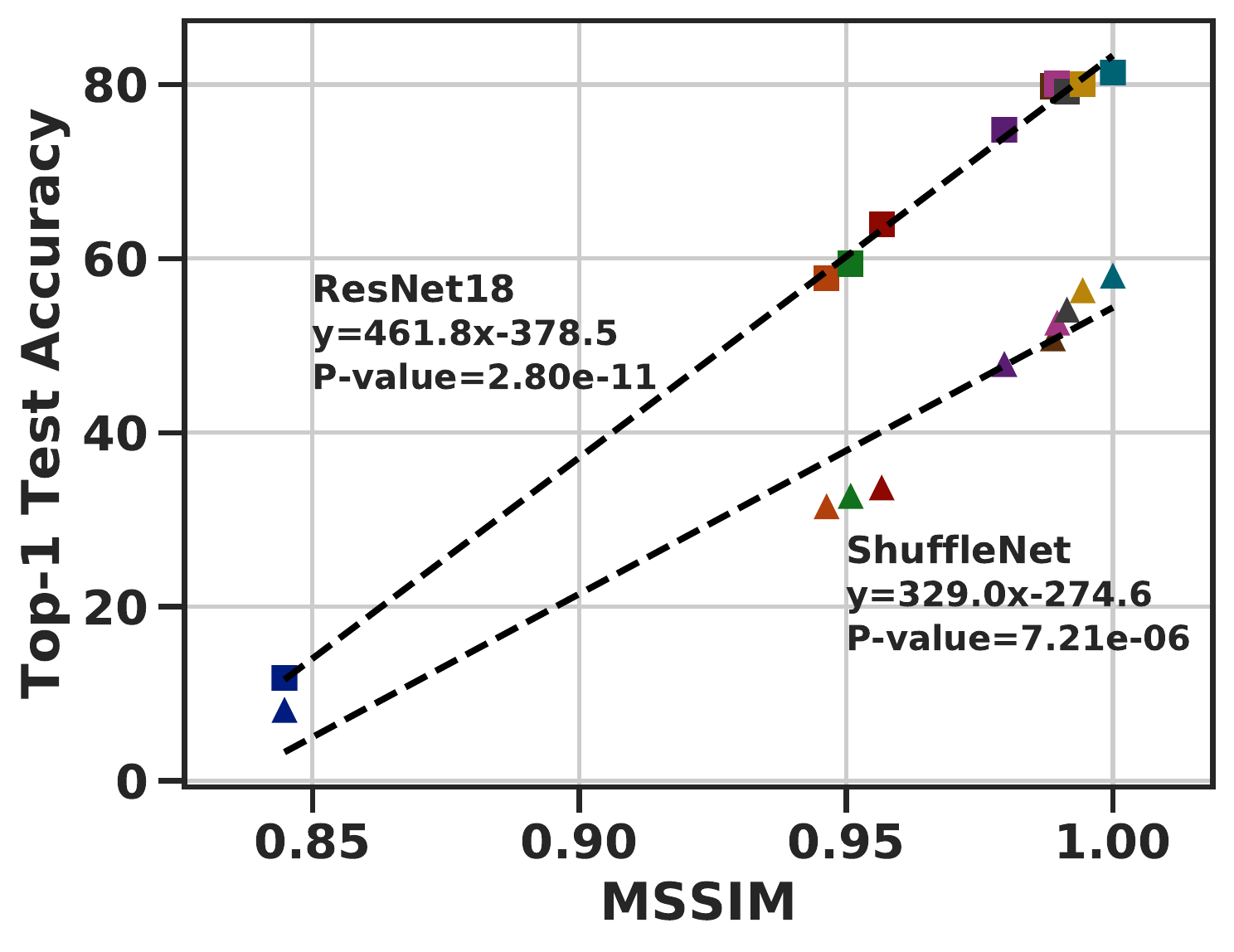}{1.03}{0.0}%
    \caption{MSSIM Regression}%
  \end{subfigure}%
  \begin{subfigure}[t]{0.23\textwidth}%
    \leftRightCrop{0.015}{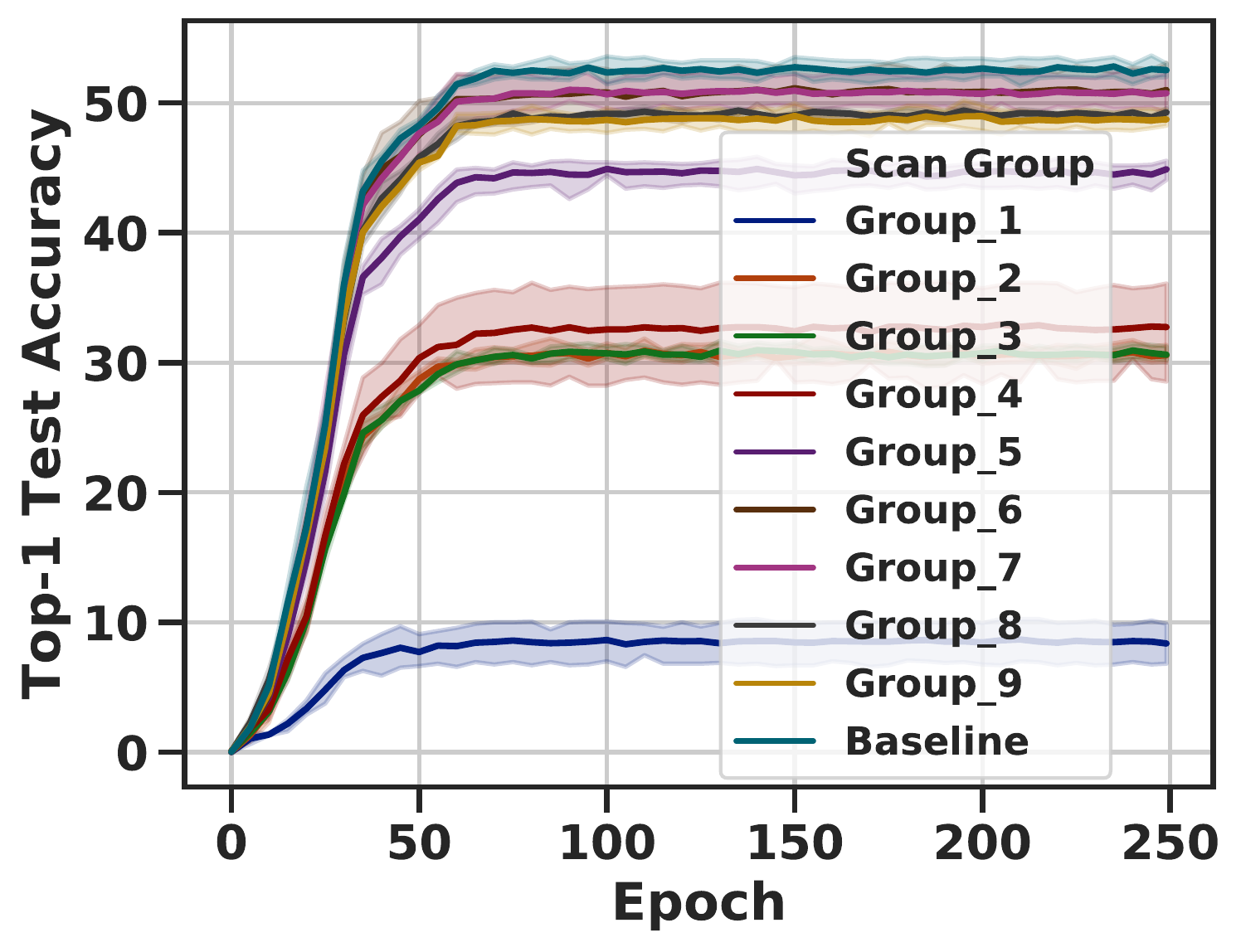}{1.03}{0.0}
    \caption{Clustered Convergence}%
  \end{subfigure}%
  \vspace{-5pt}
  \caption{%
    MSSIM vs.\ accuracy for the Cars dataset with ResNet18 and Shufflenet.
    We obtain similar results for other datasets.
    \textbf{Left:} There is a linear relationship between MSSIM and the final test accuracy.
    \textbf{Right:}
    Scan groups (ShuffleNet) cluster by MSSIM and accuracy.
  }
  \label{fig:MSSIM_vs_acc}%
\end{figure}

To better explain the effectiveness of compression, we compare how various scans approximate the reference image through MSSIM~\citep{wang2003multiscale}, a standard measure of image similarity.
We find a correlation between MSSIM and final
test accuracy, especially when comparing scan groups \textit{within} a task.
Our preliminary tests show that scan groups with similar MSSIM
achieve similar accuracy (Figure~\ref{fig:MSSIM_vs_acc}), which is why only scan groups 1, 2, 5, and the baseline are
shown. %
Due to the way progressive JPEG is coded by default, groups tend to cluster, e.g., scans 2, 3, and 4 are usually similar, while 5 introduces a difference.
Such ``banding'' or clustering
is seen in the accuracy trends; the major jumps correlate with the appearance
of Y (luminance) AC coefficients in the JPEG encoding.
Scan groups of 5 or higher have an MSSIM of $95\%+$, which
is likely why they consistently perform well.
MSSIM can therefore be used as a diagnostic for choosing scans, %
although we acknowledge that
changes in perception are hard to predict for large deviations (MSSIM $< 95\%$).
For some datasets, linear regression on MSSIM recovers final test accuracy
even with different models (Figure~\ref{fig:MSSIM_vs_acc}) or augmentations.
Test accuracy per epoch degrades with worse image
fidelity across our experiments
(Figure~\ref{fig:resnet_imagenet_epoch_mssim}), further highlighting that
time-to-accuracy speedups are caused primarily by bandwidth reduction (rather
than e.g., a form of regularization induced by lower scans).

\subsection{Autotuning Compression Level}
\label{sec:autotuning}%
\vspace{5pt}

\observation{\textit{It is possible to automatically determine an appropriate level of compression at runtime by dynamically accessing various data qualities via PCRs.}}
\vspace{5pt}

\begin{figure}
  \centering
  \begin{subfigure}[t]{0.25\textwidth}
    \includegraphics[width=1.0\linewidth]{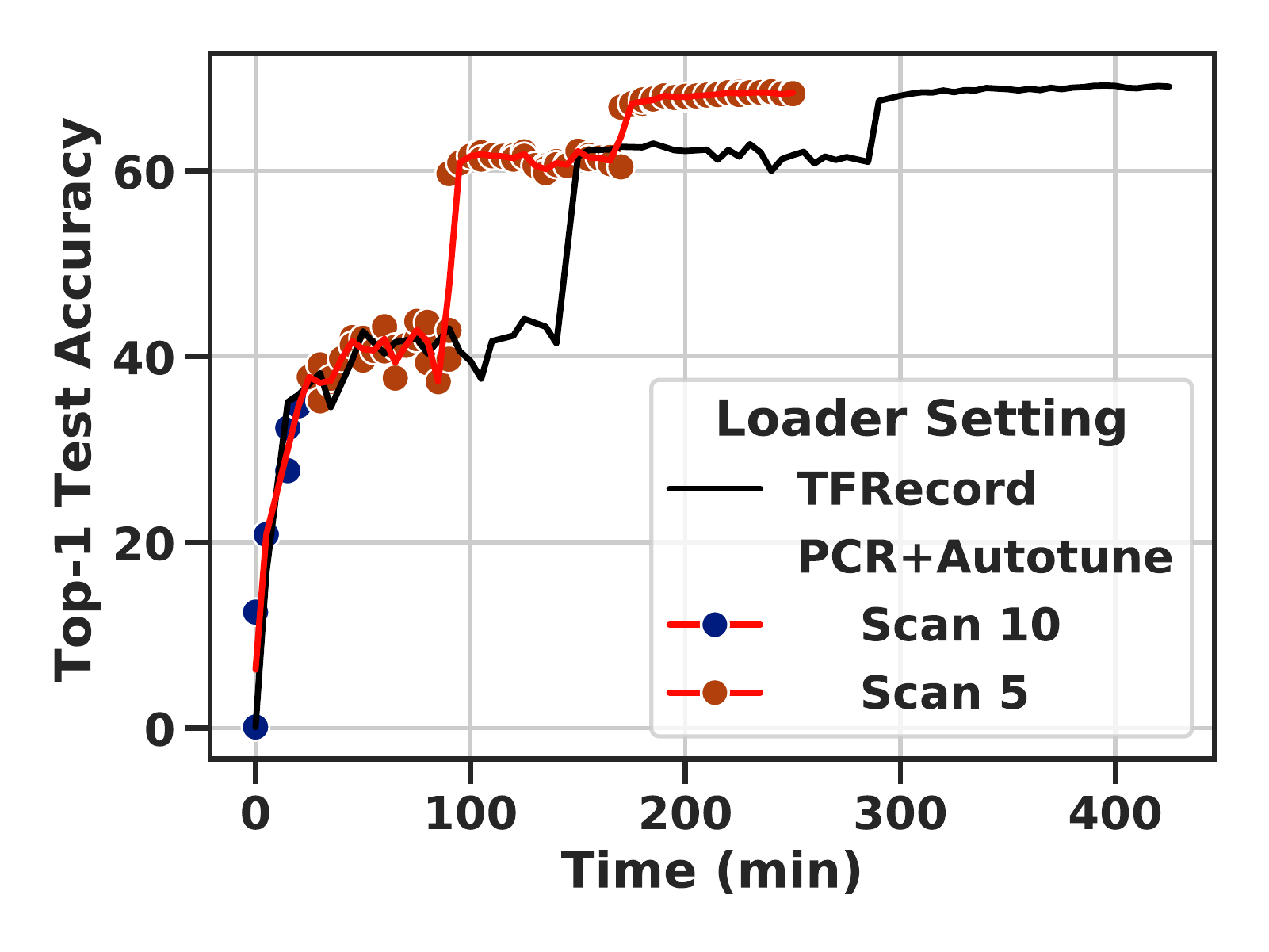}
    \caption{ResNet}
  \end{subfigure}%
  \begin{subfigure}[t]{0.25\textwidth}
    \includegraphics[width=1.0\linewidth]{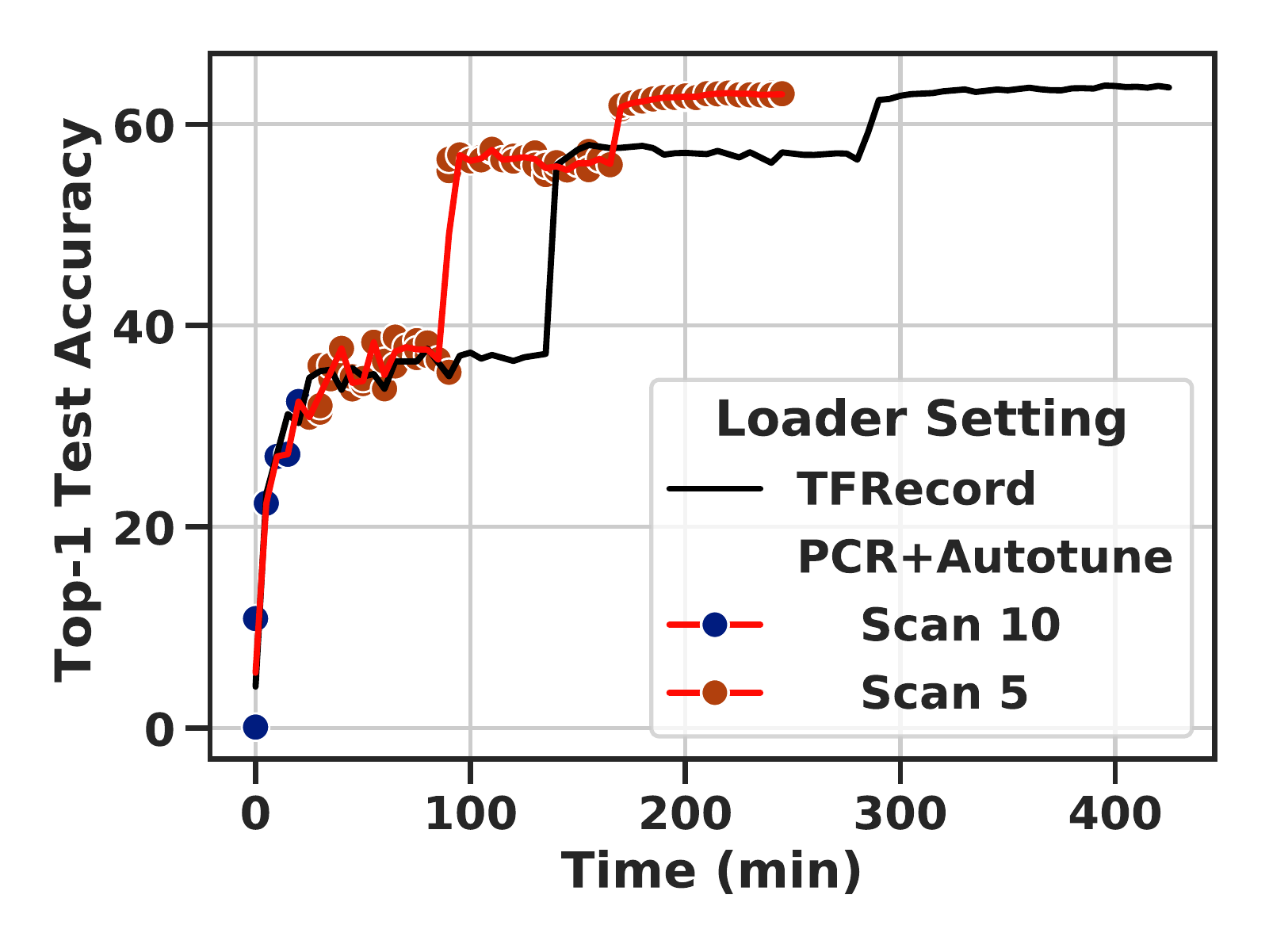}
  \caption{ShuffleNet}
  \end{subfigure}%
  \vspace{-5pt}
  \caption{%
    Adaptive tuning on ImageNet for 90 epochs compared to TFRecord training,
    which is comparable to Scan 10 training.
    For adaptive tuning, epochs are marked with scatter points.
    Training is fastest after epoch 5, when autotuning search is enabled.
    Changing the threshold from 0.8 (shown) to 0.9 results in the last few epochs
    switching to scan 10.
  }%
  \label{fig:adaptive_shufflenet}%
\end{figure}

In cases where training resolution is not structured into
learning~\citep{karras2017progressive}
or image fidelity heuristics prove
too costly to tune (\S\ref{sec:quality_analysis}), automatic tuning of the
scan hyperparameter may be desirable.
One way of doing this is by tuning with a measurement of the \textit{bias} of the gradient given a lower fidelity
image (\S\ref{sec:autotuning_quality}), intuitively measuring how the
model ``sees'' the image similarity.
As we showed in Section~\ref{sec:background},
a similarity threshold of 0.8 or higher is sufficient to avoid bad scans
throughout training---Figure~\ref{fig:PCR_similarity} clusters low-quality scans
below that point.
We apply this threshold and the rest of the procedure described in
Section~\ref{sec:autotuning_quality} to the ImageNet dataset and observe that
such autotuning repeatedly matches accuracy while almost being as fast as a pure
scan 5 approach.
The main slowdown is due to starting at scan 10 for the first 5 epochs of
training, blending the latencies of scan 10 with those of scan 5.
We note that, unlike MSSIM, which is statically concentrated above 95\% for good
quality scans, the gradient similarity changes over training.
For example, ResNet18 has a similarity of 0.88 by epoch 85, whereas it had a
similarity of 0.95 at epoch 5.
We observe that using a higher threshold, approaching 0.9,
forces scan 10 to be used for the last few epochs of training when gradient
similarity is lower, retaining
similar accuracy at slightly longer training times.

\vspace{1em}
\subsection{Image Loading Rates}%
\label{sec:loading_rates}%

\vspace{5pt}
\observation{\textit{Image loading rates are directly linked to the compression ratio, i.e., a compression ratio of $2\times$ results in a $2\times$ speedup.}}
\vspace{5pt}

\begin{figure}
  \centering
  \begin{subfigure}[t]{0.23\textwidth}
    \leftRightCrop{0.015}{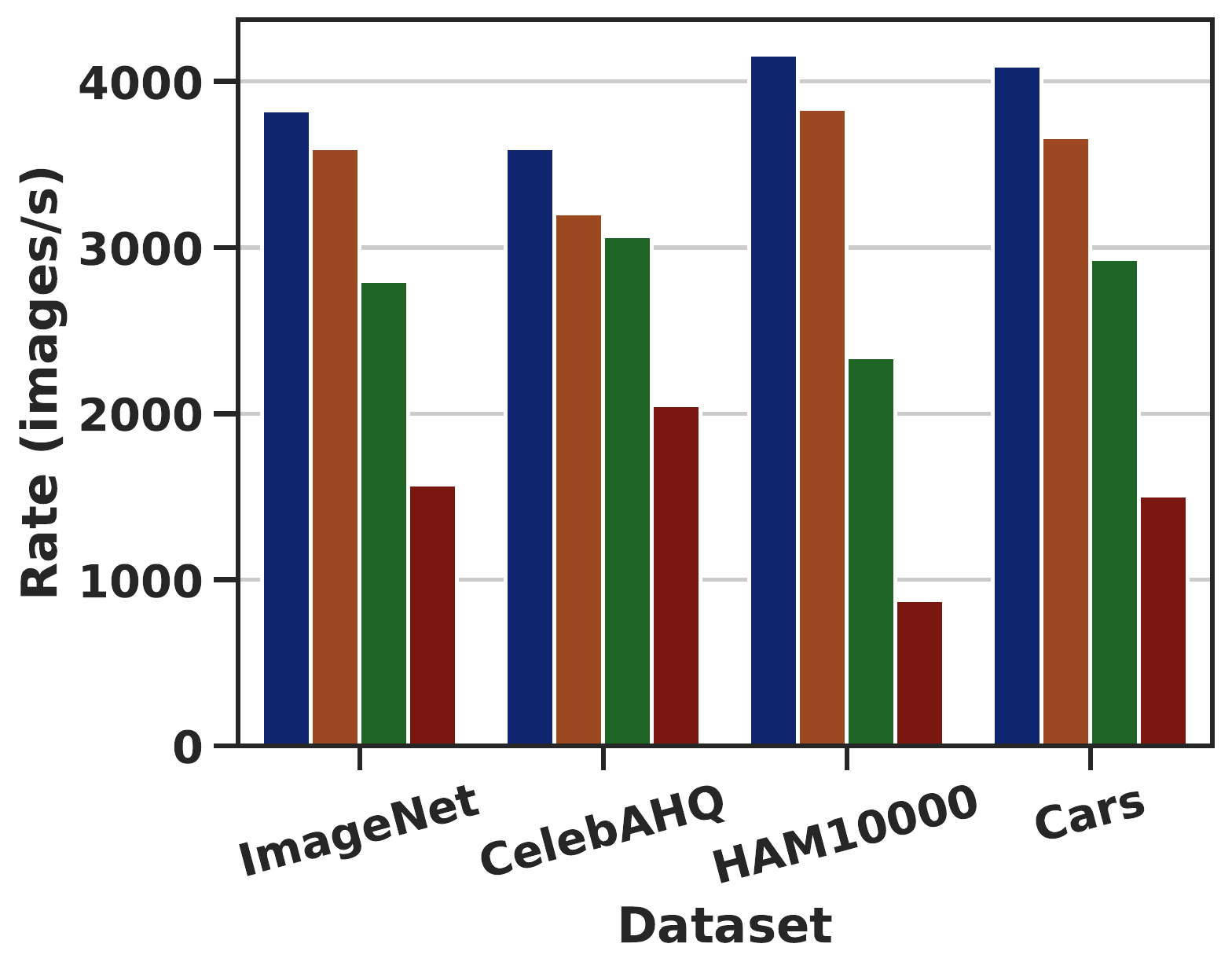}{1.0}{0.0}
    \caption{ResNet}
  \end{subfigure}%
  \begin{subfigure}[t]{0.25\textwidth}
    \leftRightCrop{0.05}{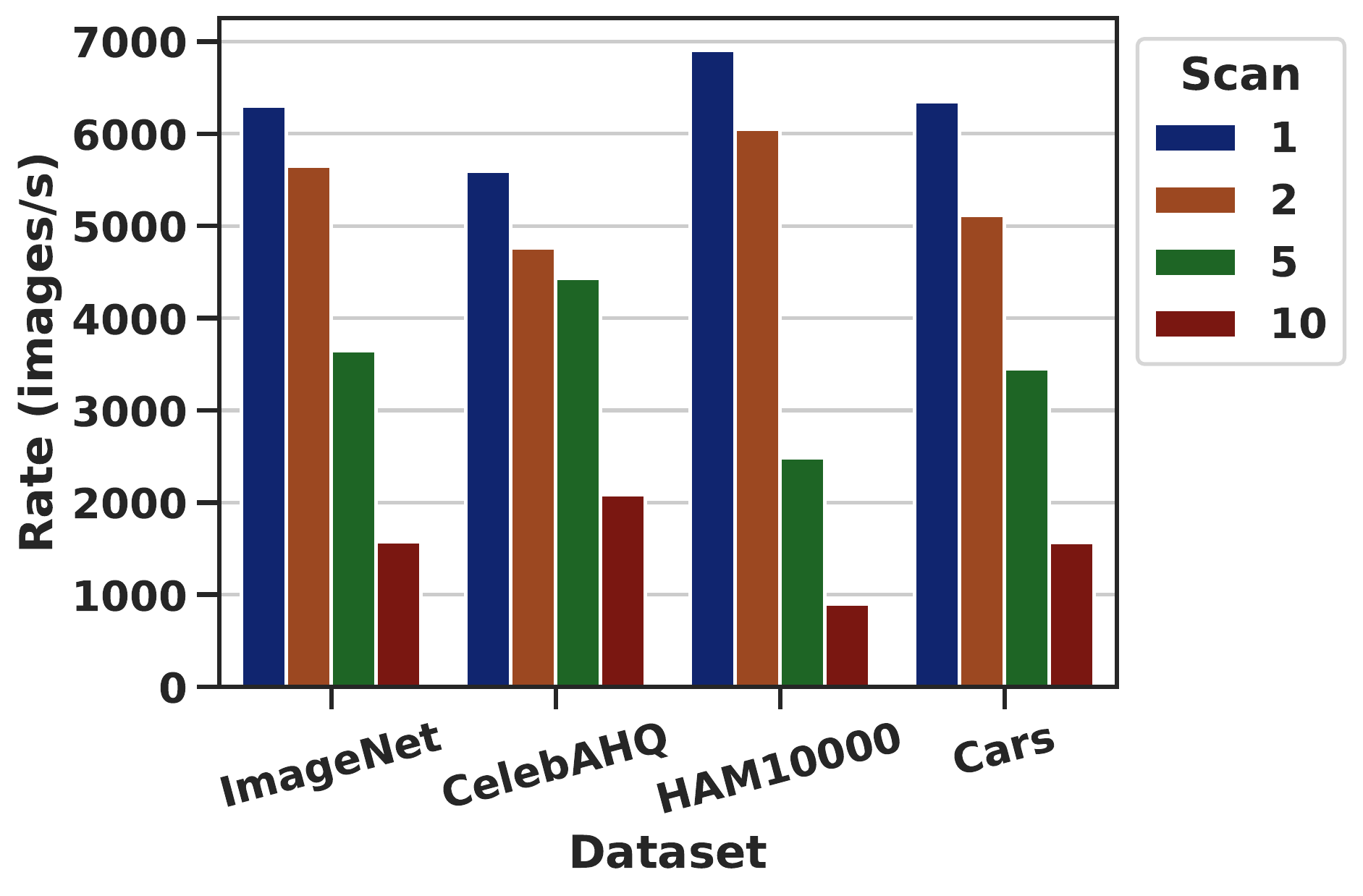}{1.1}{0.0}
    \caption{ShuffleNet}
  \end{subfigure}%
  \caption{%
    Training rates for ResNet and ShuffleNet.
    More scans reduce the rate of images/second.
    From RAM, they can process 4200/7000 images/second, respectively.}%
  \label{fig:PCR_Image_Rates}%
\end{figure}

\begin{figure*}[th]
  \centering
  \begin{subfigure}[t]{0.25\textwidth}
    \includegraphics[width=.99\linewidth]{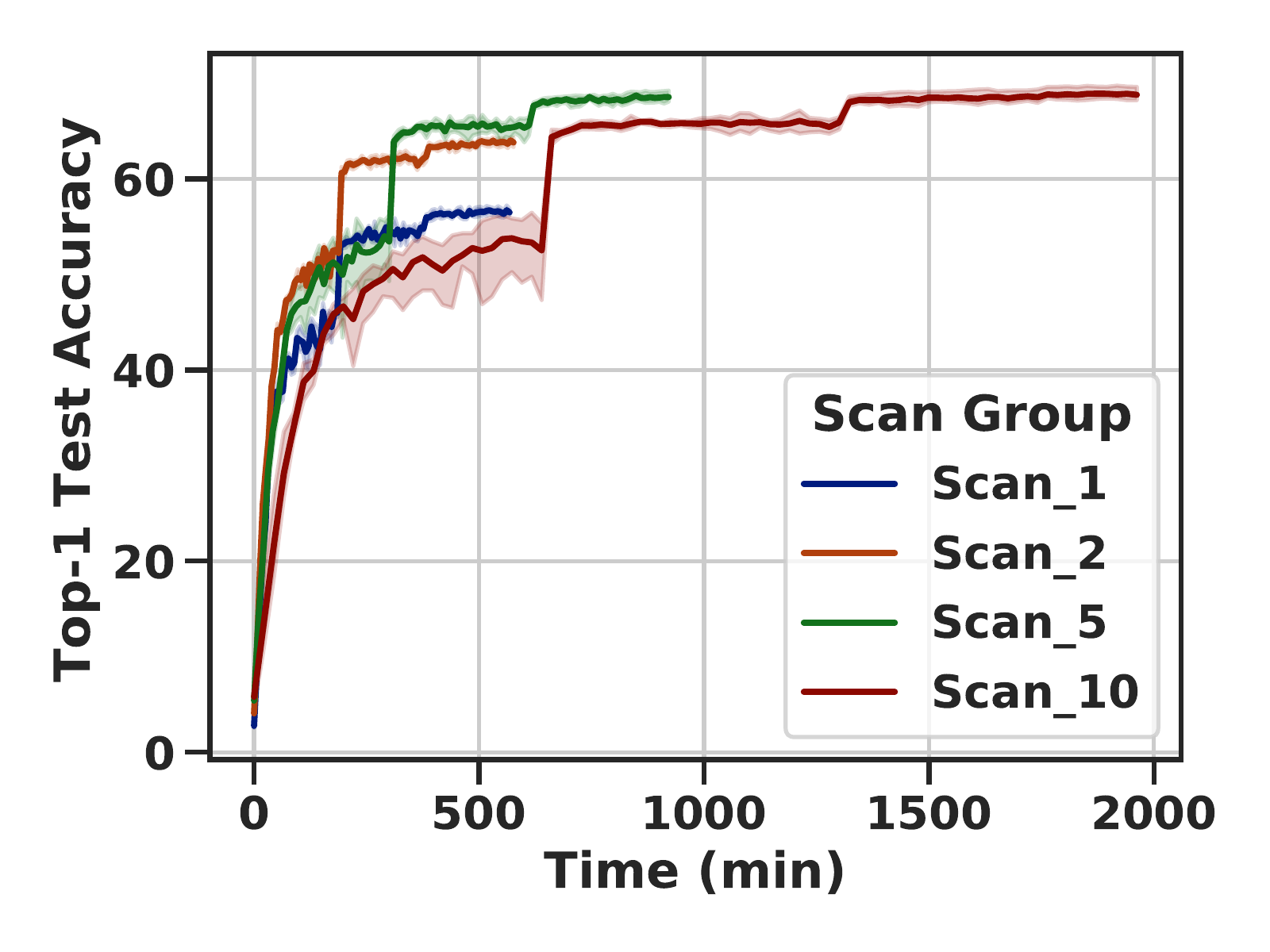}
    \caption{ResNet18 100MiB/s}
  \end{subfigure}%
  \begin{subfigure}[t]{0.25\textwidth}
    \includegraphics[width=.99\linewidth]{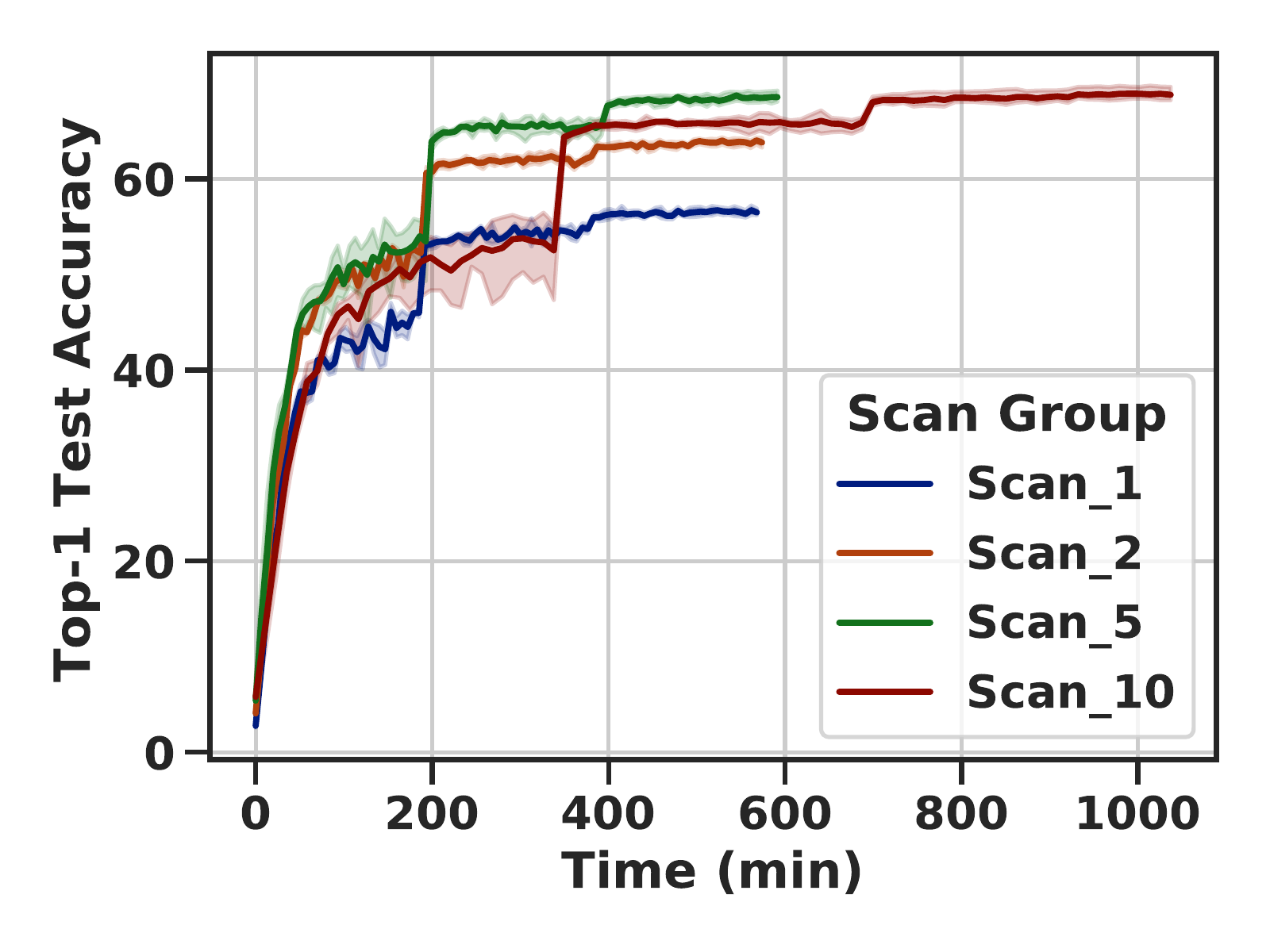}
    \caption{ResNet18 200MiB/s}
  \end{subfigure}%
  \begin{subfigure}[t]{0.25\textwidth}
    \includegraphics[width=.99\linewidth]{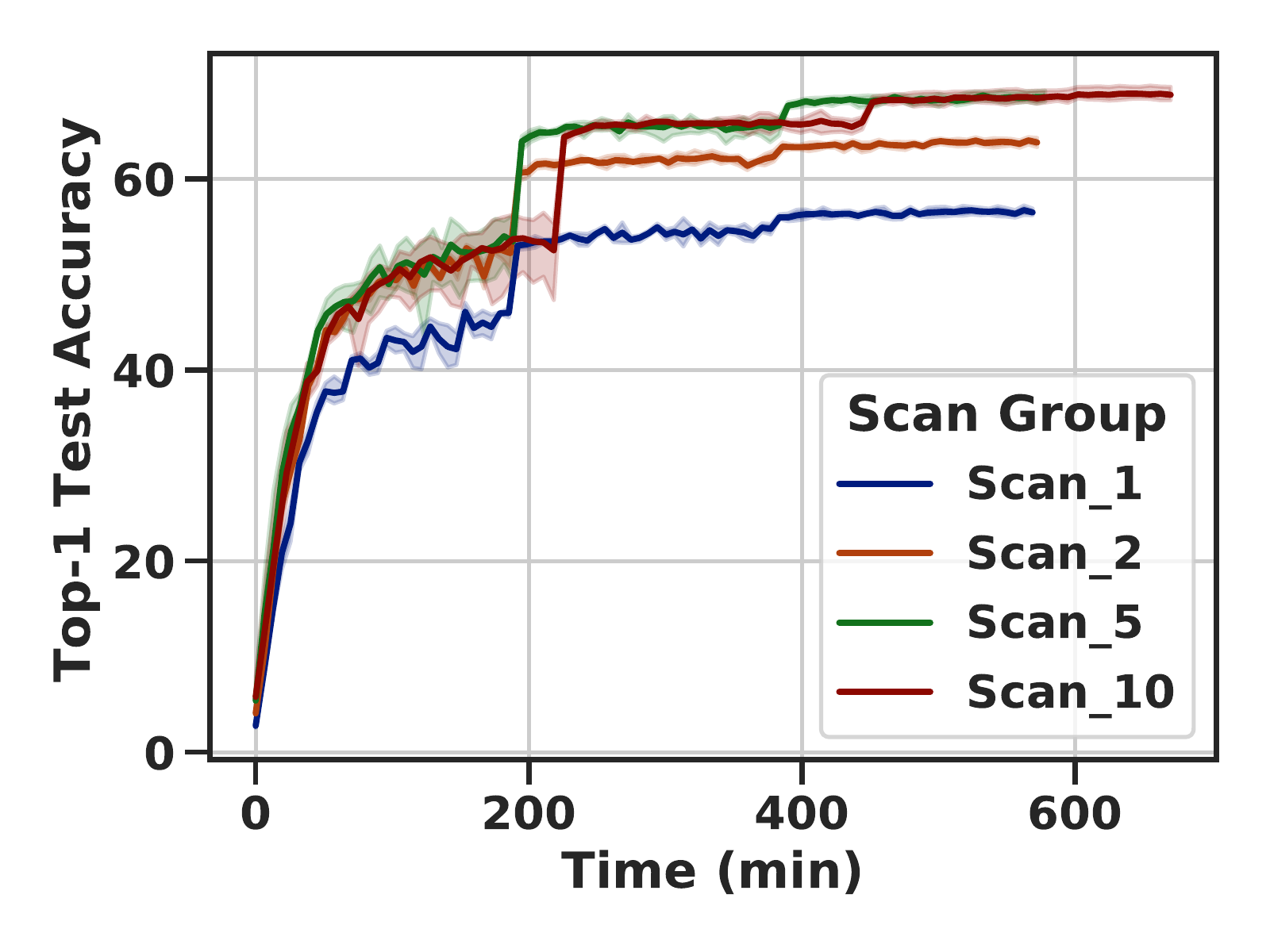}
    \caption{ResNet18 300MiB/s}
  \end{subfigure}%
  \begin{subfigure}[t]{0.25\textwidth}
    \includegraphics[width=.99\linewidth]{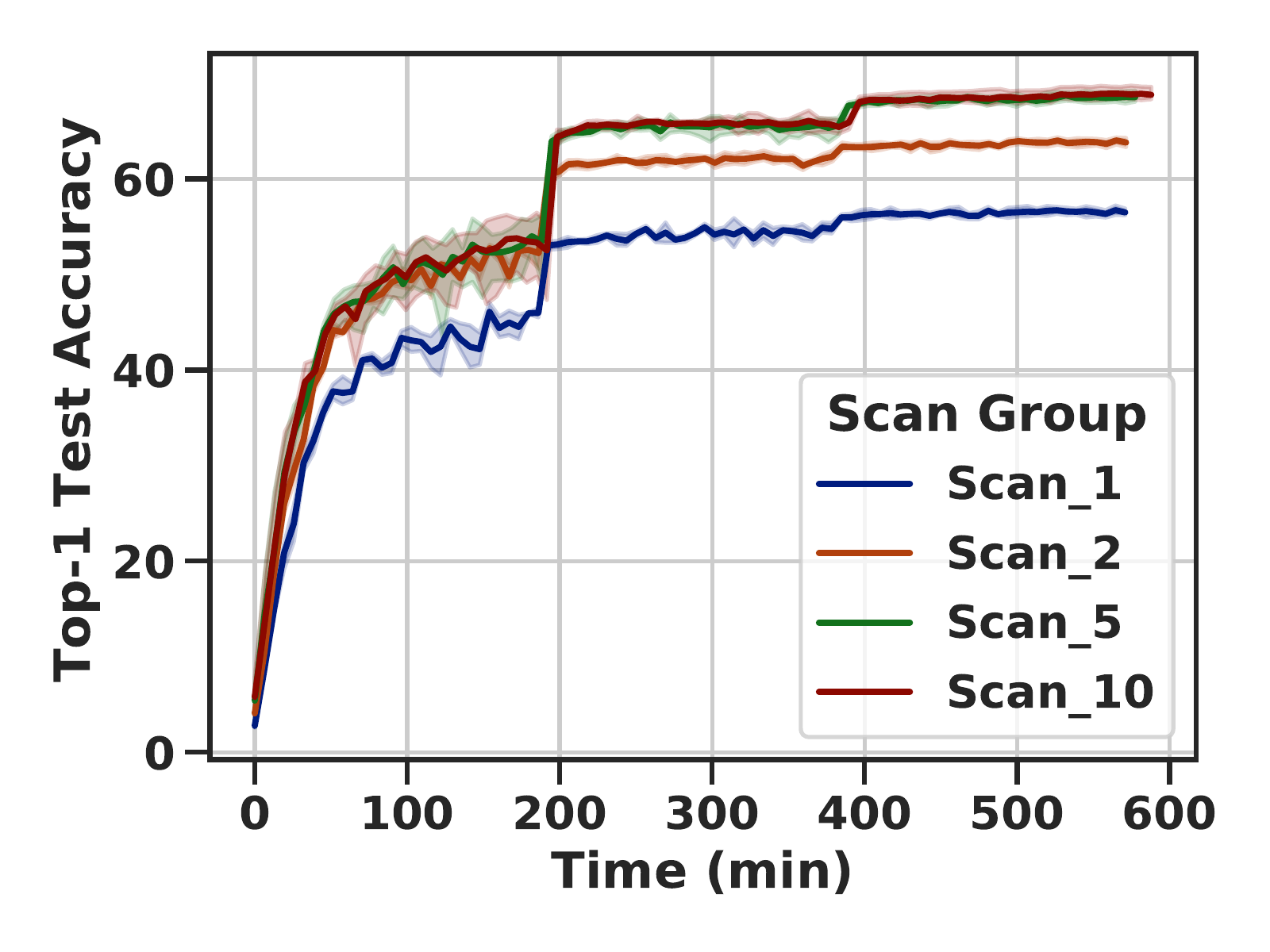}
    \caption{ResNet18 500MiB/s}
  \end{subfigure}%
  \newline
  \begin{subfigure}[t]{0.25\textwidth}
    \includegraphics[width=.99\linewidth]{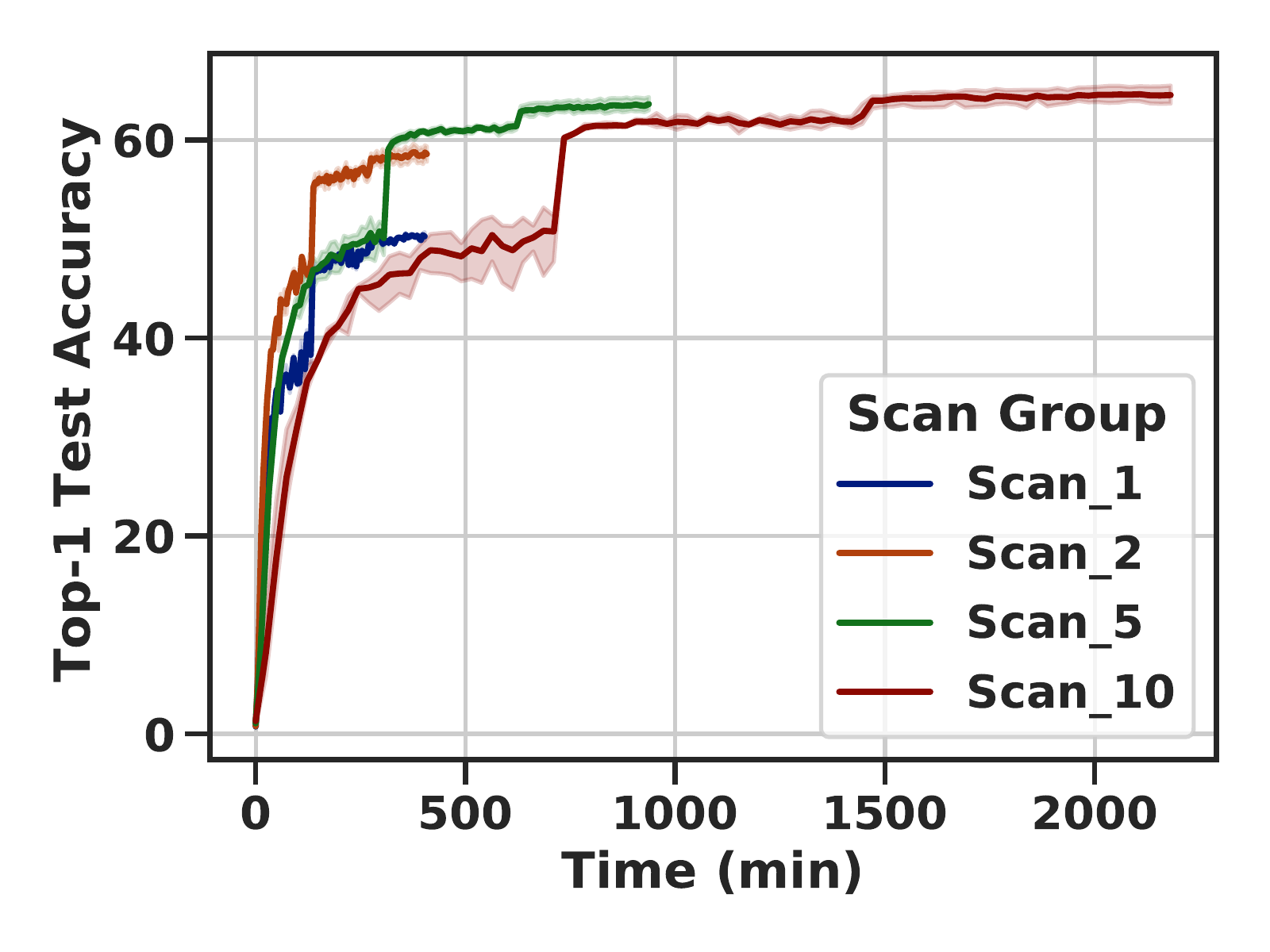}
    \caption{ShuffleNet 100MiB/s}
  \end{subfigure}%
  \begin{subfigure}[t]{0.25\textwidth}
    \includegraphics[width=.99\linewidth]{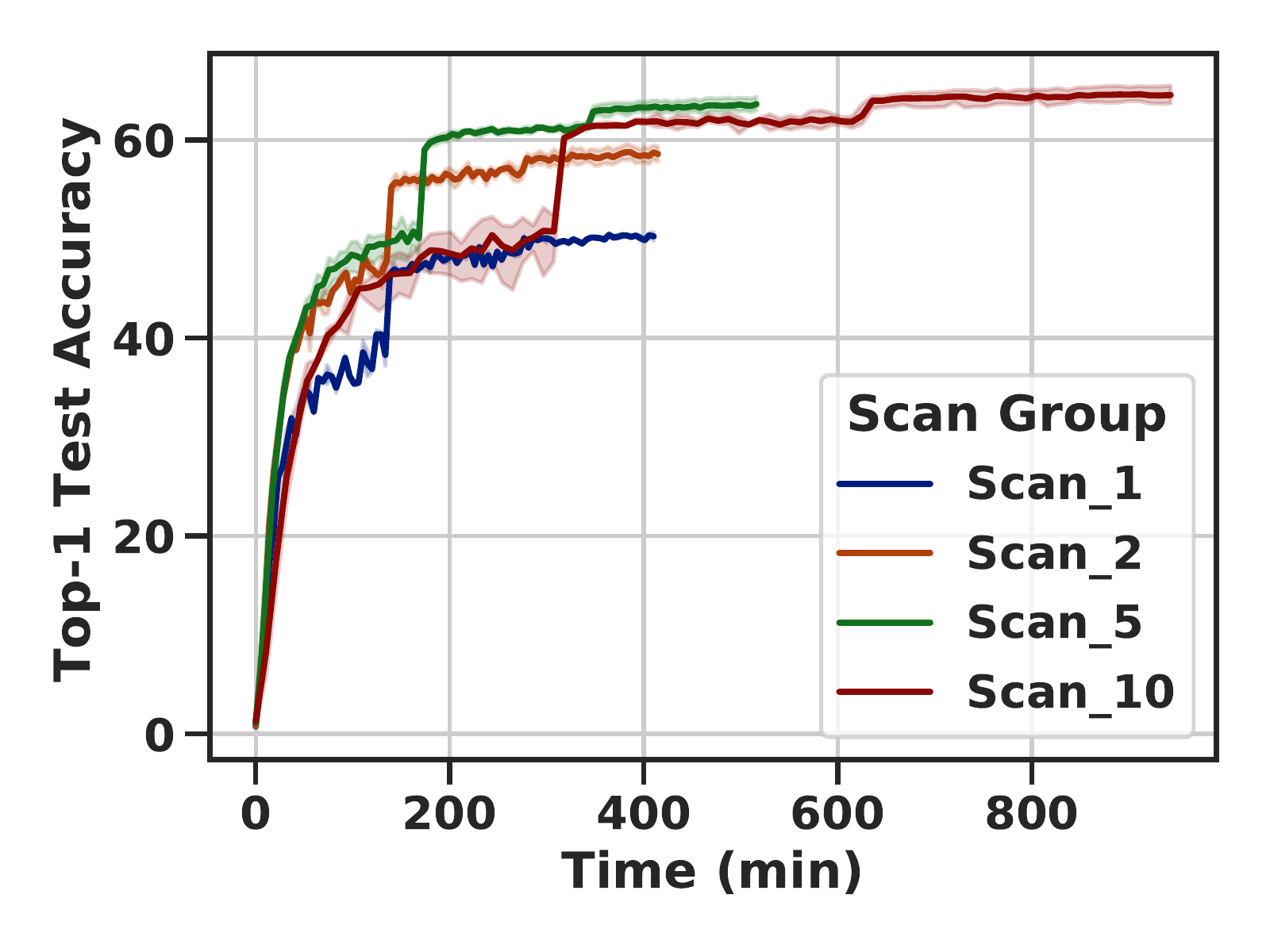}
    \caption{Shufflenet 200MiB/s}
  \end{subfigure}%
  \begin{subfigure}[t]{0.25\textwidth}
    \includegraphics[width=.99\linewidth]{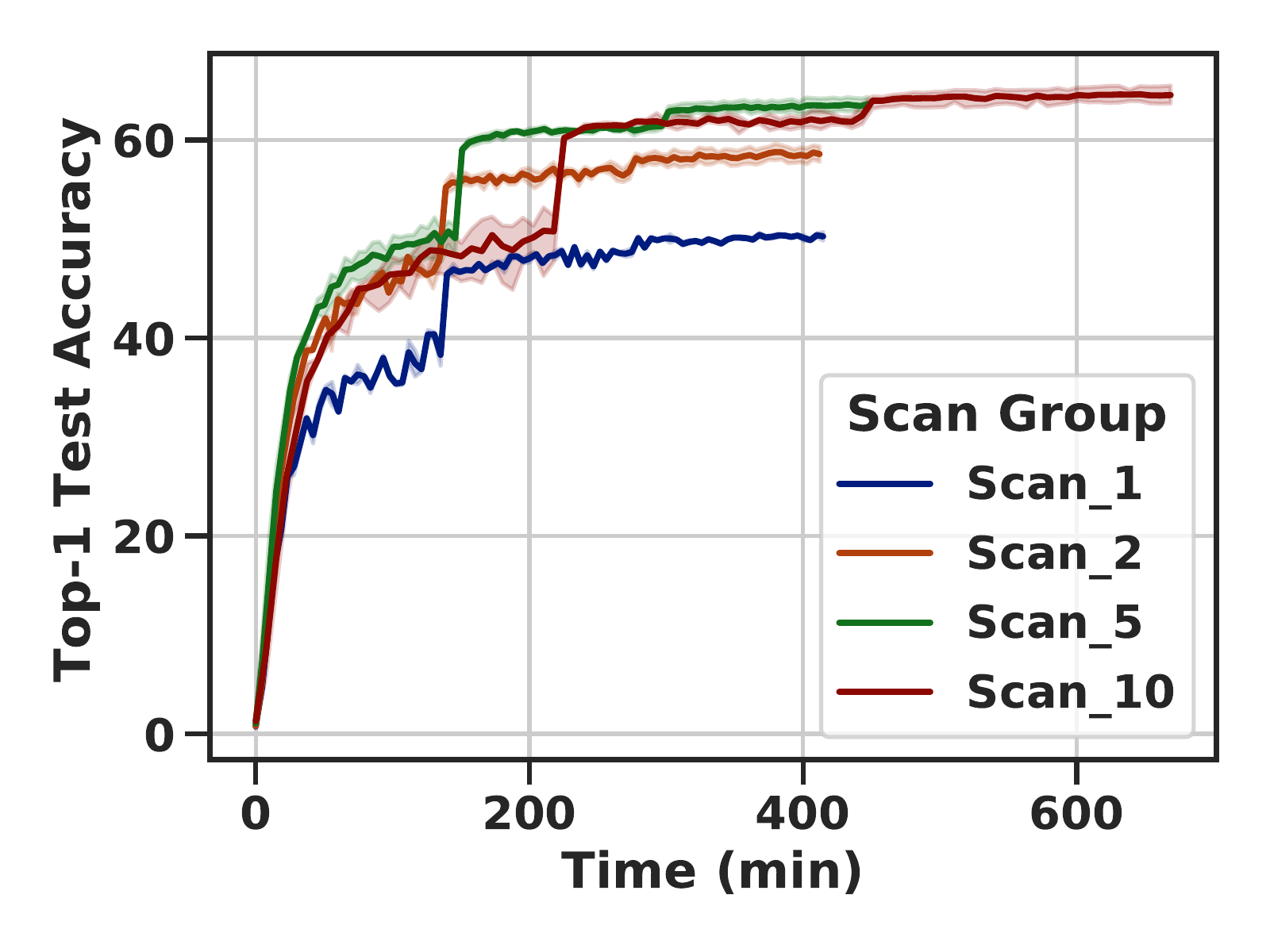}
    \caption{Shufflenet 300MiB/s}
  \end{subfigure}%
  \begin{subfigure}[t]{0.25\textwidth}
    \includegraphics[width=.99\linewidth]{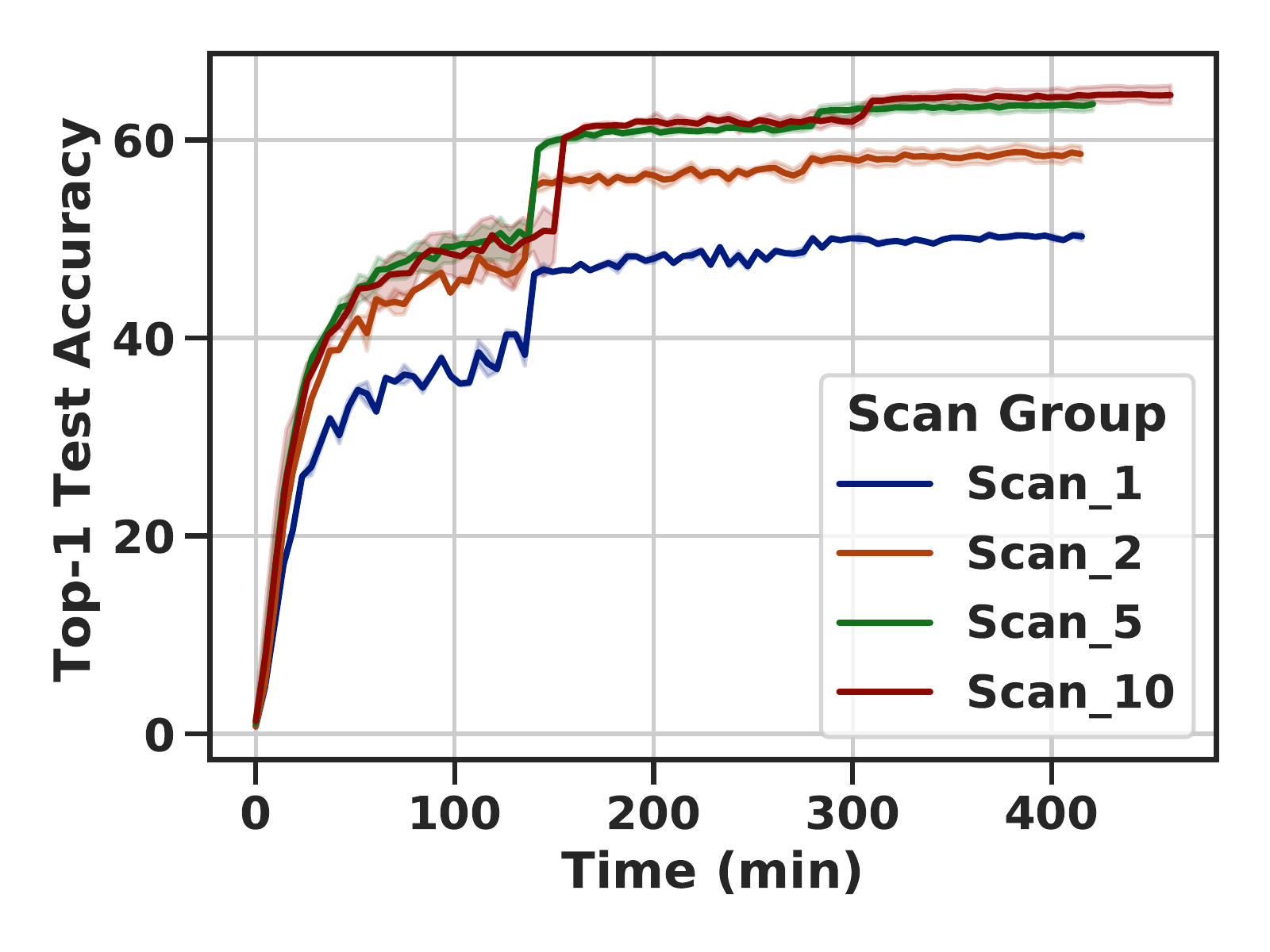}
    \caption{Shufflenet 500MiB/s}
  \end{subfigure}%
  \caption{%
    The effect of various amounts of bandwidth on a 10-node cluster running
    ResNet-18 (top) and ShuffleNet (bottom).
    At very low bandwidth, all scans provide benefits, while higher bandwidths
    provide less benefits.
    Similarly, faster models (e.g., Shufflenet) or accelerators increase I/O pressure, which
    enables low scans to obtain higher speedups---scan 1/2 are beneficial for
    ShuffleNet at 200MiB/s, but not ResNet, and the higher I/O pressure carries
    into 500MiB/s.
  }%
  \label{fig:estimated_convergence}%
\end{figure*}

Finally, we validate why lower compression levels yield faster training by observing image loading rates. Loading rates for  training are shown in
Figure~\ref{fig:PCR_Image_Rates}.
Using more scans slows down training significantly, as seen in image
rates.
Training slowdowns manifest as latency spikes from a data stall, causing rates to fluctuate
considerably.
Informally, we can
perform twice as many read operations if we decrease the data read by each
operation by $2\times$ (Section~\ref{sec:speedup_analysis}).
The speedup can be calculated through the average PCR size
(Figure~\ref{fig:predictedThroughput} and Figure~\ref{fig:PCR_Inspector_Sizes}).
Since ShuffleNetv2 is capable of a higher maximum training rate
than ResNet-18, it achieves higher speedups. %
As HAM10000 has the largest images, it is the most
bottlenecked by image loading bandwidth---scan 5 is $2.9\times$ smaller than
scan 10.
For CelebAHQ, scan 2 is roughly the
same size as scan 5; as expected, image rates are very similar.
For the 20 worker runs, we observe that scan 1 and 2 for ImageNet have a median
epoch latency of 100 seconds, while TFRecords and scan 10 have a median epoch latency
of 300 seconds---even though the size difference is over $10\times$, $3\times$ is
the best factor that can be achieved before hitting in-memory processing rates.
These results indicate that systems with large images, efficient models, and fast compute would be the biggest
benefactors of PCRs.

\vspace{.5em}
\observation{\textit{Image loading rates are directly linked to the the
underlying model and bandwidth.}}
\vspace{5pt}

While faster compute hardware can speed up a fixed model (e.g.,
Figure~\ref{fig:simulatedslowdown}), it is less clear how bandwidth impacts training rates, especially with
different hardware/model combinations.
To explore this question, we implement the token-bucket algorithm in our \texttt{tf.data} implementation.
Each second, a node accumulates a fixed amount of tokens, which are traded for
bytes read off storage, and thus
nodes will block if they use too much bandwidth in a given time.
We rate limit the bandwidth of each of the 10 nodes in the cluster across a
sweep of aggregate cluster bandwidths in Figure~\ref{fig:estimated_convergence},
calculate the time per epoch over 7 minutes (using the data shown in Figure~\ref{fig:predictedThroughput}),
and utilize the previous convergence results in Figure~\ref{fig:scan_performance_joined_orca_acc_time}
to project the
accuracy-over-time graphs for 90 epochs.
As previously observed, lower scans benefit the most from low bandwidth, and
faster models are more bottlenecked.
For instance, although ShuffleNet is typically faster than ResNet-18, it still
takes a similar amount of time at low bandwidth to finish training---simply
because nearly all time is waiting on I/O.
We do not see much benefit for lower scans at high bandwidth, but low bandwidth
(e.g., 20MiB/s) shows gains for even scan 1 over scan 2 (Figure~\ref{fig:predictedThroughput}).
\section{Discussion}%
\label{sec:discussion}
\paragraph{\textbf{Data caching.}}
Data can be partially cached in memory (e.g., OS page cache);
however, uniform sampling means that over 75\% of the data must be
cached to see the majority of speedup---caching does not significantly affect
our ImageNet results.
Specifically, the expected latency of a read is a convex combination of the disk
and memory latency~\cite{dsanalyzer}, where disk latencies are usually high.
PCRs help reduce cache pressure by reducing the number of bytes read (and thus
the size of the working set).
Further, PCRs facilitate cache sharing in a multi-user environment, as multiple
fidelities share common data, eliminating double caching.

\paragraph{\textbf{Hardware acceleration.}}
Hardware JPEG decoders are popular in mobile phones,
and PCRs could take advantage of such hardware support.
In fact, NVIDIA's A100 is the first datacenter GPU to ship with
hardware decoding~\citep{A100JPEG}; prior versions used software
acceleration~\cite{nvjpeg}.
The reasoning behind hardware support is simple: 
33\% or more of CPU time can be spent on image decoding~\cite{dsanalyzer} and 96 cores or more are
currently matched with an accelerator~\citep{MLPerf07},
which directly cuts into training cost
efficiency~\citep{jouppi2020domain}---empirically, 20\% of jobs spend over 30\%
of their compute budget on data ingest~\cite{murray2021tfdata}.
Frameworks like DALI already offload part of the JPEG decode to the GPU, namely the
Discrete Cosine Transform; mapping Huffman decoding to the GPU requires
additional
work~\cite{huffman1952method,klein2003parallel,weissenberger2018massively}.
A different practice is to avoid the CPU by caching \textit{decoded} images in
memory~\citep{MLPerf07,kumar2020exploring,dsanalyzer},
though this has limitations as datasets are large, especially when uncompressed.

\paragraph{\textbf{Non-image datasets and training tasks.}}
Our exp\-er\-i\-ments indicate that PCRs are robust (in terms of accuracy)
across a variety of tasks, and we only focus on a subset of tasks due to a limited
computation budget.
Using ResNet50~\cite{he2016deep} on the ImageNet~\cite{ILSVRC15} tasks, we obtain 75.14\% vs.\ 75.47\%
(scan 5/baseline) accuracy, and 38.37/38.80 AP on the FPN-ResNet50/COCO
task~\cite{lin2017focal,lin2014microsoft}.
While the computational difference between ResNet18 and ResNet50 is only 2--3 years of accelerator
progress (Figure~\ref{fig:simulatedslowdown}), we note that detection datasets
(like COCO) can be 10-100$\times$ slower to compute (albeit with larger images).
Apart from image-based datasets,
PCRs generalize to other datasets and modalities as long as the encoding is progressive.
For example, each component in Principal Component Analysis (PCA)~\cite{shlens2014tutorial} is a
progressive approximation of the source dataset, and removing (e.g.,
compressing) 50\% of the components loses 1\% of accuracy over YouTube videos~\cite{abu2016youtube}.
Another general encoding is quantization, which progressively encodes subsets
of the higher-order bits (e.g., the first 25\%) in a dataset's features,
an approach baked into progressive JPEG and used in YouTube-8m~\cite{abu2016youtube}.
These results suggest that if one were to implement PCRs via PCA/quantization over videos, they would
obtain $8\times$ bandwidth savings total with only a 1\% loss in accuracy.

\paragraph{\textbf{Generalizing across hardware.}}
To test the generalization of our system, we investigate applying PCRs
to a Google
Cloud~\cite{google_cloud} \texttt{n1-instance-16} with a \texttt{P100} GPU.
We attach a 150GB HDD for the operating system and use a SSD for
the data loading, which has peak bandwidth of 74MB/s (similar per-GPU load as
prior work~\cite{dsanalyzer}).
On ImageNet/ShuffleNet, we observe 650 images/second for TFRecords, and 680
(scan 10), 1540 (scan 5),
1700 (scan 2), and 1750 (scan 1) for PCRs.
The difference between TFRecords and scan 10 can be explained by progressive
compression being 6\% smaller in size.
Using either as a baseline, scan 5 is over $2\times$ faster.
Because the ratio of resources primarily matters,
we observe that doubling the CPU, GPU, and SSD resources maintains the same
relative performance advantages for PCRs, yielding a $2.2\times$ speedup from
TFRecord to scan 5, and a $2.7\times$ speedup from scan 1/2.

\section{Related Work}%
\label{sec:related_work}
\newcontent{%

Numerous works have explored methods for decreasing training time with
large
datasets~\citep{goyal2017accurate,you2018imagenet,jia2018highly,ying2018image,yamazaki2019yet,
kurth2018exascale,kumar2019scale}, motivating a need for improved I/O
internals~\citep{chowdhury2019characterization,chien2018characterizing,pumma2019scalable,Baylor:2017:TTP:3097983.3098021,aizman2019high},
formats~\cite{PyTorchWebDataset},
caching~\cite{kumar2020exploring},
and data pipeline frameworks~\cite{murray2021tfdata}.
We discuss caching, forms of compression, and frequency-domain DL literature
below.

\paragraph{Dataset Caching.}
Caching places data in faster storage tiers to offload the bandwidth burden from slower devices.
ML applications lack locality due to uniform sampling~\cite{quiver},
requiring either prohibitively large cache sizes or
weaker forms of sampling~\cite{meng2017convergence}.
However, when done correctly, caching can obtain significant speedups~\cite{dsanalyzer}.
PCRs are designed for datasets which do not fit in caches, and, by virtue of
accessing less data, can increase cache hit rates.

\paragraph{Dataset Cardinality Reduction.}
``Big data'' spawned interest in dataset reduction techniques
that aim to reduce the \textit{number} of training samples while maintaining model accuracy~\citep{pmlr-v99-karnin19a,feldman2013turning,liberty2013simple,woodruff2014sketching,45938,bachem2017practical,matsushima2012linear}.
Similarly, dataset echoing~\citep{choi2019faster,agarwal2020stochastic} re-uses subsamples
to speedup the data pipeline.
PCRs differ in that they reduce I/O burden by
modifying data representation and layout.

\paragraph{Dataset Sample Compression.}
Techniques such as compression~\citep{DeepNJPEG,abu2016youtube} or
resizing~\citep{karras2017progressive} reduce data size by lowering fidelity,
reducing I/O pressure.
Prior work has shown that resizing as a form of data reduction is particularly
effective for DL tasks, as resized data can speed up training, transfer to
high fidelity test points, and in some cases, even \textit{increase}
accuracy when combined with certain data augmentations~\cite{adascale,karras2017progressive, fixing_train_test}.
However, resizing parameters are chosen statically and heuristically (therefore
suboptimally~\cite{fixing_train_test}), and thus may not meet the needs of all
applications without duplicated work.
PCRs differ in that they provide a \textit{dynamic} mechanism for adjusting I/O load, and thus
 can adapt to both the system and task at runtime.

Similar to our work, MLWeaving~\cite{MLWeaving} has shown that
\textit{transposed layouts} (i.e., column major) can accelerate machine learning training.
However, this work differs in that we focus on I/O in the context of deep
learning models, whereas MLWeaving
focuses on memory bandwidth for general linear models.
Additionally, the compression method differs.
For image data, the three canonical dimensions of compression are 1)
quantization, 2) frequency selection, and 3) spatial
selection~\cite{wallace1992jpeg}; MLWeaving uses
the first while PCRs use the first and second (via JPEG).

Neural compression~\cite{toderici2017full}, which learns custom compression
formats using neural networks, is an interesting direction for future work and
is compatible with PCRs.
However, while neural compression can outperform JPEG in terms of quality~\cite{toderici2017full}, it
does so at significant cost.
Using state-of-the-art neural compression~\cite{mentzer2020high,nonlinear_image_compression,variational_image_compression}, we find decoding to be between
$900\times$ and $5000\times$ the cost of baseline JPEG, and thus incompatible
with real-time performance.

\paragraph{General Compression.}
Bandwidth reduction extends to databases, memory hierarchies, and the web~\citep{zukowski2006super,abadi2006integrating,pekhimenko2018tersecades,pekhimenko2012base,agababov2015flywheel}.
Progressive compression has been used in the context of
dynamically saving bandwidth for mobile phone downloads~\citep{yan2017customizing}.
HippogriffDB~\citep{li2016hippogriffdb} uses GPUs to compress
data in the context of databases, which lowers
I/O bandwidth to get a speedup.
Other work has investigated how image degradation affects
inference~\citep{dodge2016understanding,vasiljevic2016examining,peng2016fine,45227}.
In contrast, our work is focused on compression for I/O savings in deep
learning.
Reinforcement learning has been used to choose JPEG parameters for
cloud inference~\citep{adacompress}; other work has hand-designed
JPEG encodings for training~\citep{DeepNJPEG}.
These works are similar to ours in that they tune compression for
the model, though they differ in that they are static.
Other work investigates compressing
models~\citep{han2015deep,han2016eie,han2015learning,cheng2017survey,xu2018deep,hwang2014fixed,anwar2015fixed,denton2014exploiting}
or network
traffic~\citep{lim20183lc,alistarh2017qsgd,lin2017deep,wen2017terngrad,wangni2018gradient,zhang2017poseidon};
these are orthogonal to our work.

\paragraph{Frequency Domain Deep Learning.}
Prior work modifies models to directly train over compressed
representations~\citep{gueguen2018faster,torfason2018towards,fu2016using,ulicny2017using}
or with frequency-domain operators~\cite{dziedzic2019band};
our work does not modify the model.
Other work investigates generalization performance from the view of
low~\cite{xu2019training,basri2019convergence} and high~\cite{Wang_2020_CVPR} frequencies, which
provides insight into our work.
JPEG mostly filters low frequency components, and thus prior work has attempted to use
JPEG as a defense mechanism against adversarial
attacks~\cite{feature_disillation,das2018shield,dziugaite2016study}.
Motivated by adversarial attacks exploiting spurious, high-frequency
features~\cite{NEURIPS2019_e2c420d9,geirhos2018imagenettrained}
other work investigates if frequency filters can impact
model robustness~\cite{NEURIPS2019_frequency_perspective,benchmark_adv};
our work primarily focuses on retaining test accuracy under
non-adversarial conditions.
}%
\section*{Conclusion}%
\label{sec:conclusion}
We introduce a novel storage format, \textit{Progressive Compressed Records}
(PCRs), to reduce the bandwidth cost of training over large datasets.
PCRs use progressive
compression to split training data into multiple fidelity levels, while avoiding duplicating space. %
The format is easy to implement and can be applied to a
broad range of tasks dynamically. 
PCRs provide applications with the ability to trade off data fidelity
with storage and network demands, allowing the same model to be trained with
$2\times$ less bandwidth while retaining model accuracy. We introduce methodology for choosing the particular data fidelity necessary for a task, as well as a tuning heuristic that can be applied automatically. Using PCRs, our approaches can
dynamically switch between multiple data fidelities while training without loss of accuracy.
Future directions include alternative compression methods, data modalities, and hardware acceleration.

\begin{acks}
We thank the anonymous reviewers for their help improving the presentation of this paper.
We thank the members and companies of the PDL Consortium: Amazon, Facebook, Google, Hewlett Packard Enterprise, Hitachi Ltd., IBM Research, Intel Corporation, Microsoft Research, NetApp, Inc., Oracle Corporation, Pure Storage, Salesforce, Samsung Semiconductor Inc., Seagate Technology, Two Sigma, and Western Digital for their interest, insights, feedback, and support. 
Michael Kuchnik is supported by a National Defense Science and Engineering Graduate Fellowship.
This research was supported with Cloud TPUs from Google's TPU Research Cloud and research credits from Google Cloud Platform.
\end{acks}

\clearpage

\balance%
\bibliographystyle{ACM-Reference-Format}
\bibliography{ms}

\clearpage
\appendix
\section{Appendix}
\subsection{Data Pipelines}
\textbf{Data Pipeline Overview.}%
\label{sec:pipeline_overview}
A model of the training pipeline, including both the compute unit and data pipeline, is shown in
Figure~\ref{fig:data_pipeline}.
The compute unit ({\color{red}red}) calculates model updates over multiple
data points, called a \textit{minibatch}.
This data is served by the data pipeline ({\color{blue}blue}) in units of
\textit{records}.
There are often multiple minibatches contained within a record (accessed via \textit{record
splitting}), and records can
be shuffled in memory to further randomize minibatches.
However, while record splitting and randomization are important to model
convergence, their use does not change the compute time per data item once GPUs
are already saturated, and thus, we can simply abstract the computation to
operate over records directly.

Both compute and data portions of the training pipeline operate as fast as they
can; however, the compute unit needs to wait for data.
The loader operates as a \textit{closed system},
starting the next piece of
work after the last is finished.
Each finished piece of work is placed in a queue to be used
first-come-first-serve by the compute unit.
The compute unit operates as an \textit{open system}, waiting for work to be
assigned to it by the data loader.
There is a dependency between the data pipeline's output and the parameter
update input, and thus, the data pipeline may block the compute unit, which
we call a \textit{data stall}.
If the data loader cannot prefetch data before the current update is finished, there is no work for the
accelerator to work on, and thus, the parameter updates will start in lockstep
with the data fetches.
Although this is a simple model of the training computation (e.g., ignoring
multithreading),
it captures the essence of training behavior when data bandwidth is altered.

\begin{figure}
  \centering
  \includegraphics[width=.85\linewidth]{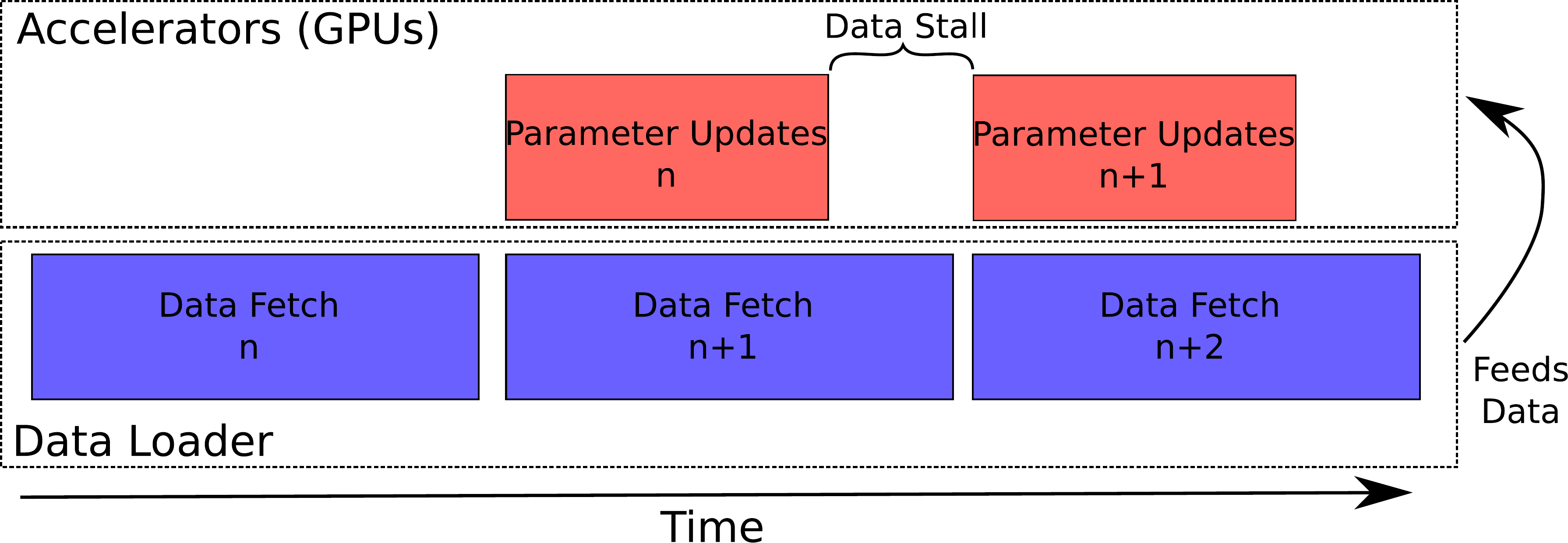}
  \vspace{-0.5em}
  \caption{%
    A data pipeline feeds the model with future data
    in parallel to model updates for the current data.
    The entire system can be modeled as two components (i.e., data loader and
    compute)  operating in sequence (i.e., data loader $\to$ compute).
    As models get faster at computing their minibatch updates, the amount of
    time available to fetch data will decrease, and eventually cause \textit{data
    stalls}, or periods of time solely spent waiting on fetching data
  }%
  \vspace{-1em}%
  \label{fig:data_pipeline}%
\end{figure}

\textbf{Relating Data Stalls and Data Bandwidth.}%
\label{sec:system_analysis}
The datasets we evaluated show that data loading can slow down the training
process by causing data stalls.
To highlight these slowdowns, and the improvements PCRs achieve by not using all
scan groups, we present the loading time of data for a ResNet18 ImageNet
run in
Figure~\ref{fig:imagenet_scan_performance_resnet18_orca_data_times}.
We obtain similar results for the other datasets.
The baseline of using all scan group results in high periodic loading stalls,
where the prefetching queue was drained.
Upon blocking, training cannot proceed until the worker threads obtain a full batch
of data.
Periods of (mostly) no stalls are caused by both threads pre-fetching the
data and single records servicing
multiple minibatches.
Using fewer scan groups reduces the amount of data read, which can result in lower magnitude
stalls.
The periods of progress and stalls average out to a throughput (e.g., images per
second) over a long period of time, a subject we analyze in
Appendix~\ref{sec:throughput_analysis}.

\begin{figure}
  \centering
  \includegraphics[width=.85\linewidth]{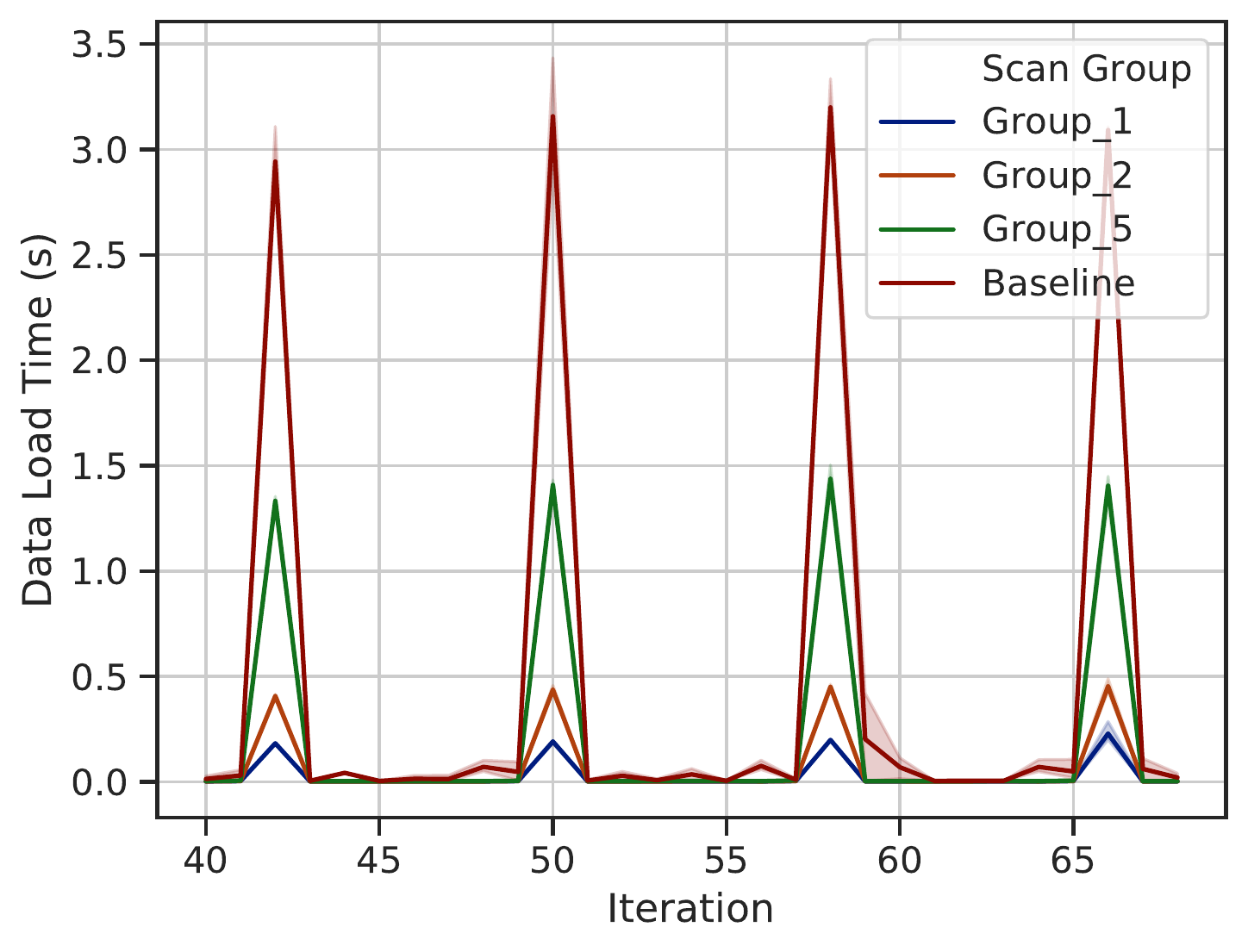}
  \caption{%
    Data loading stalls are periodic and followed by extents of prefetched data.
    Lower scan groups reduce stall time.
  }%
  \label{fig:imagenet_scan_performance_resnet18_orca_data_times}%
  \vspace{-1.5em}
\end{figure}

\subsection{Throughput Analysis}%
\label{sec:throughput_analysis}
As depicted in the bottom of Figure~\ref{fig:data_pipeline},
the data pipeline can be modeled as continuously fetching items sequentially.
We can model such a closed system with queueing theory and show that, as one
would intuitively expect, lower scan groups
increase data throughput (i.e., the rate of data items loaded per unit time).
First, we show the relationship between time to do a read and bytes read,
then we generalize this relationship to the stochastic setting by mapping the mean
time to expected throughput.
We can compute the expected behavior of the system with a few key
parameters; notably we only require to know the mean data size and not the shape
of the distribution.
To do so, we assume that read latencies are proportional to the number of bytes
read---reading at bandwidth $W$ is achieved after a input-size-independent setup
cost, which
  is a reasonable cost model for HDD/SSD~\cite{conway2017file}.
As shown in Figure~\ref{fig:imagenet_sizes}, a typical ImageNet image size is 110kB, but some images can be
larger (e.g., 200kB or more) or smaller (e.g., 4kB or less)---the main
unintuitive part of the analysis is dealing with such variance.

\begin{figure}
  \centering
  \includegraphics[width=.95\linewidth]{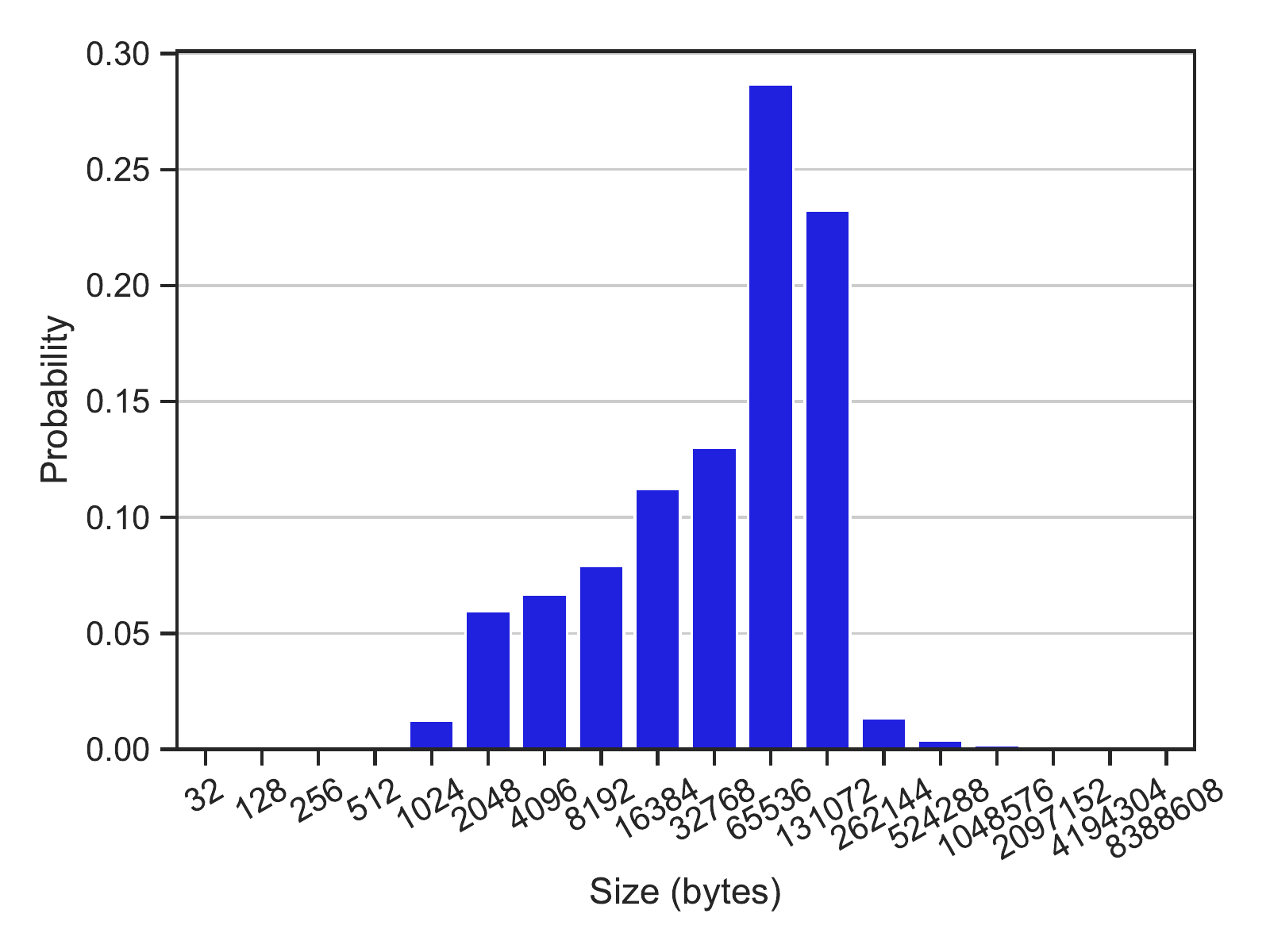}
  \vspace{-0.5em}
  \caption{%
    The sizes of images in ImageNet.
    Each image varies in size due to different dimensions and compression
    ratios.
    Most mass is concentrated close to the mode, but a small amount of mass exists
    in outliers.
  }%
  \vspace{-1em}%
  \label{fig:imagenet_sizes}%
\end{figure}

As one would intuitively expect, the average time to complete a read is
proportional to the average size of the record read (Lemma~\ref{lemma1}).
PCRs allow the system to modulate the expected record (likewise image) size
and thus decrease the
time per record accordingly.
The main text shows the bytes per record
decrease for lower scan groups.
Knowing the average time of a read allows us to calculate the average
throughput (Lemma~\ref{lemma2}), which determines how long it takes to perform any fixed amount of
model updates (e.g., epochs).
Exploiting data reduction allows the loader to obtain a throughput speedup
(Lemma~\ref{lemma3}).
However, since the entire training process is limited by the rate of both
parameter updates and data loading (Lemma~\ref{lemma4}), the training speedup is limited by any
computational bottlenecks (e.g., saturated GPU) reached by the parameter
updates, as shown in Figure~\ref{fig:scan_vs_throughput}.
Thus, for I/O bound tasks, the speedup is proportional to PCR data reduction
(Theorem~\ref{thm1}).

\begin{lemma}%
  \label{lemma1}%
  Let $\train$ be the set representing the training data, and a record,
  $\sB_n$, consist of a batch of
  $n$ elements drawn from $\train$.
  Let the notation $\sB_n \sim \train^n$ denote the batched draw (with or without
  replacement) from a distribution over $\train$, and let $x \sim
  \train$ denote the individual draws from the same distribution.
  Let $s(\cdot)$ denote the size of the input in bytes (or the sum of the sizes of
  a set).
  Let $W$ be the device read bandwidth in bytes per unit time, which we assume operates at a constant
  rate for all records and operates one record at a time.
  The expected time to complete a record read is
  $\E_{\sB_n \sim \train^n}[t]=\Theta\left(\cfrac{n \E_{x
  \sim\train}[s(x)]}{W}\right)$.
  The amortized expected time to complete an image read is
  $\E_{x \sim \train}[t]=\Theta\left(\cfrac{\E_{x
  \sim\train}[s(x)]}{W}\right)$.
\end{lemma}
\begin{proof}
  The time to read a record is $t=\cfrac{s(\sB_n)}{W} + \Theta(1)$, where the
  constant time cost is due to overhead costs (e.g., disk seeking).
  Since the record is drawn from a distribution of images (and thus imposing a
  distribution over sizes), we can
calculate time to read a record in expectation by using linearity of expectation.
\begin{align*}
  \E_{\sB_n \sim \train^n}[t] &= \E_{\sB_n \sim \train^n}\left[\cfrac{s(\sB_n)}{W} + \Theta(1) \right] \\
  &= \cfrac{\E_{\sB_n \sim \train^n}[\sum_{x \in \sB_n} s(x)]}{W} + \Theta(1) \\
  &=\cfrac{n \E_{x \sim \train}[s(x)]}{W} + \Theta(1) \\
  &=\Theta\left(\cfrac{n \E_{x \sim\train}[s(x)]}{W}\right)
\end{align*}
Since a single record yields $n$ images, dividing the right hand side by $n$
  gives the amortized cost for images.
\end{proof}

For the remained of the analysis, we drop the asymptotic notation as we assume that $n$ is sufficiently
large that constants can be safely ignored.

\begin{lemma}%
  \label{lemma2}%
  Let the size of scan group $g$ be represented by $s(\cdot, g)$.
  Let $X_b$ be the baseline data pipeline throughput and $X_g$ be the scan group
  data pipeline throughput.
  The baseline image throughput (e.g., images per second) is
  $X_b=\cfrac{W}{\E_{x \sim \train}[s(x)]}$
  and the throughput
  at scan $g$ is
  $X_g=\cfrac{W}{\E_{x \sim \train}[s(x, g)]}$.
\end{lemma}

\begin{proof}
Little's law~\citep{littles_law} for single-job closed
systems~\citep{harchol2013performance} (i.e., the number of jobs in the system
is constant and equal to one, and they arrive at the throughput rate)
states that the
expected throughput, $X$, is related to the expected time (over jobs) of a job's completion
$\E[t]$ by an
inverse relationship:
$X={\E[t]}^{-1}$.
These results hold regardless of the \textit{shape} of the distribution of data.
Let $X_b$ be the baseline throughput and $X_g$ be the scan group throughput.
By the Lemma~\ref{lemma1} and Little's Law, the baseline throughput is
$X_b=\cfrac{W}{\E_{x \sim \train}[s(x)]}$
and the throughput
at scan $g$ is
$X_g=\cfrac{W}{\E_{x \sim \train}[s(x, g)]}$.
\end{proof}

It's worth noting that the baseline rate is simply a special case of using scan
groups, and it is equivalent in size (barring entropy coding) to having all scan groups.
Thus, we can substitute references to the baseline as just the last scan group.
Also, throughputs in terms of records can be obtained by simply dividing image
rates by $n$.

\begin{lemma}%
  \label{lemma3}%
  The data pipeline throughput speedup at scan group $g$ is the ratio of the
  mean reduced data size, $\E_{x \sim \train}[s(x,g)]$,
  and the mean baseline data size, $\E_{x \sim \train}[s(x)]$.
\end{lemma}
\begin{proof}
The speedup of PCRs at group $g$ is then $\cfrac{X_g}{X_b}$, which simplifies to
  $\cfrac{\E_{x \sim \train}[s(x)]}{\E_{x \sim \train}[s(x, g)]}$ with
  Lemma~\ref{lemma2}.
\end{proof}

\begin{lemma}%
  \label{lemma4}%
  A training pipeline's throughput, $X$, is bound by the throughput of the compute
  unit, $X_c$, as well as the throughput of the data pipeline at scan
  group $g$ (folding the baseline into the last scan group),
  $X_g$, by the equation $X\leq\min(X_c, X_g)$.
\end{lemma}

\begin{figure}
    \centering
    \includegraphics[width=.75\linewidth]{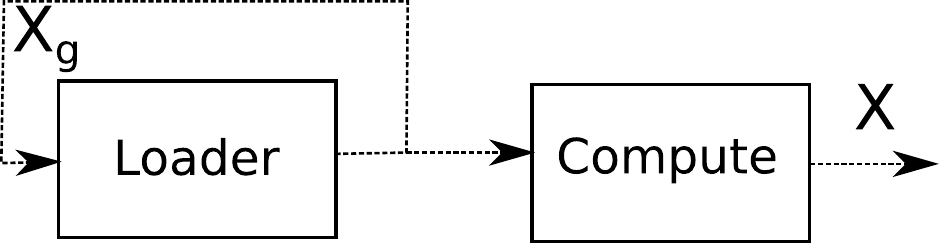}
    \vspace{-0.5em}
    \caption{%
      The queueing network for the system.
      The compute unit acts as an open system with arrivals determined by the
      loader, which acts as a closed system.
      The entire system's throughput, $X$, is determined by the maximum
      achievable throughputs of the two subsystems.
    }%
    \vspace{-1em}%
    \label{fig:system_analysis}%
\end{figure}%
\begin{figure}
  \centering
  \includegraphics[width=.75\linewidth]{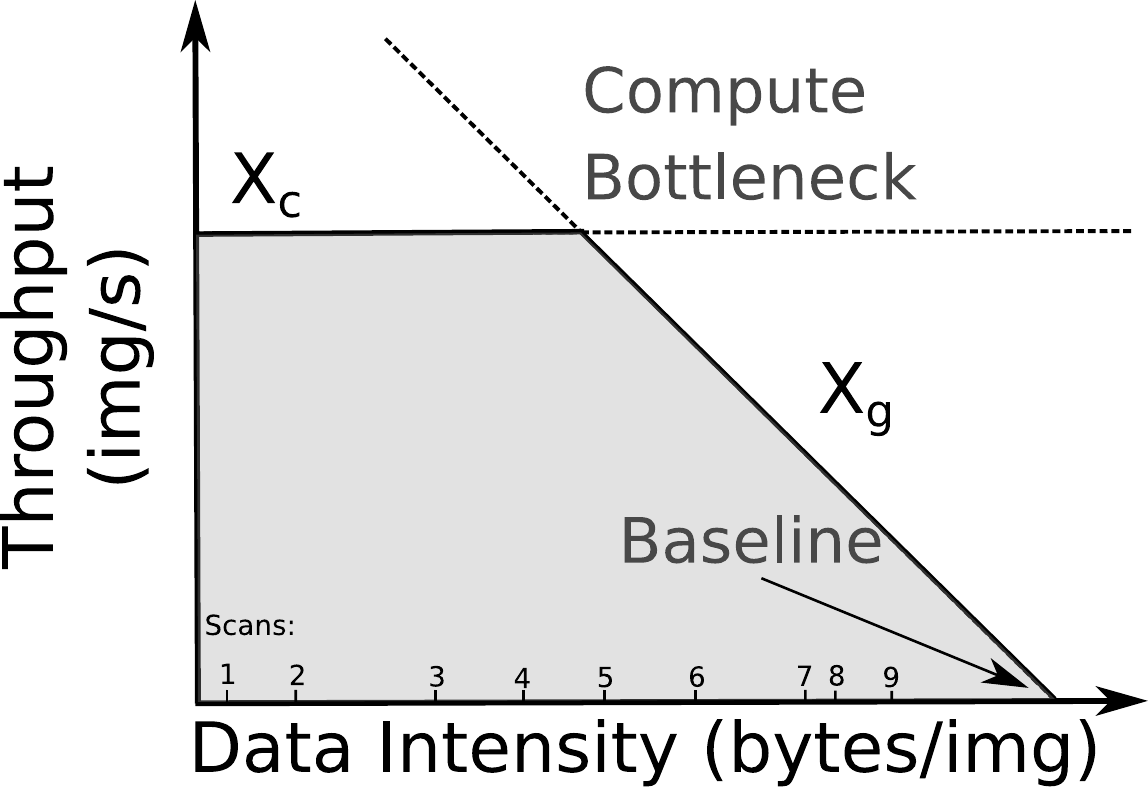}
  \vspace{-0.5em}
  \caption{%
    The system can process more images per second when a higher data rate is
    achieved via PCR data reduction.
    This trend continues until the compute units (e.g., GPUs) become saturated and
    the system becomes compute bound, which depends on the hardware and model.
    The shaded region corresponds to possible implementation throughputs.
    The bottom of the figure is marked with notches representing possible byte
    intensities for various scan groups (placed for illustrative purposes), allowing the user to increase
    throughput for bandwidth-bound workloads.
  }%
  \vspace{-1em}%
  \label{fig:scan_vs_throughput}%
\end{figure}

\begin{proof}
  The training pipeline is a system with the data pipeline feeding into
  the compute unit, as shown in Figure~\ref{fig:system_analysis}.
  Every training point has to pass through both components exactly once.
  The data pipeline at scan group $g$ operates at a rate of $X_g$,
  since it is a closed system and thus has $100\%$ utilization.
  The compute unit can operate at a maximum rate of $X_c$ when it is fully
  utilized (thus, it's service
  rate is $\mu=X_c$) and queues inputs
  from the data pipeline at rate $X_g$.
  Thus, the compute unit receives inputs from the loader as if the compute unit were
  an open system, the arrival rate was $\lambda=X_g$, and the service rate was
  $\mu=X_c$.

  Network analysis on open systems can then be used to determine $X$ for the
  whole system.
  The Utilization Law~\citep{harchol2013performance} states that for a device
  with $\lambda < \mu$, the equation holds
  $\rho=\lambda / \mu=X / \mu$, where $\rho\in[0,1)$ and is the device
  utilization.
  If $\lambda < \mu$, then the throughput, $X$, is $\lambda=X_g$.
  As $\lambda$ approaches $\mu$ from the left, $\rho$ goes to $1$ and $X=\mu=X_c$.
  If $\lambda \geq \mu$, the
  queue grows unbounded and the throughput, $X$, is limited by $\mu=X_c$ (the
  maximum achievable rate).

\end{proof}

\begin{theorem}%
  \label{thm1}%
  If a training pipeline is data bound ($X_c > X_g$ for scan group $g$), then the maximum
  achievable system speedup for switching to PCR scan group $g$ is
  $\cfrac{\E_{x \sim \train}[s(x)]}{\E_{x \sim \train}[s(x, g)]}$.
\end{theorem}

\begin{proof}
  By the assumption $X_c > X_g$, $X_g$ dominates the $\min(\cdot)$ term in
  Lemma~\ref{lemma3} and thus $X \leq X_g$.
  Substituting $X_g$ with Lemma~\ref{lemma2}, we find the speedup.
\end{proof}

Our bottleneck model is similar to that of the Roofline
model~\citep{williams2009roofline}, except we change
the graphs to highlight changes in \textit{data intensity} rather than
\textit{compute intensity}.
These data intensities are a property of the data and the progressive format:
larger images require more work per image and less scan groups require less
work per image.
Further, the derived bounds generalize to the distributed compute and storage setting
by simply measuring each subsystem's empirical throughput (e.g.,
one can measure $X_c$ by bypassing the loader entirely and using a cached
dataset).
We find these bounds to be predictive of real performance.
For example, the $2\times$
speedups correlated with the $2\times$ mean data reduction observed using half
the scan groups, and speedups taper off as they approach the compute limit.

\subsection{System Setup}%
\label{sec:system_setup_full}
We run distributed experiments on a 16-node Ceph~\citep{weil2006ceph}
cluster connected with a Cisco Nexus
 3264-Q 64-port QSFP+ 40GbE switch.
Each node has a 16--core Intel E5--2698Bv3 Xeon 2GHz CPU,
64GiB RAM,
NVIDIA TitanX,
4TB 7200RPM Seagate ST4000NM0023 HDD,
and a Mellanox MCX314A-BCCT 40GbE NIC\@.
All nodes run Linux kernel 4.15 on Ubuntu 18.04, CUDA10, and the Luminous release
(v12.2.12) of Ceph.
We use six of the nodes as Ceph nodes; five nodes are dedicated as storage nodes in
the form of Object Storage Devices (OSDs), and one node is used as a Ceph
metadata server (MDS).
The remaining 10 nodes are used as machine learning workers for the training
process.
This means there is a 2:1 ratio between compute and storage nodes.
We use PyTorch~\citep{NIPS2019_9015} (v1.12) and NVIDIA Apex~\citep{apex}
(v0.1).
We use at 4 to 8 threads to prefetch data in the loader.
As very large datasets (e.g., Petabytes) cannot fit in RAM cache,
our experiments minimize the effects of caching with \texttt{DirectIO} and reduced cache sizes.
The same effect can be observed by simply duplicating any dataset multiple times.
For TensorFlow~\citep{tensorflow2015-whitepaper} experiments, we use a fork of
v2.5 with PyTorch 1.8 and CUDA 11, and we do not utilize \texttt{DirectIO}, since ImageNet
is sufficiently large relative to memory.
The data pipeline is still an ImageNet pipeline, though it uses recently
suggested $160\times160$ training-time crops~\cite{fixing_train_test}, sets parallelism
equal to the number of cores, does not use \texttt{fp16} training, and scales
the number of workers to 20.
For P100 runs, we attach a 150GB operating-system disk and a 150GB or 300 GB
data disk to a \texttt{n1} instance with 16 cores per GPU and less than 66 GiB of
RAM---a custom instance type is required for 32 core runs to prevent caching.

\textbf{Software.}
The experiments and plots in this paper were developed with a number of open source packages.
PyTorch~\citep{NIPS2019_9015},
DALI~\citep{DALI}, Tensorflow~\cite{tensorflow2015-whitepaper}, and Python3~\citep{Numpy}
were used throughout the experiments.
JSK~\citep{JSK} was used for encoding and provided insight into the JPEG EOI truncation
trick.
Various SciPy~\citep{2020SciPy-NMeth} libraries were used for both experiments and plotting,
including Numpy~\citep{Numpy,Numpy2}, Seaborn~\citep{Seaborn}, Matplotlib~\citep{Matplotlib}. 
\subsection{Dataset Details}%
\label{sec:dataset_details_full}%

Our evaluation uses the
ImageNet ILSVRC~\citep{imagenet_cvpr09,ILSVRC15},
CelebA-HQ~\citep{karras2017progressive},
HAM10000~\citep{tschandl2018ham10000},
and
Stanford Cars~\citep{KrauseStarkDengFei-Fei_3DRR2013}
datasets.
Below, we provide further details for each dataset.
\begin{itemize}[leftmargin=*,topsep=0pt,itemsep=-1ex,partopsep=1ex,parsep=2ex]
  \item \textbf{ImageNet:} 
  We use the provided training and validation set from the 1000-way image classification task.%
  \item \textbf{CelebAHQ:}
  CelebA-HQ is a high-resolution derivative of the CelebA
  dataset~\citep{liu2015faceattributes}, which consists of 30k images of celebrity faces with dimension $1024\times1024$.
  The dataset reconstruction is saved in JPEG form, which adds a 75\% compression
  factor by default.
  We use the annotations provided by CelebA to construct a binary classification
    task (``smiling'' vs.\ ``not smiling''), and  split the 30k dataset into 80\%/20\% train/test split.
\item \textbf{HAM10000:}
  We split the HAM10000 dataset randomly 80\%/20\% between train and test.
 This dataset consists of dermatoscopic images of skin lesions (7
    classes),
  and differs from the other datasets in that it is outside the scope of natural images.
\item \textbf{Stanford Cars:}
The Stanford Cars dataset is a fine-grained
classification dataset, since all images are cars, and there are 196 classes
  spread over 16k images (only about 80 images per class).
  As this is a difficult task, we additionally explore how grouping the
  labels into coarse-grained classes affects training.%
\end{itemize}

\textbf{Record and Image Quality Details.}
We provided the dataset size details for the encoded datasets in
the main text and provide further information below.
As the original (e.g., lossless) images are hard to find, we estimate the JPEG quality setting of the training set with ImageMagick using \texttt{identify -format '\%Q'}.
The JPEG quality setting determines the level of frequency quantization outlined
in the main text.
Intuitively, one would expect that higher quality JPEG images could allow more
aggressive PCR compression rates for a fixed resolution,
since each image has more redundant
information on average.
ImageNet and HAM10000 both have high quality images.
CelebAHQ has lower quality images, but they are downscaled
to $256\times256$ for training purposes, which increases the information density
in the image (e.g., blurry images can be made to appear less blurry by
downsampling), a fact exploited in prior work~\citep{yan2017customizing}.
It's worth noting that CelebAHQ is derived from CelebA, which is already noted
to be full of compression artifacts~\citep{karras2017progressive}, and thus
careful post-processing was needed for the creation of the $1024^2$ images.
Cars is neither high JPEG quality or large resolution.
Under-compressing images (perhaps at high
resolution) during the initial
JPEG compression may allow for a larger range of viable scan groups.

\textbf{Dataset Creation Times.}
We provide bandwidth-optimized record baselines in
Figure~\ref{fig:encode_times}, where we re-encode the
images using a statically-chosen level of compression.
By default, we use 4 worker threads per core, which totals 128 threads
processing the conversion process.
We use 50\%, 75\%, 90\% and 95\% JPEG quality as the static levels of
compression to reduce dataset size at a \textit{fixed} level of fidelity.
One caveat is that these quality settings may not necessarily exactly map to the
PCR scan groups used (e.g., in terms of metrics such as MSSIM); however,
these settings are within the typical range of JPEG quality used in practice.
We note that the conversion times do not vary significantly by quality; we
observe a maximum difference of less than 16\% between 50\% and 95\%
quality.

It is worth noting that re-encoding images compounds with the original JPEG
compression, so the re-encoded image quality may be lower than the
quality obtained if images were encoded in their original lossless form.
In fact,
we note that we observe \textit{larger} image sizes with additional
compression (e.g., CelebAHQ increases from 2.1GiB to 2.6GiB with $90\%$
compression), since the multiple rounds of compression (i.e., JPEG generation loss)
induce artifacts, which
are hard to compress.
Thus, static compression may, counterintuitively, decrease quality and increase
file size (and thus bandwidth).
This is in contrast to PCRs, which losslessly convert the JPEG images into a
progressive format, allowing \textit{dynamic} access to the level of
fidelity without the complications of generation loss.

Both the static compression method of dataset bandwidth reduction and the
PCR method can take considerable encoding time, since the encoding time scales
proportionally to the dataset size.
We observe that the PCR method is competitive ($1.13\times$ to $2.05\times$)
to that of any of the static compression levels in terms of total time.
When multiple static compression levels are utilized, the \textit{sum} of each
of their encoding costs is paid.
In contrast, PCRs avoid having to re-encode a dataset at multiple fidelity levels, and,
therefore, they can save both storage space and encoding time.
Although the exact conversion times are dependent on implementation, hardware,
and the dataset,
they can be in the range of one hour of
compute time per $100$ GB\@.

\begin{figure*}[h]
  \centering
  \begin{subfigure}[t]{0.33\textwidth}
    \includegraphics[width=.99\linewidth]{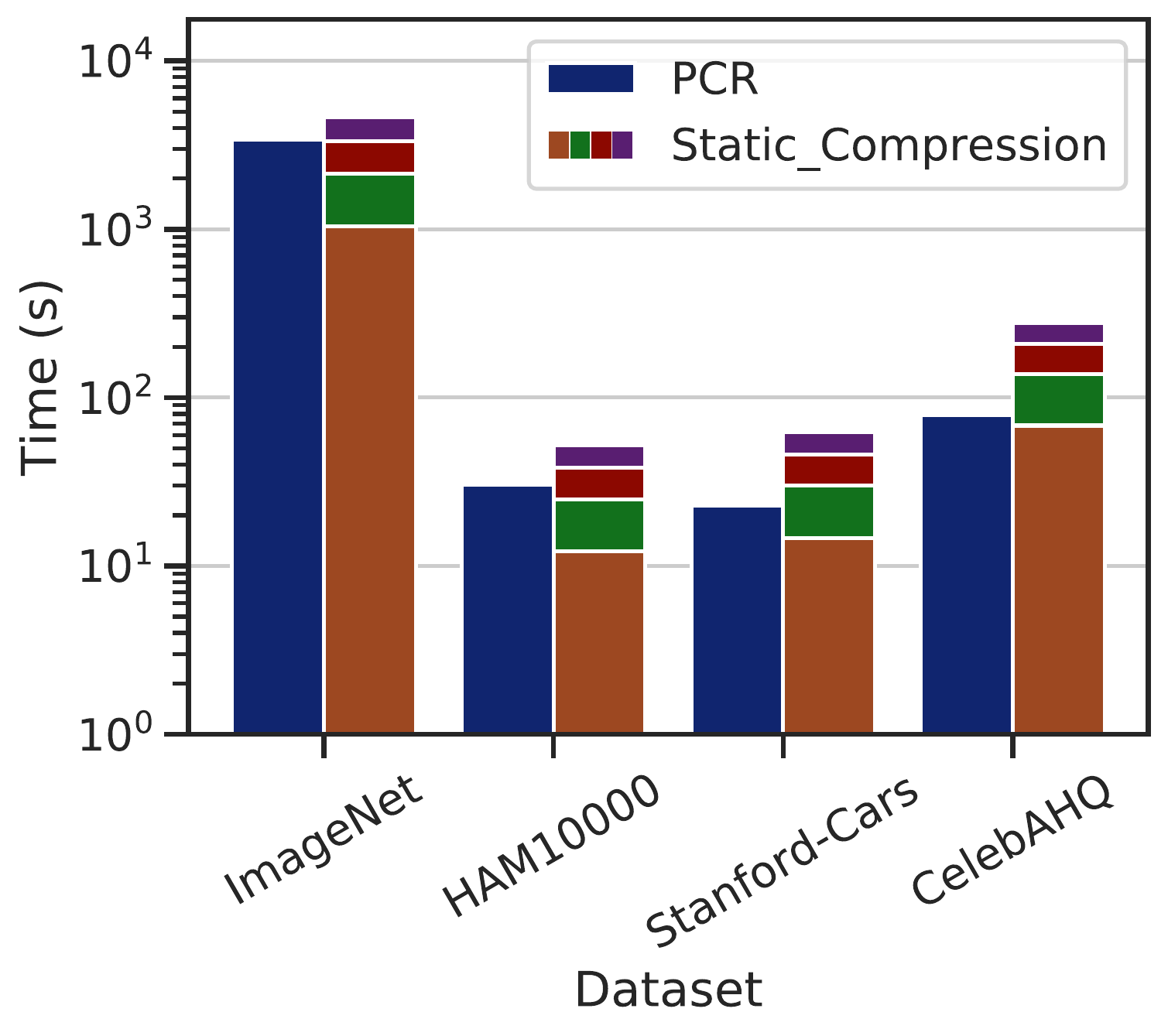}
    \caption{JPEG Conversion Times}
  \end{subfigure}%
  \begin{subfigure}[t]{0.33\textwidth}
    \includegraphics[width=.99\linewidth]{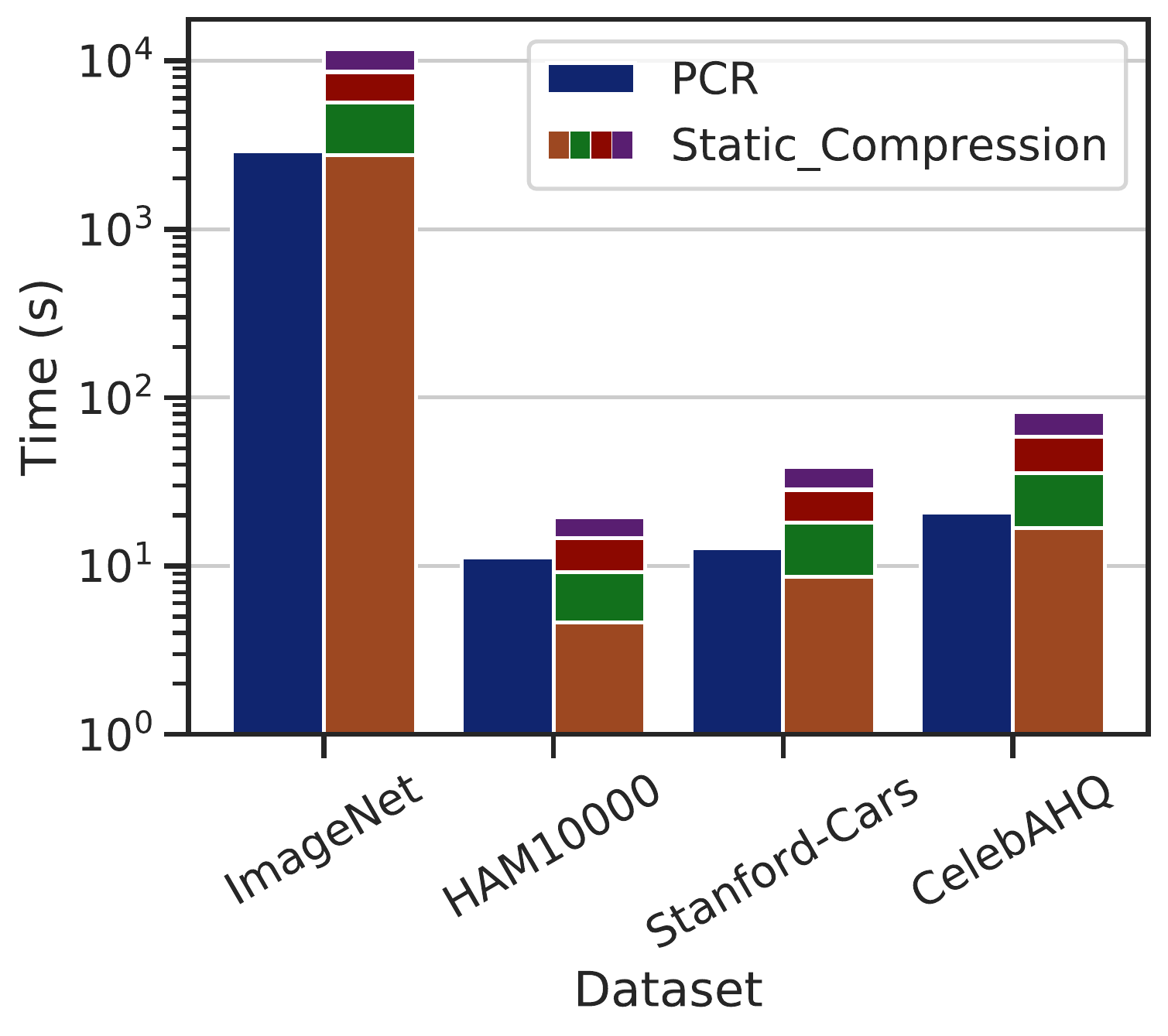}
    \caption{Record Creation Times}
  \end{subfigure}%
  \begin{subfigure}[t]{0.33\textwidth}
    \includegraphics[width=.99\linewidth]{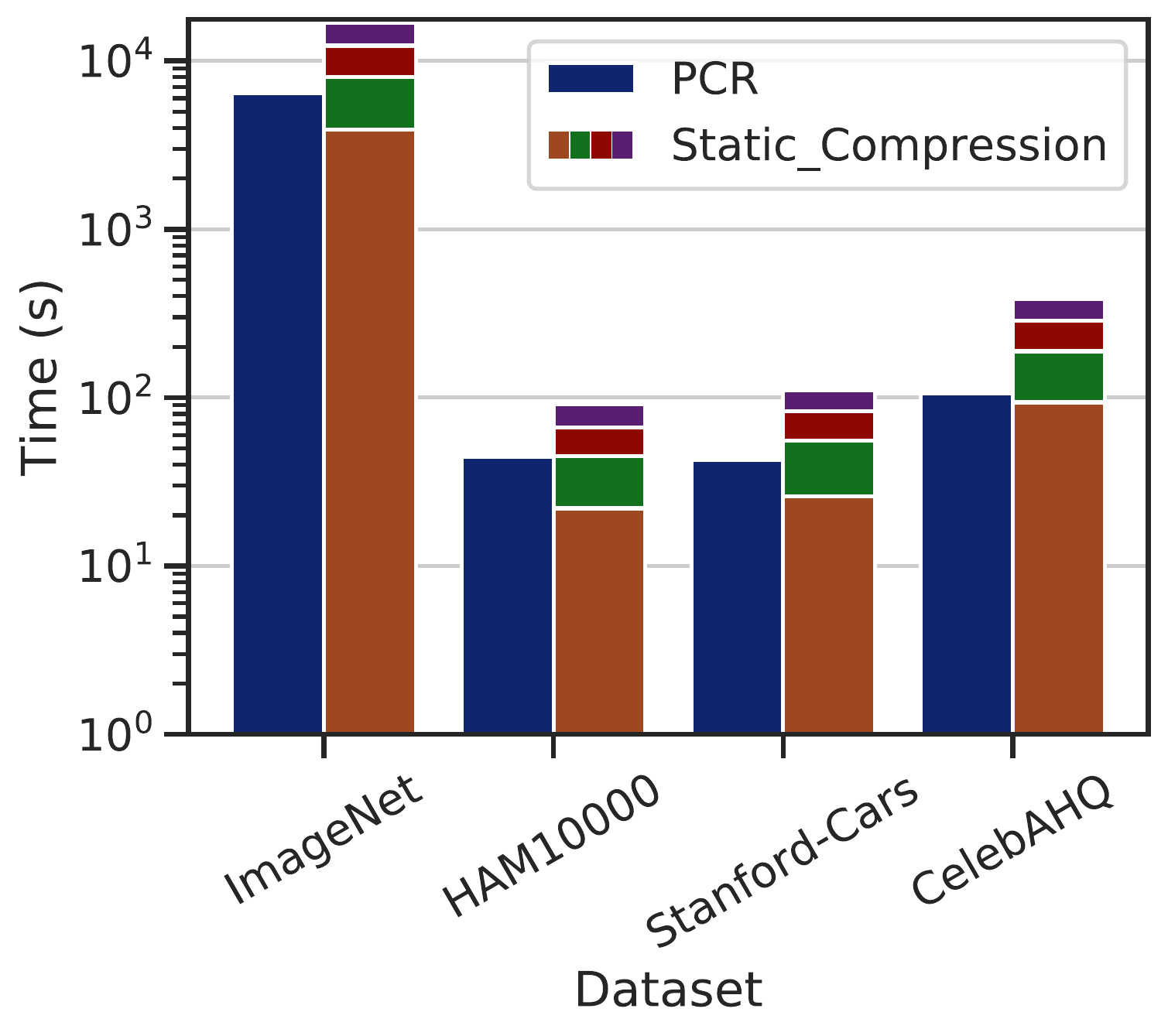}
    \caption{Total Times}
  \end{subfigure}%
  \caption{%
    Encoding times for baseline JPEG re-encoding and the PCR approach.
    The static encodings are 50\%, 75\%, 90\%, and 95\%, and they are stacked in
    that order.
    While the cost to encode PCRs is slightly larger than the cost to encode a
    single baseline record, it is significantly lower than the total cost of
    encoding the dataset at multiple quality levels.
    \textit{Total Time} is the combination of \textit{JPEG Conversion Time} and
    the \textit{Record Creation Time}.
    \textit{JPEG Conversion Time} is the amount of time required to convert the
    JPEG to progressive form or re-encode it to a lower quality JPEG\@.
    \textit{Record Creation Time} is the amount of time required to write the
    images to record format.
  }%
  \label{fig:encode_times}%
\end{figure*}

\textbf{Example Application of PCRs.}
Using the Progressive GAN
repository\footnote{\url{https://github.com/tkarras/progressive_growing_of_gans}}, 4 worker threads, and default settings,
we are able to convert the entire CelebAHQ dataset into TFRecord form in 109 minutes;
this process
generates a total of 118 GiB.
The process is run on a Intel i7-6700K, 16 GiB memory, and a Micron 2TB SSD
(1100 MTFDDAK2T0TBN).
There are 9 records generated in total, consuming 2.9 MiB, 7.4 MiB, 24.7 MiB,
93.8 MiB, 370.3 MiB,
1.5 GiB, 5.9 GiB, 23.6 GiB, and 94.4 GiB, respectively.
The reason for the large space amplification is two-fold: compression
is not used to store the images, and each record corresponds to a different power of two
resolution.
Meanwhile, PCRs take less than 6 minutes to make the conversion, and they only
produce 2.6 GiB.
Using $100\%$ quality JPEG compression resulted in 123 minutes of processing time
and a $4\times$ space amplification.
In this case, the 9 records are of size: 14 MiB, 14.7 MiB, 20 MiB, 45.6 MiB,
132 MiB, 422.8 MiB, 1.4 GiB, 4.4 GiB, and 7.0 GiB\@.
75\% compression took 117 minutes and $1.5\times$ space amplification, and the
process created records of sizes 11.2 MiB, 11.6 MiB,
13.7 MiB, 21.0 MiB, 41.8 MiB, 103.8 MiB, 305.6 MiB, 959.9 MiB, and 2.7 GiB (3.9 GiB total).
Thus, applications resorting to static compression and encoding schemes may
create a $1.5\times$ space amplification in a good case
or $40\times$ space
amplification in a worst case.
PCRs minimize space amplification as they only need one copy of the dataset for
various task requirements.

\textbf{MSSIMs.}
Figure~\ref{fig:PCR_Conversion_MSSIM}
shows the scan MSSIM results, respectively, for all datasets.
MSSIM decreases as the scan gets lower, as expected.

\begin{figure*}
  \centering
  \begin{subfigure}[t]{0.25\textwidth}
    \includegraphics[width=.99\linewidth]{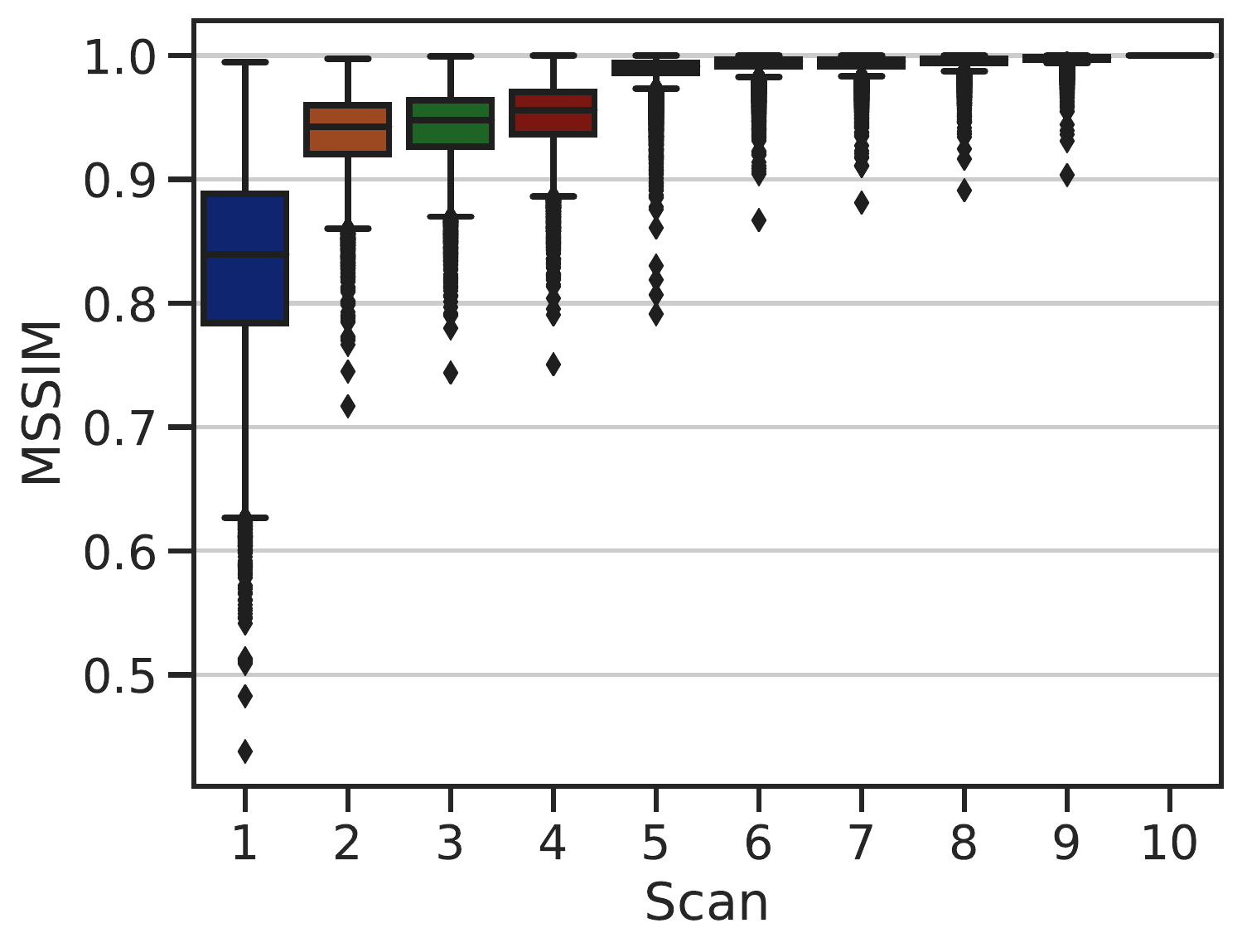}
    \caption{ImageNet}
  \end{subfigure}%
  \begin{subfigure}[t]{0.25\textwidth}
    \includegraphics[width=1.012\linewidth]{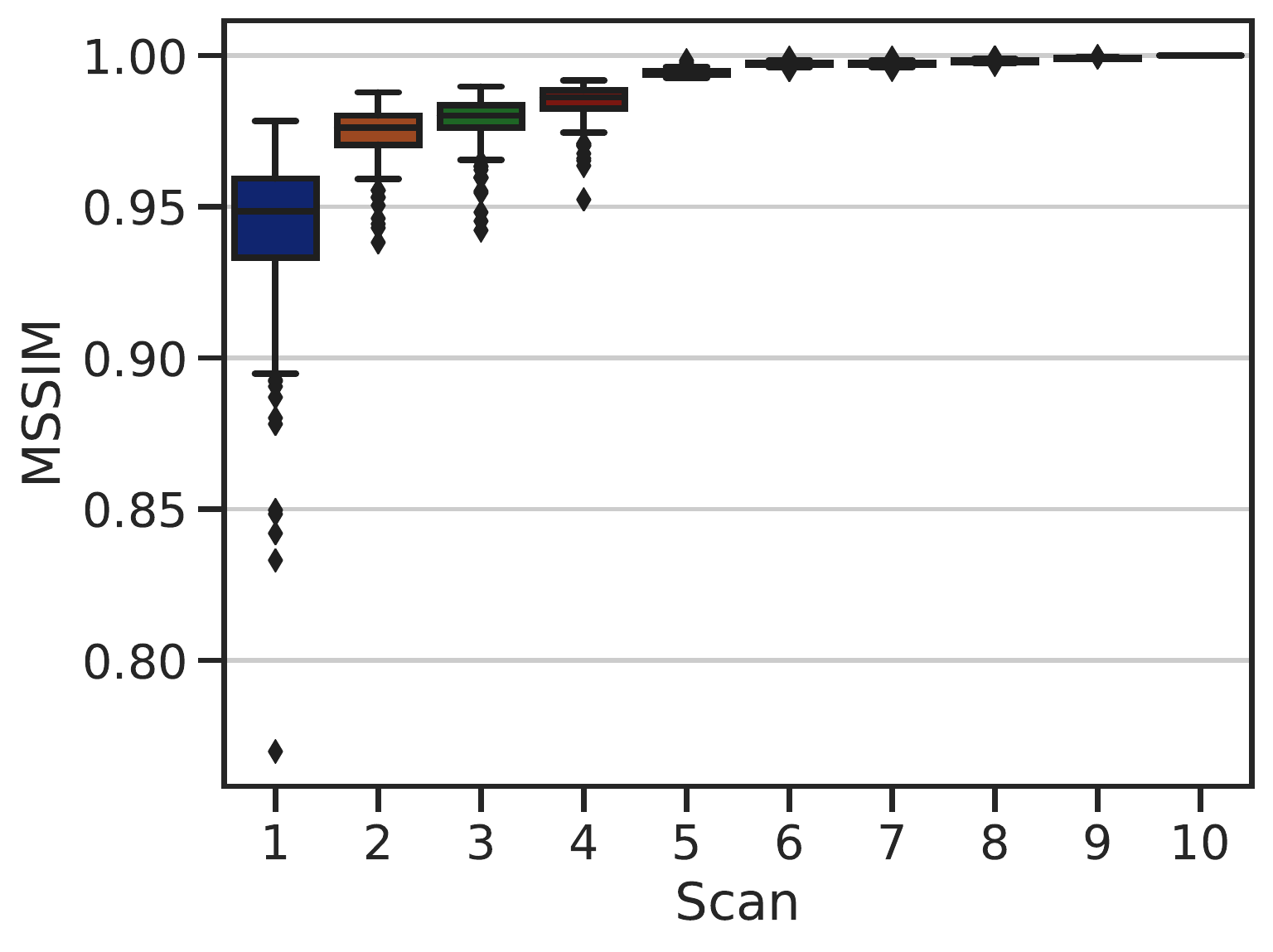}
    \caption{HAM10000}
  \end{subfigure}%
  \begin{subfigure}[t]{0.25\textwidth}
    \includegraphics[width=1.012\linewidth]{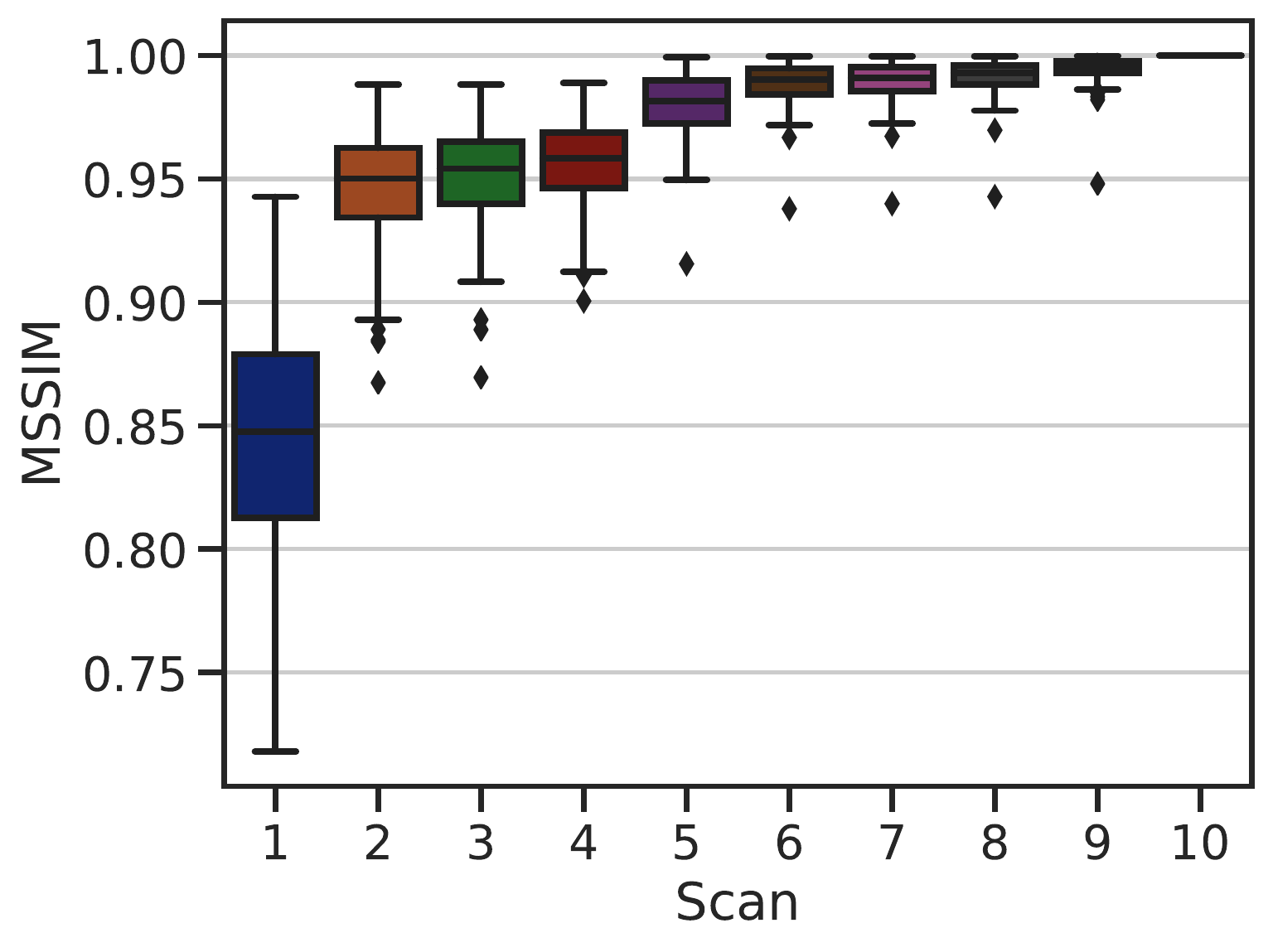}
    \caption{Stanford Cars}
  \end{subfigure}%
  \begin{subfigure}[t]{0.25\textwidth}
    \includegraphics[width=1.013\linewidth]{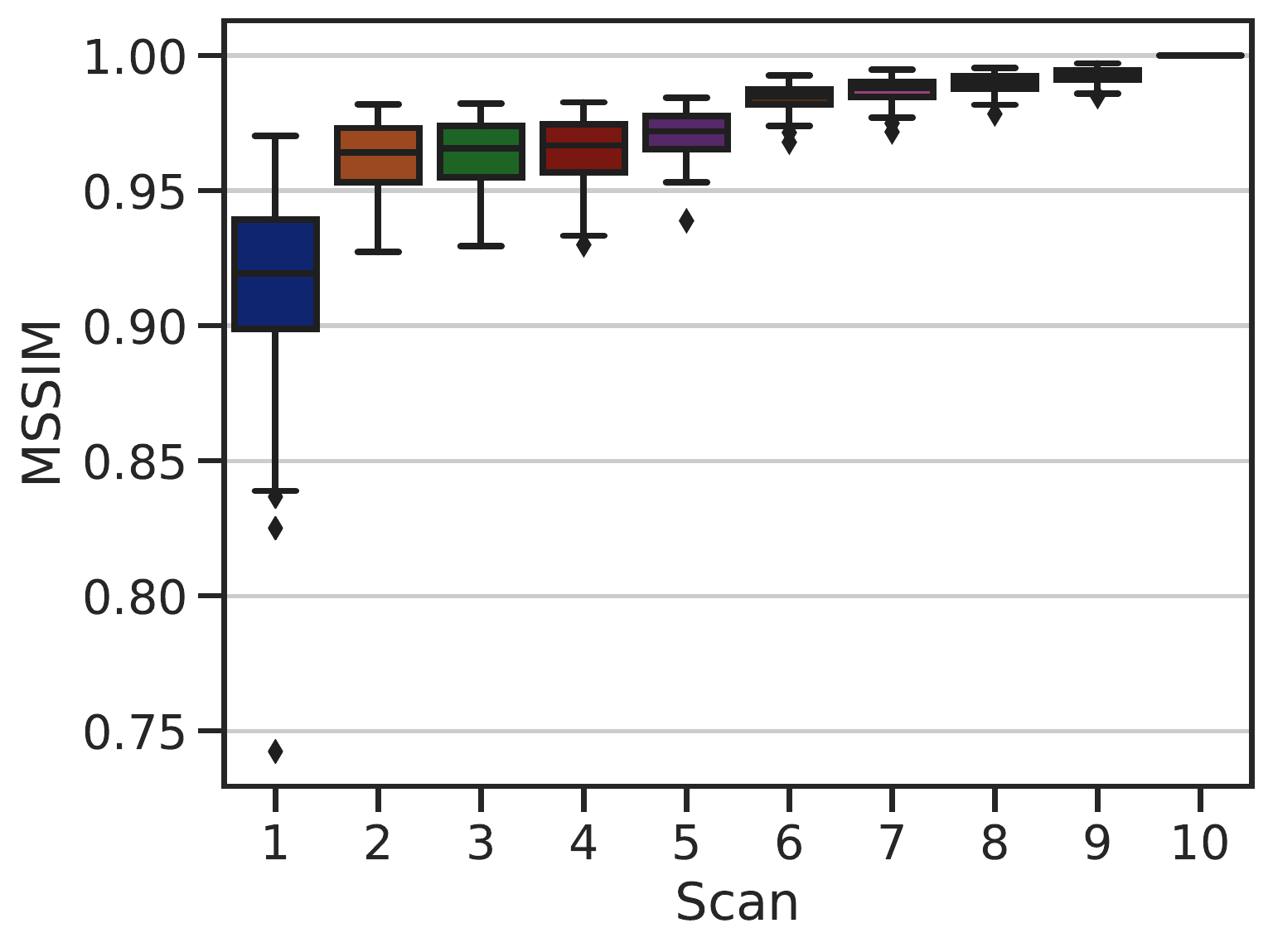}
    \caption{CelebAHQ}
  \end{subfigure}%
  \caption{%
    The reconstruction quality (measured with MSSIM) of using various amounts of
    scans.
    Perfect reconstruction has an MSSIM of $1$.
    Higher scans have diminishing fidelity returns.
    Interquartile ranges are shown.
  }%
  \label{fig:PCR_Conversion_MSSIM}%
  \vspace{-10pt}
\end{figure*}

\subsection{Experiment Setup}%
\label{sec:experiment_setup}
Below we describe details of how the experiments were run, such as hardware
characteristics and software configurations.

\textbf{Benchmark Cluster Speeds.}
As noted in the main text, we utilize a NVIDIA TitanX Graphics Processing Unit
(GPU) on each node for the
model training.
This GPU allows us to train
ResNet-18 at
450 images per second and ShuffleNetv2
at 750 images per second.
With a cached, decoded dataset of $224\times224$ resolution images,
we achieve a cluster-wide
4200 images per second for ResNet-18
and
7000 images per second for ShuffleNetv2.
ImageNet images are around 110kB on average;
with 10 GPUs, the cluster can consume $450\ \si{megabytes/s}$ (ResNet-18) and
$750\
\si{megabytes/s}$ (ShuffleNetv2) of storage system bandwidth.

\textbf{Reader Microbenchmarks.}
To highlight that PCRs can be implemented efficiently, we demonstrate how fast
PCR-encoded images can be read without any further data pipeline processing.
To do so, we instantiate a PCR loader for all scans for progressively encoded
PCR records as well as the single scan of baseline records.
We show the resulting images per second in Figure~\ref{fig:loader_scans_vs_throughput} for
the CelebAHQ dataset.
It's worth noting that regardless of encoding (i.e., progressive JPEG or
baseline JPEG) or scan number, all reads saturate the drive with 8 threads.
The test is run on a Intel i7-6700K, 16 GiB memory, and a Micron 2TB SSD (1100 MTFDDAK2T0TBN).
We use 1000 minibatches per scan to get an accurate measurement and records have 1024 images.
They do this while utilizing less than 1 core of CPU time (roughly
$75\%$ of single thread's system time), since
most of the time is waiting for IO requests to complete.
No reader implementation can go faster than these read rates as the IO device is already a
bottleneck.

We note that, as predicted by Theorem~\ref{thm1}, the speedup relative to the
baseline is simply the ratio of the mean data sizes for the two scan groups,
which can be found in the main text.
As the number of scans is increased, the number of
bytes read per image is increased, and thus the throughput in images per second
is correspondingly decreased.
Baseline JPEG performs within 4\% of scan 10 due to baseline images being within
5\% of the size of progressive images in practice.
Progressively encoded PCRs, baseline encoded PCRs, and TFRecords are all about
90MB\@.
Similarly, all dataset directories are 2.1GB with 24 records.
Thus, as long as bandwidth is fixed (which effectively depends on caching, the parallelism of data reads, and the
underlying storage system), PCRs increase throughput until the workload becomes
compute bound.

\begin{figure*}
  \centering
  \begin{subfigure}[t]{0.33\textwidth}
    \leftRightCrop{0.0}{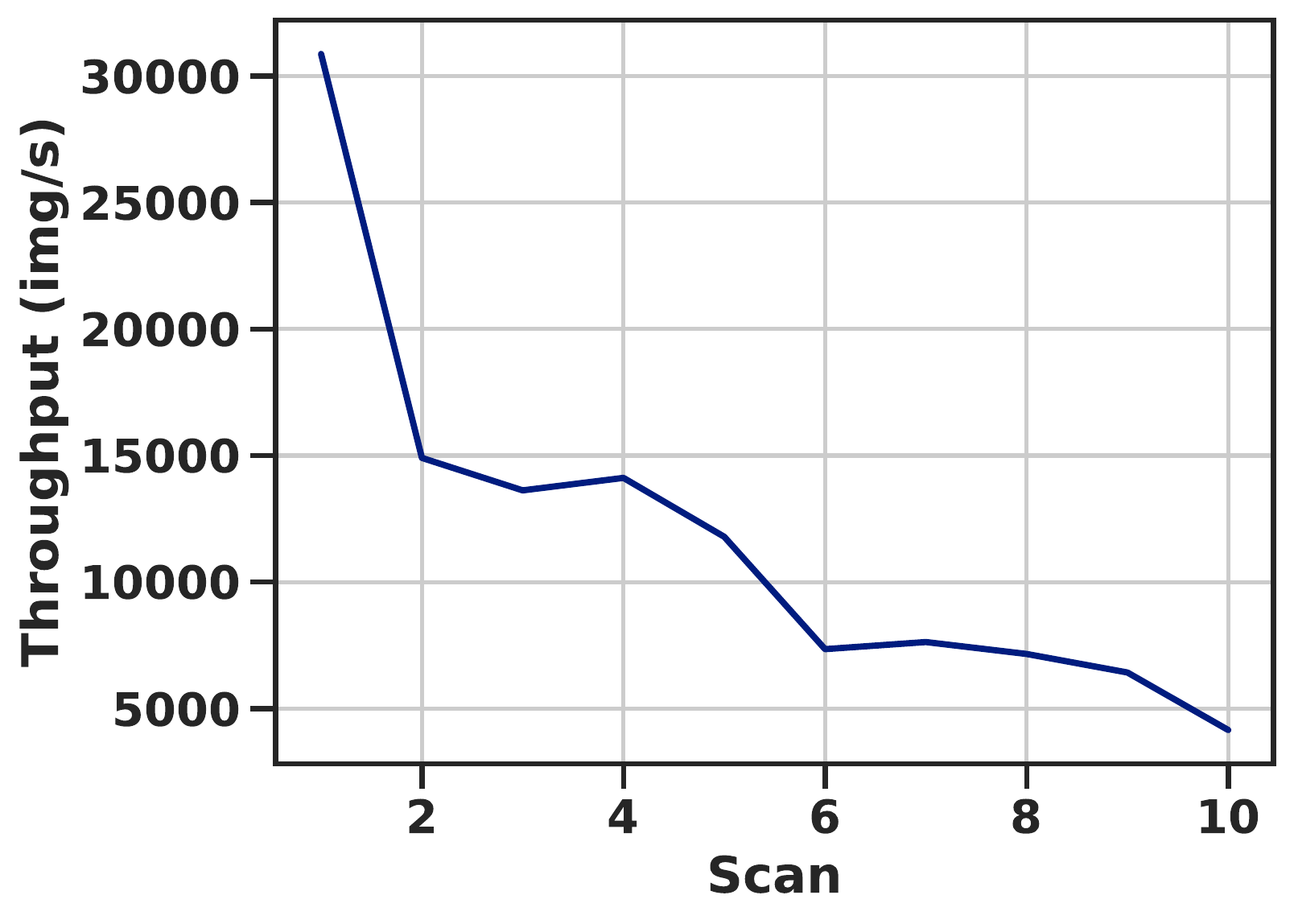}{1.0}{0.0}
    \caption{Mean Throughput}
  \end{subfigure}%
  \begin{subfigure}[t]{0.33\textwidth}
    \leftRightCrop{0.0}{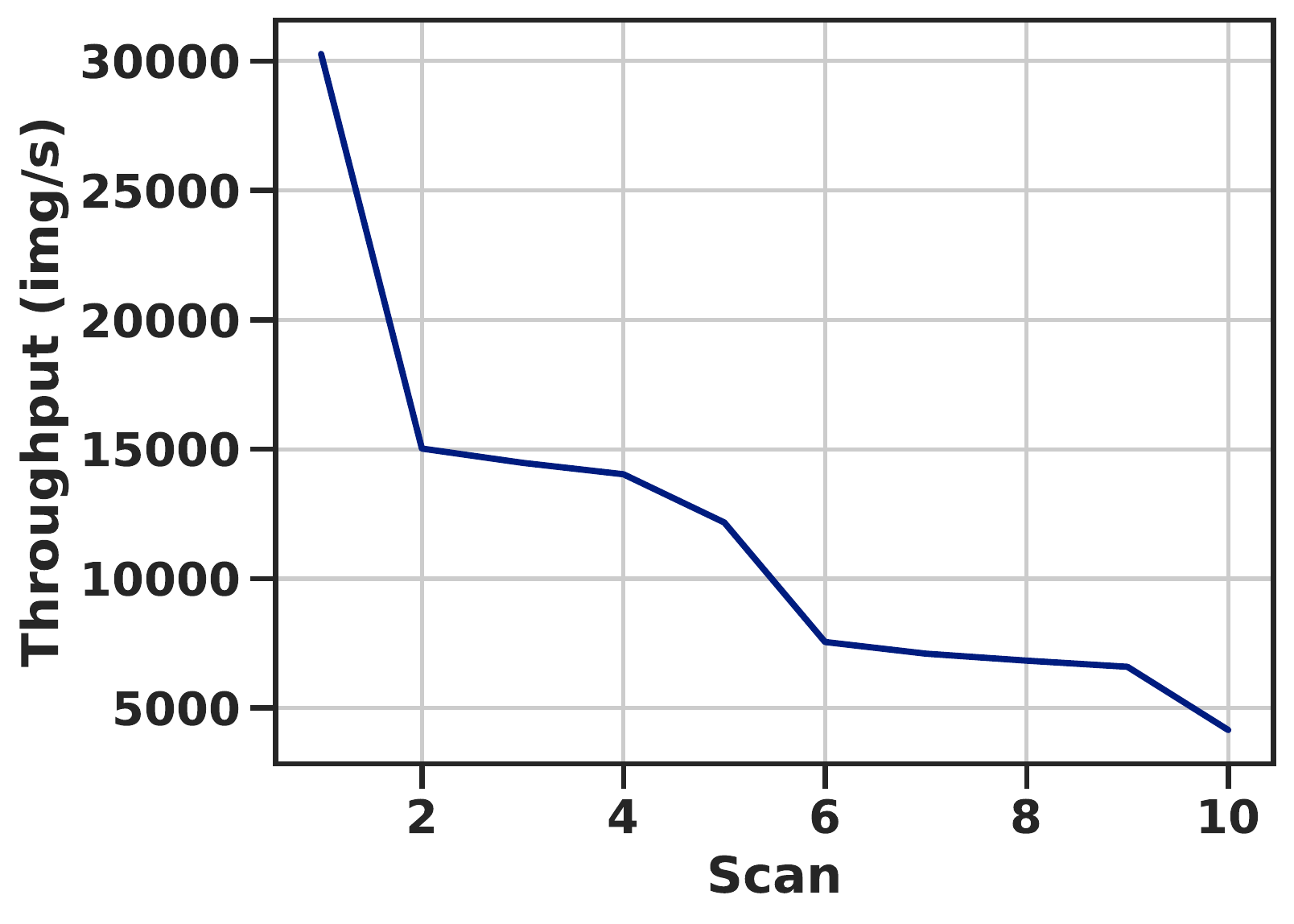}{1.0}{0.0}
    \caption{Predicted Throughput}
  \end{subfigure}%
  \begin{subfigure}[t]{0.33\textwidth}
    \leftRightCrop{0.0}{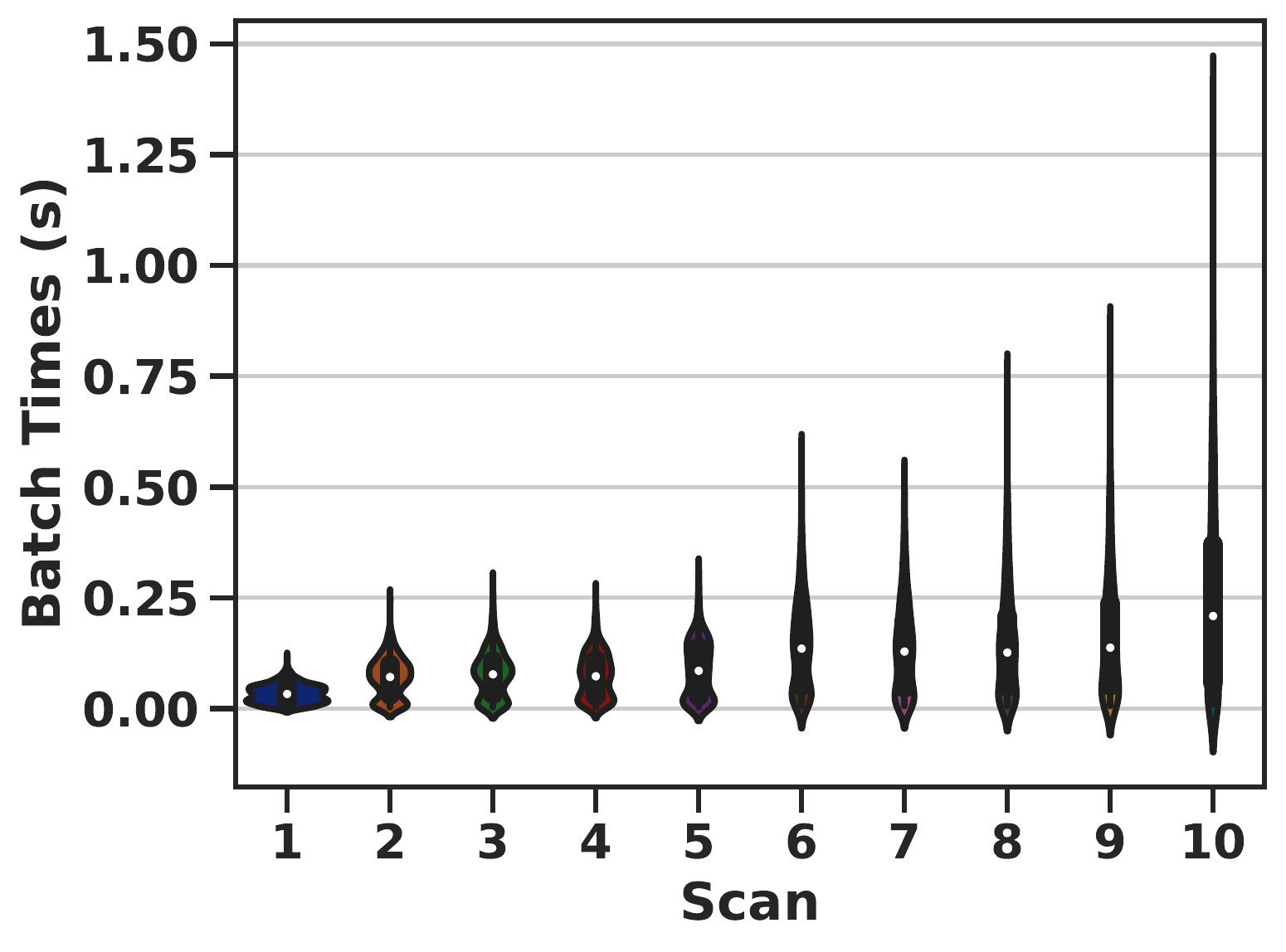}{1.0}{0.0}
    \caption{Batch Times}
  \end{subfigure}%
  \caption{%
    \textbf{Left:}
    PCR Loader with 8 threads reading CelebAHQ images from a 400MB/s SSD\@.
    Bandwidth utilization saturates the drive for all scans.
    Baseline encoded JPEG images are read at 4290 images/sec, which is within 4\% of the scan
    10 rate of 4150 images/sec.
    We obtain similar results for other datasets and hardware.
    The main factor in system performance is bandwidth utilization prior to
    decoding, thus reading less data increases throughput in images per second.
    \textbf{Middle:}
    The predicted throughput using the ratios of mean scan sizes to extrapolate
    from scan 10 rates.
    Predictions closely match empirical throughput.
    \textbf{Right:}
    The empirical corresponding batch times for each scan.
    Higher scans cause latency spikes for batches since the drive is saturated
    and batch requests must wait for other requests to finish.
    Latency spikes lead to lower aggregate throughput.
  }%
  \label{fig:loader_scans_vs_throughput}%
\end{figure*}

\begin{figure}
  \begin{center}
  \leftRightCrop{0.0}{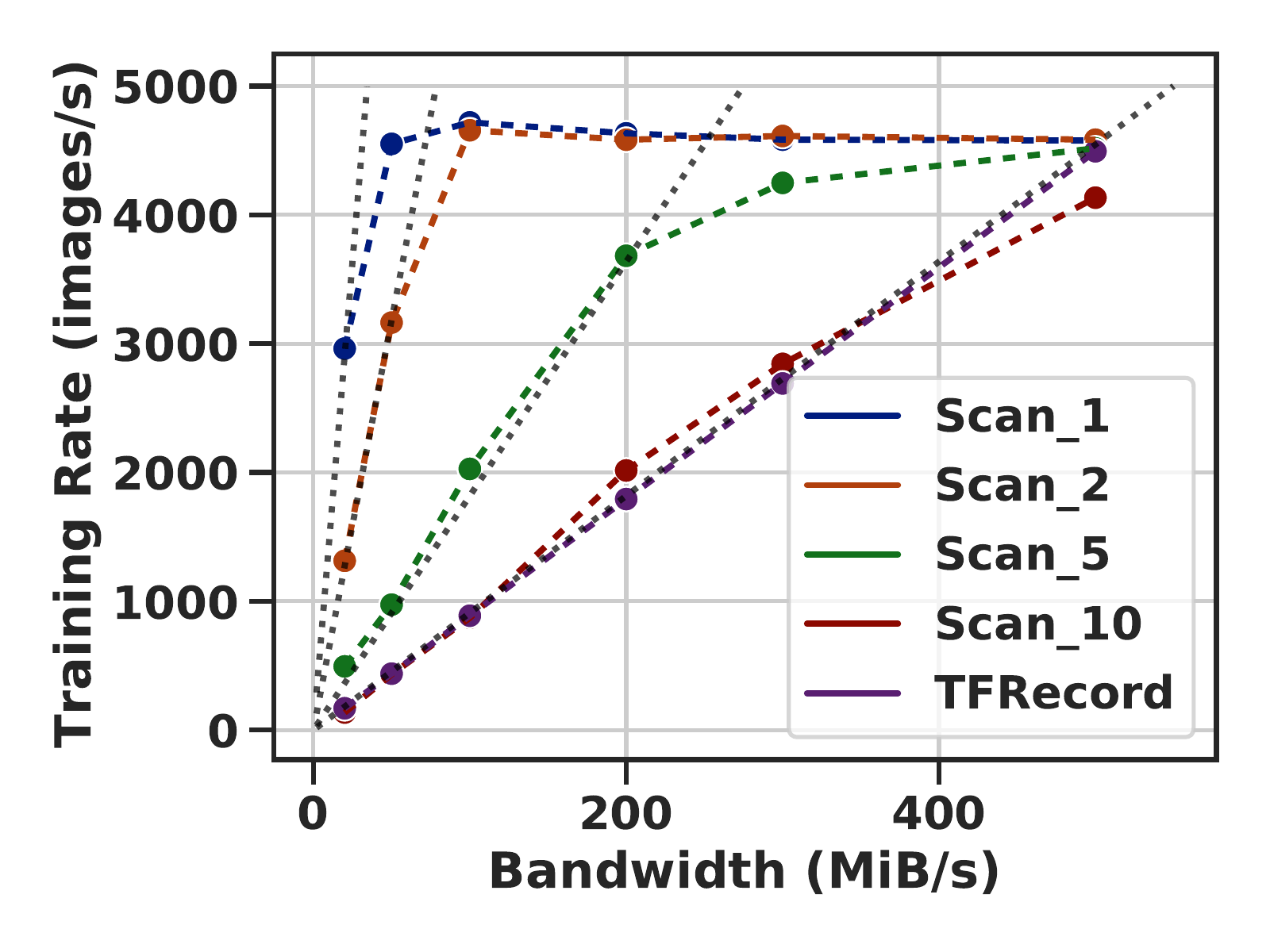}{1.0}{0.0}
  \end{center}
  \caption{%
      The training rate of a 10-node \texttt{TitanX} GPU cluster with a
      ShuffleNet workload using
      PCRs (the scans) and TFRecord.
      The throughput of the training process is dominated by bandwidth until
      the compute limit is reached, which is equivalent to in-memory processing.
      PCRs at scan 10 are approximately the same size as TFRecord,
      and thus have similar performance.
      The training rates predicted from bytes/image calculations are shown, and
      they are a close fit until the compute limit is hit.
    }%
  \vspace{-1pt}%
  \label{fig:predictedThroughput_shufflenet}%
\end{figure}

\subsection{ImageNet ShuffleNet Rates using \texttt{tf.data}}
For completeness, we pair the rates measures using the \texttt{tf.data} loader
on ResNet-18 with those of
Shufflenetv2.
We use the same methodology as in the ResNet-18 Figure, where we
benchmark for 7 minutes to obtain an estimate of the training rate.
Results are shown in
Figure~\ref{fig:predictedThroughput_shufflenet}.
Our experiment setup differs slightly from the others in that we use \texttt{tf.data}
rather than DALI for the data pipeline, and we also do not utilize \texttt{FP16}
training.
The trends are close to the bandwidth bounds predicted by the dataset size.

\subsection{Tuning PCRs: The Scan Group Parameter}%
\label{sec:tuning}%

\subsubsection{Dynamic Tuning}
While PCRs are stable as hyperparameters, it is also possible to tune the
  hyperparameter at runtime.
In particular, we find that a consistently effective method involves measuring the
gradient of the loss with respect to each scan group and comparing that to the
gradient of the loss with respect to the true data.
We choose to use the cosine distance between these vectors
To allow the model to warm up (and get accurate measurements for the scan
groups), training starts at scan 10 with an initial tuning at epoch 5.
The gradient similarity is set to be at least 90\% to accept the scan group.
The result on training rates can be seen in
  Figure~\ref{fig:dynamic_celeba_rates}.
We have used this technique to autotune ImageNet and the other datasets while
retaining accuracy, though ImageNet sees the largest benefit due to its long
runtime.

\begin{figure*}
  \centering
  \leftRightCrop{0.0}{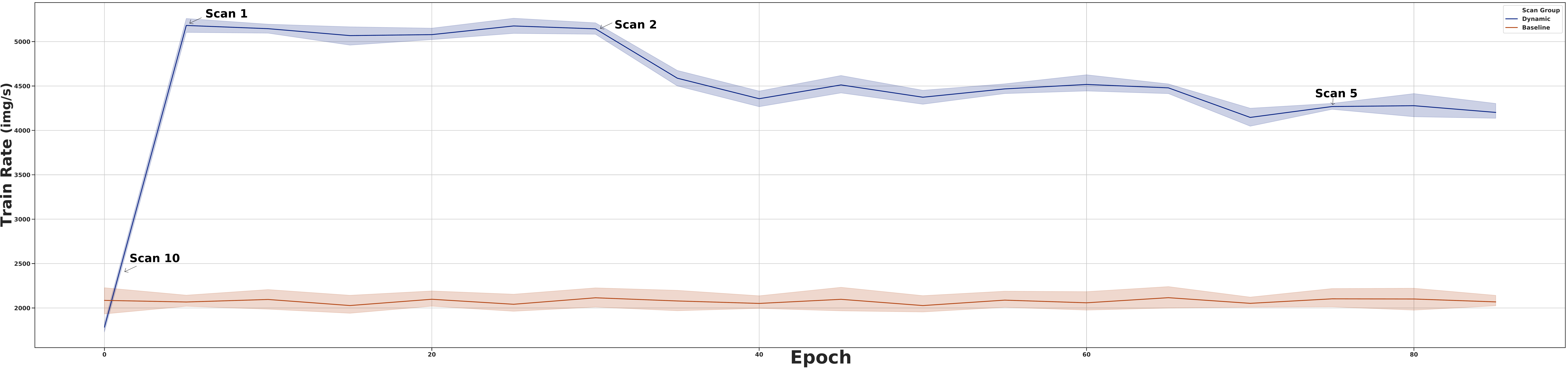}{1.0}{0.0}
  \caption{%
    Epochs vs.\ training rate (speed in images/second) of a CelebA ShuffleNet run using dynamic tuning
    with cosine distance and the static, baseline training.
    Dynamic training starts at a high scan to initialize training, after which
    scan 1, 2, 5 are used, respectively.
    Static rate remains at scan 10, and thus has a slow rate throughout.
  }%
  \label{fig:dynamic_celeba_rates}%
\end{figure*}

\subsection{Complete Experimental Plots}%
\label{sec:additional_experimental_plots}
Below, we provide additional experiment plots that
were omitted in the main text.
Figure~\ref{fig:scan_performance_resnet18_orca_acc_time} and
Figure~\ref{fig:scan_performance_shufflenet_orca_acc_time} give the accuracy
over time plots for all datasets.
Figure~\ref{fig:scan_performance_resnet18_orca_loss_time}
and
Figure~\ref{fig:scan_performance_shufflenet_orca_loss_time}
contain the loss over time for the ResNet-18 and ShuffleNetv2 experiments shown
in the main text.
It is worth noting that Top-5 accuracies mirror the Top-1 accuracies trends for ImageNet and Cars.

To measure the effect of compression without accounting for time, we show
accuracy vs.\ epoch plots in
Figure~\ref{fig:scan_performance_resnet18_orca_epoch}
and
Figure~\ref{fig:scan_performance_shufflenet_orca_epoch}.
While compression can itself be viewed as a data augmentation (e.g., removing high
frequency features that can possibly cause overfitting), we notice that it does
not usually improve accuracy.
Rather, most of the gains in time-to-accuracy are from faster image rates.

\begin{figure*}
  \centering
  \begin{subfigure}[t]{0.25\textwidth}
  \includegraphics[width=.98\linewidth]{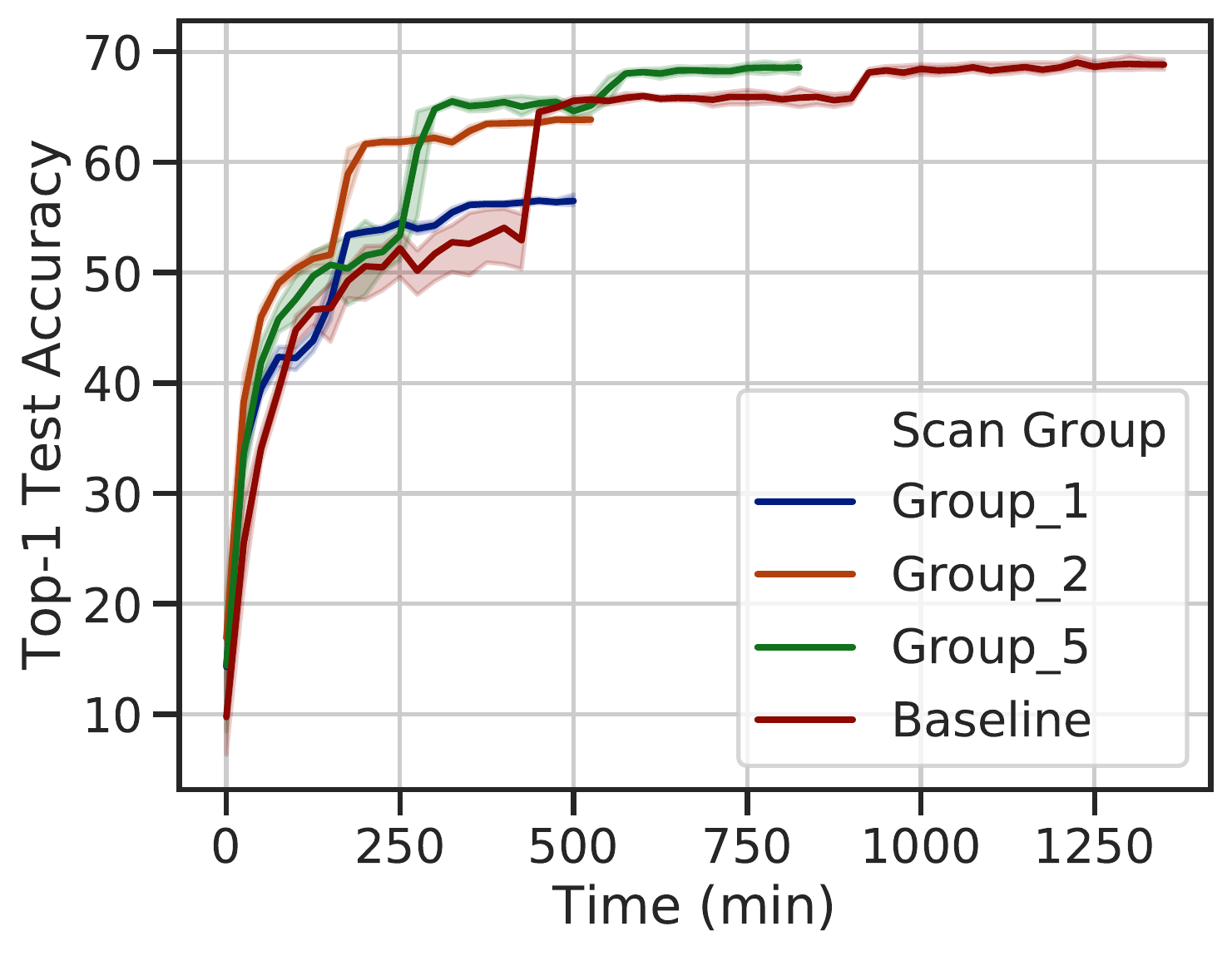}
    \caption{ImageNet}%
  \end{subfigure}%
  \begin{subfigure}[t]{0.25\textwidth}
  \includegraphics[width=.99\linewidth]{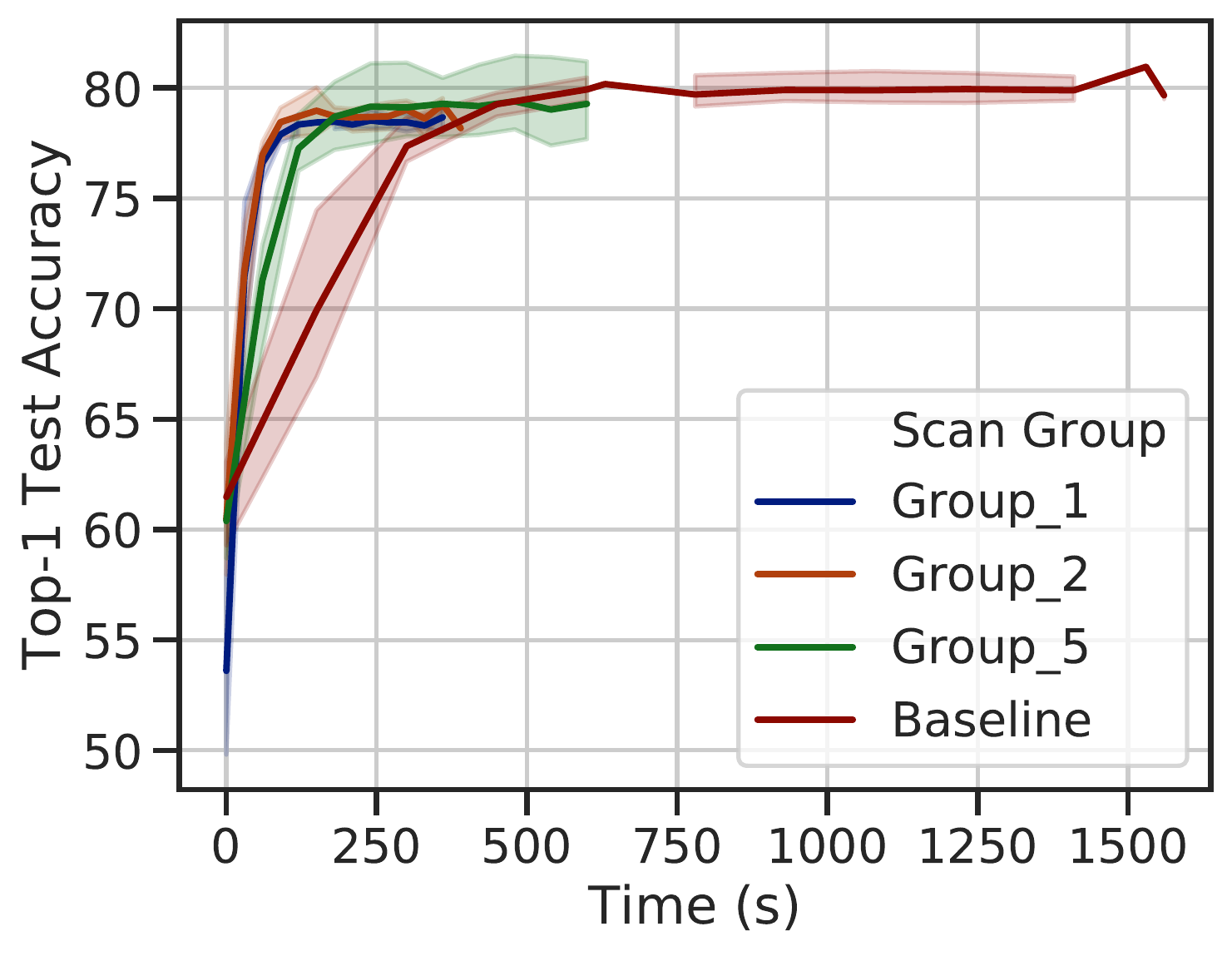}
    \caption{HAM10000}
  \end{subfigure}%
  \begin{subfigure}[t]{0.25\textwidth}
  \includegraphics[width=.99\linewidth]{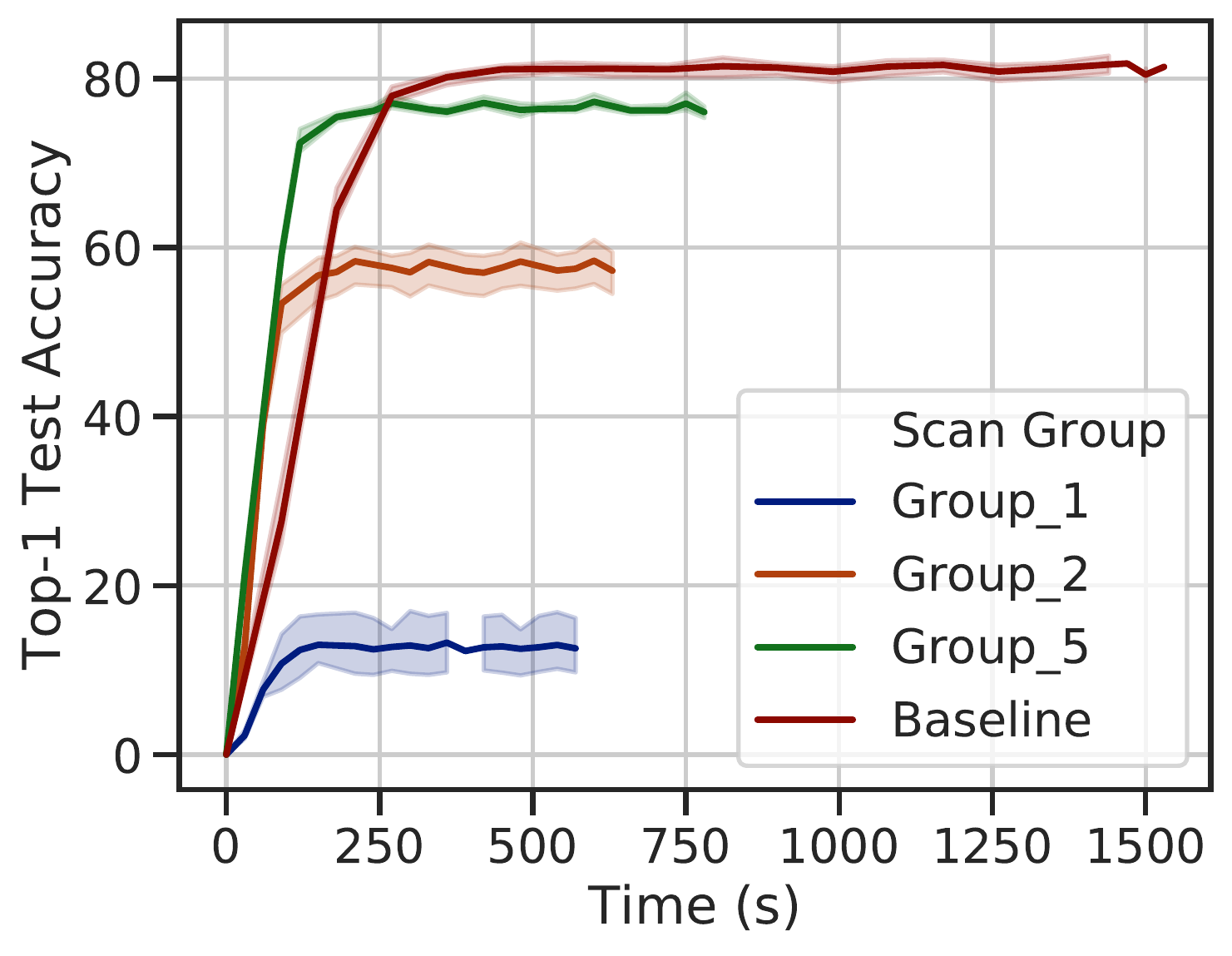}
    \caption{Stanford Cars}
  \end{subfigure}%
  \begin{subfigure}[t]{0.25\textwidth}
  \includegraphics[width=.99\linewidth]{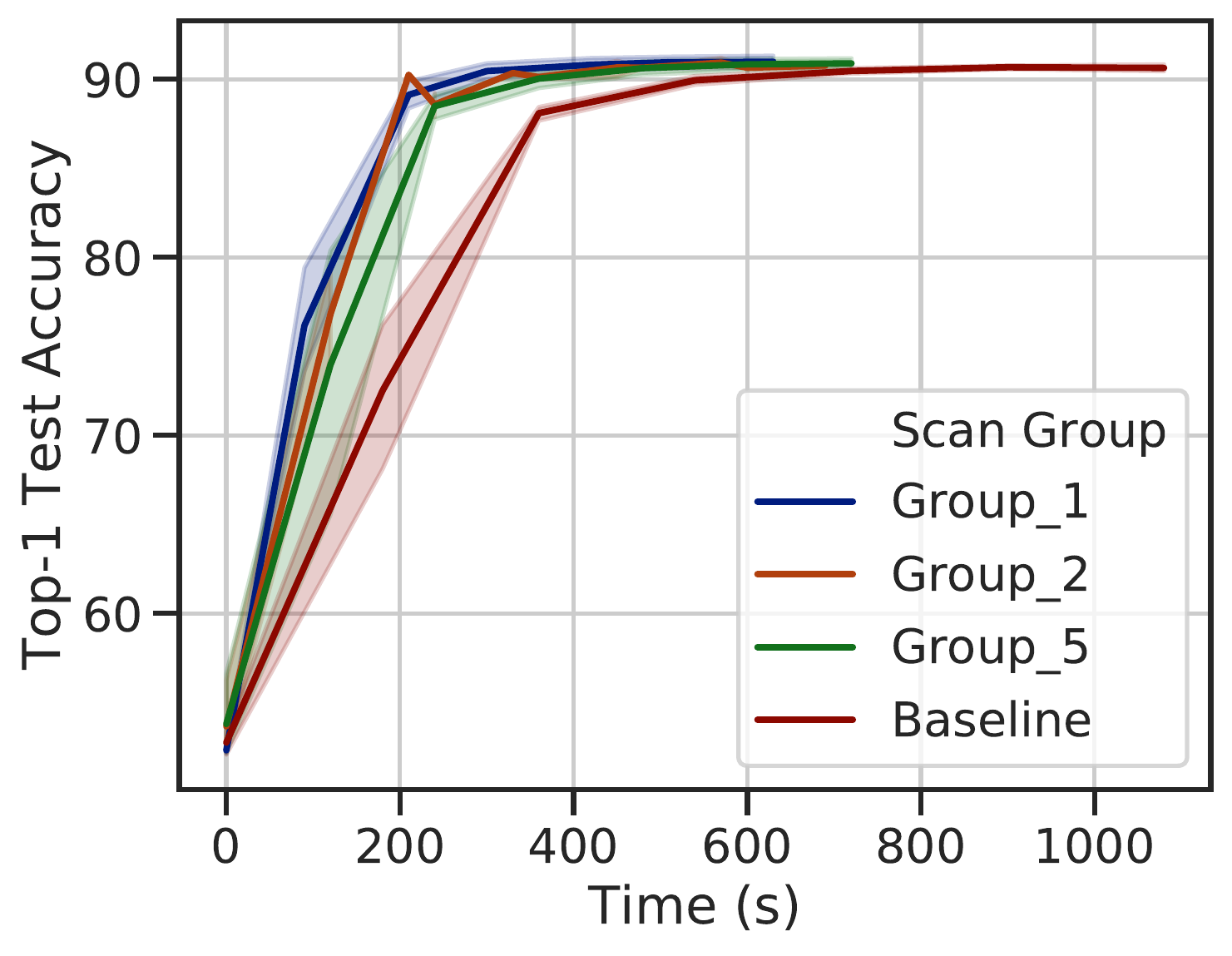}
    \caption{CelebAHQ}
  \end{subfigure}%
  \caption{%
    Top-1 test performance with ResNet18.
    Lower scan groups speed up training by reducing bandwidth.
    Time is the x-axis (seconds) and is relative to first epoch.
    95\% confidence intervals are shown.
    Higher scan groups are less compressed.
  }%
  \label{fig:scan_performance_resnet18_orca_acc_time}%
  \vspace{-10pt}
\end{figure*}

\begin{figure*}
  \centering
  \begin{subfigure}[t]{0.25\textwidth}
  \includegraphics[width=.97\linewidth]{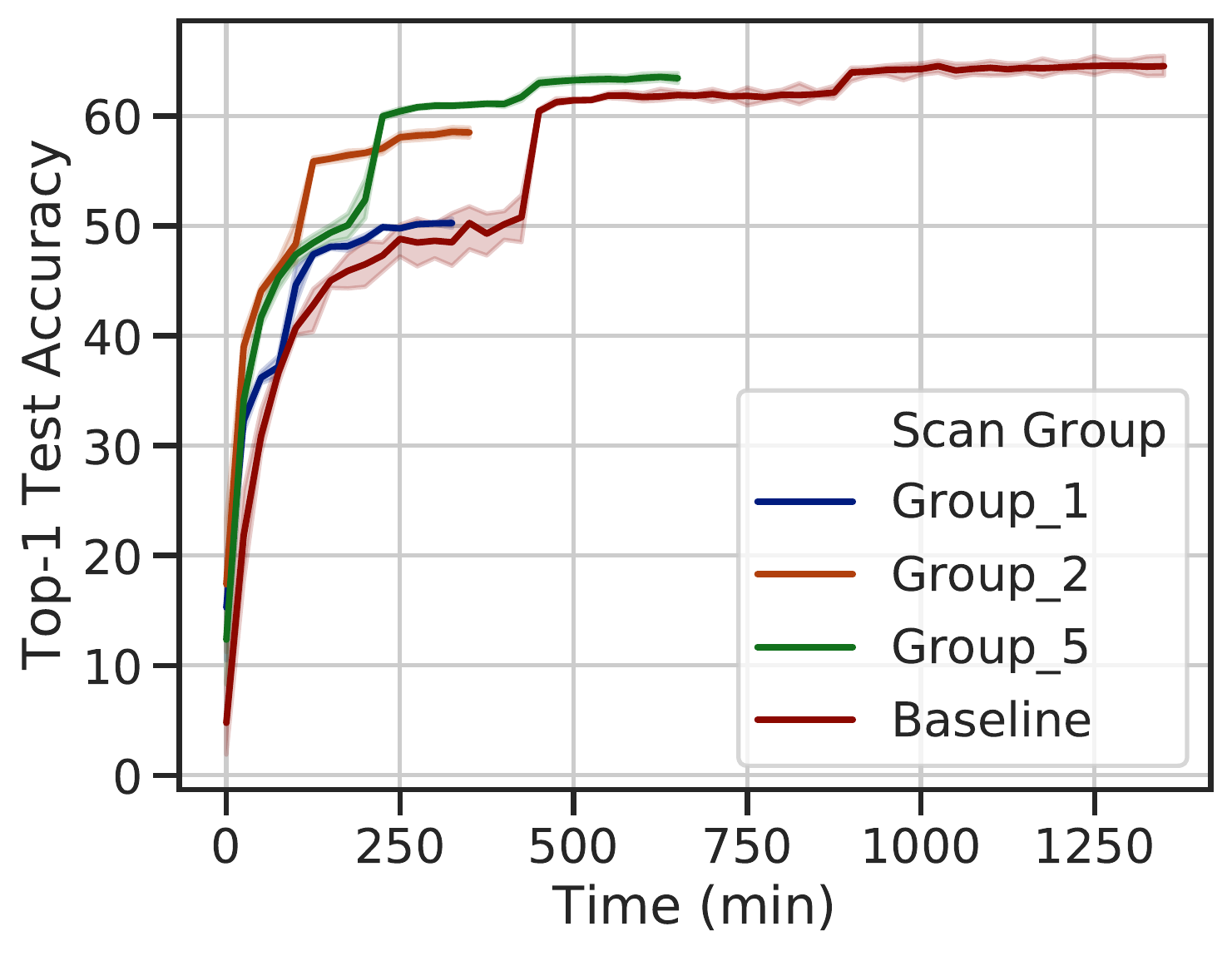}
    \caption{ImageNet}%
  \end{subfigure}%
  \begin{subfigure}[t]{0.25\textwidth}
  \includegraphics[width=1.02\linewidth]{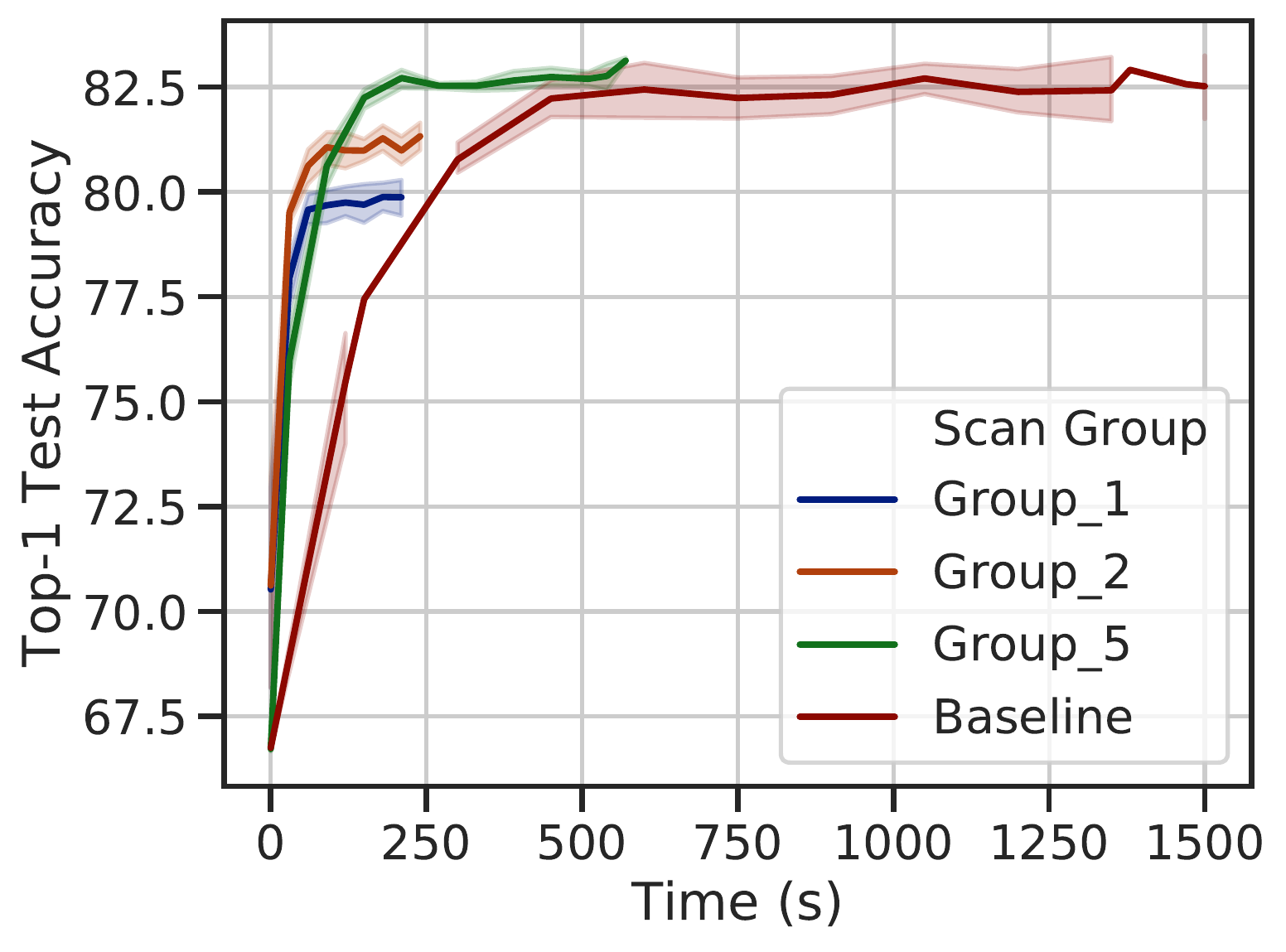}
    \caption{HAM10000}
  \end{subfigure}%
  \begin{subfigure}[t]{0.25\textwidth}
  \includegraphics[width=1.02\linewidth]{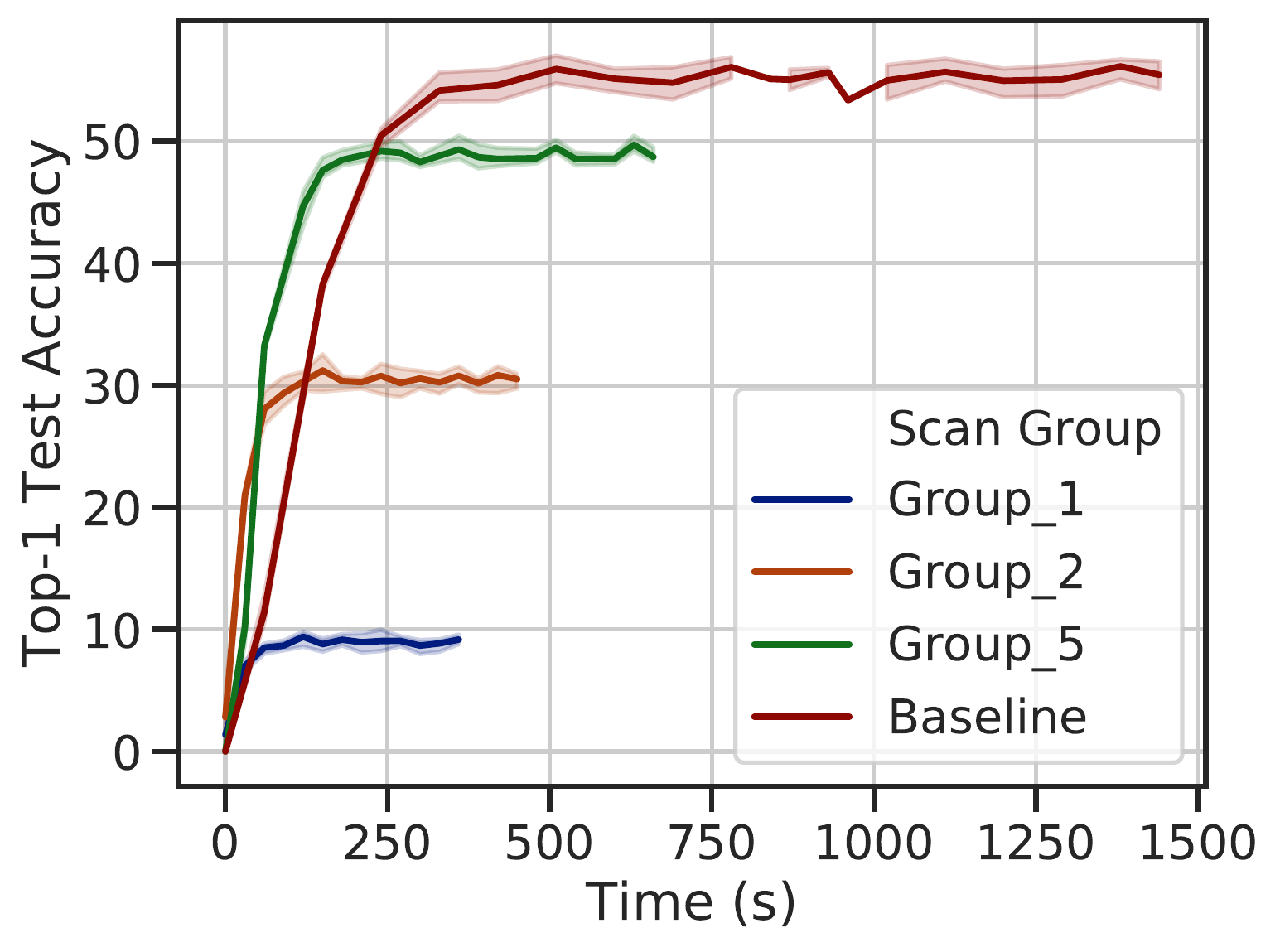}
    \caption{Stanford Cars}
  \end{subfigure}%
  \begin{subfigure}[t]{0.25\textwidth}
  \includegraphics[width=.98\linewidth]{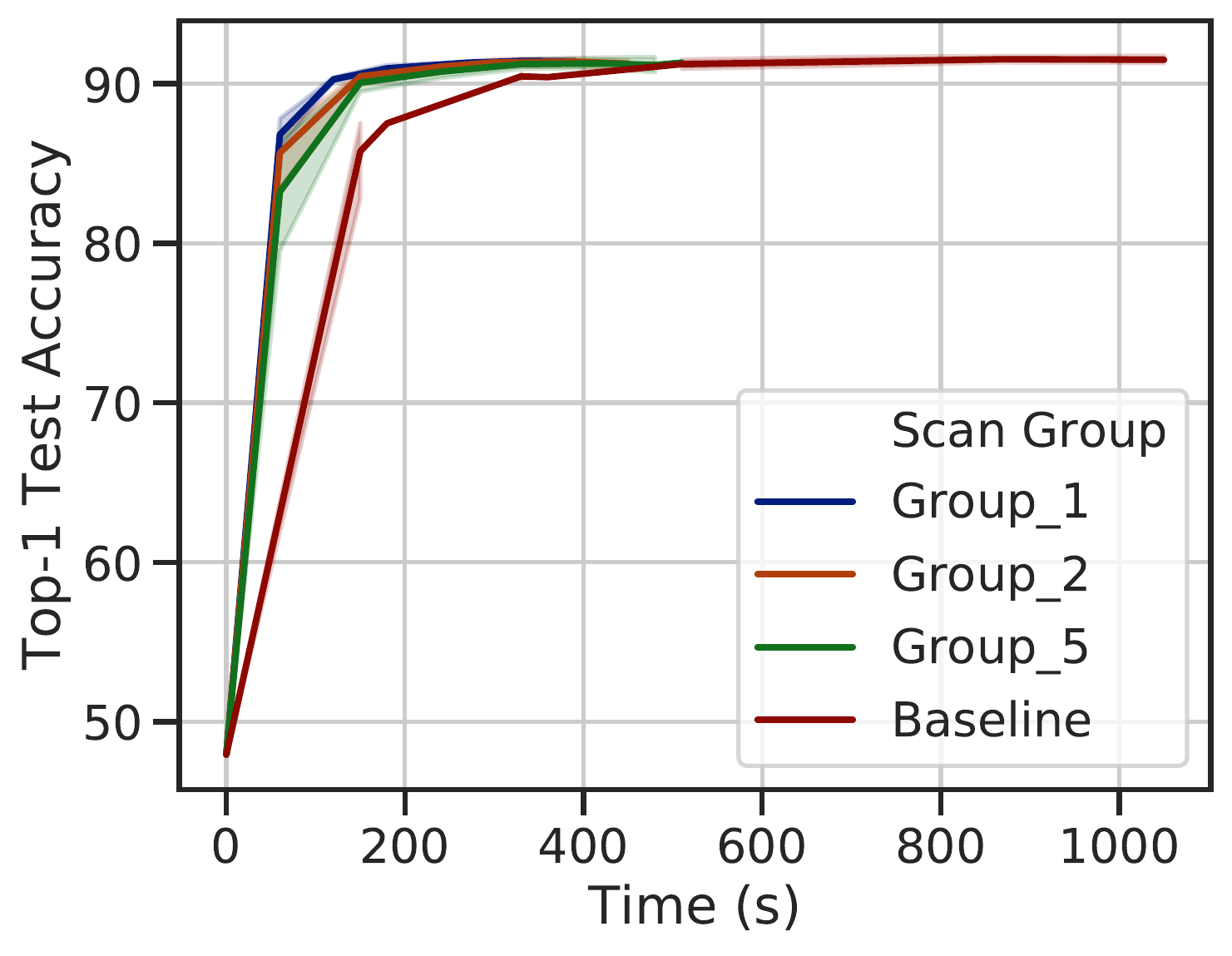}
    \caption{CelebAHQ}
  \end{subfigure}%
  \caption{%
    Top-1 test performance with ShuffleNetv2.
    Lower scan groups speed up training by reducing bandwidth.
    ShuffleNetv2 is more bandwidth bound since it runs faster than ResNet18.
    Time is the x-axis (seconds) and is relative to first epoch.
    95\% confidence intervals are shown.
    Higher scan groups are less compressed.
  }%
  \label{fig:scan_performance_shufflenet_orca_acc_time}%
  \vspace{-10pt}
\end{figure*}

\begin{figure*}
  \centering
  \begin{subfigure}[t]{0.25\textwidth}
    \includegraphics[width=.99\linewidth]{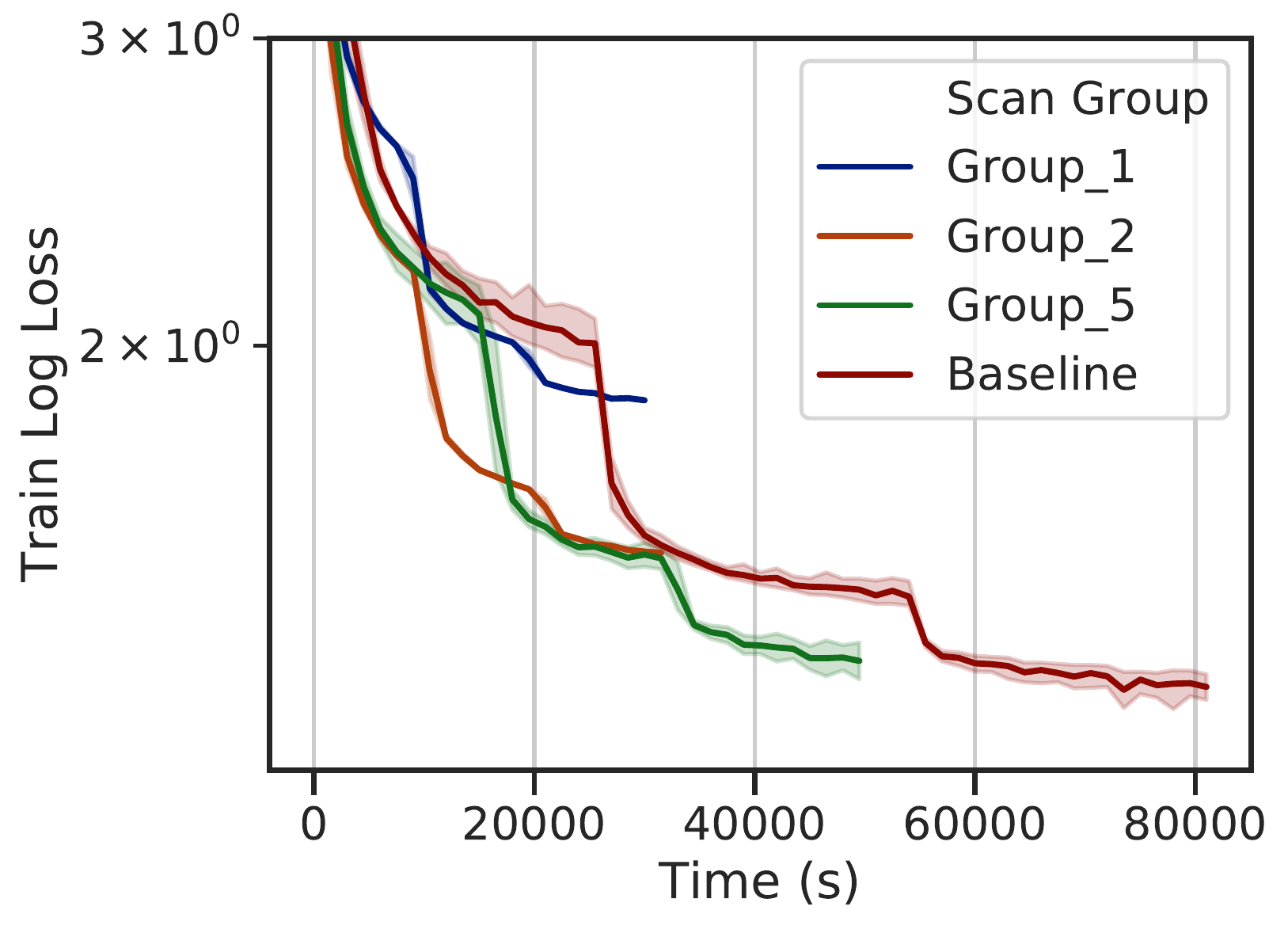}
    \caption{ImageNet}
  \end{subfigure}%
  \begin{subfigure}[t]{0.25\textwidth}
  \includegraphics[width=1.0\linewidth]{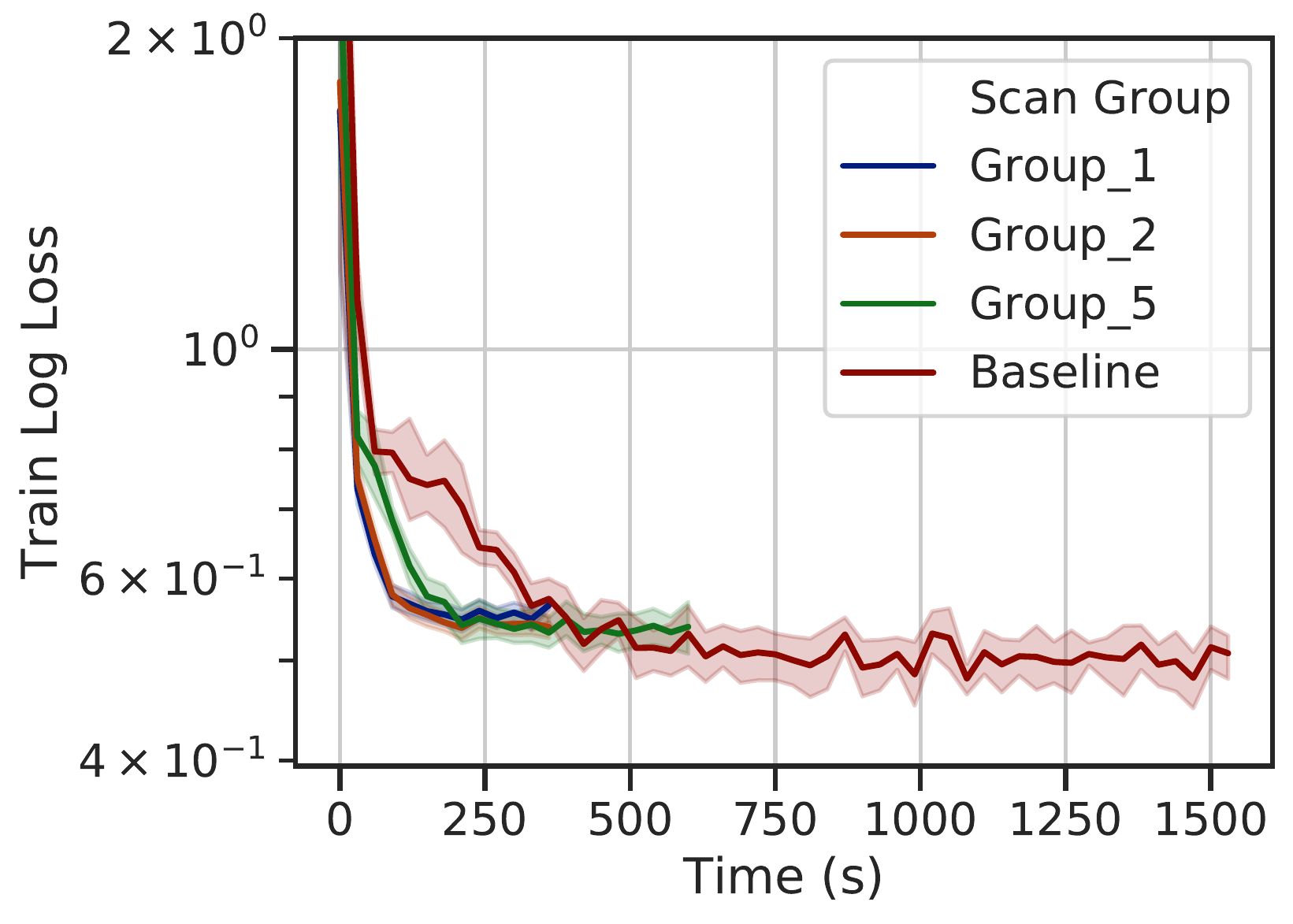}
    \caption{HAM10000}
  \end{subfigure}%
  \begin{subfigure}[t]{0.25\textwidth}
  \includegraphics[width=.99\linewidth]{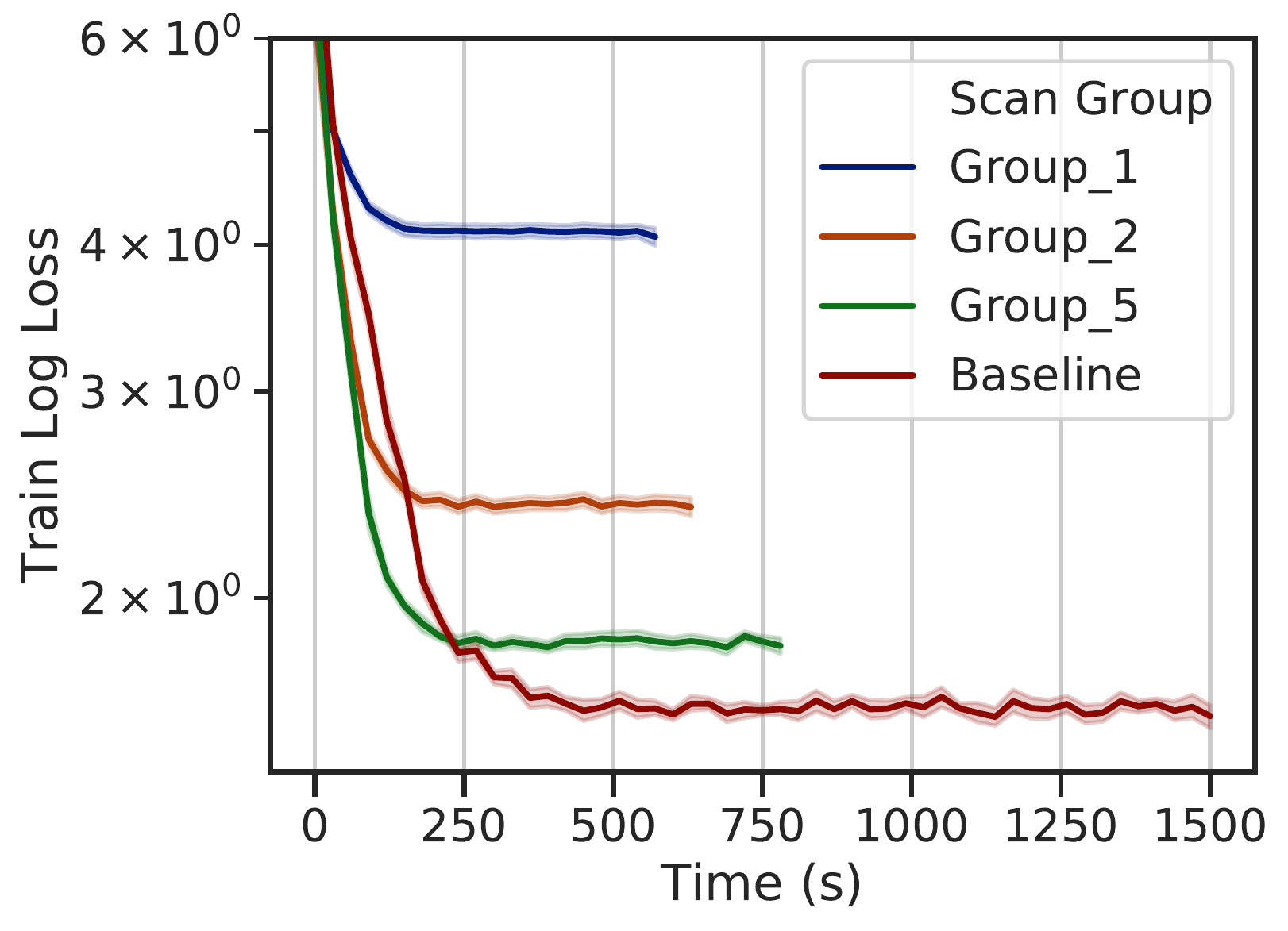}
    \caption{Stanford Cars}
  \end{subfigure}%
  \begin{subfigure}[t]{0.25\textwidth}
  \includegraphics[width=.92\linewidth]{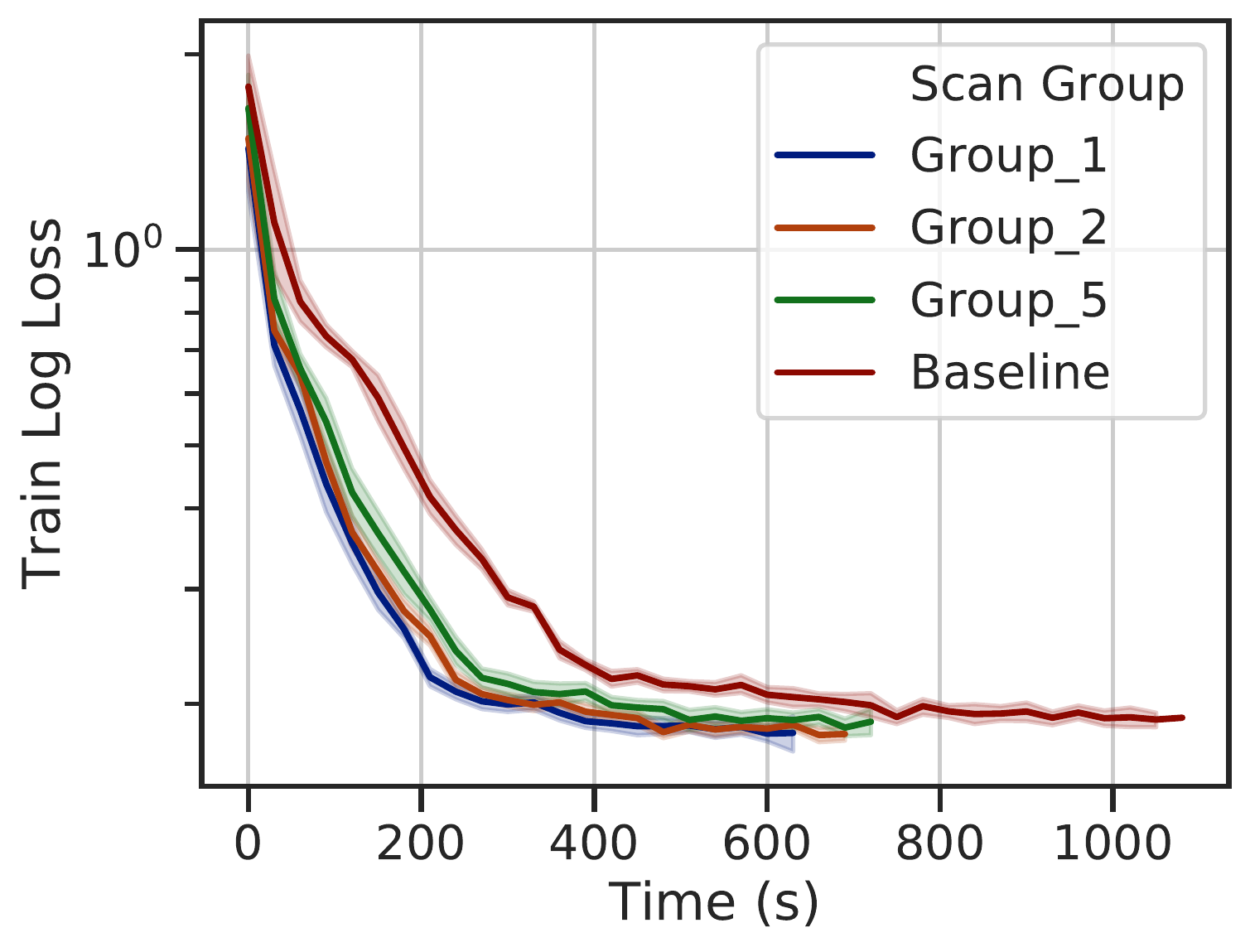}
    \caption{CelebAHQ}
  \end{subfigure}%
  \caption{%
    Training loss with ResNet-18.
    Lower scan groups speed up training, but they may impact loss minimization.
    Time is the x-axis (seconds) and is relative to first epoch.
    95\% confidence intervals are shown.
  }%
  \label{fig:scan_performance_resnet18_orca_loss_time}%
\end{figure*}

\begin{figure*}
  \centering
  \begin{subfigure}[t]{0.25\textwidth}
    \includegraphics[width=.99\linewidth]{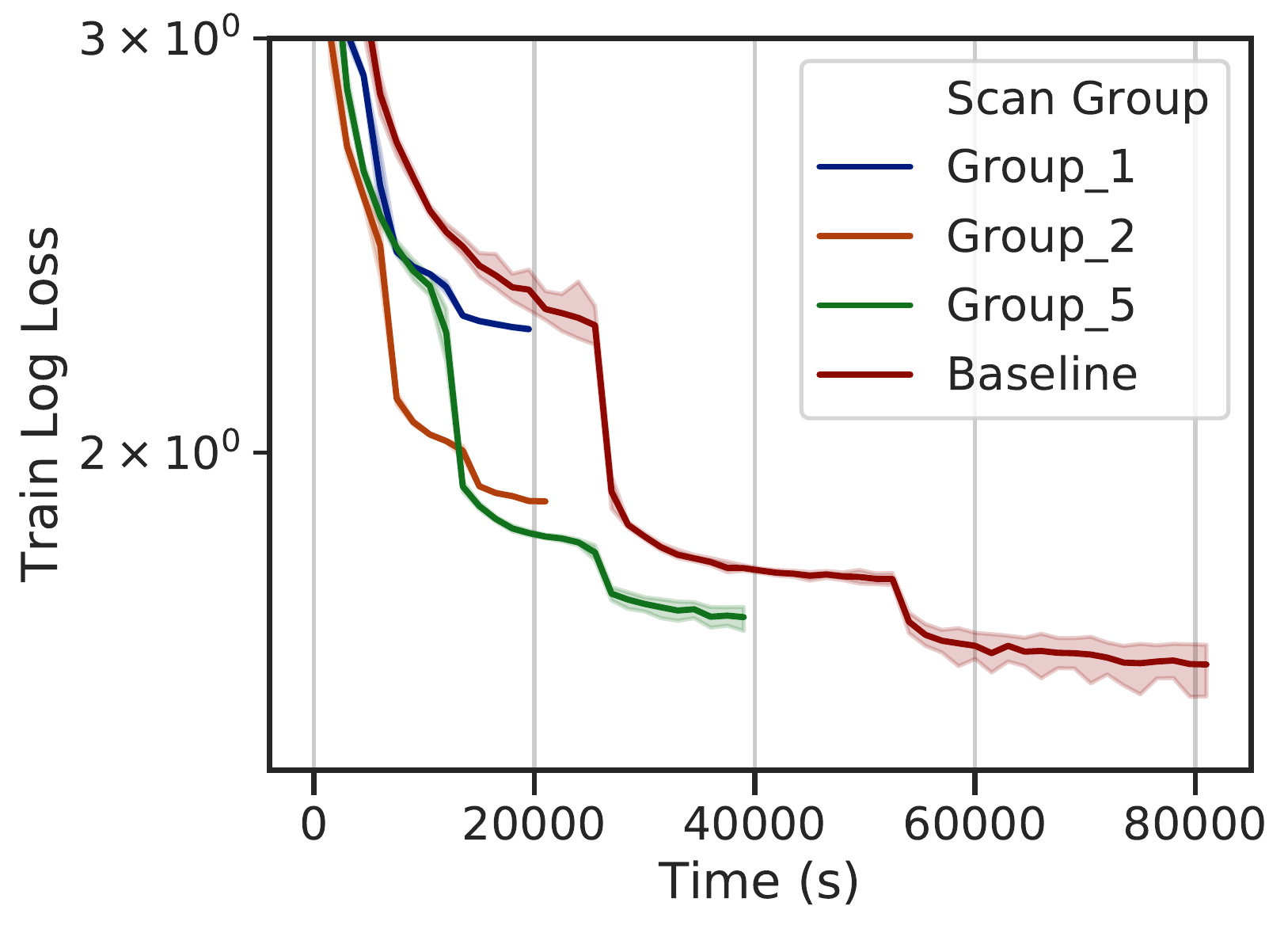}
    \caption{ImageNet}
  \end{subfigure}%
  \begin{subfigure}[t]{0.25\textwidth}
  \includegraphics[width=1.00\linewidth]{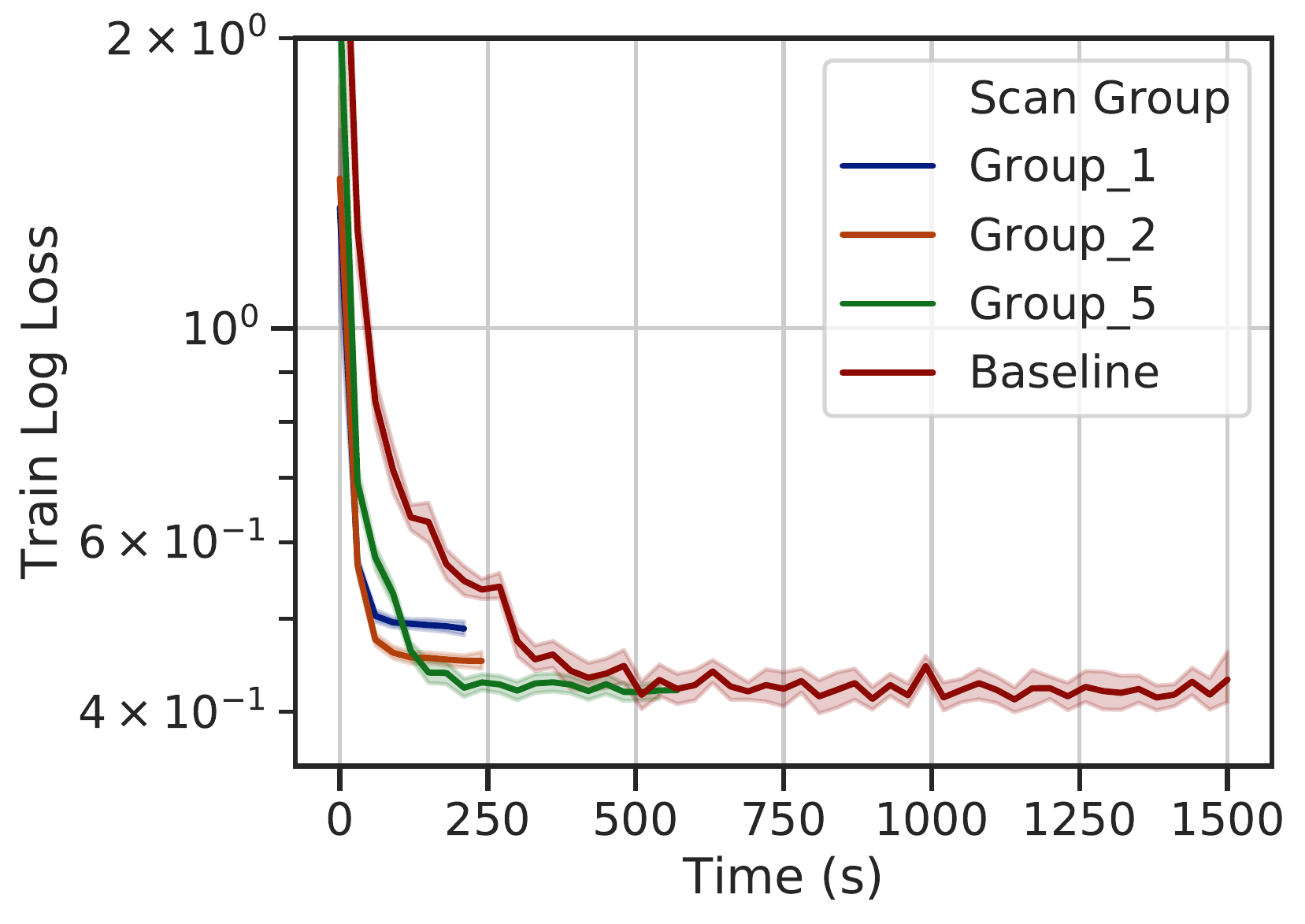}
    \caption{HAM10000}
  \end{subfigure}%
  \begin{subfigure}[t]{0.25\textwidth}
  \includegraphics[width=.99\linewidth]{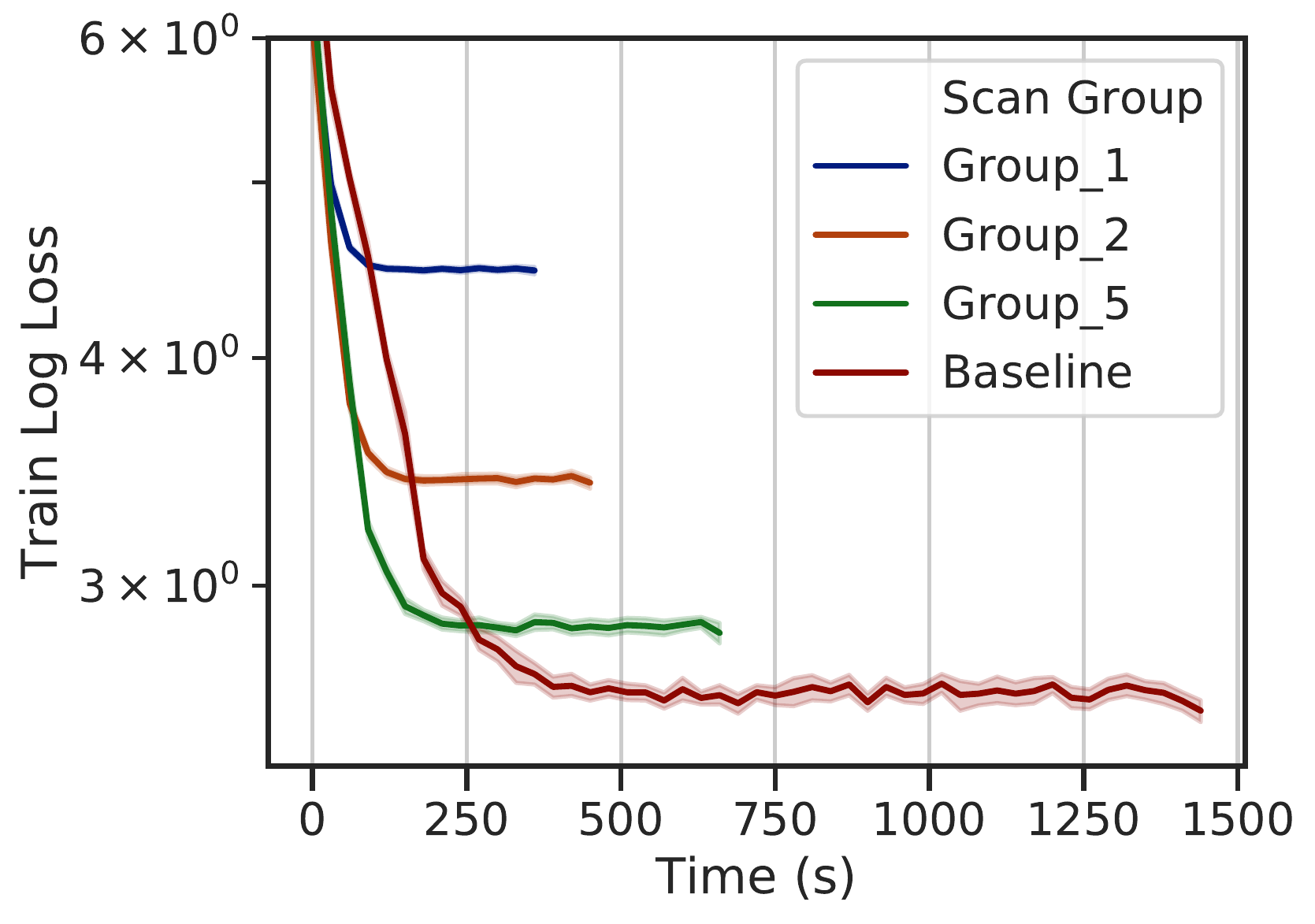}
    \caption{Stanford Cars}
  \end{subfigure}%
  \begin{subfigure}[t]{0.25\textwidth}
  \includegraphics[width=.91\linewidth]{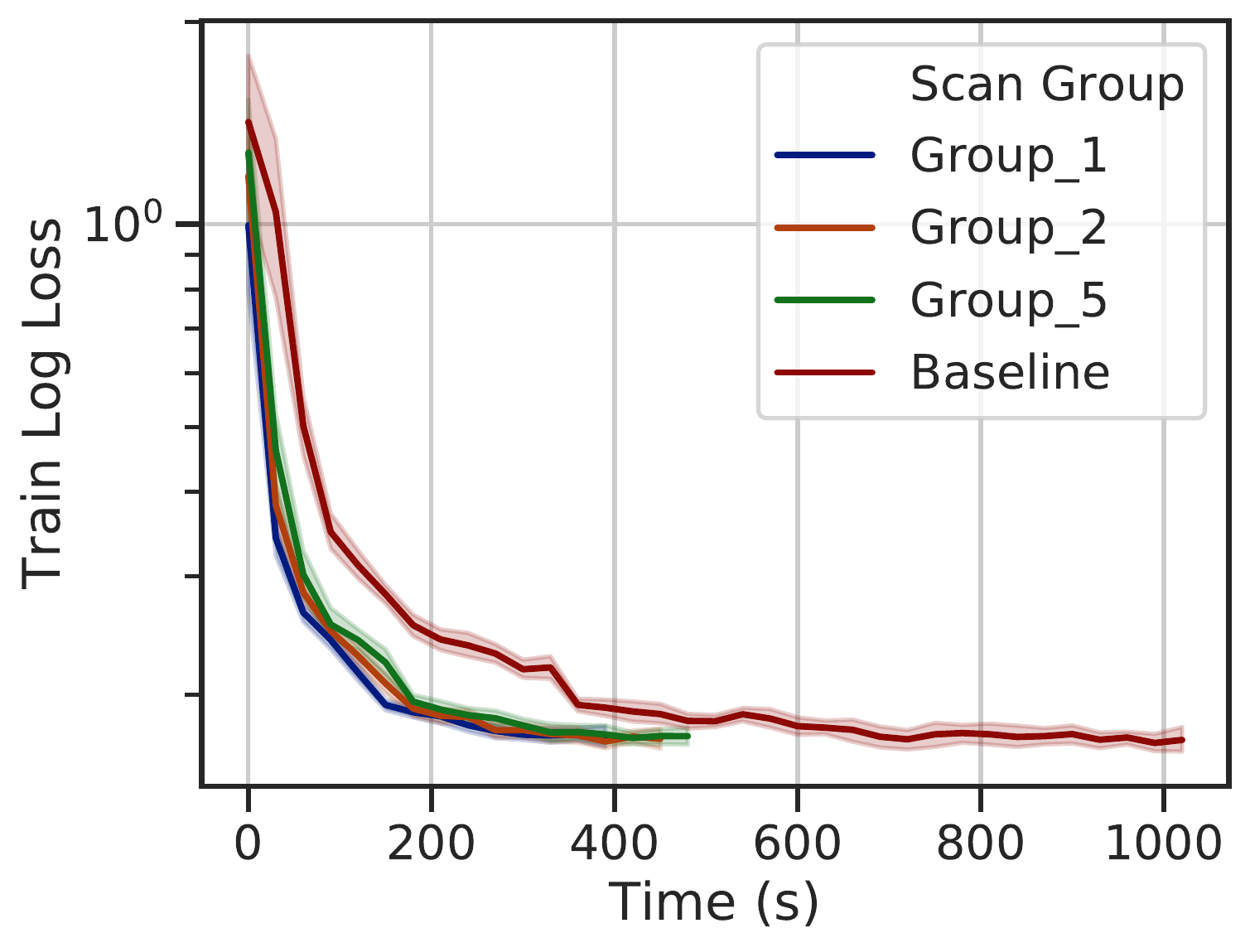}
    \caption{CelebAHQ}
  \end{subfigure}%
  \caption{%
    Training loss with ShuffleNetv2.
    Lower scan groups speed up training, but they may impact loss minimization.
    Time is the x-axis (seconds) and is relative to first epoch.
    95\% confidence intervals are shown.
  }%
  \label{fig:scan_performance_shufflenet_orca_loss_time}%
\end{figure*}

\begin{figure*}
  \centering
  \begin{subfigure}[t]{0.25\textwidth}
    \includegraphics[width=.99\linewidth]{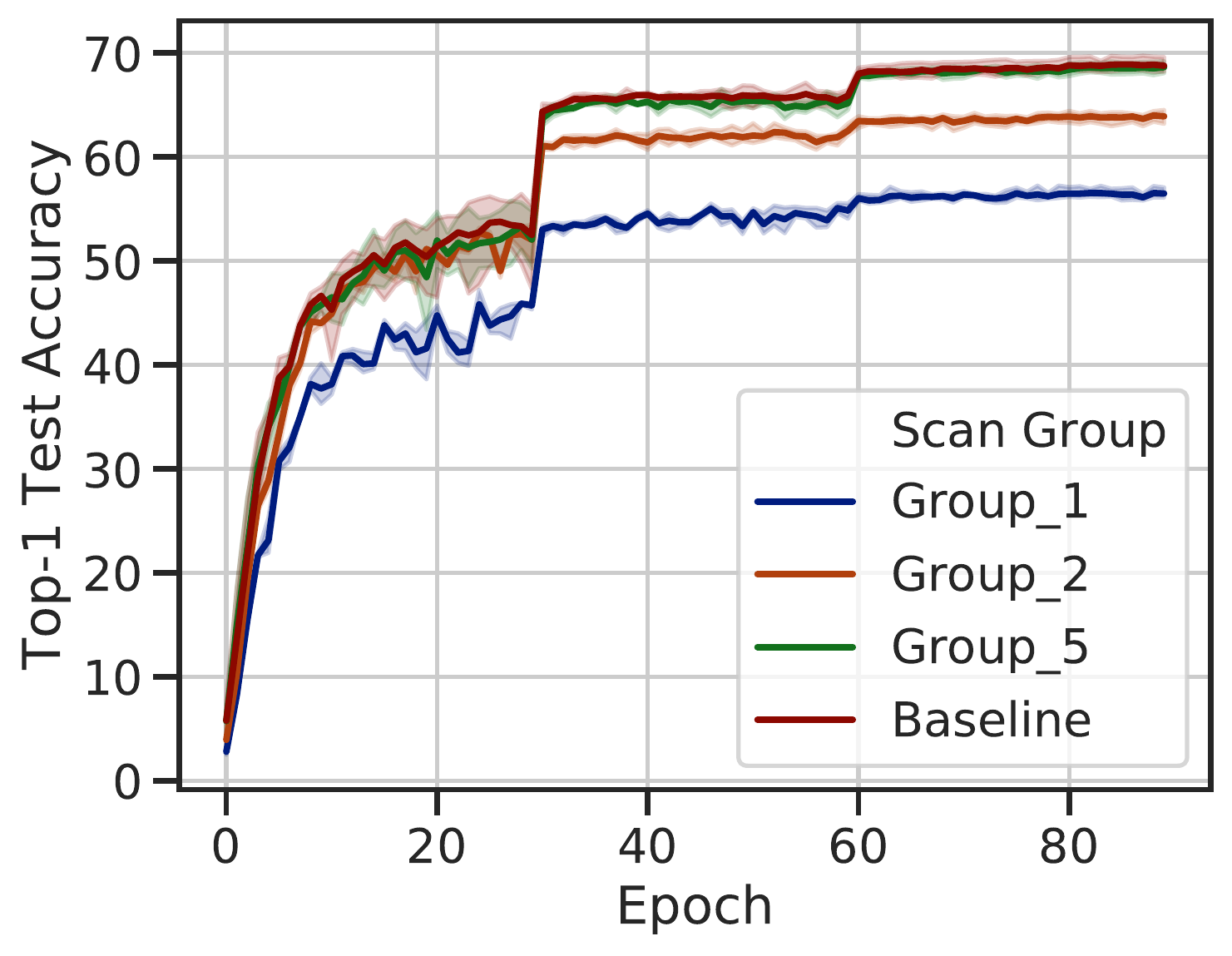}
    \caption{ImageNet}
  \end{subfigure}%
  \begin{subfigure}[t]{0.25\textwidth}
  \includegraphics[width=.99\linewidth]{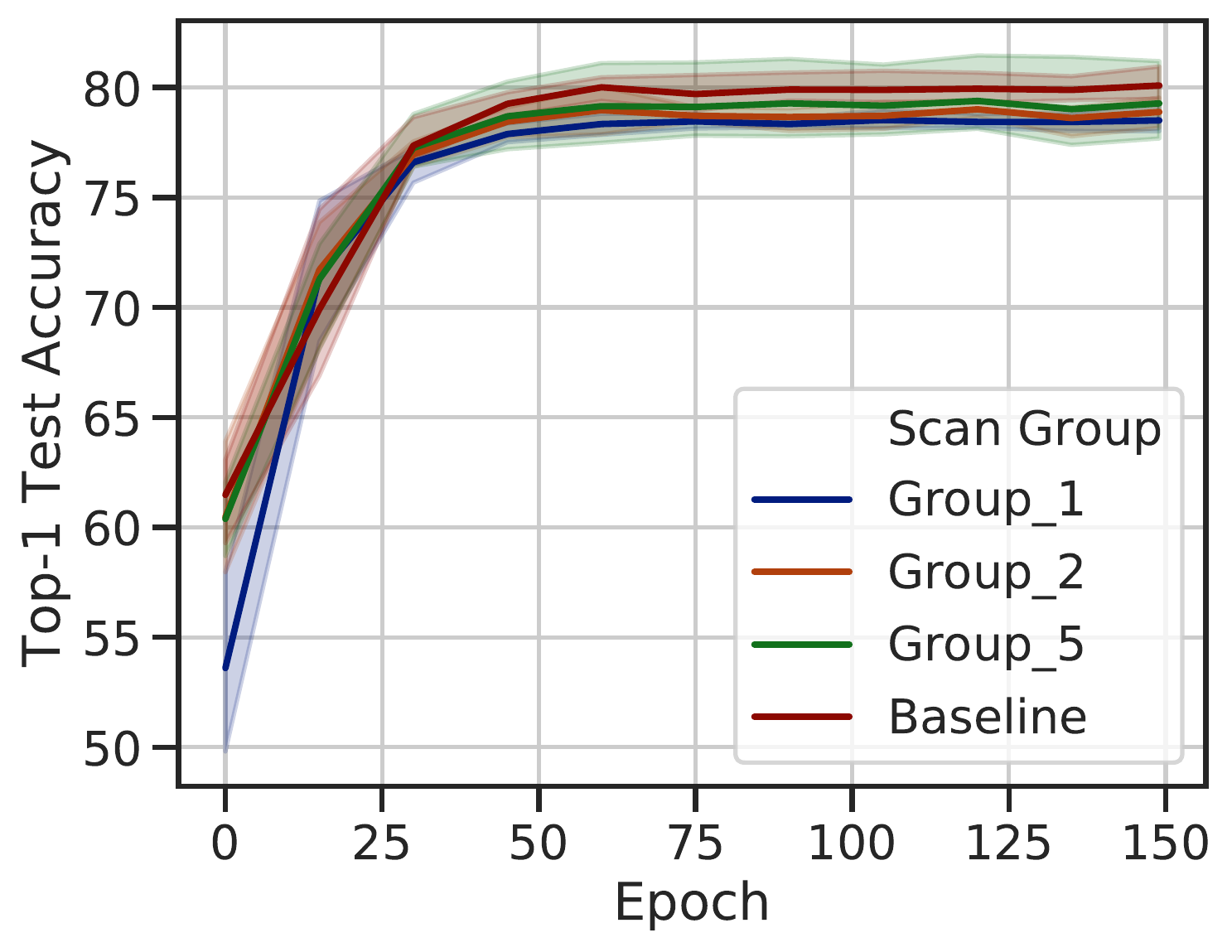}
    \caption{HAM10000}
  \end{subfigure}%
  \begin{subfigure}[t]{0.25\textwidth}
  \includegraphics[width=.99\linewidth]{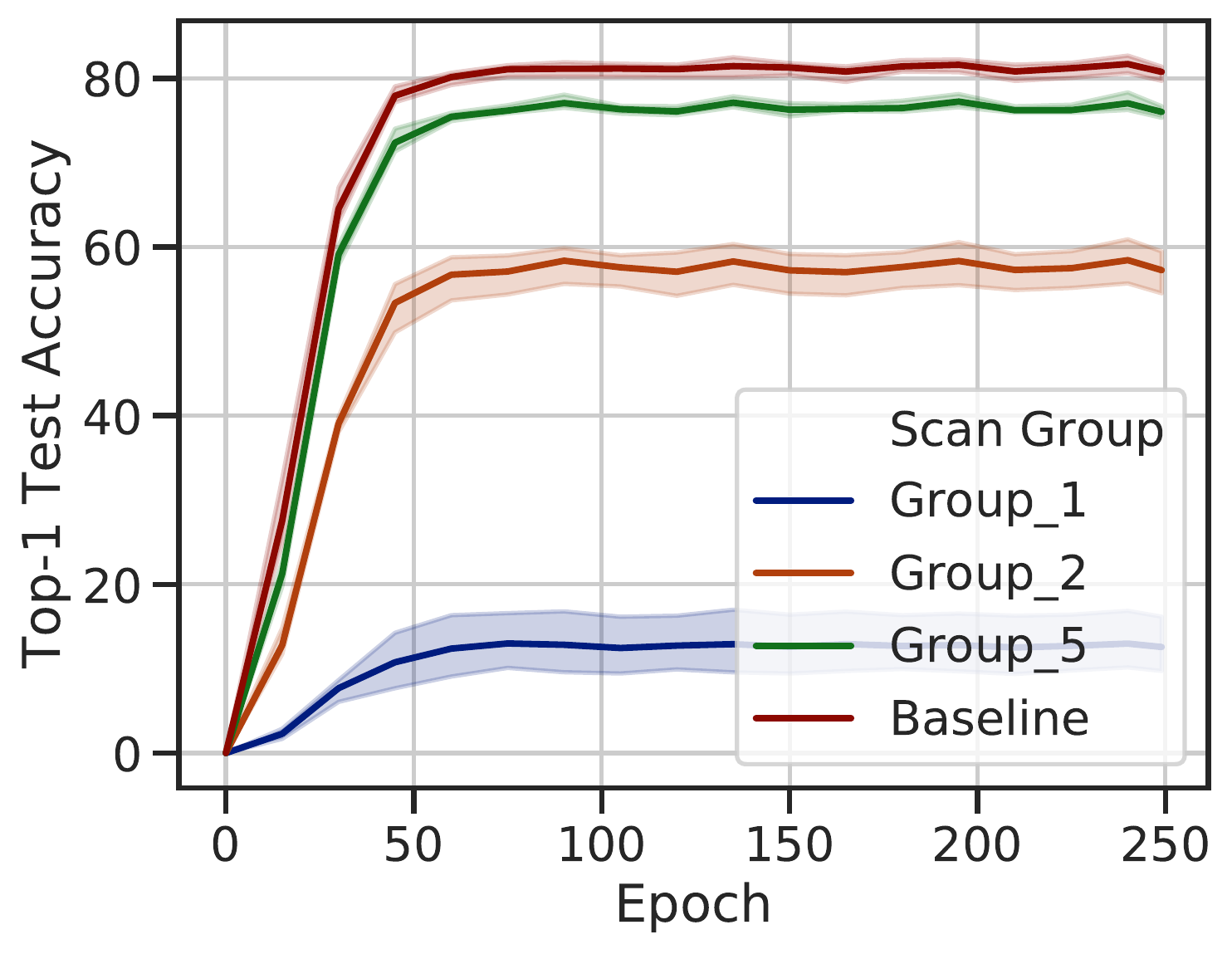}
    \caption{Stanford Cars}
  \end{subfigure}%
  \begin{subfigure}[t]{0.25\textwidth}
  \includegraphics[width=.99\linewidth]{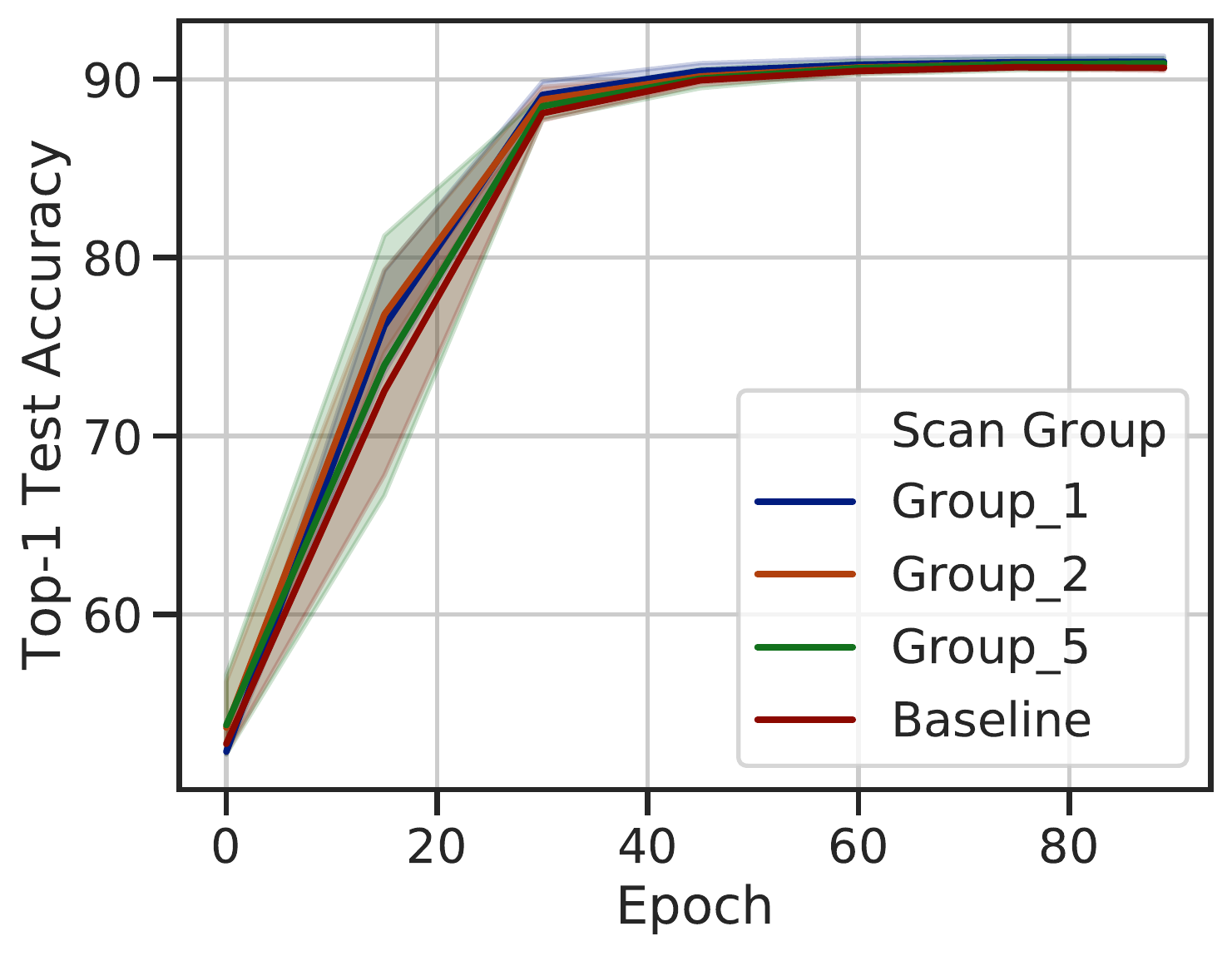}
    \caption{CelebAHQ}
  \end{subfigure}%
  \caption{%
    Testing accuracy with ResNet-18.
    Lower scan groups don't improve accuracy (e.g., if compression was a
    regularizer).
    Epochs are the x-axis.
    95\% confidence intervals are shown.
  }%
  \label{fig:scan_performance_resnet18_orca_epoch}%
\end{figure*}

\begin{figure*}
  \centering
  \begin{subfigure}[t]{0.25\textwidth}
  \includegraphics[width=.975\linewidth]{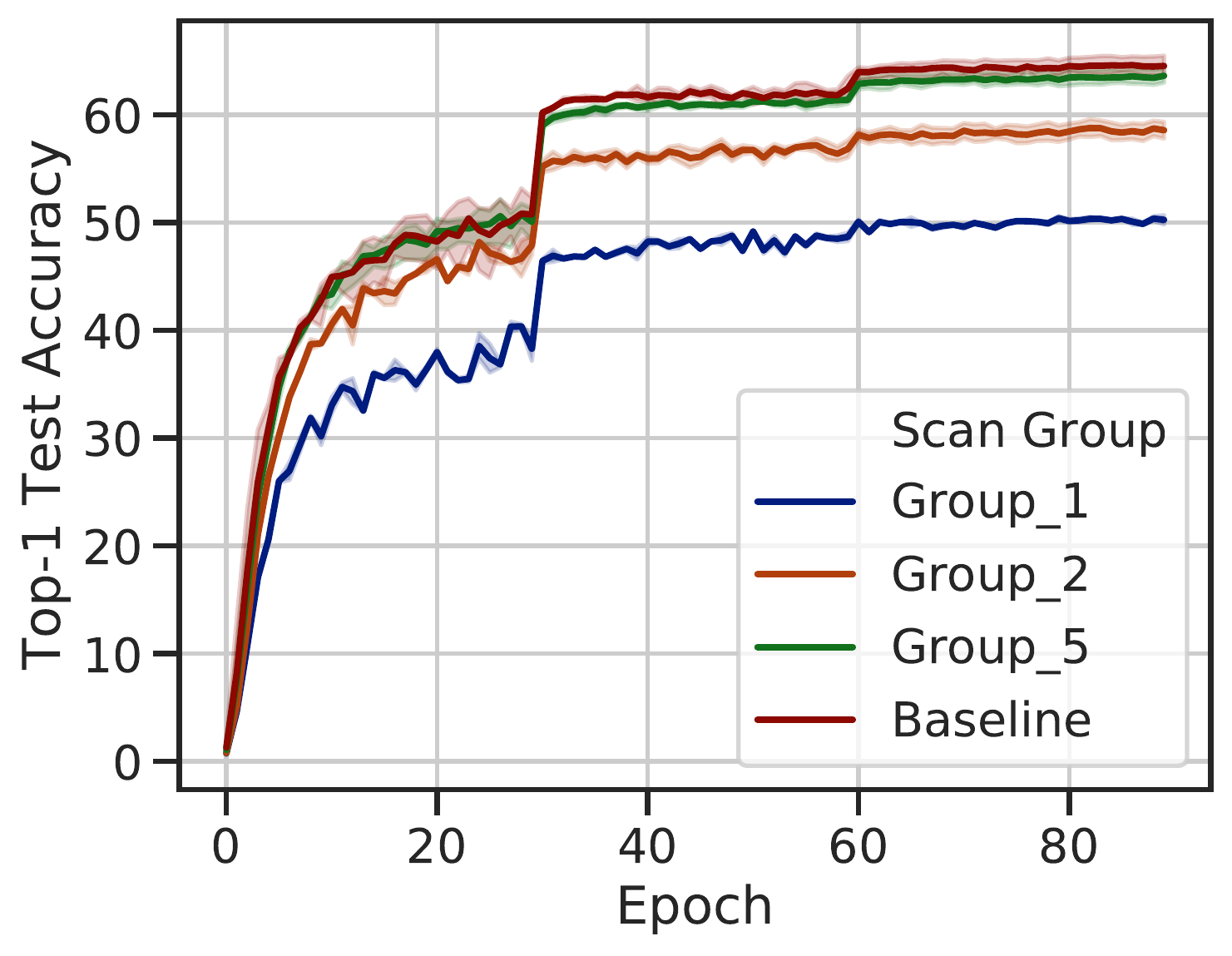}
    \caption{ImageNet}
  \end{subfigure}%
  \begin{subfigure}[t]{0.25\textwidth}
  \includegraphics[width=1.01\linewidth]{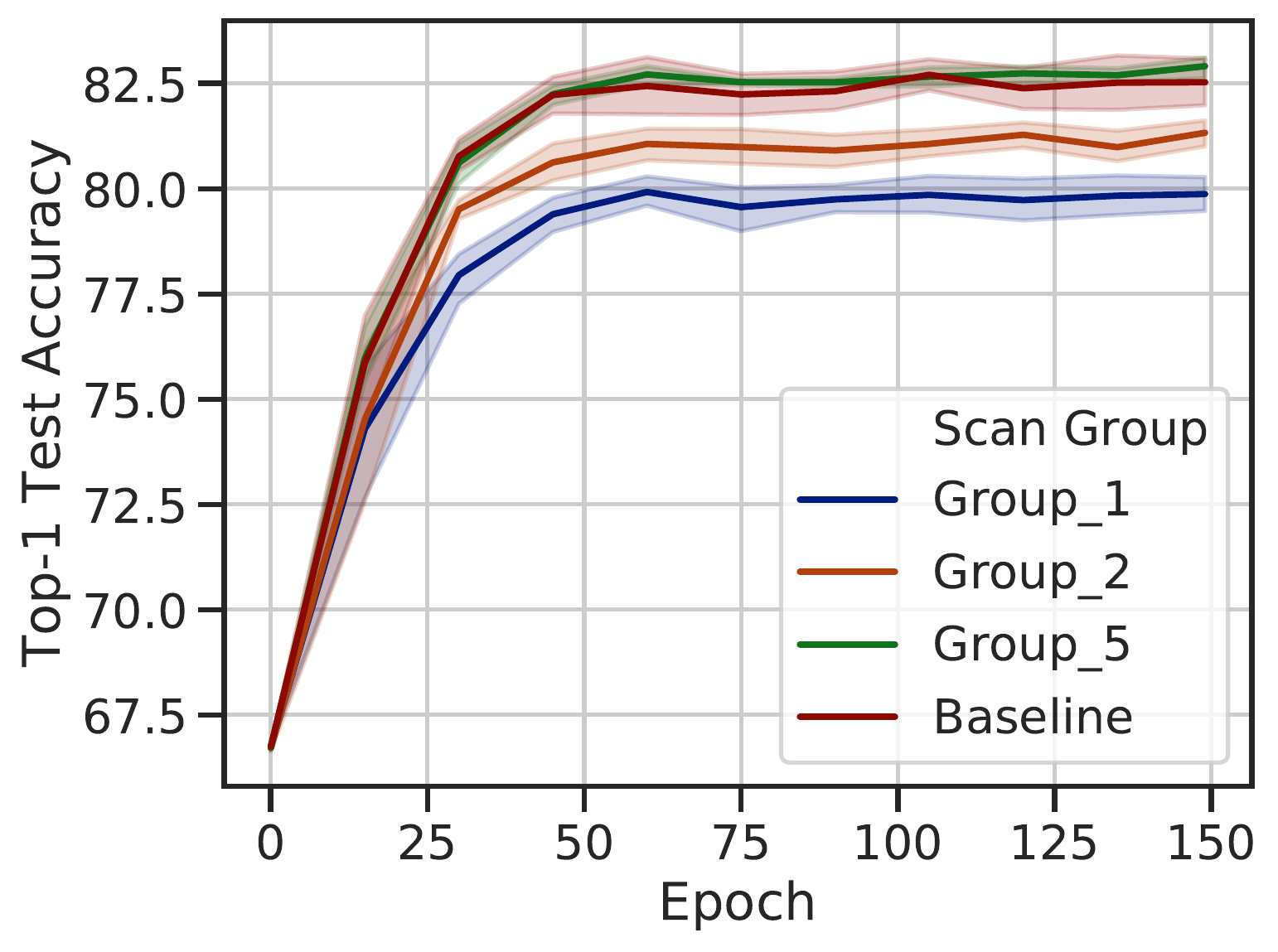}
    \caption{HAM10000}
  \end{subfigure}%
  \begin{subfigure}[t]{0.25\textwidth}
  \includegraphics[width=.98\linewidth]{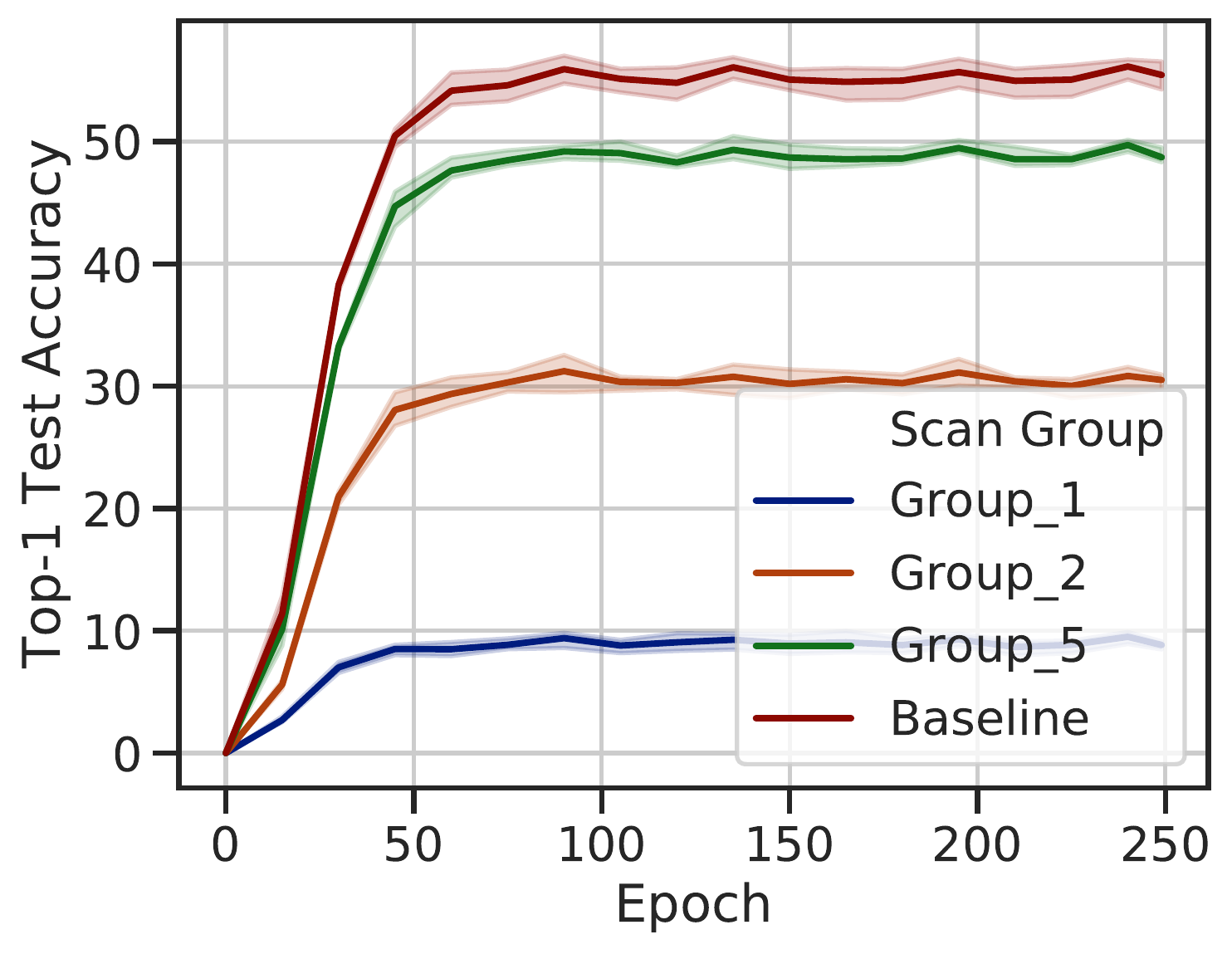}
    \caption{Stanford Cars}
  \end{subfigure}%
  \begin{subfigure}[t]{0.25\textwidth}
  \includegraphics[width=.985\linewidth]{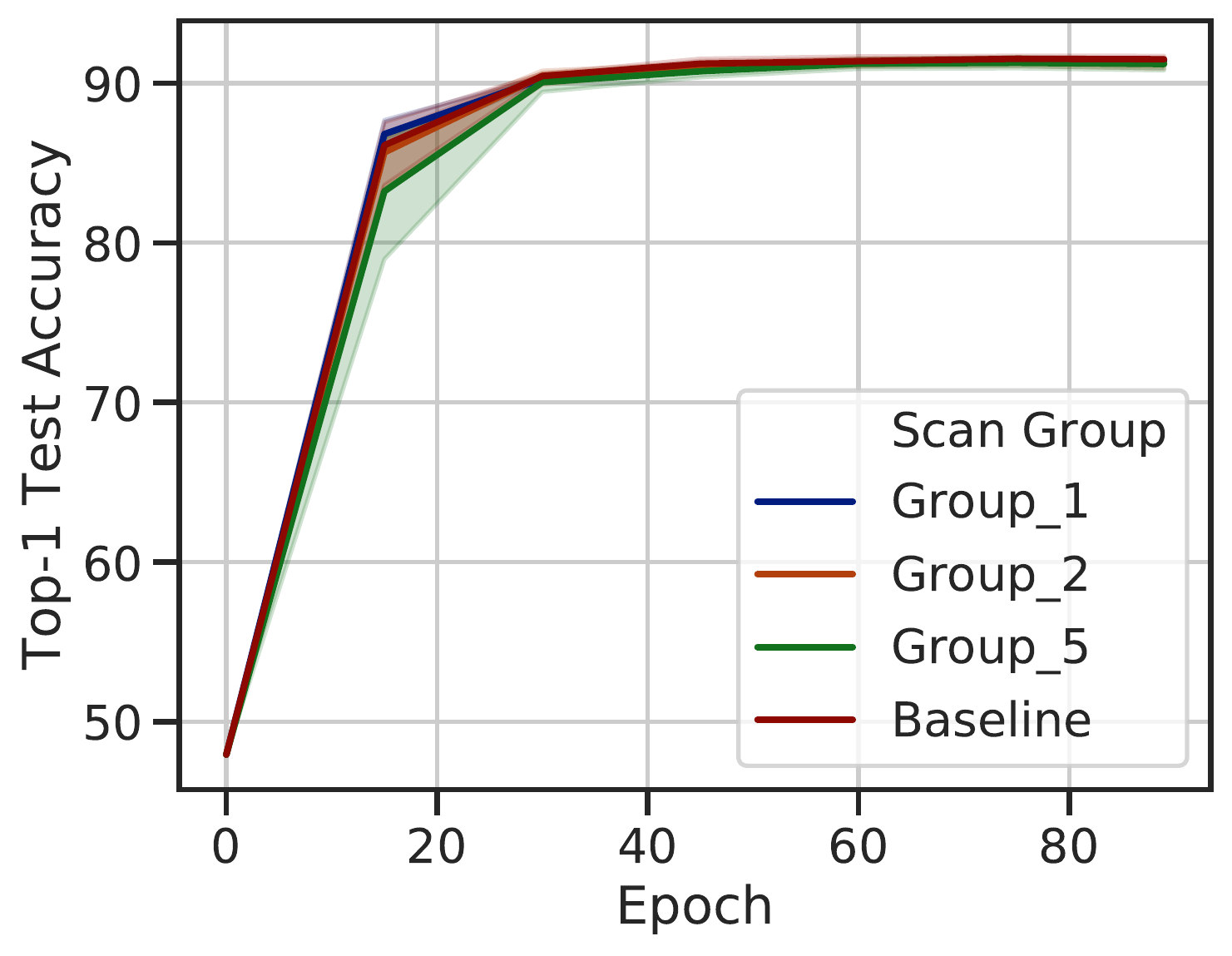}
    \caption{CelebAHQ}
  \end{subfigure}%
  \caption{%
    Testing accuracy with ShuffleNetv2.
    Lower scan groups don't improve accuracy (e.g., if compression was a
    regularizer).
    Epochs are the x-axis.
    95\% confidence intervals are shown.
  }%
  \label{fig:scan_performance_shufflenet_orca_epoch}%
\end{figure*}

\textbf{Coarse Grained vs.\ Fine Grained Cars Experiments.}
We provide the accuracy figures for reduced label sets in
Figure~\ref{fig:coarse_make_cars_performance_resnet18_orca_acc_time_resnet}
and
Figure~\ref{fig:coarse_make_cars_performance_resnet18_orca_acc_time_shufflenet}.
Make-only has 22 classes.

\begin{figure*}
  \centering
  \begin{subfigure}[t]{0.33\textwidth}
    \leftRightCrop{0.0}{cars_scan_performance_resnet18_orca_acc_time_app.pdf}{1.0}{0.0}
    \caption{Original Multiclass}
  \end{subfigure}%
  \begin{subfigure}[t]{0.33\textwidth}
    \leftRightCrop{0.0}{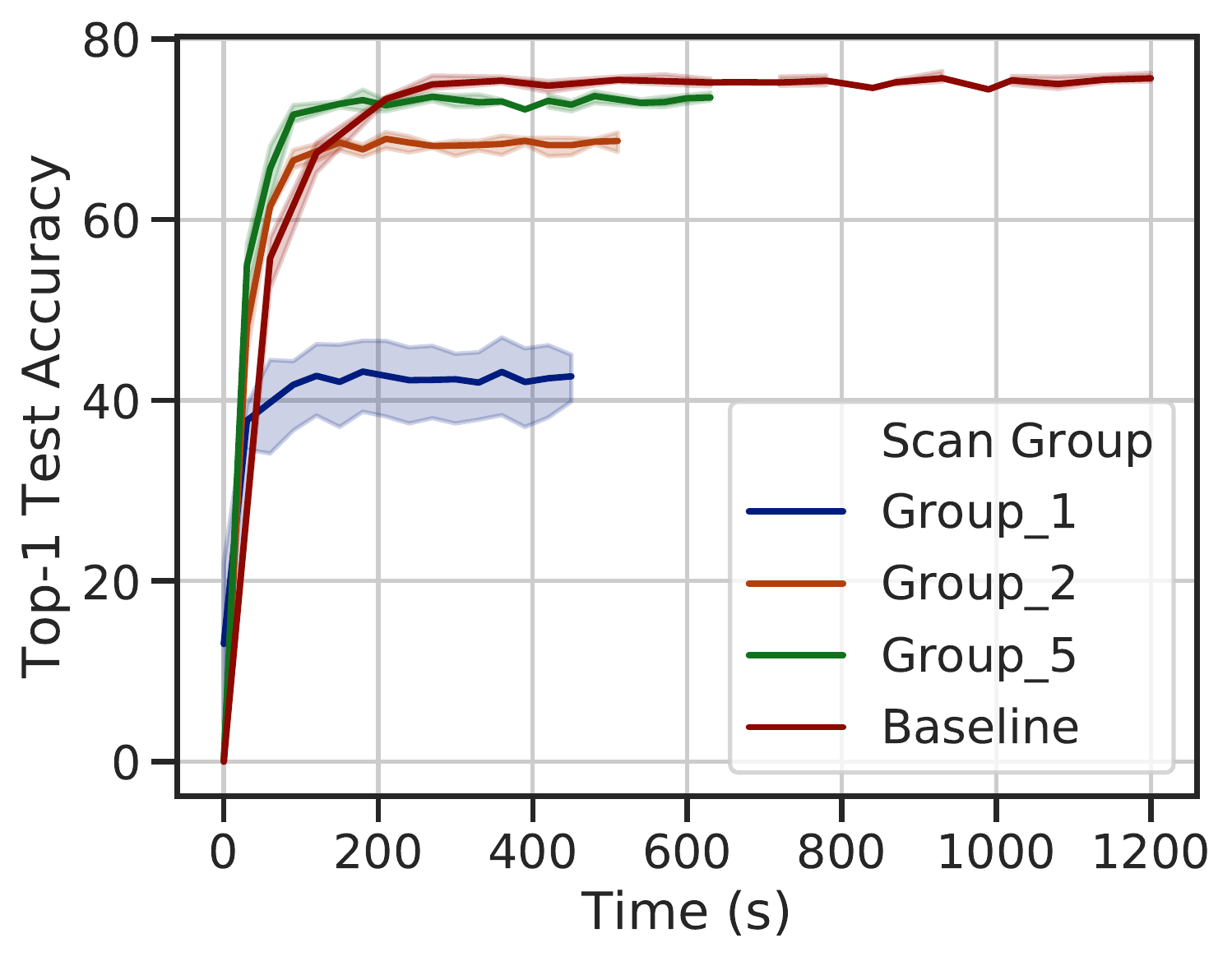}{1.0}{0.0}
    \caption{Make-Only}
  \end{subfigure}%
  \begin{subfigure}[t]{0.33\linewidth}
    \leftRightCrop{0.0}{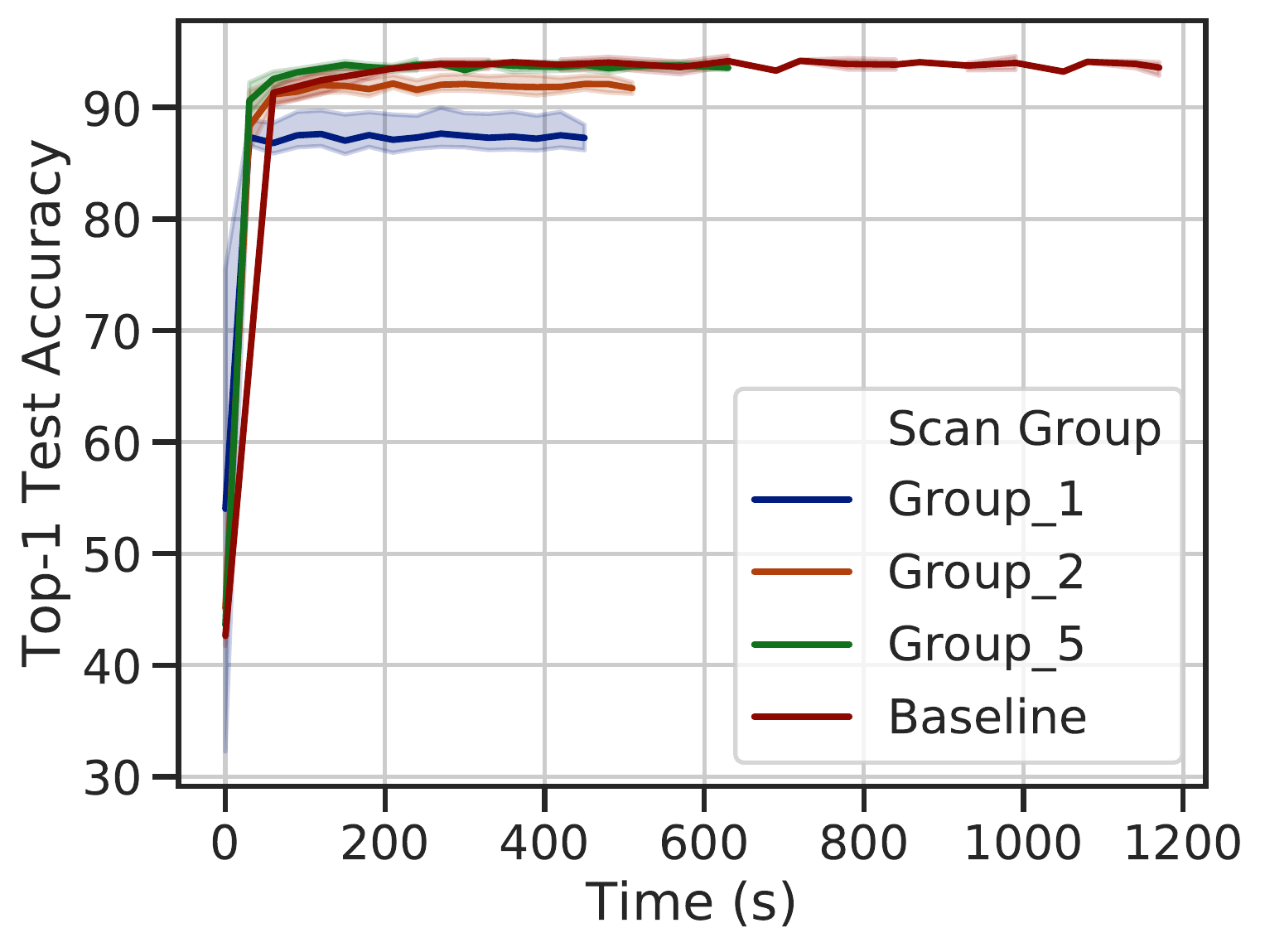}{1.0}{0.0}
    \caption{Binary Is-Corvette}
  \end{subfigure}%
  \vspace{-1pt}%
  \caption{%
    Test accuracy with ResNet-18 
    on a coarser version of the Stanford Cars dataset.
    The full range of classes is used for \textit{Baseline} (i.e, car make,
    model, and year create a unique class),
    only car make is used for \textit{Make-Only},
    and a binary classification task of Corvette detection is used for
    \textit{Is-Corvette}.
    The gap between scan groups closes as the task is made more simple.
    Time is the x-axis (seconds) and is relative to first epoch.
    95\% confidence intervals are shown.
  }%
  \label{fig:coarse_make_cars_performance_resnet18_orca_acc_time_resnet}%
  \vspace{-5pt}%
\end{figure*}

\begin{figure*}[h]
  \begin{subfigure}[t]{0.33\textwidth}
  \includegraphics[width=1.01\linewidth]{cars_scan_performance_shufflenet_orca_acc_time.pdf}
    \caption{Original Multiclass}
  \end{subfigure}%
  \begin{subfigure}[t]{0.33\textwidth}
    \includegraphics[width=.993\linewidth]{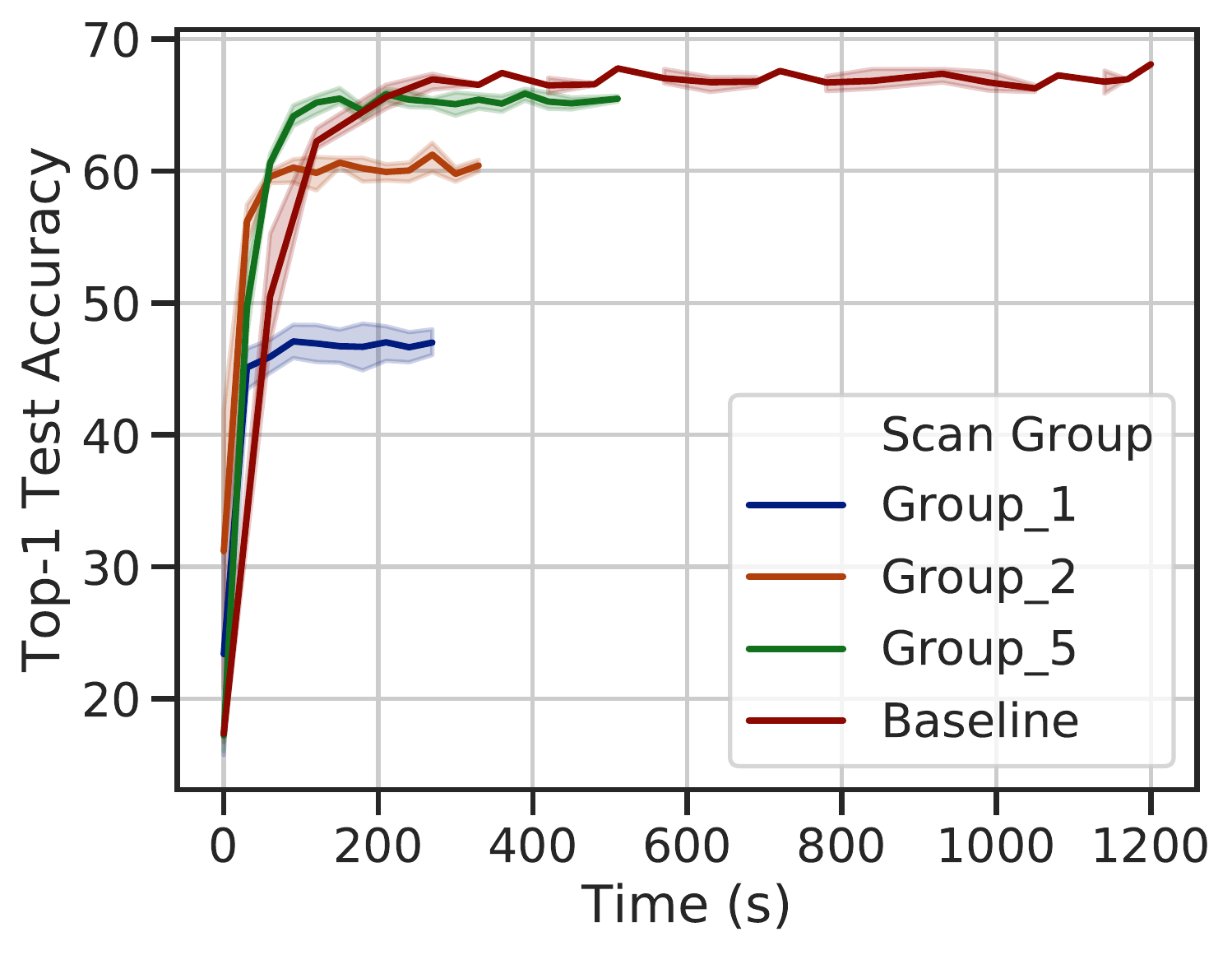}
    \caption{Make-Only}
  \end{subfigure}%
  \begin{subfigure}[t]{0.33\textwidth}
    \includegraphics[width=1.015\linewidth]{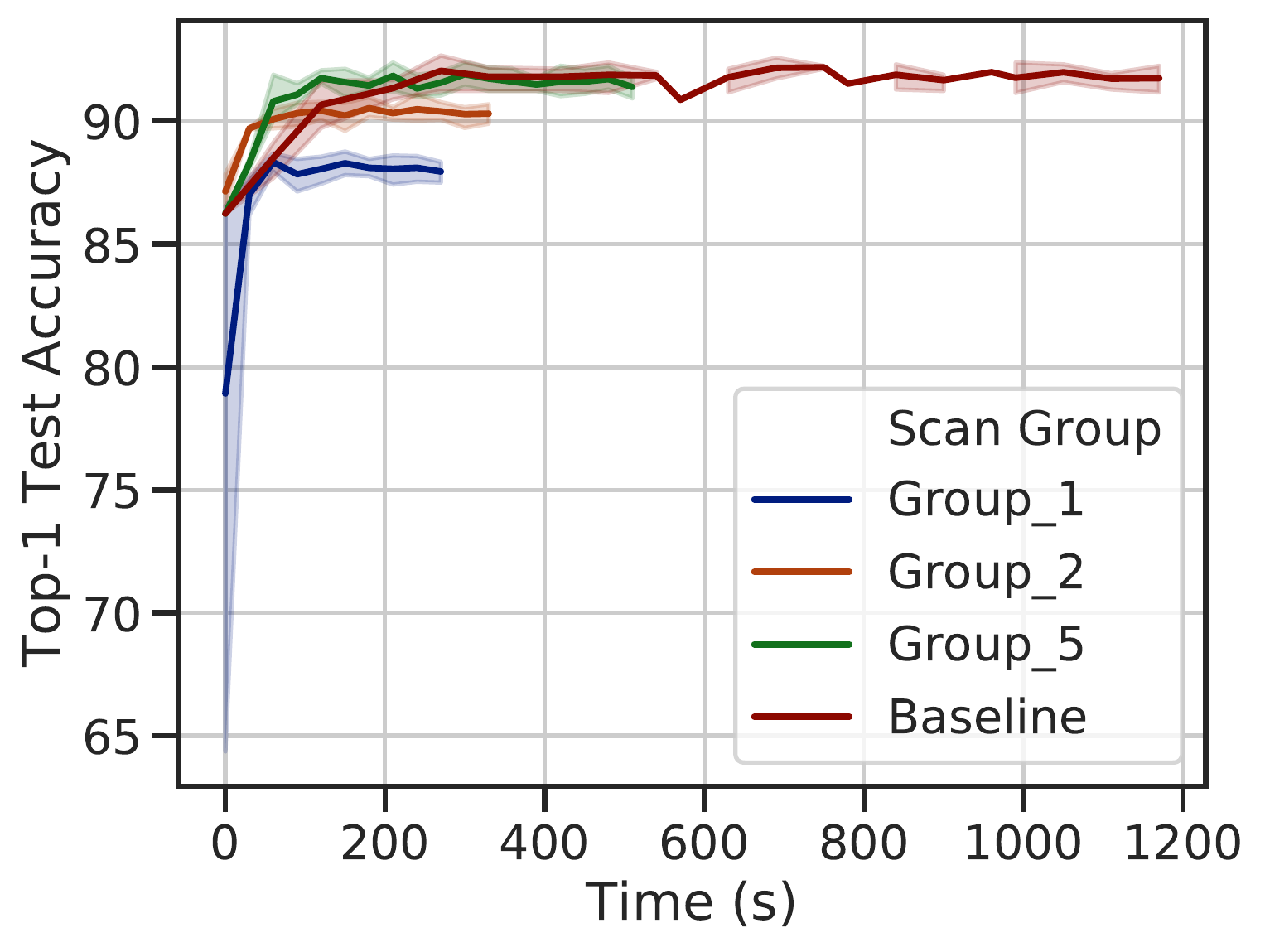}
    \caption{Binary Is-Corvette}
  \end{subfigure}%
  \caption{%
    Training accuracy with
    ShuffleNetv2
    on a coarser version of the Stanford Cars dataset.
    The full range of classes is used for \textit{Baseline} (i.e, car make,
    model, and year create a unique class),
    only car make is used for \textit{Make-Only},
    and a binary classification task of Corvette detection is used for
    \textit{Is-Corvette}.
    The gap between scan groups closes as the task is made more simple.
    Time is the x-axis (seconds) and is relative to first epoch.
    95\% confidence intervals are shown.
  }%
  \label{fig:coarse_make_cars_performance_resnet18_orca_acc_time_shufflenet}%
\end{figure*}

\subsection{Image Examples by Scan}%
\label{sec:examples}
We provide image examples from each dataset that illustrate each scan
group in Figure~\ref{fig:jpeg_scans_big_all}.
Reading more scans, and, thus, data, from a progressive image results in higher
fidelity images, but there are diminishing returns.
Images can use a remarkably low amount of scan groups without impacting visual
quality, which manifests in bandwidth savings if used accordingly.
\addtocounter{figure}{1}%
\begin{figure*}
  \setlength\tabcolsep{1pt}
  \settowidth\rotheadsize{ImageNet}
  \renewcommand\thesubfigure{\arabic{subfigure}}
  \setcounter{subfigure}{0}%
  \begin{tabularx}{\linewidth}{l XXX}%
    \rothead{\textbf{\mbox{ImageNet}}} &
  \begin{subfigure}[t]{0.2\hsize}
    \centerCropEx{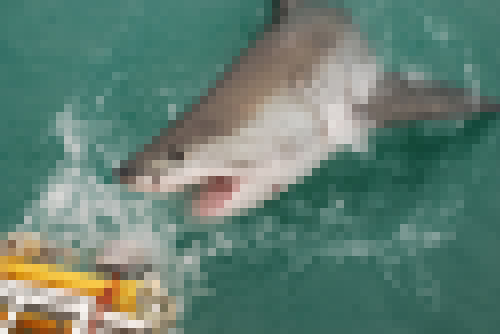}
    \caption{6.3KiB}
  \end{subfigure}%
  \begin{subfigure}[t]{0.2\hsize}
    \centerCropEx{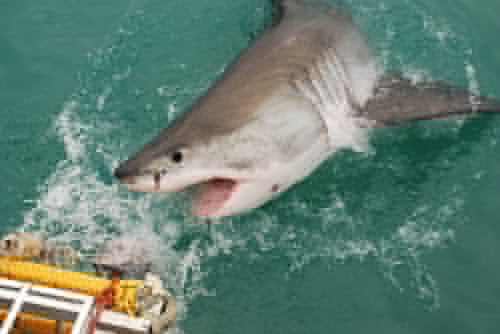}
    \caption{14.7KiB}
  \end{subfigure}%
  \begin{subfigure}[t]{0.2\hsize}
    \centerCropEx{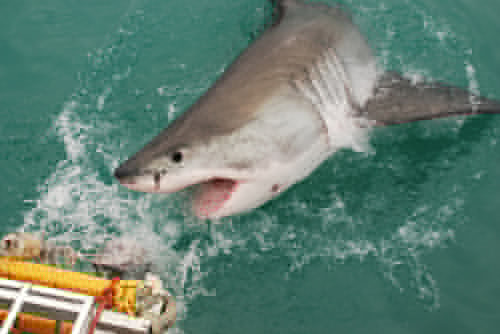}
    \caption{23.8KiB}
  \end{subfigure}%
  \begin{subfigure}[t]{0.2\hsize}
    \centerCropEx{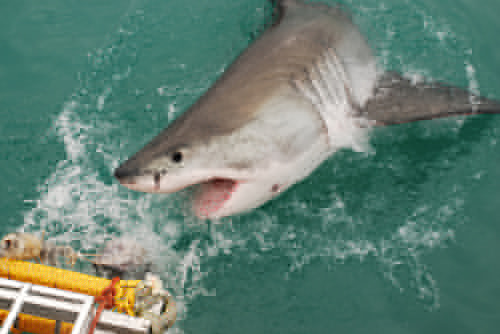}
    \caption{28.8KiB}
  \end{subfigure}%
  \begin{subfigure}[t]{0.2\hsize}
    \centerCropEx{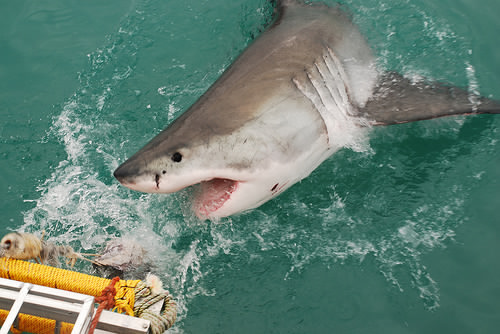}
    \caption{53.3KiB}
  \end{subfigure}%
  \newline
  \begin{subfigure}[t]{0.2\hsize}
    \centerCropEx{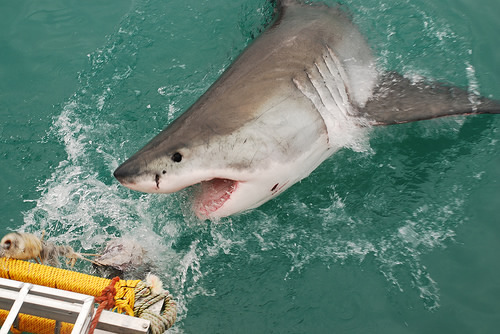}
    \caption{72.0KiB}
  \end{subfigure}%
  \begin{subfigure}[t]{0.2\hsize}
    \centerCropEx{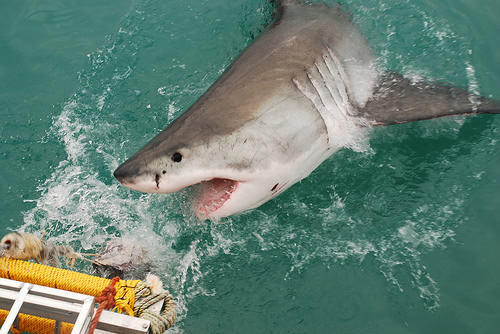}
    \caption{73.1KiB}
  \end{subfigure}%
  \begin{subfigure}[t]{0.2\hsize}
    \centerCropEx{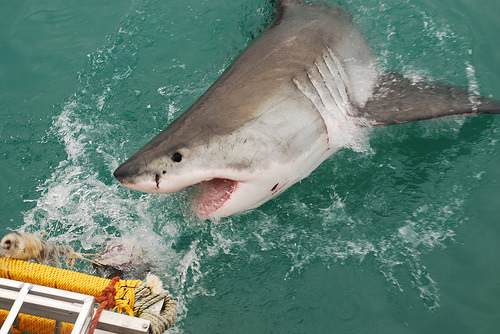}
    \caption{84.6KiB}
  \end{subfigure}%
  \begin{subfigure}[t]{0.2\hsize}
    \centerCropEx{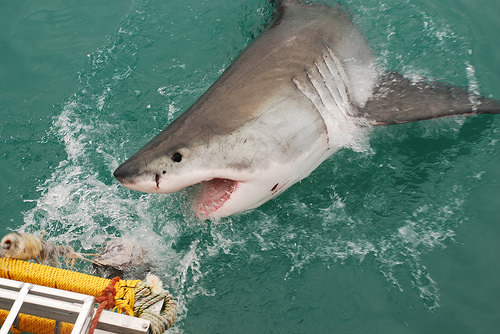}
    \caption{91.0KiB}
  \end{subfigure}%
  \begin{subfigure}[t]{0.2\hsize}
    \centerCropEx{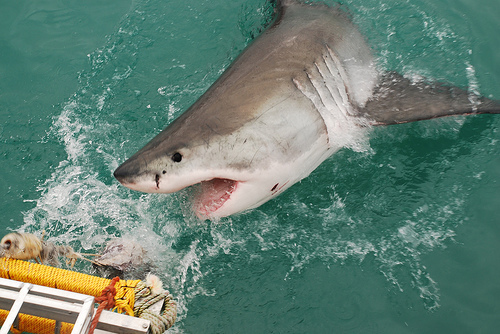}
    \caption{115.3KiB}
  \end{subfigure} \\
  \setcounter{subfigure}{0}%
    \rothead{\textbf{HAM10000}} &
  \begin{subfigure}[t]{0.2\hsize}
    \centerCropEx{ISIC_0024643_001.jpg}
    \caption{7.3KiB}
  \end{subfigure}%
  \begin{subfigure}[t]{0.2\hsize}
    \centerCropEx{ISIC_0024643_002.jpg}
    \caption{13.9KiB}
  \end{subfigure}%
  \begin{subfigure}[t]{0.2\hsize}
    \centerCropEx{ISIC_0024643_003.jpg}
    \caption{43.6KiB}
  \end{subfigure}%
  \begin{subfigure}[t]{0.2\hsize}
    \centerCropEx{ISIC_0024643_004.jpg}
    \caption{81.5KiB}
  \end{subfigure}%
  \begin{subfigure}[t]{0.2\hsize}
    \centerCropEx{ISIC_0024643_005.jpg}
    \caption{97.7KiB}
  \end{subfigure}%
  \newline
  \begin{subfigure}[t]{0.2\hsize}
    \centerCropEx{ISIC_0024643_006.jpg}
    \caption{139.9KiB}
  \end{subfigure}%
  \begin{subfigure}[t]{0.2\hsize}
    \centerCropEx{ISIC_0024643_007.jpg}
    \caption{141.5KiB}
  \end{subfigure}%
  \begin{subfigure}[t]{0.2\hsize}
    \centerCropEx{ISIC_0024643_008.jpg}
    \caption{191.3KiB}
  \end{subfigure}%
  \begin{subfigure}[t]{0.2\hsize}
    \centerCropEx{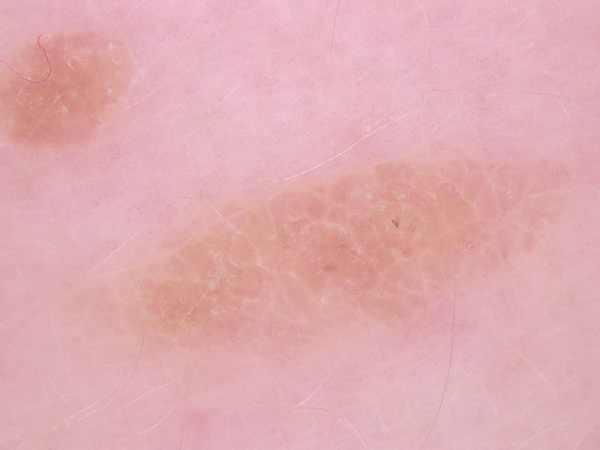}
    \caption{240.6KiB}
  \end{subfigure}%
  \begin{subfigure}[t]{0.2\hsize}
    \centerCropEx{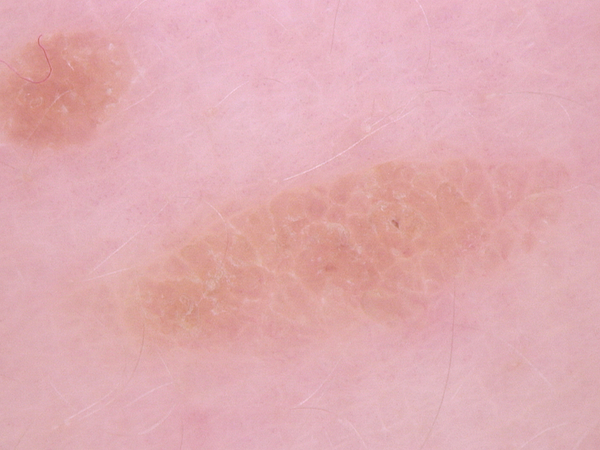}
    \caption{287.4KiB}
  \end{subfigure} \\
  \setcounter{subfigure}{0}%
  \rothead{\textbf{\mbox{Stanford Cars}}} &
  \begin{subfigure}[t]{0.2\hsize}
    \centerCropEx{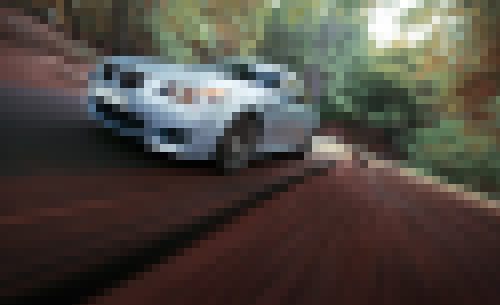}
    \caption{5.9KiB}
  \end{subfigure}%
  \begin{subfigure}[t]{0.2\hsize}
    \centerCropEx{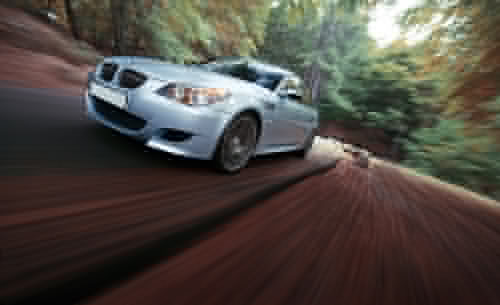}
    \caption{13.3KiB}
  \end{subfigure}%
  \begin{subfigure}[t]{0.2\hsize}
    \centerCropEx{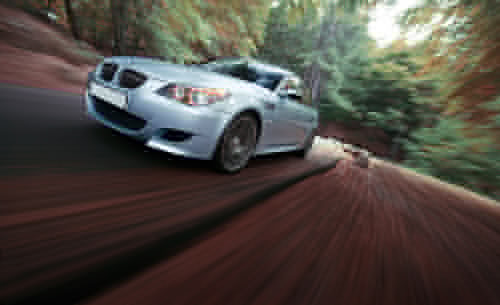}
    \caption{21.0KiB}
  \end{subfigure}%
  \begin{subfigure}[t]{0.2\hsize}
    \centerCropEx{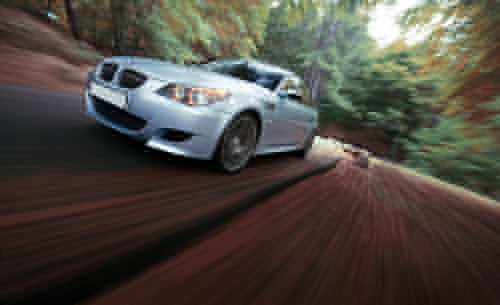}
    \caption{26.9KiB}
  \end{subfigure}%
  \begin{subfigure}[t]{0.2\hsize}
    \centerCropEx{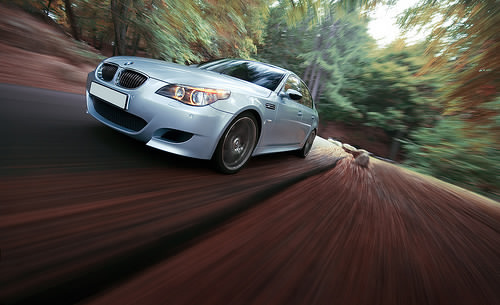}
    \caption{44.8KiB}
  \end{subfigure}%
  \newline
  \begin{subfigure}[t]{0.2\hsize}
    \centerCropEx{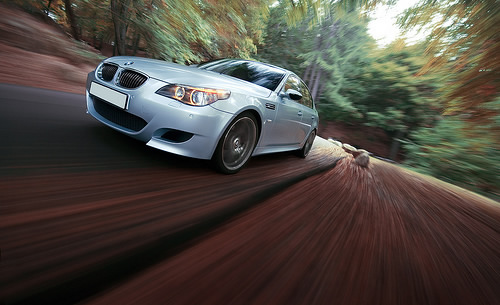}
    \caption{58.8KiB}
  \end{subfigure}%
  \begin{subfigure}[t]{0.2\hsize}
    \centerCropEx{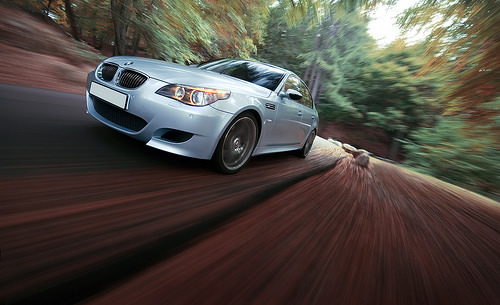}
    \caption{59.7KiB}
  \end{subfigure}%
  \begin{subfigure}[t]{0.2\hsize}
    \centerCropEx{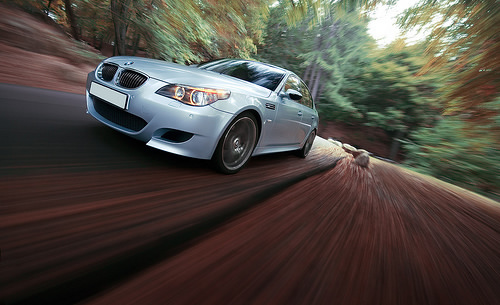}
    \caption{69.7KiB}
  \end{subfigure}%
  \begin{subfigure}[t]{0.2\hsize}
    \centerCropEx{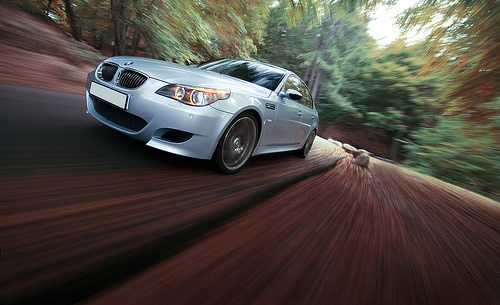}
    \caption{77.0KiB}
  \end{subfigure}%
  \begin{subfigure}[t]{0.2\hsize}
    \centerCropEx{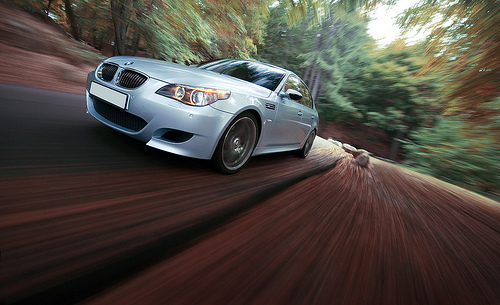}
    \caption{96.6KiB}
  \end{subfigure} \\
  \setcounter{subfigure}{0}%
    \rothead{\textbf{\mbox{CelebAHQ}}} &
  \begin{subfigure}[t]{0.2\hsize}
    \centerCropEx{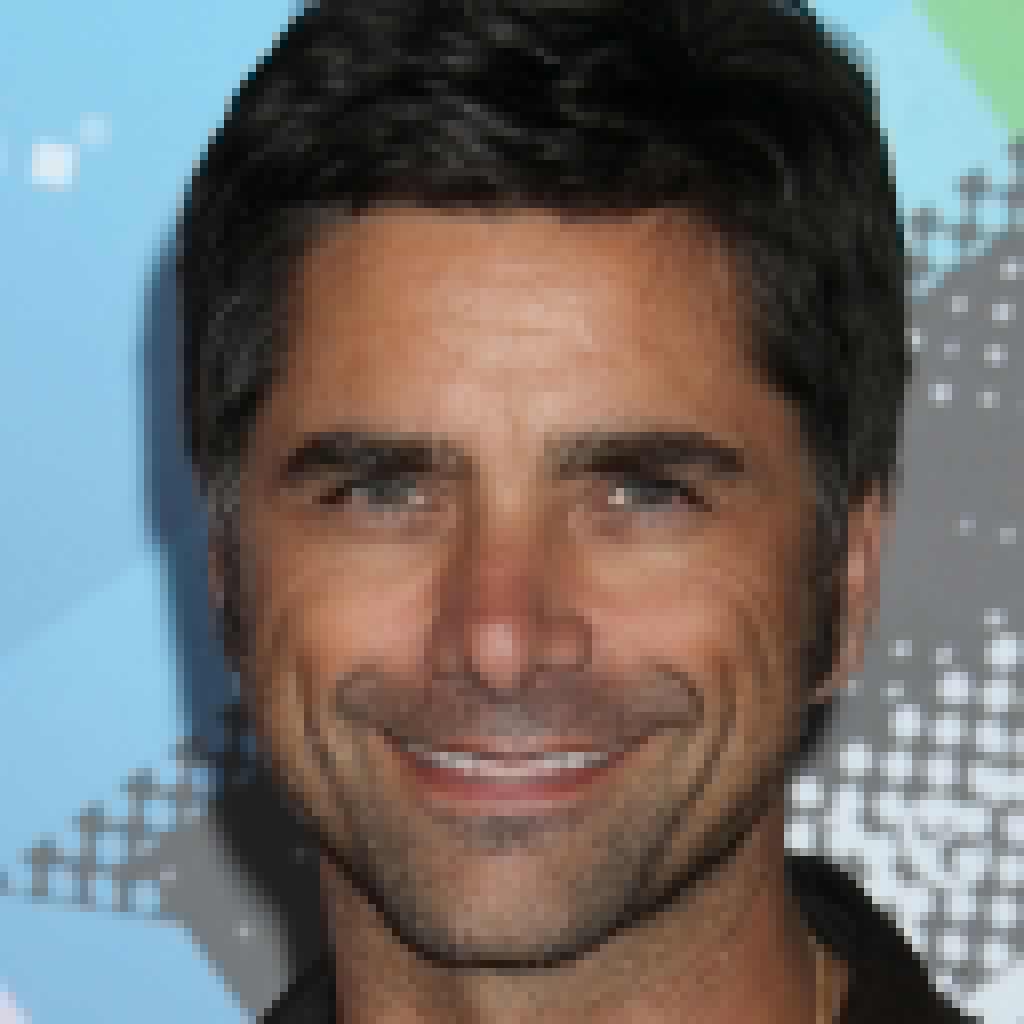}
    \caption{12.6KiB}
  \end{subfigure}%
  \begin{subfigure}[t]{0.2\hsize}
    \centerCropEx{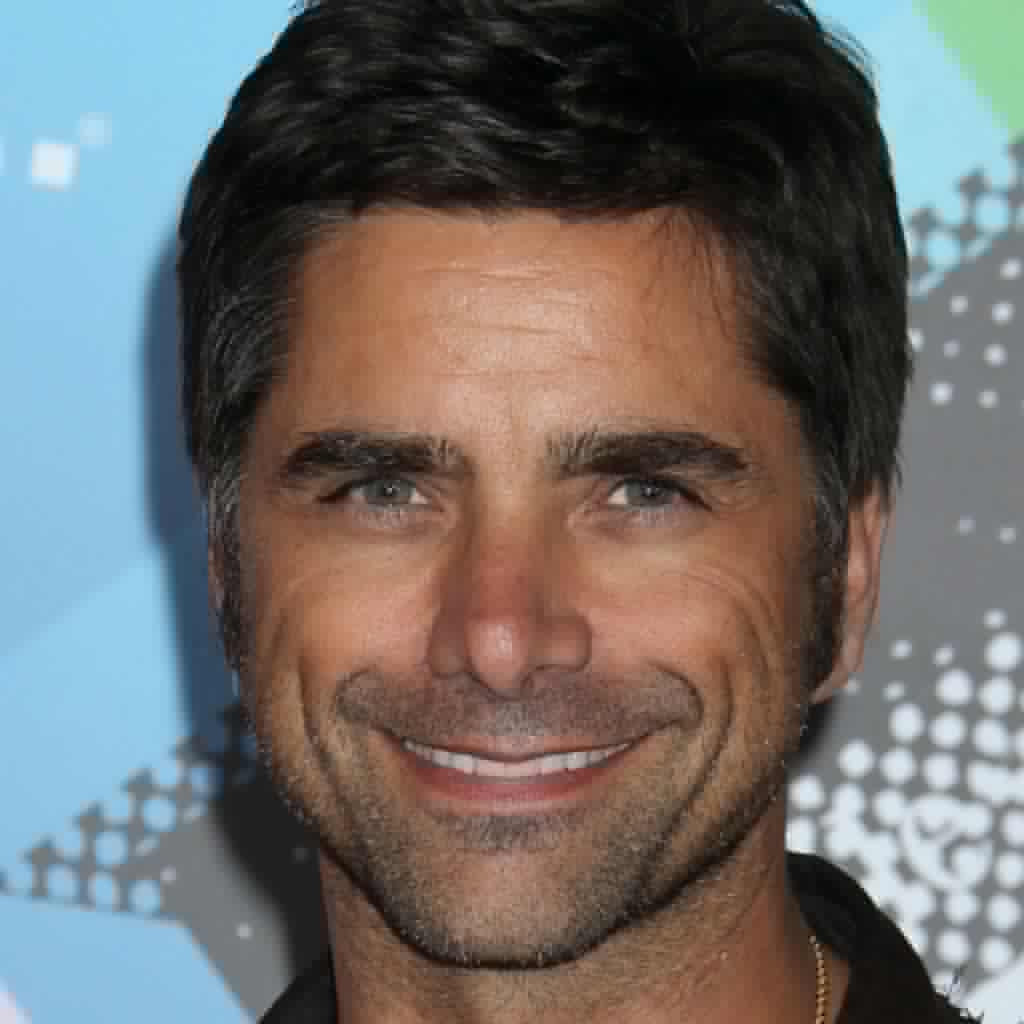}
    \caption{25.9KiB}
  \end{subfigure}%
  \begin{subfigure}[t]{0.2\hsize}
    \centerCropEx{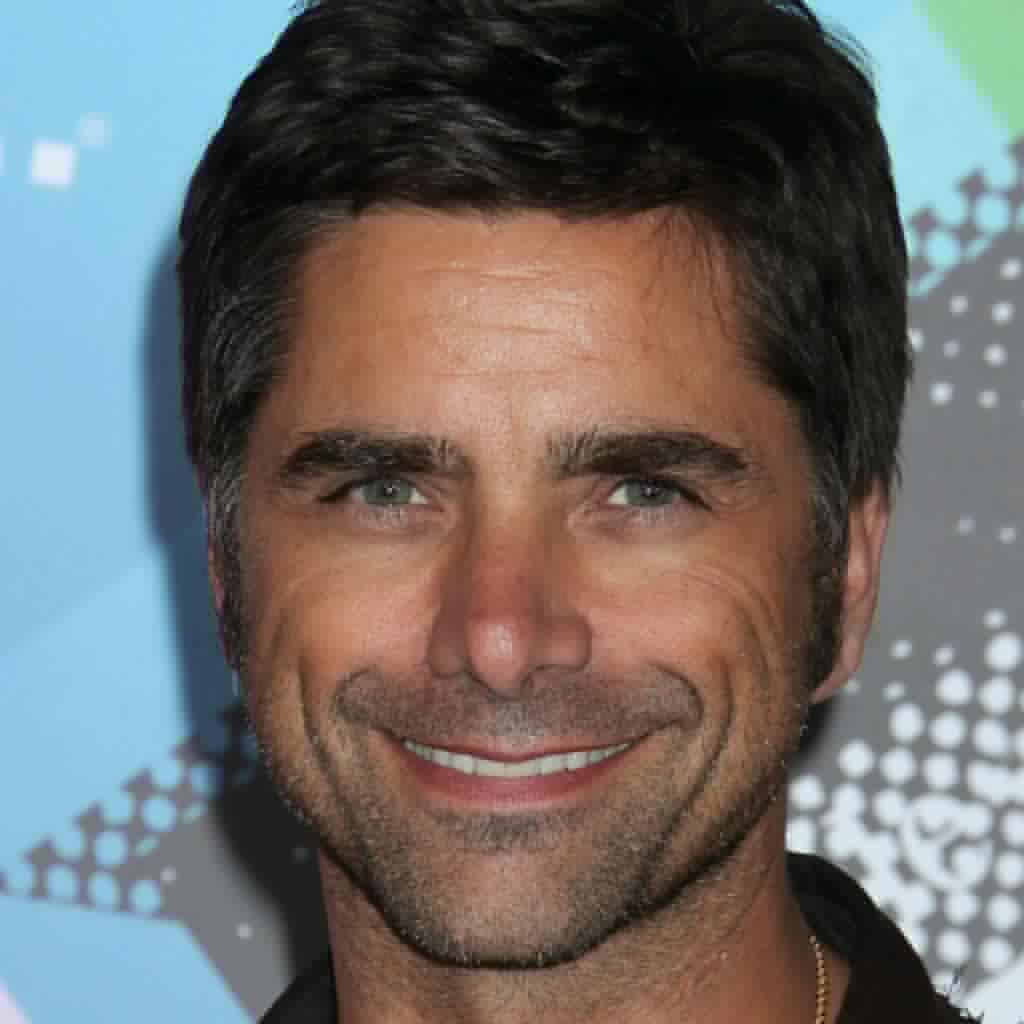}
    \caption{27.1KiB}
  \end{subfigure}%
  \begin{subfigure}[t]{0.2\hsize}
    \centerCropEx{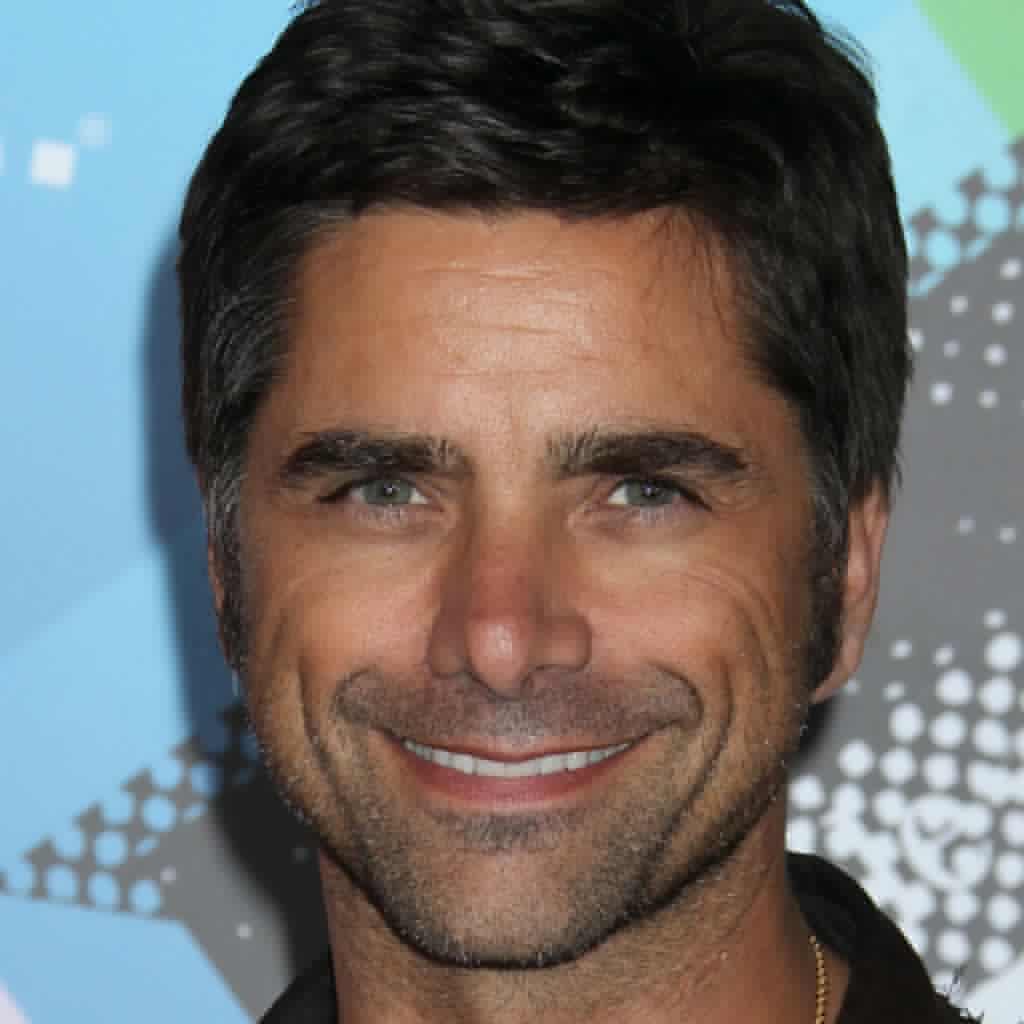}
    \caption{28.0KiB}
  \end{subfigure}%
  \begin{subfigure}[t]{0.2\hsize}
    \centerCropEx{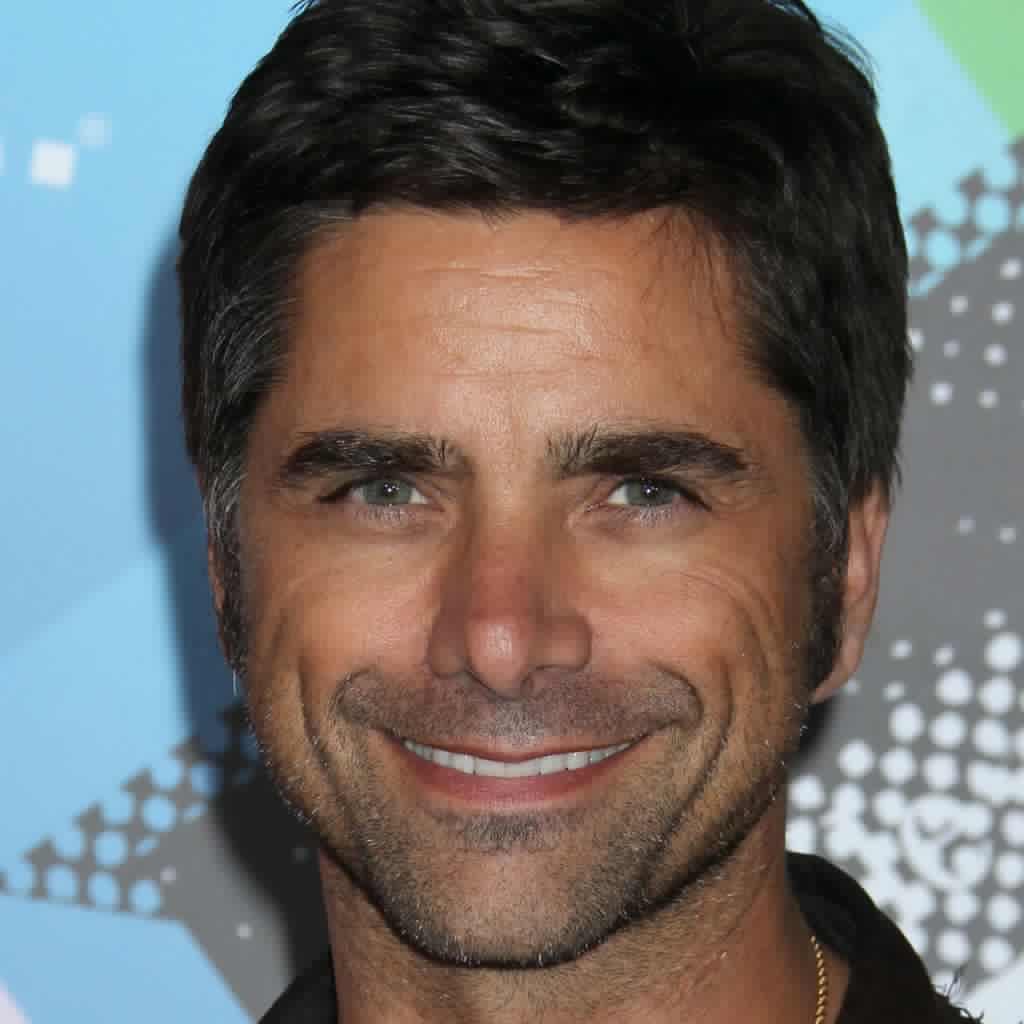}
    \caption{32.8KiB}
  \end{subfigure}%
  \newline
  \begin{subfigure}[t]{0.2\hsize}
    \centerCropEx{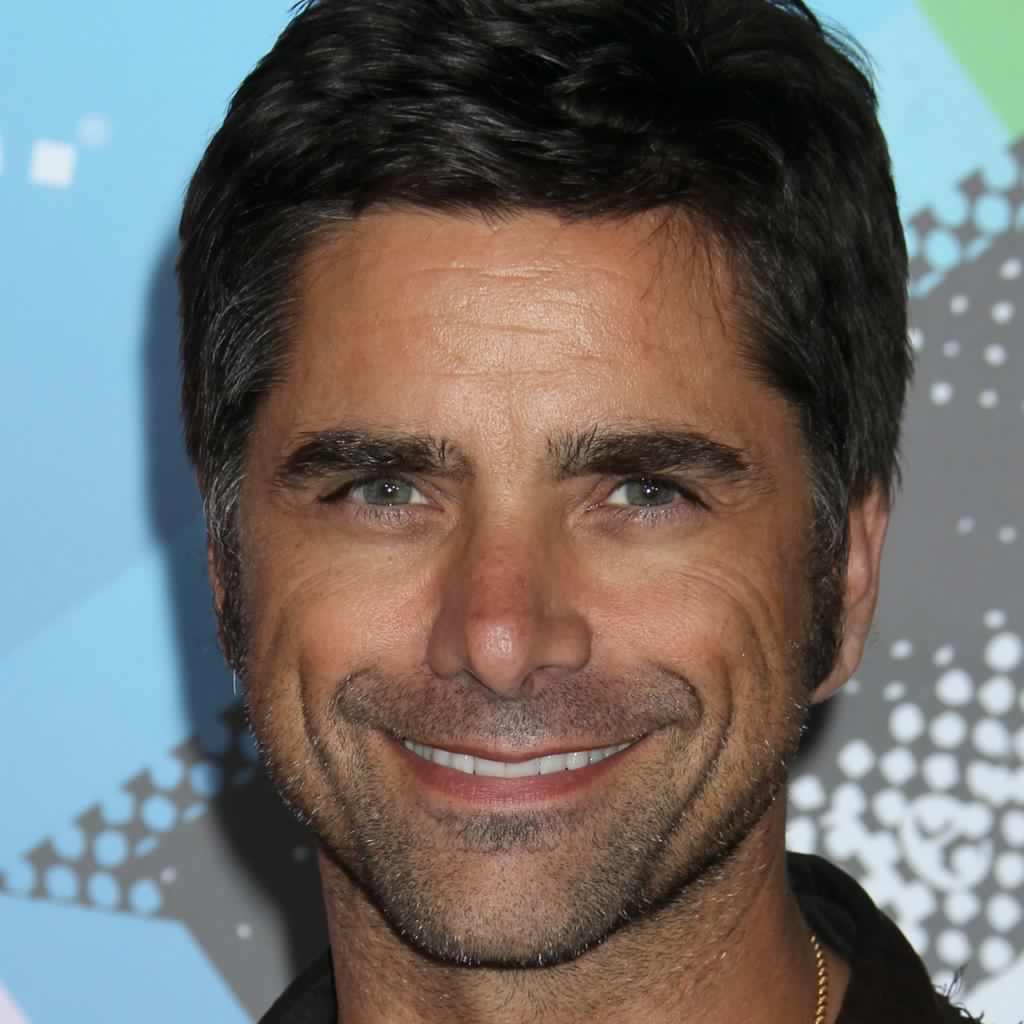}
    \caption{57.5KiB}
  \end{subfigure}%
  \begin{subfigure}[t]{0.2\hsize}
    \centerCropEx{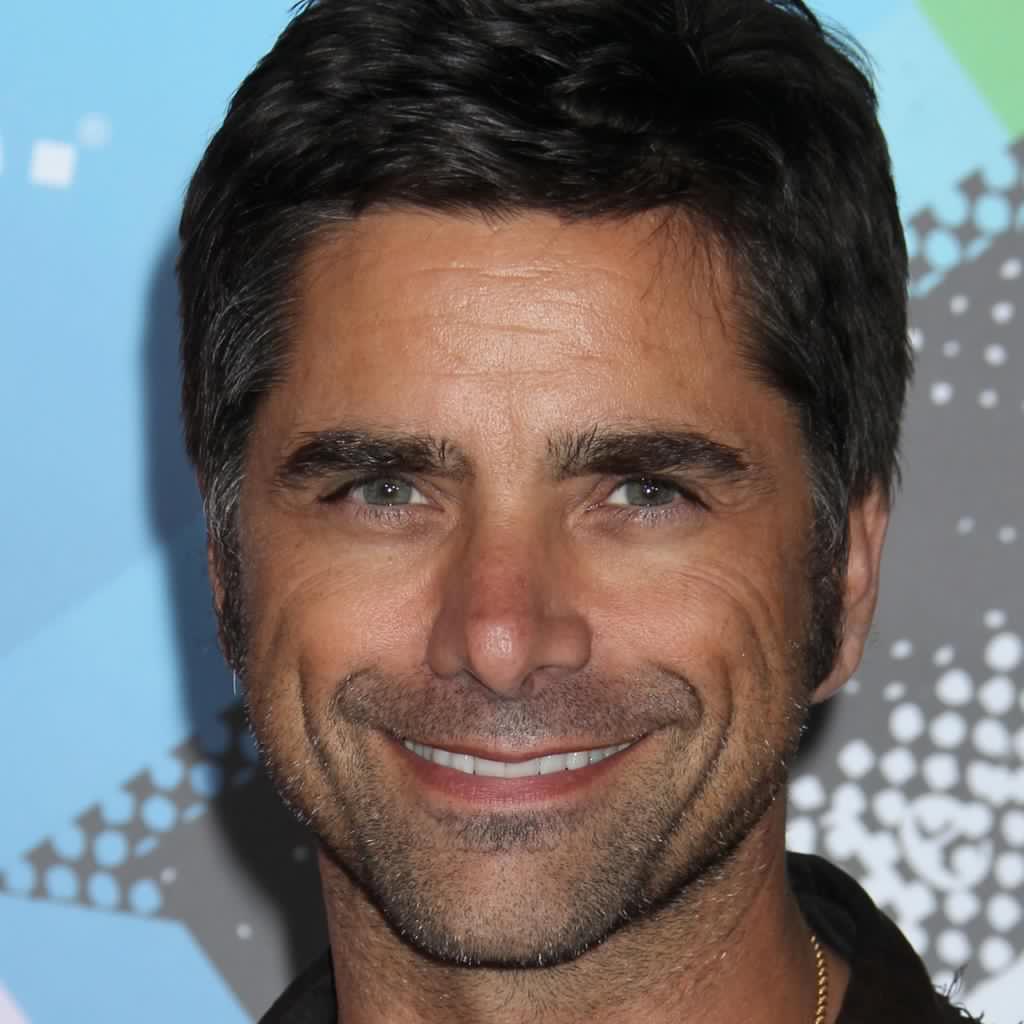}
    \caption{60.6KiB}
  \end{subfigure}%
  \begin{subfigure}[t]{0.2\hsize}
    \centerCropEx{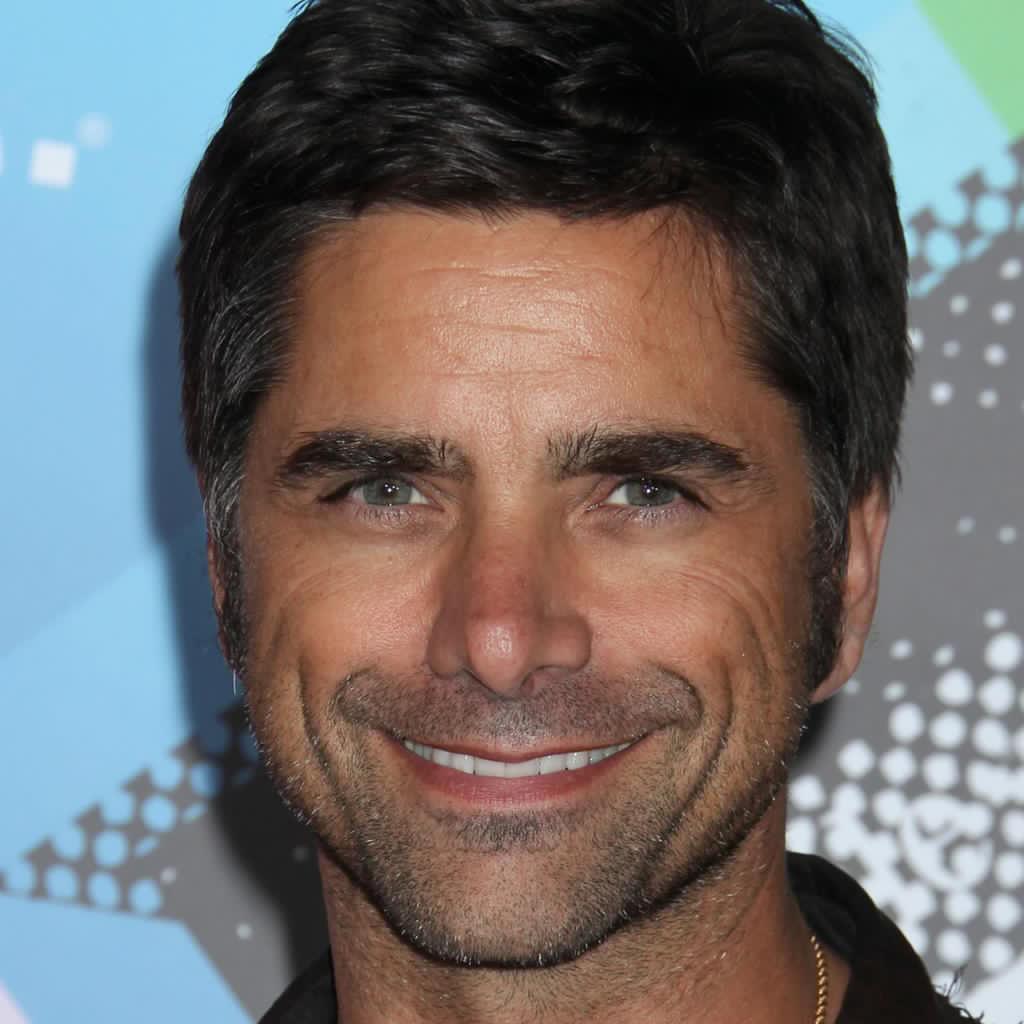}
    \caption{62.9KiB}
  \end{subfigure}%
  \begin{subfigure}[t]{0.2\hsize}
    \centerCropEx{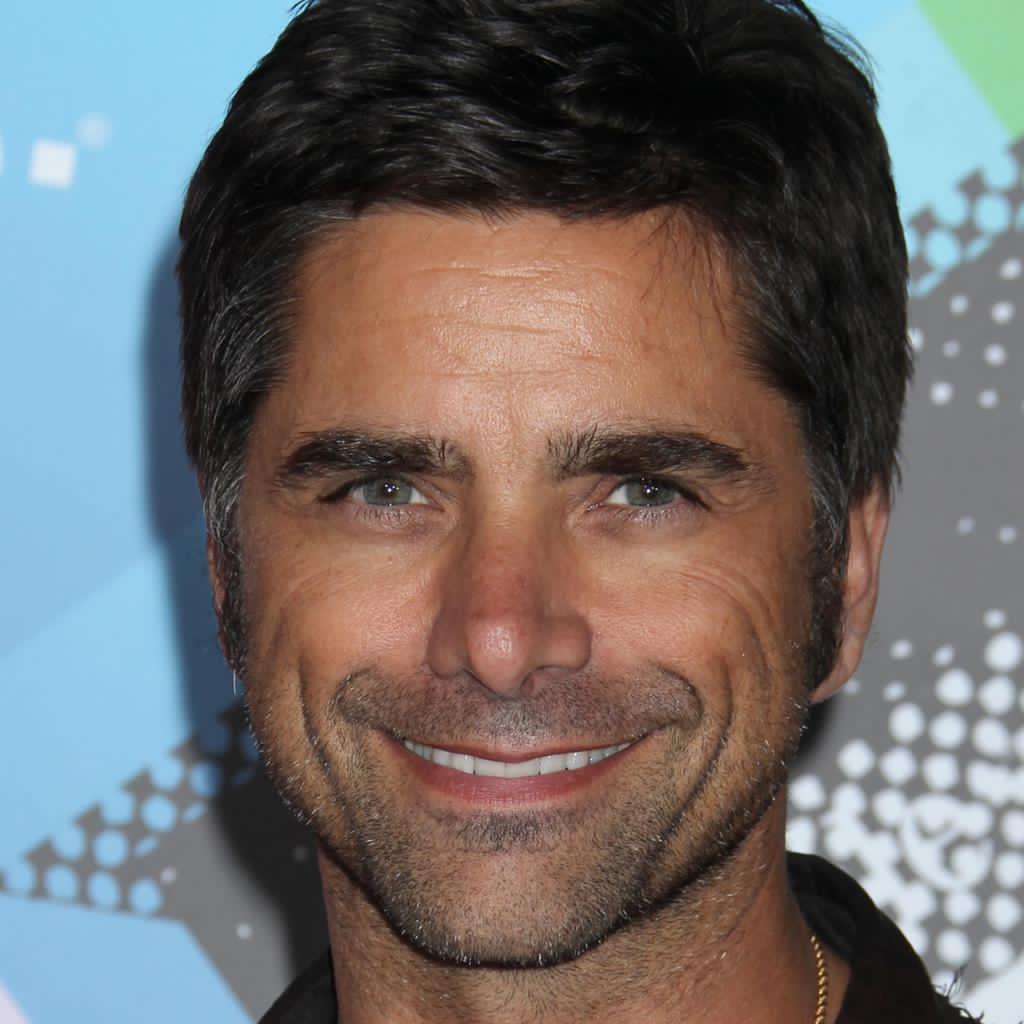}
    \caption{65.1KiB}
  \end{subfigure}%
  \begin{subfigure}[t]{0.2\hsize}
    \centerCropEx{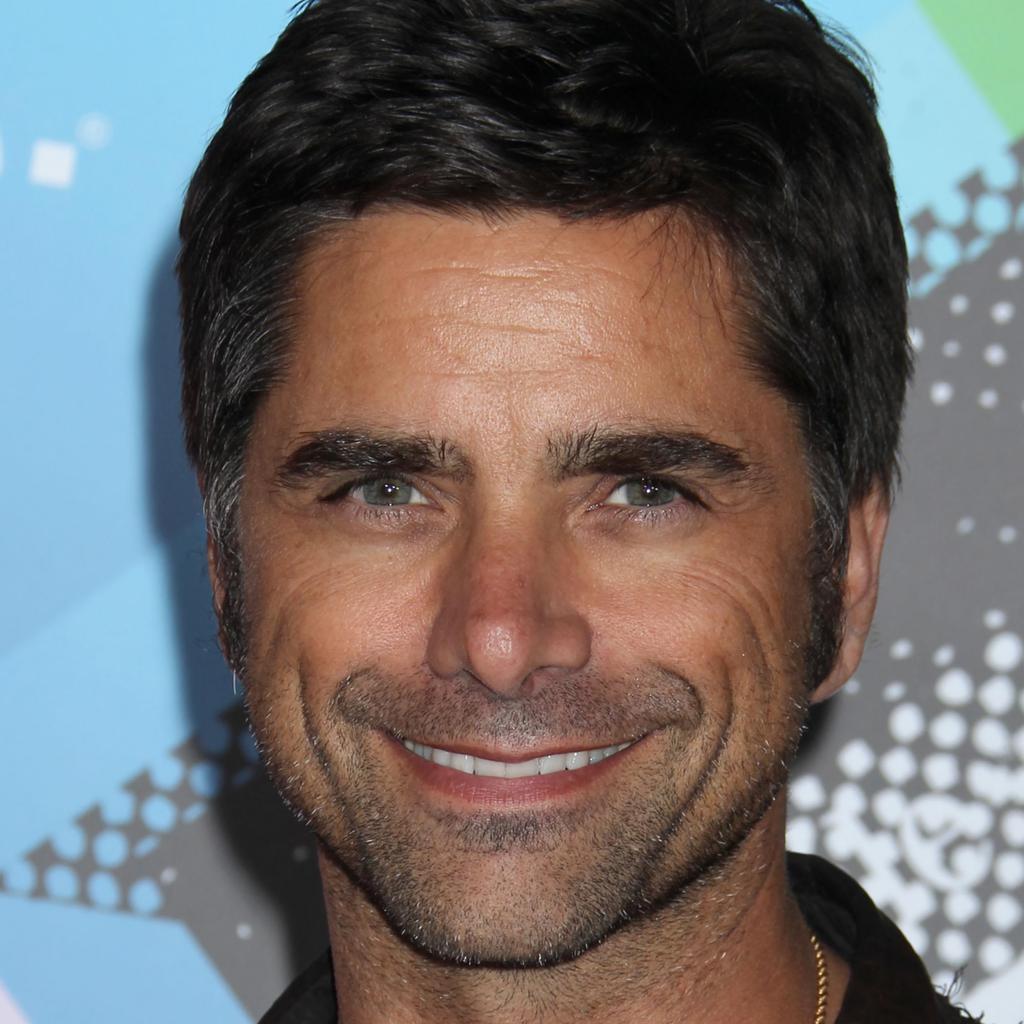}
    \caption{110.7KiB}
  \end{subfigure}%
  \end{tabularx}
  \caption{%
    Examples of scans with the corresponding file size.
    Images are center cropped for demonstration.
    The amount of scans needed to hit an acceptable level of fidelity is small.
    Having a larger final size
    results in more
    savings for earlier scans.
  }%
  \label{fig:jpeg_scans_big_all}%
\end{figure*}

\end{document}